%% file: bare_jrnl_new_sample4.tex
\documentclass[10pt，journal,compsoc]{IEEEtran}
\usepackage{algorithmic}
\usepackage{algorithm}
\usepackage{array}
\usepackage{textcomp}
\usepackage{stfloats}
\usepackage{url}
\usepackage{verbatim}
\usepackage{graphicx}
\usepackage{cite}
\usepackage{booktabs}
\usepackage{multirow}
\usepackage{amsmath}
\usepackage{amssymb}
\usepackage{mathtools}
\usepackage{amsthm}
\usepackage{bm}
\usepackage{caption}
\usepackage{stmaryrd}
\usepackage{subcaption}
\usepackage{microtype}
\usepackage[colorlinks=true]{hyperref}
\usepackage[nameinlink,noabbrev]{cleveref}
\theoremstyle{plain}
\newtheorem{theorem}{Theorem}[section]

\newtheorem{lemma}[theorem]{Lemma}
\newtheorem{corollary}[theorem]{Corollary}
\theoremstyle{definition}
\newtheorem{definition}[theorem]{Definition}
\newtheorem{assumption}[theorem]{Assumption}
\theoremstyle{remark}
\newtheorem{remark}[theorem]{Remark}

\DeclareMathOperator*{\argmin}{arg\,min}

\makeatletter
\newenvironment{breakablealgorithm}
  {
   \begin{center}
     \refstepcounter{algorithm}
     \hrule height.8pt depth0pt \kern2pt
     \renewcommand{\caption}[2][\relax]{
       {\raggedright\textbf{\ALG@name~\thealgorithm} ##2\par}%
       \ifx\relax##1\relax 
         \addcontentsline{loa}{algorithm}{\protect\numberline{\thealgorithm}##2}%
       \else 
         \addcontentsline{loa}{algorithm}{\protect\numberline{\thealgorithm}##1}%
       \fi
       \kern2pt\hrule\kern2pt
     }
  }{
     \kern2pt\hrule\relax
   \end{center}
  }
\makeatother

\usepackage[textsize=tiny]{todonotes}

\begin{document}

\title{Controllable Concept Bottleneck Models}

\author{Hongbin Lin*,
        Chenyang Ren*,
        Juangui Xu*,
        Zhengyu Hu,
        Cheng-Long Wang,
        Yao Shu,
        Hui Xiong,~\IEEEmembership{Fellow,~IEEE,}
        Jingfeng Zhang,~\IEEEmembership{Member,~IEEE,}
        Di Wang,
        and~Lijie Hu$^{\dagger}$%
\IEEEcompsocitemizethanks{
\IEEEcompsocthanksitem * H. Lin, C. Ren, and J. Xu contributed equally to this work.
\IEEEcompsocthanksitem $\dagger$ L. Hu is the corresponding author.
\IEEEcompsocthanksitem Manuscript received [Month] [Day], [Year]; revised [Month] [Day], [Year].
}%
}

\markboth{Journal of \LaTeX\ Class Files,~Vol.~14, No.~8, August~2021}%
{Shell \MakeLowercase{\textit{et al.}}: A Sample Article Using IEEEtran.cls for IEEE Journals}

\IEEEpubid{0000--0000/00\$00.00~\copyright~2021 IEEE}

\maketitle

\begin{abstract}
Concept Bottleneck Models (CBMs) have garnered much attention for their ability to elucidate the prediction process through a human-understandable concept layer. However, most previous studies focused on static scenarios where the data and concepts are assumed to be fixed and clean. In real-world applications, deployed models require continuous maintenance: we often need to remove erroneous or sensitive data (unlearning), correct mislabeled concepts, or incorporate newly acquired samples (incremental learning) to adapt to evolving environments. Thus, deriving efficient editable CBMs without retraining from scratch remains a significant challenge, particularly in large-scale applications. To address these challenges, we propose Controllable Concept Bottleneck Models (CCBMs). Specifically, CCBMs support three granularities of model editing: concept-label-level, concept-level, and data-level, the latter of which encompasses both \emph{data removal} and \emph{data addition}. CCBMs enjoy mathematically rigorous closed-form approximations derived from influence functions that obviate the need for retraining. Experimental results demonstrate the efficiency and adaptability of our CCBMs, affirming their practical value in enabling dynamic and trustworthy CBMs.
\end{abstract}


\input{sections/1_introduction}
\input{sections/2_related_work}
\input{sections/3_preliminary}
\input{sections/4_method}

\input{sections/5_experiment}

\input{sections/6_conclusion}
\input{sections/7_impact_statement}

\bibliography{references}
\bibliographystyle{IEEEtran}
\input{IEEEbiography}

\include{appendix/7_appendix_a}

\include{appendix/7_appendix_alg}
\include{appendix/7_appendix_b_bound}
\include{appendix/7_appendix_b}
\include{appendix/7_appendix_c_bound}
\include{appendix/7_appendix_c}
\include{appendix/7_appendix_c_update}
\include{appendix/7_appendix_d_bound}
\include{appendix/7_appendix_d}

\include{appendix/7_appendix_e}

\include{appendix/7_appendix_exp}

\end{document}

%% file: sections/1_introduction.tex
\section{Introduction}
 
\IEEEPARstart{T}{he} rapid proliferation of deep learning across high-stakes domains, including healthcare~\cite{ahmad2018interpretable,yu2018artificial} and finance~\cite{cao2022ai}, has been accompanied by growing concerns regarding the opaque nature of these systems. While large language models~\cite{zhao2023survey,yang2024moral,yang2024human,xu2023llm,yang2024dialectical} and multimodal architectures~\cite{yin2023survey,ali2024prompt,cheng2024multi} achieve remarkable performance, their complex, non-linear decision-making processes often remain impenetrable to end-users. To bridge this trust gap, Explainable Artificial Intelligence (XAI)~\cite{das2020opportunities,hu2023seat,hu2023improving} has become a focal point of research. Among various XAI paradigms, Concept Bottleneck Models (CBMs)~\cite{koh2020concept} stand out by explicitly structuring the prediction process: inputs are first mapped to human-understandable concepts, which then serve as the sole basis for the final classification. This architectural constraint ensures that the model's reasoning is transparent and aligned with human semantic knowledge.
 
\input{fig/fig1}
 
Despite their interpretability advantages, the practical deployment of CBMs faces significant hurdles. Prior literature has predominantly focused on two aspects: mitigating the high cost of concept annotations through semi-supervised or unsupervised learning~\cite{oikarinen2023label,yuksekgonul2022post,lai2023faithful}, and narrowing the performance gap between CBMs and standard black-box models~\cite{NEURIPS2023_555479a2, yuksekgonul2022post, NEURIPS2022_867c0682}. However, these works typically treat the CBM as a static entity once trained. This assumption contradicts the dynamic reality of real-world applications, where models must continuously adapt to correct errors, evolve knowledge bases, and respond to fluctuating data governance policies.
 
In this work, we argue that \textbf{controllability}—the ability to modify model behavior post-deployment without expensive retraining precisely—is as critical as interpretability. We identify three distinct granularities where such control is required, necessitating a transition from static CBMs to dynamic frameworks (see Figure~\ref{fig:intro}):
 
\begin{itemize}
 \item \emph{Concept-label-level Controllability:} 
 Human annotations are prone to error. In scenarios where expert feedback identifies specific mislabeled concepts in the training data, stakeholders need the ability to rectify these label-specific errors to correct the model's logic, while preserving the valid information contained in the rest of the sample.
 
 \item \emph{Concept-level Controllability:} 
 The definition of "valid knowledge" evolves over time. Medical research may uncover new risk factors (e.g., specific comorbidities for lung cancer) that need to be \emph{added} to an existing model's concept set. Conversely, previously relied-upon concepts might be deemed spurious or discriminatory (e.g., demographic biases in pandemic prediction~\cite{rozenfeld2020model}), necessitating their immediate \emph{removal}. A controllable CBM must support such ontology evolution in real-time.
 
\emph{Data-level Controllability:}
 We further categorize data-level control into two bidirectional streams:

 \item \emph{Data Unlearning (Removal):} To comply with privacy regulations (e.g., GDPR) or to discard poisoned data samples, the model must be capable of completely erasing the influence of specific data points from its parameters.
 \item \emph{Data Addition (Incremental Learning):} Conversely, as fresh data becomes available, the model should seamlessly integrate this new information to enhance robustness or address under-representation, avoiding the downtime associated with full model retraining.

\end{itemize}
 
Addressing these needs via naive retraining is computationally prohibitive, especially for large-scale backbones. To overcome this bottleneck, we propose Controllable Concept Bottleneck Models (CCBMs). Unlike traditional approaches, CCBMs enable efficient, mathematically grounded updates to the model parameters in response to the aforementioned changes. We derive closed-form approximations based on influence functions~\cite{cook2000detection,cook1980characterizations}, extending their application to the composite architecture of CBMs—a domain where their utility has been largely unexplored.
 
Our contributions to the field are fourfold:
\begin{itemize}
 \item We conceptualize the problem of \emph{Post-Hoc Controllability} in CBMs, defining a comprehensive taxonomy that spans concept-label correction, concept set evolution, and bidirectional data flow (addition and removal).
 \item We introduce the CCBM framework, which leverages influence functions to provide precise, low-cost model updates, effectively bypassing the need for retraining from scratch.
 \item We propose an accelerated algorithm integrating Eigenvalue-corrected Kronecker-Factored Approximate Curvature (EK-FAC) to scale our approach to high-dimensional parameter spaces.
 \item Extensive experiments validate that CCBMs achieve a superior trade-off between utility retention and computational efficiency, demonstrating their potential as a foundational framework for sustainable and trustworthy AI.
\end{itemize}

%% file: fig/fig1.tex
\begin{figure}[ht]
\centering
\vspace{-6pt}
\includegraphics[width=\linewidth]{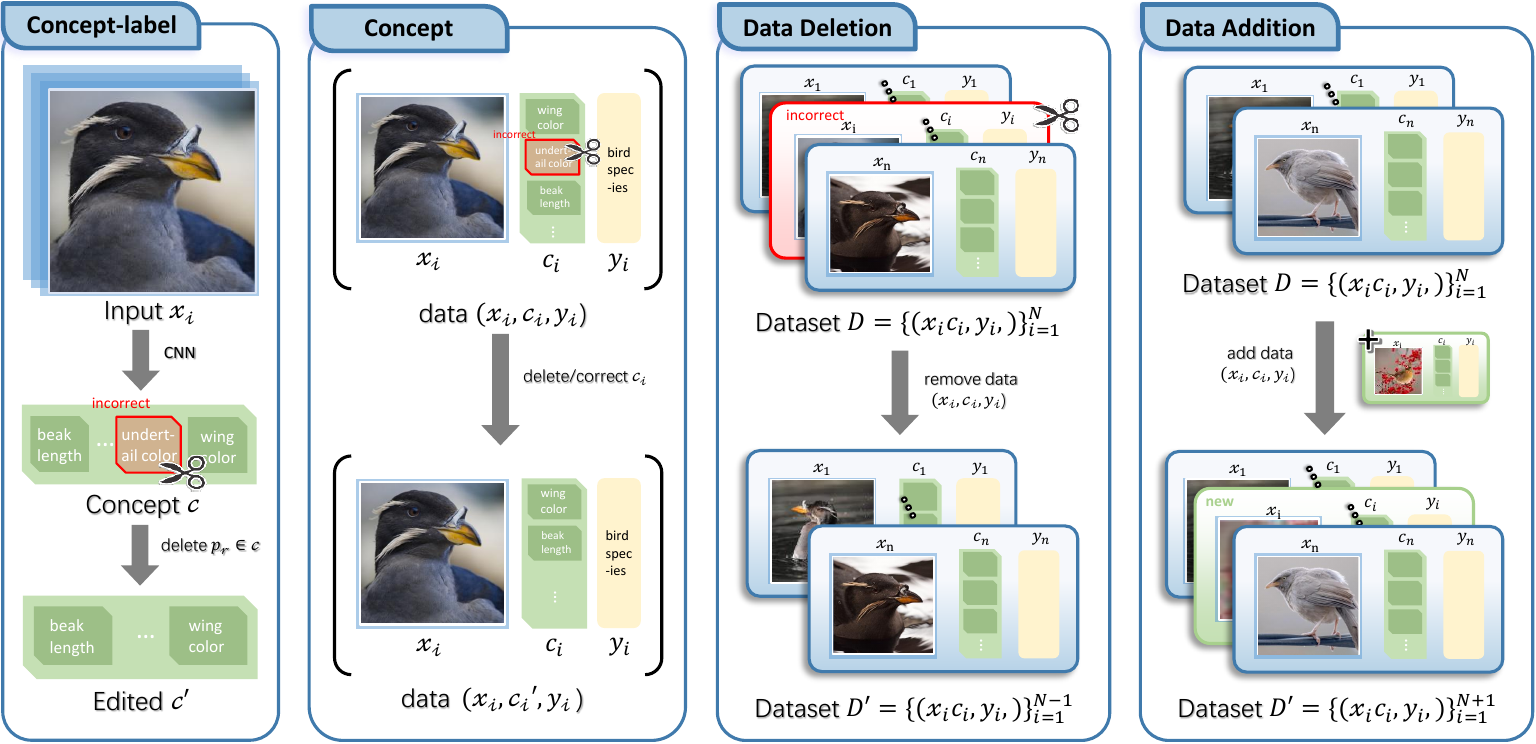}
\vspace{-20pt}
\caption{An illustration of Controllable Concept Bottleneck Models with four settings. \label{fig:intro}}
\vspace{-10pt}
\end{figure}

%% file: sections/2_related_work.tex
\section{Related Work}
\label{sec:related work}

\noindent{\bf Concept Bottleneck Models.}
CBM~\cite{koh2020concept} stands out as an innovative deep-learning approach for image classification and visual reasoning. It introduces a concept bottleneck layer into deep neural networks, enhancing model generalization and interpretability by learning specific concepts. However, CBM faces two primary challenges: its performance often lags behind that of original models lacking the concept bottleneck layer, attributed to incomplete information extraction from the original data to bottleneck features. Additionally, CBM relies on laborious dataset annotation. Researchers have explored solutions to these challenges. \cite{chauhan2023interactive} extend CBM into interactive prediction settings, introducing an interaction policy to determine which concepts to label, thereby improving final predictions. \cite{oikarinen2023label} address CBM limitations and propose a novel framework called Label-free CBM. This innovative approach enables the transformation of any neural network into an interpretable CBM without requiring labeled concept data, all while maintaining high accuracy. Post-hoc Concept Bottleneck models \cite{yuksekgonul2022post} can be applied to various neural networks without compromising model performance, preserving interpretability advantages. CBMs work on the image field also includes the works of \cite{havasi2022addressing},\cite{kim2023probabilistic},\cite{keser2023interpretable},\cite{sawada2022concept} and \cite{sheth2023auxiliary}. Despite many works on CBMs, we are the first to investigate the interactive influence between concepts through influence functions. Our research endeavors to bridge this gap by utilizing influence functions in CBMs, thereby deciphering the interaction of concept models and providing an adaptive solution to concept editing. For more related work, please refer to Appendix \ref{app:sec:more_related}.

%% file: sections/3_preliminary.tex
\section{Preliminaries}
\label{sec:preliminary}

\noindent{\bf Concept Bottleneck Models.} In this paper, we consider the original CBM, and we adopt the notations used by \cite{koh2020concept}. We consider a classification task with a concept set denoted as $\{p_1, \cdots, p_k\}$ with each $p_i$ being a concept given by experts or LLMs, and a training dataset represented as $\mathcal{D} = \{ z_i\}_{i=1}^{n}$, where $z_i = (x_i, y_i, c_i)$. Here, for $i\in [n]$, $x_i\in \mathbb{R}^{d_i}$ represents the input feature vector, $y_i\in \mathbb{R}^{d_o}$ denotes the label (with $d_o$ corresponding to the number of classes) and ${c}_i=(c_i^1, \cdots, c_i^k) \in \mathbb{R}^k$ represents the concept vector. In this context, $c_i^j$ represents the label of the concept $p_j$ of the $i$-th data. In CBMs, our goal is to learn two representations: one called concept predictor that transforms the input space into the concept space, denoted as $g: \mathbb{R}^d_i \to \mathbb{R}^k$, and the other called label predictor which maps the concept space to the prediction space, denoted as $f: \mathbb{R}^k \to \mathbb{R}^{d_o}$. Usually, here the map $f$ is linear. For each training sample $z_i = (x_i, y_i, {c}_i)$, we consider two empirical loss functions: concept predictor $\hat{g}$ and label predictor $\hat{f}$:

\begin{equation}\label{eq:def_g}
\begin{split}
    \hat{g} &= \argmin_g \sum_{i=1}^n\sum_{j=1}^k g^j(x_i)^\top \log({c^j_i}),
\end{split}
\end{equation}

where $g^j(*)$ is the predicted $j$-th concept. For brevity, write the loss function as $L_{C}(g(x_i), c_i) = \sum_{j=1}^k L_{C}^j(g(x_i), c_i)$ for data $(x_i, c_i)$. Once we obtain the concept predictor $\hat{g}$, the label predictor is defined as:

\begin{align}
\label{eq:hatf}
\hat{f} = \argmin_{f} \sum_{i=1}^n L_Y\big(f(\hat{g}(x_i)), y_i\big),
\end{align}

where \( L_Y \) represents the cross-entropy loss, similar to \eqref{eq:def_g}.
CBMs enforce dual precision in predicting interpretable concept vectors \( \hat{c} = \hat{g}(x) \) (matching concept \( {c} \)) and final outputs \( \hat{y} = \hat{f}(\hat{c}) \) (matching label \( y \)), ensuring transparent reasoning through explicit concept mediation. Furthermore, in this paper, we focus primarily on the scenarios in which the label predictor ff is a linear transformation, motivated by their interpretability advantages in tracing concept-to-label relationships. For details on the symbols used, please refer to the notation table in Appendix~\ref{tab:notation}.

\noindent {\bf Influence Function.} 
The influence function measures the dependence of an estimator on the value of individual point in the sample. Consider a neural network $\hat{\theta}=\arg\min_\theta \sum_{i=1}^n \ell(z_i;\theta)$ with loss function $\ell$ and dataset $D=\{z_i\}_{i=1}^n$. If we remove $z_m$ from the training dataset, the parameters become $\hat{\theta}_{-z_m} = \arg\min_\theta \sum_{i\neq m} \ell(z_i; \theta)$. The influence function provides an efficient model approximation by defining a series of $\epsilon$-parameterized models as $\hat{\theta}_{\epsilon, -z_m} ={\operatorname{argmin}} \sum_{i=1}^{n} \ell(z_{i} ; \theta) + \epsilon \ell(z_m ; \theta).$ 
By performing a first-order Taylor expansion on the gradient of the objective function corresponding to the \(\argmin\) process, the influence function is defined as:

\begin{equation*}
    \mathcal{I}_{\hat{\theta}} \left( z_m \right)\triangleq \left.\frac{\mathrm{d} {\hat{\theta}_{\epsilon,-z_m}} }{\mathrm{d}\epsilon}\right|_{\epsilon = 0} = -H^{-1}_{\hat{\theta}} \cdot \nabla_{\theta} \ell (z_m; {\hat{\theta}}),
\end{equation*}

where \(H^{-1}_{\hat{\theta}} = \nabla^2_{\theta} \sum_{i=1}^n \ell(z_i; \hat{\theta})\) is the Hessian matrix. When the loss function $\ell$ is twice-differentiable and strongly convex in $\theta$, the Hessian $H_{{\hat{\theta}}}$ is positive definite and thus the influence function is well-defined. For non-convex loss functions, \cite{bartlett1953approximate} proposed replacing the Hessian \(H_{\hat{\theta}}\) with \(\hat{H} = G_{\hat{\theta}} + \delta I\), where \(G_{\hat{\theta}}\) is the Fisher information matrix defined as \( \sum_{i=1}^n \nabla_{\theta} \ell(z_i; \theta) ^{\mathrm{T}}\nabla_{\theta} \ell(z_i; \theta)\), and \(\delta \) is the damping term used to ensure the positive definiteness of \(\hat{H}\). We can employ the Eigenvalue-corrected Kronecker-Factored Approximate Curvature (EK-FAC) method to further accelerate the computation. See Appendix \ref{sec:further} for additional details. 

%% file: sections/4_method.tex
\section{Controllable Concept Bottleneck Models}
\label{sec:method}

In this section, we introduce our CCBMs for the three settings mentioned in the introduction, leveraging the influence function. Specifically, at the concept-label level, we calculate the influence of a set of data samples' individual concept labels; at the concept level, we calculate the influence of multiple concepts; and at the data level, we calculate the influence of multiple samples. Figure~\ref{fig:workflow} illustrates the workflows of our CCBMs in the corresponding settings.

\input{fig/figWorkFlow}

\subsection{Concept Label-level Controllable CBM}

In many cases, certain data samples contain erroneous annotations for specific concepts, yet their other information remains valuable. This is particularly relevant in domains such as medical imaging, where acquiring data is often costly and time-consuming. In such scenarios, it is common to correct the erroneous concept annotations rather than removing the entire data from the dataset. Estimating the retrained model parameter is crucial in this context. We refer to this scenario as the concept label-level Controllable CBM.

Mathematically, we have a set of erroneous data $D_e$ and its associated index set $S_e\subseteq [n]\times [k]$ such that for each $(w, r)\in S_e$, $(x_w, y_w, {c}_w)\in D_e$ with $c_w^r$ is mislabeled and $\tilde{c}_w^r$ is corrected concept label. Our goal is to estimate the retrained CBM. The retrained concept predictor and label predictor are represented as follows:

\begin{equation}
    \begin{split}
        \hat{g}_{e} = \argmin_{g}  & \sum_{(i, j) \notin S_e} L^j_{C}\left(g(x_i),{c}_i\right) \\
       +&\sum_{(i, j) \in S_e} L^j_{C}\left(g(x_i), \tilde{c}_i \right), \label{concept-label:g}
    \end{split}
\end{equation}
\vspace{-5pt}
\begin{equation}
     \label{concept-label:f}
    \hat{f}_{e} = \argmin_{f} \sum_{i=1}^n L_{Y} \left(f\left(\hat{g}_{e}\left(x_i\right)\right), y_i\right).
\end{equation}

For simple neural networks, we can use the influence function approach directly to estimate the retrained model. However, for CBM architecture, if we intervene with the true concepts, the concept predictor $\hat{g}$ fluctuates to $\hat{g}_e$ accordingly. Observe that the input data of the label predictor comes from the output of the concept predictor, which is also subject to change. Therefore, we need to adopt a two-stage editing approach. Here we consider the influence function for \eqref{concept-label:g} and \eqref{concept-label:f} separately. We first edit the concept predictor from $\hat{g}$ to $\bar{g}_{e}$, and then edit from $\hat{f}$ to $\bar{f}_e$ based on our approximated concept predictor. To begin, we provide the following definitions:

\begin{definition}
Define the gradient of the \(j\)-th concept predictor and the label predictor for the \(i\)-th data point \(x_i\) as:
\begin{align*}
    &G^j_C(x_i,{c}_i;{g}) \triangleq \nabla_{{g}}L^j_C\left({g}(x_i),  {c}_i\right),\\
     &G_Y(x_i;{g},f) \triangleq \nabla_{{f}}L_Y\left(f({g}(x_i)),y_i\right).
\end{align*}
\end{definition}

\begin{theorem}\label{th:concept-label:g}
The retrained concept predictor $\hat{g}_{e}$ defined by (\ref{concept-label:g}) can be approximated by $\bar{g}_{e}$, defined by: 
\begin{equation*}
     \hat{g} -H^{-1}_{\hat{g}}  \cdot \sum_{(w,r)\in S_e} \left(  G^r_C(x_w,\tilde{c}_w;\hat{g}) - G^r_C(x_w,  {c}_w;\hat{g}) \right),
\end{equation*}
where $H_{\hat{g}} = \nabla_{\hat{g}}  \sum_{i,j}  G^j_C(x_i,{c}_i;\hat{g})$ is the Hessian matrix of the loss function with respect to $\hat{g}$.
\end{theorem}

\begin{theorem}
The retrained label predictor $\hat{f}_{e}$ defined by \eqref{concept-label:f} can be approximated by $\bar{f}_{e}$, defined by: 
\begin{align*}
       \hat{f} + H^{-1}_{\hat{f}} \cdot  \sum_{i=1}^n \left(G_Y(x_i;\hat{g},\hat{f})   -G_Y(x_i;\bar{g}_e,\hat{f})  \right),
   \end{align*}
where $H_{\hat{f}} = \nabla_{\hat{f}}\sum_{i=1}^n   G_Y(x_i;\hat{g},\hat{f})$ is the Hessian matrix, and $\bar{g}_{e}$ is given in Theorem \ref{th:concept-label:g}.
\end{theorem}

\noindent {\bf Difference from Test-Time Intervention.}
The ability to intervene in CBMs allows human users to interact with the model during the prediction process. For example, a medical expert can directly replace an erroneously predicted concept value \(\hat{c}\) and observe its impact on the final prediction \(\hat{y}\).
However, the underlying flaws in the concept predictor remain unaddressed, meaning similar errors may persist when applied to new test data. In contrast, under the Controllable CBM framework, not only can test-time interventions be performed, but the concept predictor of the CBM can also be further refined based on test data that repeatedly produces errors. 
Our CCBM method incorporates the corrected test data into the training dataset without requiring full retraining. This approach extends the rectification process from the data level to the model level.

\subsection{Concept-level Controllable CBM}
In this case, a set of concepts is removed due to incorrect attribution or spurious concepts, termed concept-level edit. \footnote{For convenience, in this paper, we only consider concept removal; our method can directly extend to concept insertion.}Specifically, for the concept set, denote the erroneous concept index set as $M\subset [k]$, we aim to delete these concept labels in all training samples. We aim to investigate the impact of updating the concept set within the training data on the model's predictions.  
It is notable that compared to the above concept label case,  the dimension of output (input) of the retrained concept predictor (label predictor) will change. If we delete $t$ concepts from the dataset, then ${g}$ becomes ${g}^{\prime}:\mathbb{R}^{d_i}\rightarrow\mathbb{R}^{k-t}$ and ${f}$ becomes ${f}^{\prime}:\mathbb{R}^{k-t}\rightarrow\mathbb{R}^{d_o}$.
More specifically, if we retrain the CBM with the revised dataset, the corresponding concept predictor becomes:

\begin{equation}\label{concept-level:g}
\hat{g}_{-p_M} = \argmin_{g'} \sum_{j\notin M}\sum_{i=1}^n L^j_{C}({g'}(x_i),c_i). 
\end{equation}

The variation of the parameters in dimension renders the application of influence function-based editing challenging for the concept predictor. This is because the influence function implements the editorial predictor by approximate parameter change from the original base after $\epsilon$-weighting the corresponding loss for a given sample, and thus, it is unable to deal with changes in parameter dimensions.

To overcome the challenge, our strategy is to develop some transformations that need to be performed on $\hat{g}_{-p_M}$ to align its dimension with $\hat{g}$ so that we can apply the influence function to edit the CBM. We achieve this by mapping $\hat{g}_{-p_M}$ to $\hat{g}^{*}_{-p_M} \triangleq \mathrm{P}(\hat{g}_{-p_M})$, which has the same amount of parameters as $\hat{g}$ and has the same predicted concepts $\hat{g}^{*}_{-p_M}(j)$ as $\hat{g}_{-p_M}(j)$ for all $j\in [d_i]-M$. We achieve this effect by inserting a zero row vector into the $r$-th row of the matrix in the final layer of $\hat{g}_{-p_M}$ for $r\in M$. Thus, we can see that the mapping $P$ is one-to-one. Moreover, assume the parameter space of $\hat{g}$ is $T$ and that of $\hat{g}^{*}_{-p_M}$, $T_0$ is the subset of $T$. Noting that $\hat{g}^{*}_{-p_M}$ is the optimal model of the following objective function:

\begin{equation}\label{concept-level:g^*}
\hat{g}^{*}_{-p_M} = \argmin_{g^{\prime} \in T_0} \sum_{j \notin M} \sum_{i=1}^{n} L^j_{C} ( g'(x_i), c_i), 
\end{equation}

i.e., it is the optimal model of the concept predictor loss on the remaining concepts under the constraint $T_0$.  Now we can apply the influence function to edit $\hat{g}$ to approximate $\hat{g}^{*}_{-p_M}$ with the restriction on the value of 0 for rows indexed by $M$ with the last layer of the neural network, denoted as $\bar{g}^*_{-p_M}$. After that, we remove from $\bar{g}^*_{-p_M}$ the parameters initially inserted to fill in the dimensional difference, which always equals 0 because of the restriction we applied in the editing stage, thus approximating the true edited concept predictor $\hat{g}_{-p_M}$. We now detail the editing process from $\hat{g}$ to $\hat{g}^{*}_{-p_M}$ using the following theorem. 

\begin{theorem}\label{thm:4.3}
For the retrained concept predictor $\hat{g}_{-p_M}$ defined in \eqref{concept-level:g}, we map it to $\hat{g}^{*}_{-p_M}$ as \eqref{concept-level:g^*}. And we can edit the initial $\hat{g}$ to $\hat{g}^*_{-p_M}$, defined as:
\begin{equation}
    \bar{g}^*_{-p_M}\triangleq \hat{g}  - H^{-1}_{\hat{g}} \cdot \sum_{j\notin M}\sum_{i=1}^n G_C^j(x_i,c_i;\hat{g}),
\end{equation}
where $H_{\hat{g}} = \nabla_{g} \sum_{j\notin M} \sum_{i=1}^{n} G^j_{C} (x_i, c_i; \hat{g})$.
Then, by removing all zero rows inserted during the mapping phase, we can naturally approximate $\hat{g}_{-p_M}\approx  \mathrm{P}^{-1}(\hat{g}^{*}_{-p_M})$.
\end{theorem}

For the second stage of training, assume we aim to remove concept $p_r$ for $r\in M$ and the new optimal model is $\hat{f}_{-p_M}$.
We will encounter the same difficulty as in the first stage, i.e., the number of parameters of the label predictor will change. To address the issue, our key observation is that in the existing literature on CBMs, we always use linear transformation for the label predictor, meaning that the dimensions of the input with values of $0$ will have no contribution to the final prediction. 
To leverage this property, we fill the missing values in the input of the updated predictor with $0$, that is, replacing $\hat{g}_{-p_M}$ with $\hat{g}^*_{-p_M}$ and consider $\hat{f}_{p_M = 0}$ defined by

\begin{equation}\label{concept-level:f^*}
    \hat{f}_{p_M=0} = \argmin_{f} \sum_{i = 1}^n L_{Y} \left(f\left(\hat{g}^*_{-p_M}(x_i)\right), y_i\right).
\end{equation}

In total, we have the following lemma:

\begin{lemma}\label{method:lm:1}
In the CBM, if the label predictor utilizes linear transformations of the form $\hat{f} \cdot c$ with input $c$, then, for each $r\in M$, we remove the $r$-th concept from c and denote the new input as $c^{\prime}$; set the $r$-th concept to $0$ and denote the new input as $c^0$. Then we have $\hat{f}_{-p_M} \cdot c^{\prime} = \hat{f}_{p_M=0} \cdot c^0$ for any input $c$. 
\end{lemma}

Lemma \ref{method:lm:1} demonstrates that the retrained $\hat{f}_{-p_M}$ and $\hat{f}_{p_M=0}$, when given inputs $\hat{g}_{-p_M}(x)$ and $\hat{g}^*_{-p_M}(x)$ respectively, yield identical outputs. Consequently, we can utilize $\hat{f}_{p_M=0}$ as the editing target in place of $\hat{f}_{-p_M}$.

\begin{theorem}
For the revised retrained label predictor $\hat{f}_{p_M=0}$ defined by \eqref{concept-level:f^*}, we can edit the initial label predictor $\hat{f}$ to $\bar{f}_{p_M=0}$ by the following equation as a substitute for $\hat{f}_{p_M=0}$:
\begin{equation}
    \hat{f}_{p_M=0} \approx \bar{f}_{p_M=0} \triangleq \hat{f}-H_{\hat{f}}^{-1} \cdot    \sum_{l=1}^{n}G_Y(x_l;\bar{g}^*_{-p_M},\hat{f}), 
\end{equation}
where $H_{\hat{f}} = \nabla_{\hat{f}}\sum_{i=1}^n G_Y(x_l;\bar{g}^*_{-p_M},\hat{f})$ is the Hessian matrix. Deleting the $r$-th dimension of $\bar{f}_{p_M=0}$ for $r\in M$, then we can map it to $\bar{f}_{-p_M}$, which is the approximation of the final edited label predictor $\hat{f}_{-p_M}$ under concept level.
\end{theorem}

\subsection{Data-level Controllable CBM}
In this section, we address the Controllable CBM at the data level. We identify two distinct but equally critical operational requirements in the lifecycle of a deployed model: \textit{Data Unlearning} (removing specific samples) and \textit{Data Addition} (incorporating new samples). We derive specialized approximations for each scenario.

\subsubsection{Data Unlearning (Removal)}
In this scenario, the primary objective is to completely eliminate the influence of specific data samples on the trained CBM. This requirement typically stems from privacy compliance (e.g., "Right to be Forgotten"), or the necessity to cleanse the model of poisoned or mislabeled training examples.

Mathematically, let $\mathcal{D}$ be the original dataset. We identify a subset of samples to be removed, denoted as $G \subset \mathcal{D}$, containing elements $z_r = (x_r, y_r, c_r)$. The goal is to approximate the model parameters as if they were trained on the sanitized dataset $\mathcal{D} \setminus G$.
The retrained concept predictor is defined as:
\begin{equation}\label{eq:unlearn:g}
    \hat{g}_{-z_G} = \argmin_{g}\sum_{j=1}^k\sum_{z_i \in \mathcal{D} \setminus G} L^j_{C}(g(x_i), c_i).
\end{equation}
Since retraining is computationally expensive, we employ influence functions to estimate the parameter change caused by removing the loss gradients of $G$.

\begin{theorem}\label{thm:unlearn:g}
Given the subset $G$ to be removed, the sanitized concept predictor $\hat{g}_{-z_G}$ defined in \eqref{eq:unlearn:g} can be approximated by:
\begin{equation}
\hat{g}_{-z_G} \approx \bar{g}_{-z_G} \triangleq \hat{g} + H^{-1}_{\hat{g}} \cdot \sum_{z_r \in G} \sum_{j=1}^M \nabla_g L^j_C(x_r, c_r; \hat{g}),
\end{equation}
where $H_{\hat{g}}$ is the Hessian of the total concept loss. The addition sign reflects the inverse operation of removing the negative influence of the gradients in $G$.
\end{theorem}

Subsequently, the label predictor must define its decision boundary based on the updated concept representations. The retrained label predictor corresponds to:
\begin{equation}\label{eq:unlearn:f}
    \hat{f}_{-z_G} = \argmin_{f}\sum_{z_i \in \mathcal{D} \setminus G} L_{Y}\left(f(\hat{g}_{-z_G}(x_i)), y_i\right).
\end{equation}
Direct estimation is intractable due to the concurrent shift in both the training set size and the input features (updated concepts). We employ the two-stage estimation strategy. First, we define an intermediate predictor $\Tilde{f}_{-z_G}$ that accounts only for the data removal while fixing the concept predictor:
\begin{equation}\label{eq:unlearn:tilde-f}
    \Tilde{f}_{-z_G} = \argmin_{f} \sum_{z_i \in \mathcal{D} \setminus G} L_{Y}(f(\hat{g}(x_i)), y_i).
\end{equation}

\begin{theorem}\label{thm:unlearn:f}
The parameter change for the label predictor in the unlearning setting is derived as follows.
First, the intermediate predictor is approximated by removing the gradient influence of $G$:
\begin{equation*}
\begin{split}
    \Tilde{f}_{-z_G} & \approx \bar{f}^*_{-z_G} \\
    & \triangleq \hat{f} + H^{-1}_{\hat{f}} \sum_{z_r \in G} \nabla_f L_Y(x_r; \hat{g}, \hat{f}) \\
    & \triangleq \hat{f} + A_{unlearn}.
\end{split}
\end{equation*}
Second, to account for the concept drift from $\hat{g}$ to $\bar{g}_{-z_G}$, we compute the correction term $B_{unlearn}$:
\begin{equation*}
\begin{split}
    B_{unlearn} = -H^{-1}_{\bar{f}^*} \sum_{z_i \in \mathcal{D} \setminus G} \bigg( & \nabla_f L_Y(x_i; \bar{g}_{-z_G}, \bar{f}^*) \\
    & - \nabla_f L_Y(x_i; \hat{g}, \bar{f}^*) \bigg),
\end{split}
\end{equation*}
where $\bar{f}^*$ denotes $\bar{f}^*_{-z_G}$. The final unlearned label predictor is $\bar{f}_{-z_G} = \hat{f} + A_{unlearn} + B_{unlearn}$.
\end{theorem}

\subsubsection{Data Addition (Incremental Adaptation)}
Different from unlearning, real-world systems must often adapt to continuous data streams. In this \textit{Data Addition} scenario, the CBM is expanded by assimilating a new batch of samples, denoted as $\mathcal{S}_{new} = \{(\tilde{x}_k, \tilde{y}_k, \tilde{c}_k)\}_{k=1}^m$, where $m \ll n$.
Unlike removal, which seeks to revert knowledge, addition seeks to \textit{augment} the model's capability without the prohibitive cost of reprocessing the entire historical dataset. This creates an \textit{Incremental Learning} problem.

The target retrained concept predictor, incorporating both historical and new data, is defined as:
\begin{equation}\label{eq:add:defg}
    \hat{g}_{+\mathcal{S}_{new}} = \argmin_{g} \left[ \sum_{i=1}^n L_C(g(x_i), c_i) + \sum_{k=1}^m L_C(g(\tilde{x}_k), \tilde{c}_k) \right].
\end{equation}
Instead of viewing this as a retraining task, we formulate it as a Newton-step update, where the new data provides a gradient descent direction to correct the existing parameters.

\begin{theorem}\label{thm:add:g}
When augmenting the dataset with $\mathcal{S}_{new}$, the updated concept predictor $\hat{g}_{+\mathcal{S}_{new}}$ can be efficiently approximated by:
\begin{equation}
\hat{g}_{+\mathcal{S}_{new}} \approx \bar{g}_{+\mathcal{S}_{new}} \triangleq \hat{g} - H^{-1}_{\hat{g}} \cdot \sum_{z_k \in \mathcal{S}_{new}} \sum_{j=1}^M \nabla_g L^j_C(\tilde{x}_k, \tilde{c}_k; \hat{g}).
\end{equation}
Here, the subtraction sign indicates that we are moving parameters in the descent direction of the new data's gradients to minimize the augmented loss.
\end{theorem}

\vspace{0.5em}
\noindent \textbf{Remark (Incremental Learning Perspective).}
It is worth noting that Theorem \ref{thm:add:g} effectively performs a second-order online update. By utilizing the Hessian $H_{\hat{g}}$ computed from the original data (which can be pre-computed and stored), we project the gradients of the new samples $\mathcal{S}_{new}$ onto the curvature of the existing manifold. This approach allows the CBM to "absorb" new knowledge while maintaining stability with respect to previous concepts.

Similarly, the label predictor must adjust to define boundaries for the augmented dataset. The objective is:
\begin{equation}\label{eq:add:f}
\begin{split}
    \hat{f}_{+\mathcal{S}_{new}} = \argmin_{f} \bigg[ & \sum_{i=1}^n L_{Y}(f(\hat{g}_{+\mathcal{S}_{new}}(x_i)), y_i) \\
    & + \sum_{k=1}^m L_{Y}(f(\hat{g}_{+\mathcal{S}_{new}}(\tilde{x}_k)), \tilde{y}_k) \bigg].
\end{split}
\end{equation}
We define the intermediate predictor $\Tilde{f}_{+\mathcal{S}_{new}}$ that assimilates new data under the original concept mapping $\hat{g}$.

\begin{theorem}\label{thm:add:f}
The adaptation of the label predictor involves two distinct shifts.
First, we estimate the impact of the new samples on the decision boundary:
\begin{equation*}
\Tilde{f}_{+\mathcal{S}_{new}} - \hat{f} \approx - H^{-1}_{\hat{f}} \sum_{z_k \in \mathcal{S}_{new}} \nabla_f L_Y(\tilde{x}_k; \hat{g}, \hat{f}) \triangleq A_{add}.
\end{equation*}
Let $\bar{f}^*_{+\mathcal{S}_{new}} = \hat{f} + A_{add}$.
Second, we correct for the shift in concept representations caused by the new data. The correction term $B_{add}$ is derived as:
\begin{equation*}
\begin{split}
    B_{add} = -H^{-1}_{\bar{f}^*} \sum_{z \in \mathcal{D} \cup \mathcal{S}_{new}} \bigg( & \nabla_f L_Y(z; \bar{g}_{+\mathcal{S}_{new}}, \bar{f}^*) \\
    & - \nabla_f L_Y(z; \hat{g}, \bar{f}^*) \bigg),
\end{split}
\end{equation*}
where $H_{\bar{f}^*}$ is the Hessian with respect to the intermediate predictor. The final updated label predictor is $\bar{f}_{+\mathcal{S}_{new}} = \hat{f} + A_{add} + B_{add}$.
\end{theorem}

\subsubsection{Acceleration via EK-FAC.} As mentioned in Section \ref{sec:preliminary}, the loss function in CBMs is non-convex, meaning the Hessian matrices in all our theorems may not be well-defined. To address this, we adopt the EK-FAC approach, where the Hessian is approximated as \(\hat{H}_{\theta} = G_{\theta} + \delta I\). Here, \(G_{\theta}\) represents the Fisher information matrix of the model \(\theta\), and \(\delta\) is a small damping term introduced to ensure positive definiteness. For details on applying EK-FAC to CBMs, see Appendix \ref{sec:FAC_CBM}. Additionally, refer to Algorithms \ref{alg:4}-\ref{alg:6} in the Appendix for the EK-FAC-based algorithms corresponding to our three levels (where the Data Level includes additions and removals), with their original (Hessian-based) versions provided in Algorithms \ref{alg:1}-\ref{alg:3}, respectively.

\noindent{\bf Theoretical Bounds.} We provide error bounds for the concept predictor between retraining and CCBM across all three levels; see Appendix \ref{app:bound_cc}, \ref{app:bound_c} and \ref{app:bound_d_the} for details.  We show that under certain scenarios, the approximation error becomes tolerable theoretically when leveraging some damping term $\delta$ regularized in the Hessian matrix.  

%% file: fig/figWorkFlow.tex
\begin{figure*}[ht]
\centering
\vspace{-6pt}
\includegraphics[width=\linewidth]{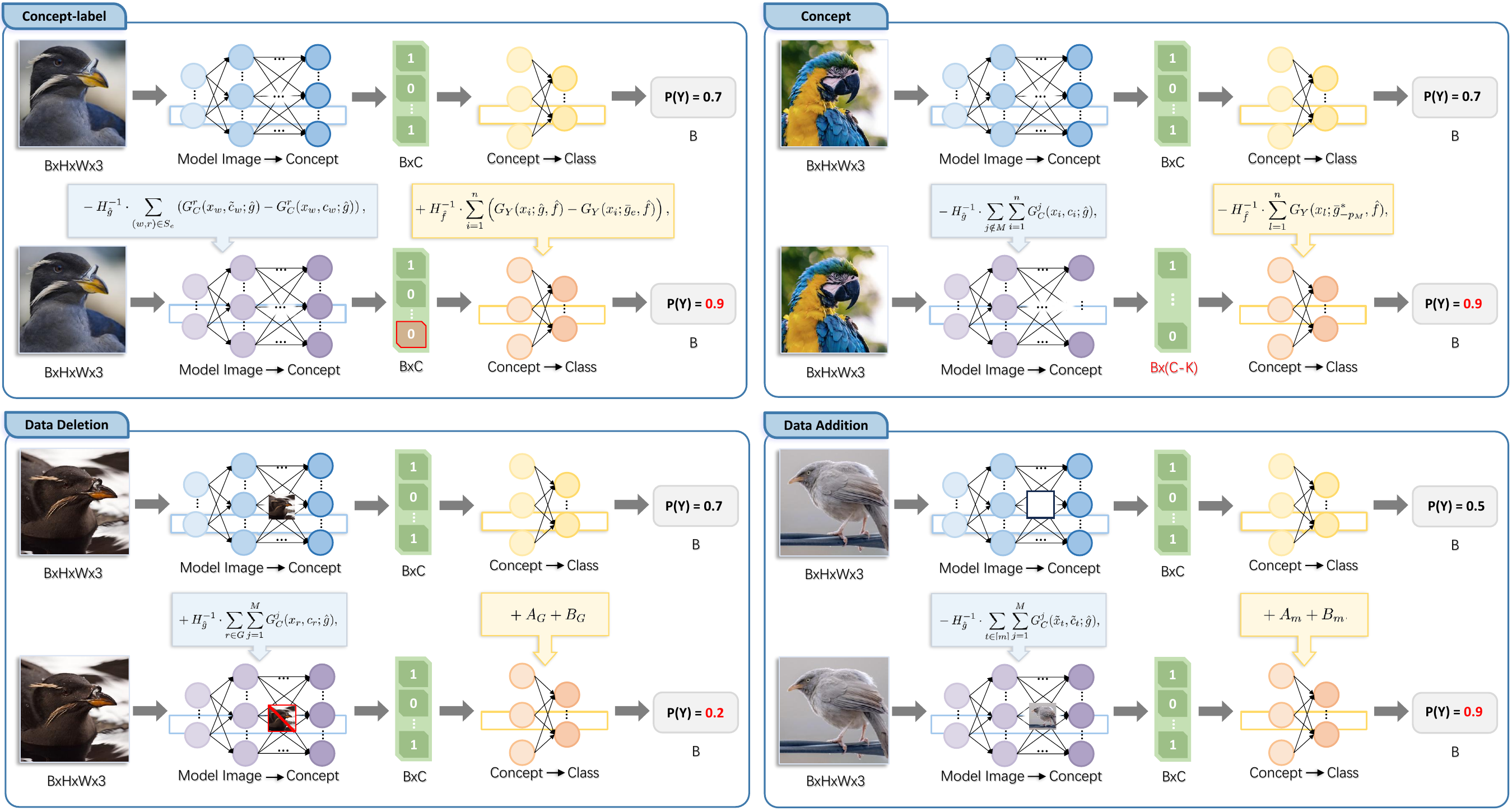}
\vspace{-20pt}
\caption{The workflows of Controllable Concept Bottleneck Models with four settings. \label{fig:workflow}}
\vspace{-10pt}
\end{figure*}

%% file: sections/5_experiment.tex
\section{Experiments}
\label{sec:exp}

In this section, we present a comprehensive evaluation of the proposed Controllable Concept Bottleneck Models (CCBMs). We assess the framework from three primary dimensions: utility preservation, editing efficiency, and interpretability. Due to space constraints, additional experimental details and supplementary results are provided in Appendix \ref{sec:appendix:exp}.

\subsection{Experimental Settings}

\noindent\textbf{Datasets.} We evaluate our method on four diverse benchmarks: \textit{X-ray Grading (OAI)}~\cite{nevitt2006osteoarthritis}, \textit{Bird Identification (CUB)}~\cite{wah2011caltech}, \textit{Large-scale CelebFaces Attributes (CelebA)}~\cite{liu2015deep}, and \textit{Derm7pt}~\cite{combalia2019dermoscopy}.

\begin{itemize}
    \item \textbf{OAI:} A multi-center observational study of knee osteoarthritis comprising 36,369 data points. We configure $n=10$ concepts characterizing crucial osteoarthritis indicators (e.g., joint space narrowing, osteophytes). Following the protocol in~\cite{koh2020concept}, we preprocess the data to align with standard CBM setups.
    
    \item \textbf{CUB:} This dataset consists of 11,788 bird images across 200 classes, annotated with 112 binary attributes (concepts). We follow the processing setting in~\cite{koh2020concept}. Specifically, we aggregate instance-level concept annotations into class-level concepts via majority voting: for instance, if more than 50\% of samples in a class (e.g., crows) exhibit a feature (e.g., black wings), that concept is assigned to the entire class.
    
    \item \textbf{CelebA:} A large-scale face attributes dataset with 202,599 images. Each image is annotated with 40 binary attributes. Following~\cite{NEURIPS2022_867c0682}, we designate 8 attributes as classification targets (labels) and the remaining 32 attributes as the bottleneck concepts.
    
    \item \textbf{Derm7pt:} A dermoscopic benchmark designed for melanoma diagnosis, containing 1,011 images.  It is annotated with 7 clinically defined dermoscopic criteria (e.g., pigment network, streaks, blue-whitish veil) serving as concepts, and a binary diagnostic label (melanoma vs. non-melanoma). Following~\cite{wang2024concept}, We utilize the official data split and standard normalization procedures.
\end{itemize}

\noindent\textbf{Evaluation Metrics.} We utilize two primary metrics:
\begin{enumerate}
    \item \textit{F1 Score:} Measures the model utility by balancing precision and recall.
    \item \textit{Runtime (RT):} Measured in minutes, evaluating the computational time required to complete the model update.
\end{enumerate}

\noindent\textbf{Implementation Details.} Experiments were conducted on an Intel Xeon CPU and an NVIDIA RTX 3090 GPU.   
For \textit{concept-label-level} editing, we randomly selected 3\% of data points and flipped one concept label for each, repeated 10 times.
For \textit{concept-level} editing, we randomly removed one concept for OAI and ten concepts for CUB, repeating the process with five different seeds.
For \textit{data-level} editing, we evaluate both unlearning and addition scenarios:
(1) For \textit{removal}, 3\% of the training data was randomly deleted from the full dataset, repeated 10 times.
(2) For \textit{addition}, we initially held out a random 10\% subset of the training data during the first training phase and subsequently added these samples back to update the model.
To test resilience in \textit{multiple sample editing} (Figure \ref{fig:hyperpara}), we removed concepts (from 2 to 20) and deleted data points (ratios from 1\% to 10\%) sequentially.

\input{tabs/tab_bs}

\subsection{Evaluation of Utility and Editing Efficiency}

\noindent\textbf{Comparison with Retraining.}
Our main results, summarized in Table \ref{tab:results_app}, highlight the superiority of CCBMs over traditional retraining and the baseline CBM-IF method. The most significant advantage lies in computational efficiency. For instance, on the OAI dataset, CCBMs reduce the update time from 297.77 minutes (Retrain) to just 2.36 minutes—a speedup of over $100\times$—while maintaining an F1 score (0.8808) highly comparable to retraining (0.8825). This trend is consistent across the CUB dataset, where runtime drops from 85.56 minutes to 0.65 minutes with negligible performance loss (0.7971 vs. 0.7963).
Compared to the standard influence function baseline (CBM-IF), CCBMs (empowered by EK-FAC) achieve not only faster updates but also consistently higher F1 scores. This improvement suggests that the EK-FAC approximation effectively stabilizes the Hessian inversion, which is often ill-conditioned in non-convex settings, thereby providing a more accurate update direction.

\input{fig/fig3}

\noindent\textbf{Editing Multiple Samples.}
To assess the robustness of CCBMs under extensive modifications, we conducted stress tests by editing varying ratios of samples and concepts. As illustrated in Figure \ref{fig:hyperpara}, CCBMs demonstrate remarkable stability. 
In the Data-level and Concept-label-level settings, the performance degradation of CCBM relative to Retraining is minimal (gap $< 0.0025$) even when editing up to 10\% of the data.
At the Concept-level, while there is a slight divergence as the number of removed concepts increases, CCBM consistently outperforms CBM-IF. 
Considering the substantial time savings (more than $3\times$ faster than retraining in these batched scenarios), CCBM represents an optimal trade-off between editing speed and model utility.

\subsection{Results on Interpretability}

\noindent\textbf{Quantifying Concept Importance via Influence Functions.}
A core advantage of CBMs is their semantic transparency. CCBMs extend this by providing a rigorous mathematical framework to quantify the *contribution* of each concept to the final prediction. We validated this capability through a "Concept Ablation" experiment on the CUB dataset. We utilized the proposed influence functions to rank all concepts by their impact and performed two contrasting tests: removing the top-10 \textit{most influential} concepts versus the bottom-10 \textit{least influential} concepts.

Figure \ref{fig:positive} and Figure \ref{fig:negative} (in Fig \ref{fig:2}) illustrate the F1 score trajectories. We observe three critical findings:
\begin{enumerate}
    \item \textbf{discriminative Power:} Removing the \textit{most influential} concepts (Figure \ref{fig:positive}) causes a sharp, monotonic degradation in model performance. The F1 score drops by over 0.025 (from $\approx 0.780$ to $<0.755$) after removing just 10 concepts. This confirms that the influence function correctly identifies the "load-bearing" semantic features that drive the model's decision-making.
    \item \textbf{Noise Identification:} Conversely, removing the \textit{least influential} concepts (Figure \ref{fig:negative}) results in a negligible performance shift (change $< 0.005$). The flat trajectory indicates that CCBM effectively isolates redundant or non-informative concepts, preventing the model from over-reliance on spurious features.
    \item \textbf{Alignment with Ground Truth:} Most importantly, the performance degradation curve of CCBM (dark blue bars) closely mirrors that of Retraining (light blue bars). The divergence between the two methods is consistently below $0.005$. This high fidelity demonstrates that the influence-based approximation accurately captures the true causal effect of concepts on the loss landscape, validating CCBM as a reliable tool for explaining model behavior.
\end{enumerate}

\input{fig/fig2}
\input{fig/figMIA}

\noindent\textbf{CCBMs can erase data influence.}
To verify the effectiveness of data-level controllability (specifically unlearning), we employed Membership Inference Attacks (MIA). In a privacy-compliant unlearning scenario, a removed training sample should become indistinguishable from a sample the model has never seen (a non-member). 

To rigorously quantify this, we utilize the \textbf{RMIA (Removed Membership Inference Attack) Score}. The RMIA score approximates the likelihood ratio of a sample belonging to the training set versus the non-member distribution, defined as:
\begin{equation}
\small
\begin{aligned}
LR_\theta(x, z) \approx \frac{\text{Pr}(f_\theta(x)|\mathcal{N}(\mu_{x,\bar{z}}(x), \sigma^2_{x,\bar{z}}(x)))}{\text{Pr}(f_\theta(x)|\mathcal{N}(\mu_{\bar{x},z}(x), \sigma^2_{\bar{x},z}(x)))} \\
\times \frac{\text{Pr}(f_\theta(z)|\mathcal{N}(\mu_{x,\bar{z}}(z), \sigma^2_{x,\bar{z}}(z)))}{\text{Pr}(f_\theta(z)|\mathcal{N}(\mu_{\bar{x},z}(z), \sigma^2_{\bar{x},z}(z)))},
\end{aligned}
\end{equation}
where $f_\theta(x)$ denotes the model logits, and $\mathcal{N}(\cdot)$ represents the Gaussian distribution fitted to member/non-member hypotheses~\cite{zarifzadeh2024low}. Higher RMIA scores indicate a high probability of being a training member.

We computed the RMIA scores for 200 training members and 200 non-members. The frequency distributions are plotted in Figure \ref{fig:rmia-1} (Before Editing) and Figure \ref{fig:rmia-2} (After Editing).
\begin{itemize}
    \item \textbf{Before Editing (Figure \ref{fig:rmia-1}):} There is a distinct separation between the distributions. The members (blue bars, red fitted curve) exhibit significantly higher RMIA scores (Mean $\mu=0.0635$) compared to non-members (green curve, $\mu=0.0495$). This separation represents the "privacy leak" or the data footprint stored in the model.
    \item \textbf{After Editing (Figure \ref{fig:rmia-2}):} After applying CCBM to unlearn the member samples, the distribution of the removed members shifts markedly to the left. The new mean RMIA score decreases to $\mu=0.0521$, converging towards the non-member baseline. Visually, the red curve (removed members) now significantly overlaps with the green curve (non-members).
\end{itemize}
This statistical shift confirms that CCBM successfully "sanitizes" the model. The removed samples no longer trigger high-confidence membership signals, effectively mitigating privacy risks and satisfying the unlearning requirement.


\subsection{Improvement via Harmful Data Removal}

To rigorously evaluate the capability of CCBMs in restoring model performance degraded by data corruption, we conducted additional experiments on the CUB dataset using a "Noise Injection and Removal" protocol. Specifically, we synthetically introduced noise into the training set at three distinct granularities:
\begin{itemize}
    \item \textbf{Concept Level:} We randomly selected 10\% of the concept attributes and flipped their binary values for a subset of the training data.
    \item \textbf{Data Level:} We randomly selected 10\% of the training samples and corrupted their ground-truth classification labels ($y$).
    \item \textbf{Concept-Label Level:} We introduced fine-grained noise by randomly flipping 10\% of the individual concept annotations ($c_{ij}$) across the entire dataset.
\end{itemize}

\begin{figure}[h]
    \centering
    \begin{tabular}{c}
        \includegraphics[width=0.95\linewidth]{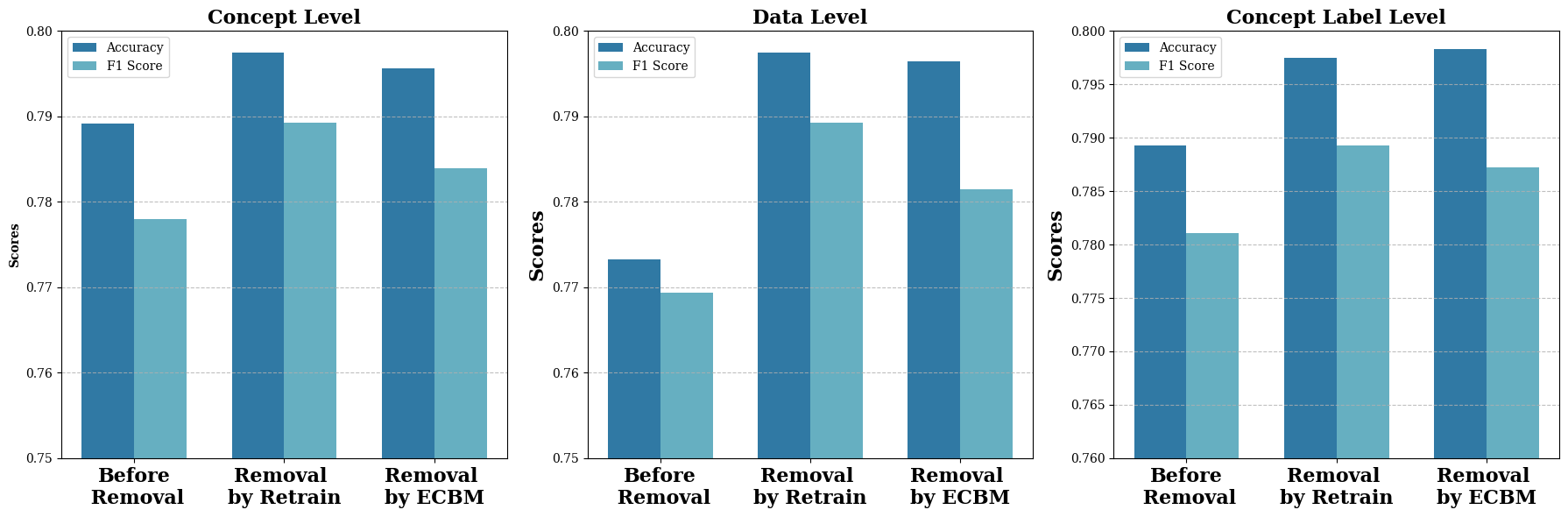} 
    \end{tabular}
    \caption{Model performance after the removal of harmful data.}
    \label{fig:harmful_removal}
\end{figure}

Following the noise injection, we trained a baseline CBM on the corrupted dataset ("Before Removal"). Subsequently, we identified the noisy components and removed them using two methods: (1) \textbf{Retraining} from scratch on the cleaned dataset (serving as the Ground Truth), and (2) \textbf{CCBM}, which edits the trained noisy model to unlearn the harmful components directly.

The results are visualized in Figure \ref{fig:harmful_removal}. We present a detailed analysis of the findings below:

\noindent\textbf{Significant Performance Recovery.} 
As observed in the "Before Removal" bars, the injection of noise noticeably degrades model utility. This impact is most pronounced at the \textit{Data Level}, where label corruption drops the Accuracy to approximately 0.773. However, after applying CCBM to remove these harmful samples, the model performance exhibits a sharp recovery, with Accuracy rebounding to $\approx$ 0.796 and F1 Score improving from $\approx$ 0.769 to $\approx$ 0.782. Similar restorative trends are observed at the Concept and Concept-Label levels, confirming that CCBM effectively neutralizes the negative impact of erroneous data.

\noindent\textbf{High Fidelity to Retraining.} 
A critical criterion for editable models is their ability to approximate the gold-standard Retraining outcome. Figure \ref{fig:harmful_removal} demonstrates that CCBM achieves this with high fidelity. Across all three settings, the performance gap between "Removal by CCBM" and "Removal by Retrain" is marginal. For instance, at the Concept-Label level, the F1 Score difference is less than 0.003. While there is a slight performance drop compared to complete retraining (attributed to the approximation error of the influence function), CCBM retains over 99\% of the recovery gain. This suggests that CCBM serves as a robust and efficient alternative to retraining for data sanitation tasks, allowing for rapid correction of dataset errors without the computational overhead of training from scratch.

\subsection{Periodic Editing Performance}

\input{fig_ICLR/app_fig_period}
In real-world lifecycle management, models are not edited just once; they undergo continuous, periodic updates. A critical challenge in such scenarios is \textit{error accumulation}: since influence functions provide a first-order approximation of parameter changes, repeated application of these approximations could theoretically lead to a divergence from the optimal model parameters over time.

To rigorously evaluate the stability of CCBM under sequential editing, we designed a "Multi-Stage Restoration" experiment. 
\textbf{Setup:} We initially injected 10\% noise into the CUB dataset across the three levels (Concept, Data, and Concept-Label), identical to the settings in the previous section.
\textbf{Protocol:} Instead of removing all noise at once, we performed a sequential cleanup process consisting of 10 rounds. in each round, we identified and removed/corrected 1\% of the noisy data (or concepts). Crucially, for the CCBM method, the update at round $t$ was applied directly to the approximate model resulting from round $t-1$, testing the method's resilience to cumulative approximation errors.

The trajectories of Accuracy and F1 Score across these 10 rounds are plotted in Figure~\ref{fig:pe_con} (Concept Level), Figure~\ref{fig:pe_da} (Data Level), and Figure~\ref{fig:pe_con_la} (Concept-Label Level). We observe the following key trends:

\noindent\textbf{Consistent Monotonic Restoration.} 
As we progress from 10\% noise (left side of x-axis) to 0\% noise (right side), both Retraining (Blue) and CCBM (Orange) exhibit a consistent upward trajectory in performance. For instance, at the Data Level (Figure \ref{fig:pe_da}), the Accuracy improves from $\sim$0.773 to $\sim$0.796. This confirms that CCBM effectively translates the removal of harmful components into tangible performance gains at every step of the sequence.

\noindent\textbf{Robustness to Error Accumulation.} 
The most significant finding is the stability of the CCBM trajectory. Despite being updated sequentially 10 times, the CCBM performance curve remains parallel to the Retraining curve without significant divergence. 
At the \textit{Data Level} (Figure \ref{fig:pe_da}), the CCBM curve tightly hugs the Retraining curve, with the final Accuracy gap being negligible ($<0.002$).
At the \textit{Concept Level} (Figure \ref{fig:pe_con}), while there is a systematic offset due to the hardness of the concept removal problem, the gap remains constant throughout the 10 stages. The CCBM model does not collapse or exhibit erratic behavior.

This empirical evidence suggests that the Hessian-based approximation in CCBM is sufficiently robust to support long-term, periodic maintenance of the model without requiring frequent full retraining resets.

%% file: tabs/tab_bs.tex
\begin{table*}[ht]
    \centering
    \caption{Performance comparison of different methods on the three datasets.}
    \vspace{-6pt}
    \label{tab:results_app}
    \resizebox{0.95\linewidth}{!}{
        \begin{tabular}{llcccccccc}
        \toprule
        \multirow{2}{*}{\textbf{Edit Level}} & \multirow{2}{*}{\textbf{Method}} & \multicolumn{2}{c}{\textbf{OAI}} & \multicolumn{2}{c}{\textbf{CUB}} & \multicolumn{2}{c}{\textbf{CelebA}} & \multicolumn{2}{c}{\textbf{Derm7pt}}\\
        \cmidrule(r){3-4} \cmidrule(r){5-6} \cmidrule(r){7-8} \cmidrule(r){9-10}
        & & \textbf{F1 score} & \textbf{RT (minute)} & \textbf{F1 score} & \textbf{RT (minute)} & \textbf{F1 score} & \textbf{RT (minute)} & \textbf{F1 score} & \textbf{RT (minute)} \\
        \midrule
        \multirow{3}{*}{Concept Label} & Retrain & 0.8825$\pm$0.0054 & 297.77 & 0.7971$\pm$0.0066 & 85.56 & 0.3827$\pm$0.0272	& 304.71 & 0.7826$\pm$0.0059 & 33.79 \\
        & CBM-IF(Ours) & 0.8650$\pm$0.0030 & 4.58 & 0.7710$\pm$0.0033 & 1.28 & 0.3597$\pm$0.0128 & 5.50 & 0.7689$\pm$0.0060 & 0.44 \\
        & CCBM(Ours) & \textbf{0.8809$\pm$0.0033} & \textbf{2.28} & \textbf{0.7918$\pm$0.0042} & \textbf{0.63} & \textbf{0.3831$\pm$0.0326} & \textbf{2.40} & 0.7813$\pm$0.0054 & 0.23 \\
        \midrule
        \multirow{3}{*}{Concept} & Retrain & 0.8448$\pm$0.0191 & 258.84 & 0.7811$\pm$0.0047 & 87.21 & 0.3776$\pm$0.0350 & 355.85& 0.7697$\pm$0.0065 & 36.11 \\
        & CBM-IF(Ours) & 0.8248$\pm$0.0065 & 4.87 & 0.7580$\pm$0.0059 & 1.44 & 0.3642$\pm$0.0195 &5.44 & 0.7421$\pm$0.0064 & 0.38 \\
        & CCBM(Ours) & \textbf{0.8411$\pm$0.0081} & \textbf{2.29} & \textbf{0.7794$\pm$0.0056} & \textbf{0.51} & \textbf{0.3768$\pm$0.0278} & \textbf{2.40} & 0.7609$\pm$0.0068 & 0.30 \\
        \midrule
        \multirow{3}{*}{Data Removal} & Retrain & 0.8811$\pm$0.0065 & 319.37 & 0.7838$\pm$0.0051 & 86.20 & 0.3797$\pm$0.0375 &  325.62 & 0.7753$\pm$0.0053 &  35.47 \\
        & CBM-IF(Ours) & 0.8477$\pm$0.0039 & 4.99 & 0.7625$\pm$0.0022 & 1.41 & 0.3546$\pm$0.0163 & 5.89 & 0.7648$\pm$0.0035 & 0.56 \\
        & CCBM(Ours) & \textbf{0.8799$\pm$0.0038} & \textbf{2.42} & \textbf{0.7851$\pm$0.0079} & \textbf{0.60} & \textbf{0.3754$\pm$0.0337} & \textbf{2.41} & 0.7728$\pm$0.0072 & 0.37 \\
        \midrule
        \multirow{3}{*}{Data Addition} & Retrain & 0.8801$\pm$0.0058 & 323.83 & 0.7987$\pm$0.0071 & 87.93 & 0.3830$\pm$0.0285 &  331.42 & 0.7850$\pm$0.0068 &  36.39 \\
        & CBM-IF(Ours) & 0.8479$\pm$0.0039 & 5.03 & 0.7635$\pm$0.0022 & 1.44 & 0.3558$\pm$0.0163 & 5.95 & 0.7581$\pm$0.0036 & 0.60 \\
        & CCBM(Ours) & \textbf{0.8806$\pm$0.0030} & \textbf{2.51} & \textbf{0.7967$\pm$0.0080} & \textbf{0.63} & \textbf{0.3817$\pm$0.0340} & \textbf{2.47} & 0.7831$\pm$0.0078 & 0.43 \\
        \bottomrule
    \end{tabular}}
    \vspace{-12pt}
\end{table*}

%% file: fig/fig3.tex
\begin{figure*}[ht]
\centering
\begin{tabular}{cccc}
\includegraphics[width=0.22\linewidth]{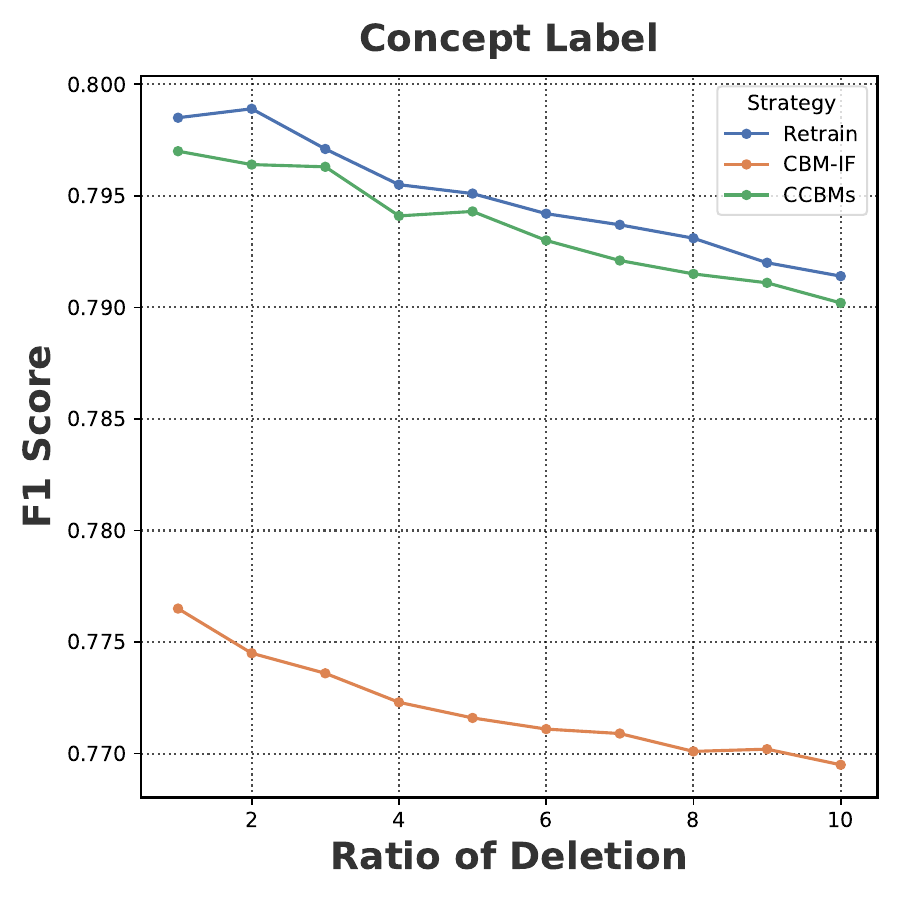} & 
\includegraphics[width=0.22\linewidth]{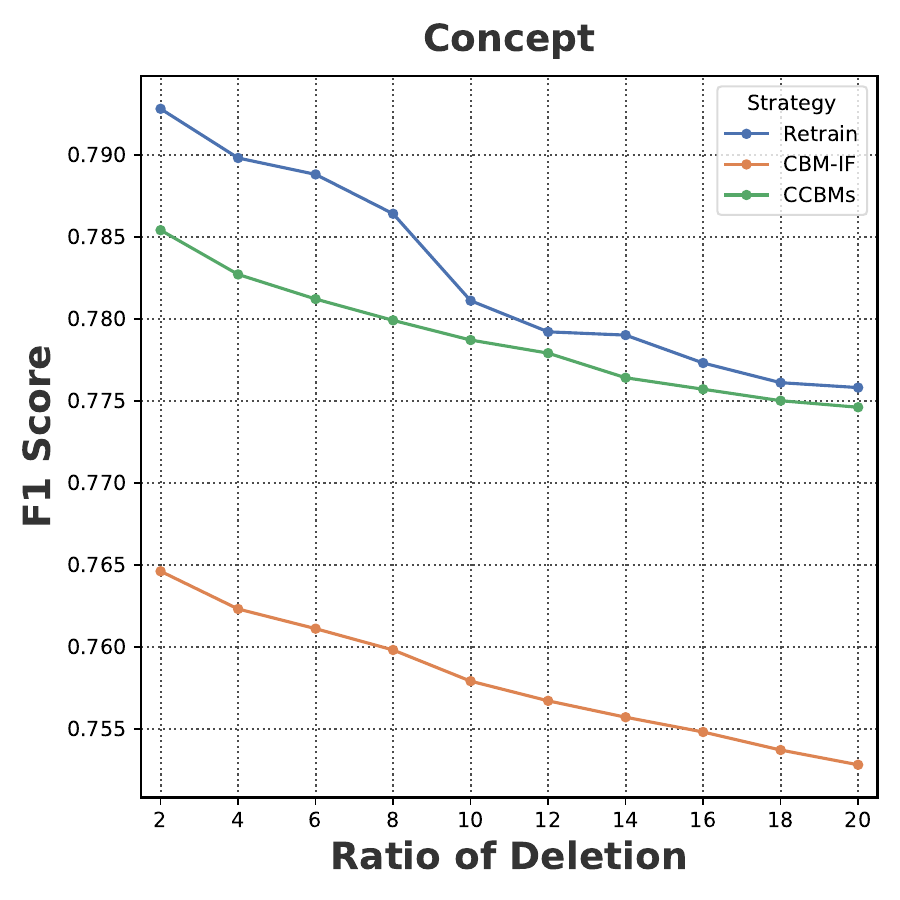} & 
\includegraphics[width=0.22\linewidth]{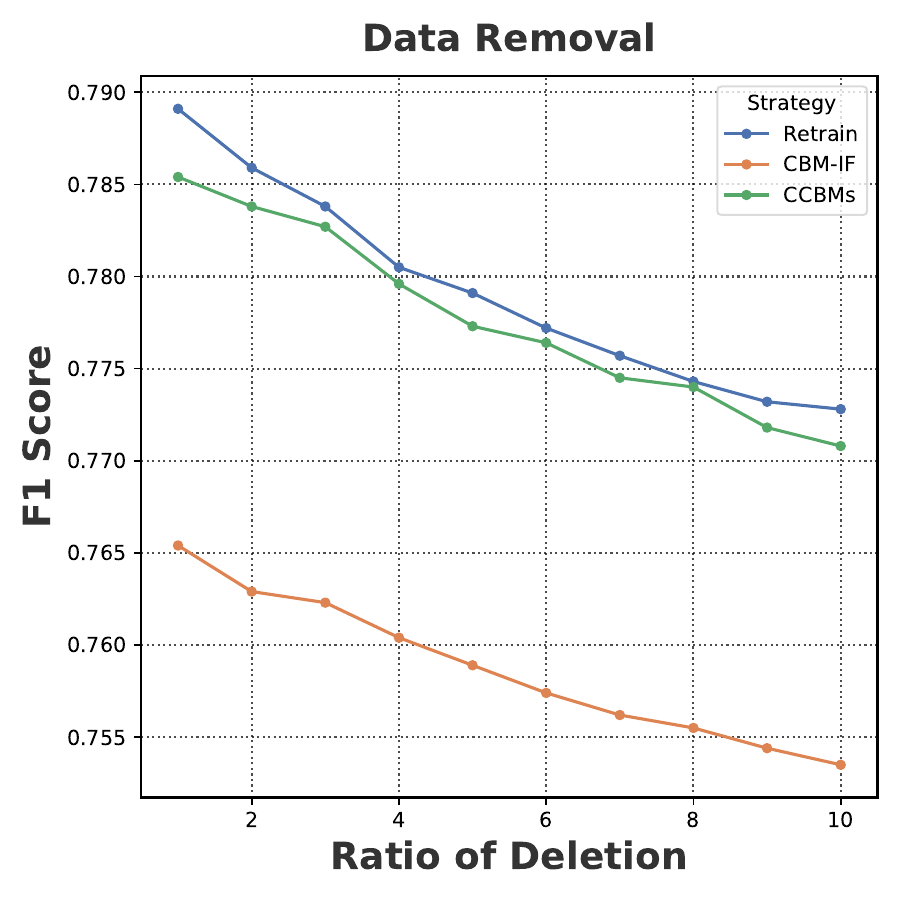} &
\includegraphics[width=0.22\linewidth]{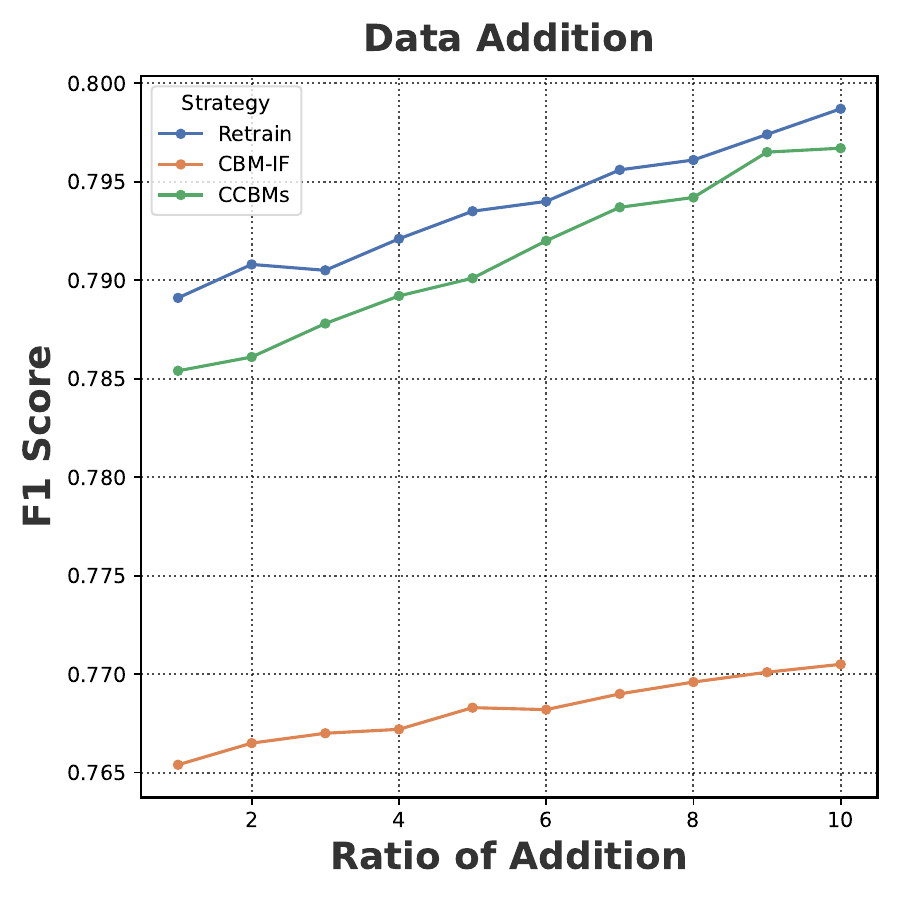}
\end{tabular}
\caption{Impact of edition ratio on three settings on CUB dataset. \label{fig:hyperpara}}
\end{figure*}

%% file: fig/fig2.tex
\begin{figure*}[!t]
\centering
\subfloat[Results on the 1-10 most influential concepts]{\includegraphics[width=2.5in]{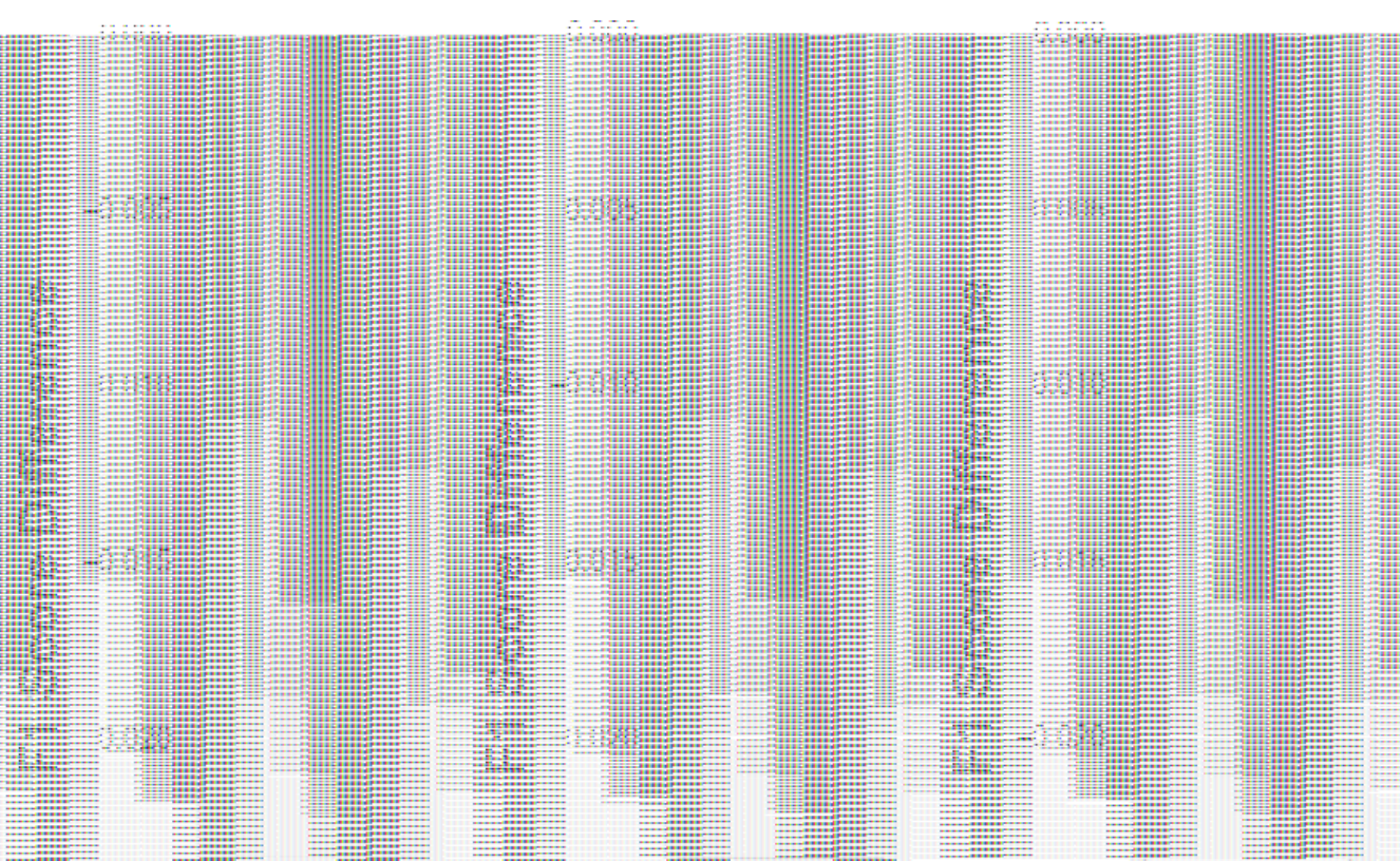}%
\label{fig:positive}}
\hfil
\subfloat[Results on the 1-10 least influential concepts]{\includegraphics[width=2.5in]{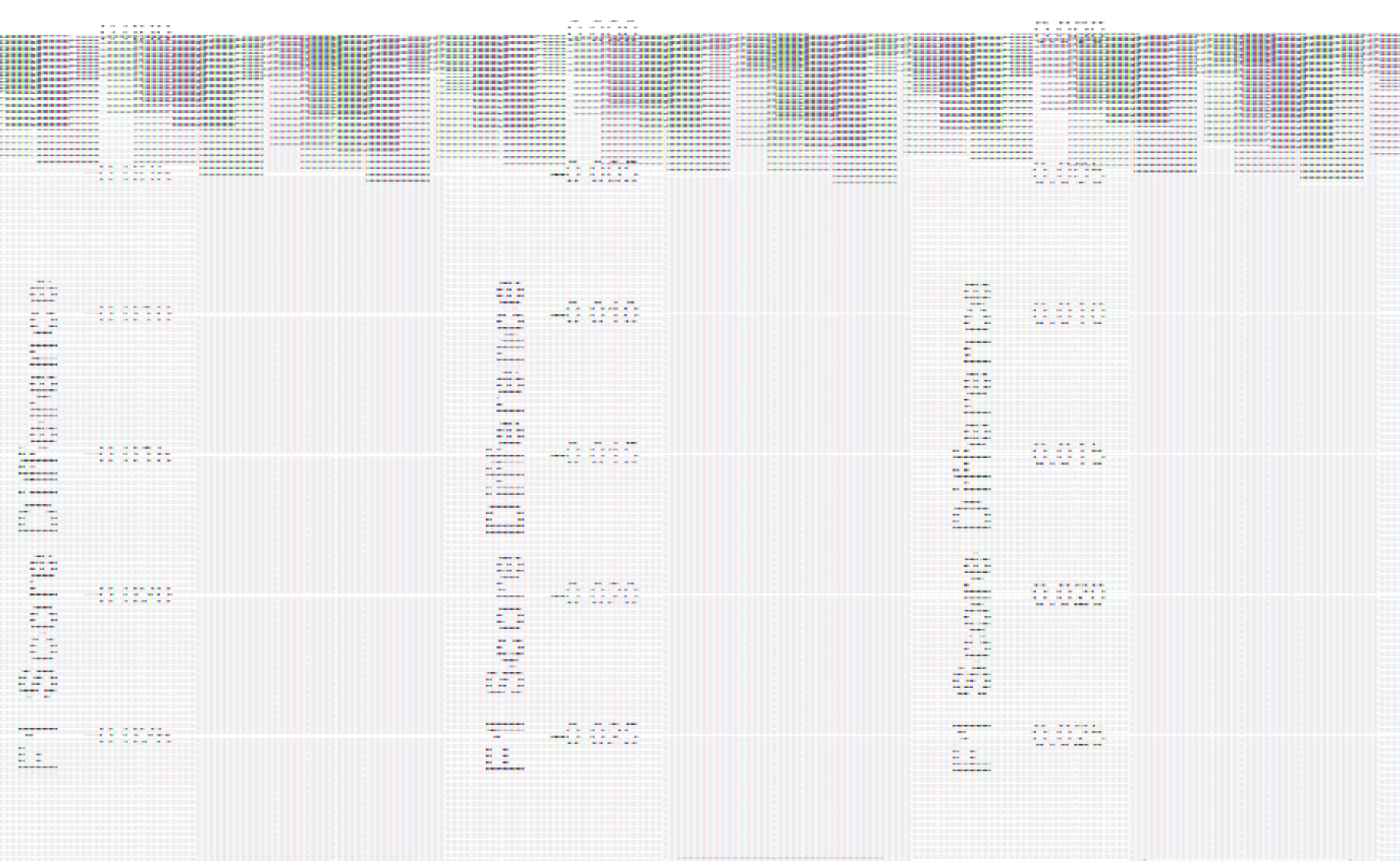}%
\label{fig:negative}}
\caption{F1 score difference after removing most and least influential concepts given by CCBM.}
\label{fig_sim}
\end{figure*}

%% file: fig/figMIA.tex
\begin{figure*}[!t]
\centering
\subfloat[RMIA Score Before Editing]{\includegraphics[width=2.5in]{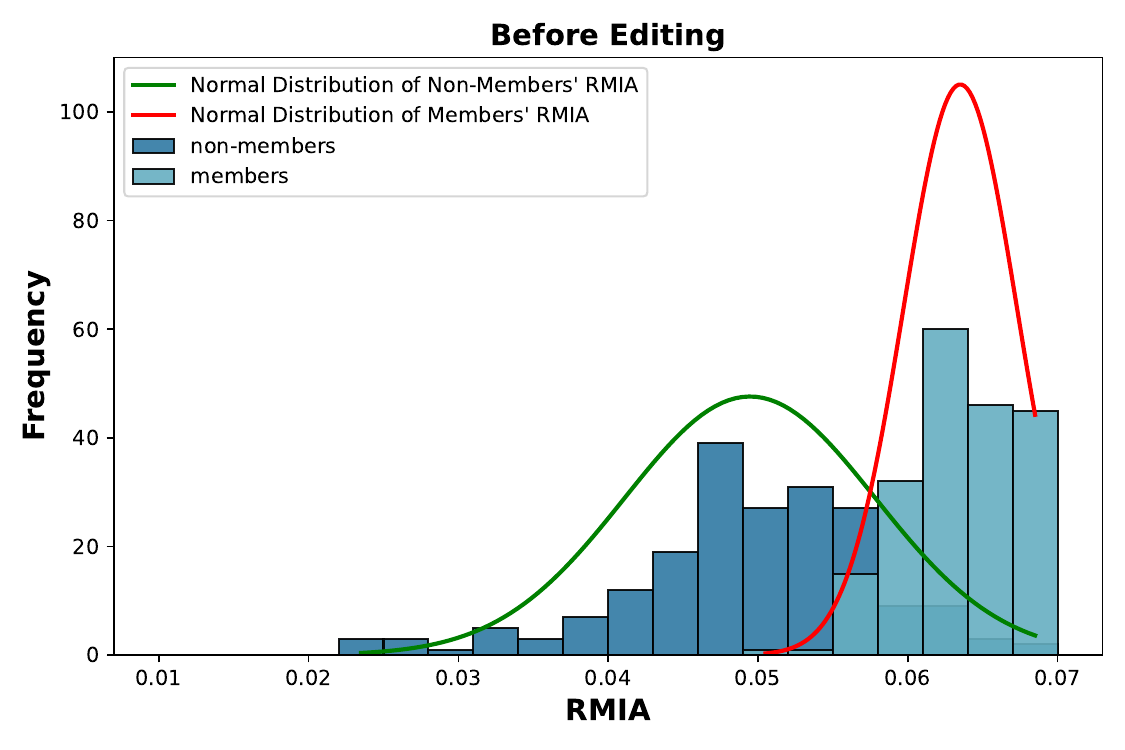}%
\label{fig:rmia-1}}
\hfil
\subfloat[RMIA Score After Editing]{\includegraphics[width=2.5in]{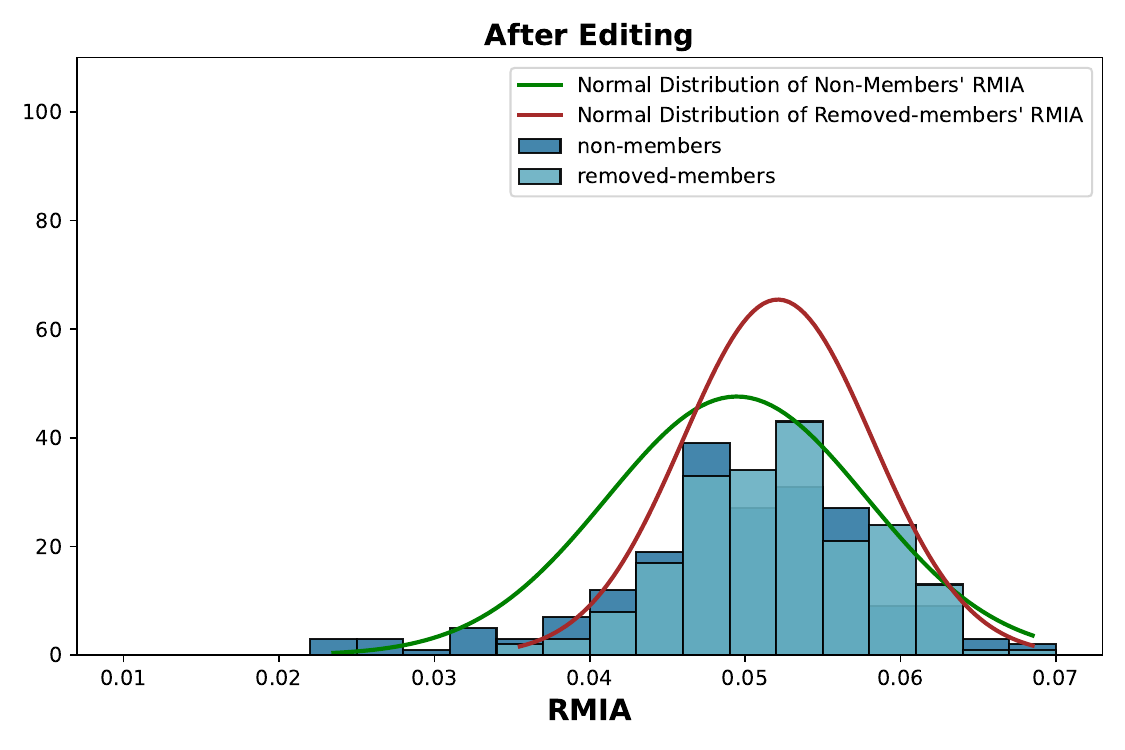}%
\label{fig:rmia-2}}
\caption{RMIA scores of data before and after removal.}
\label{fig_sim}
\end{figure*}

%% file: fig_ICLR/app_fig_period.tex
\begin{figure}[!ht]
\centering
\begin{subfigure}[b]{0.4\textwidth}
    \includegraphics[width=\linewidth]{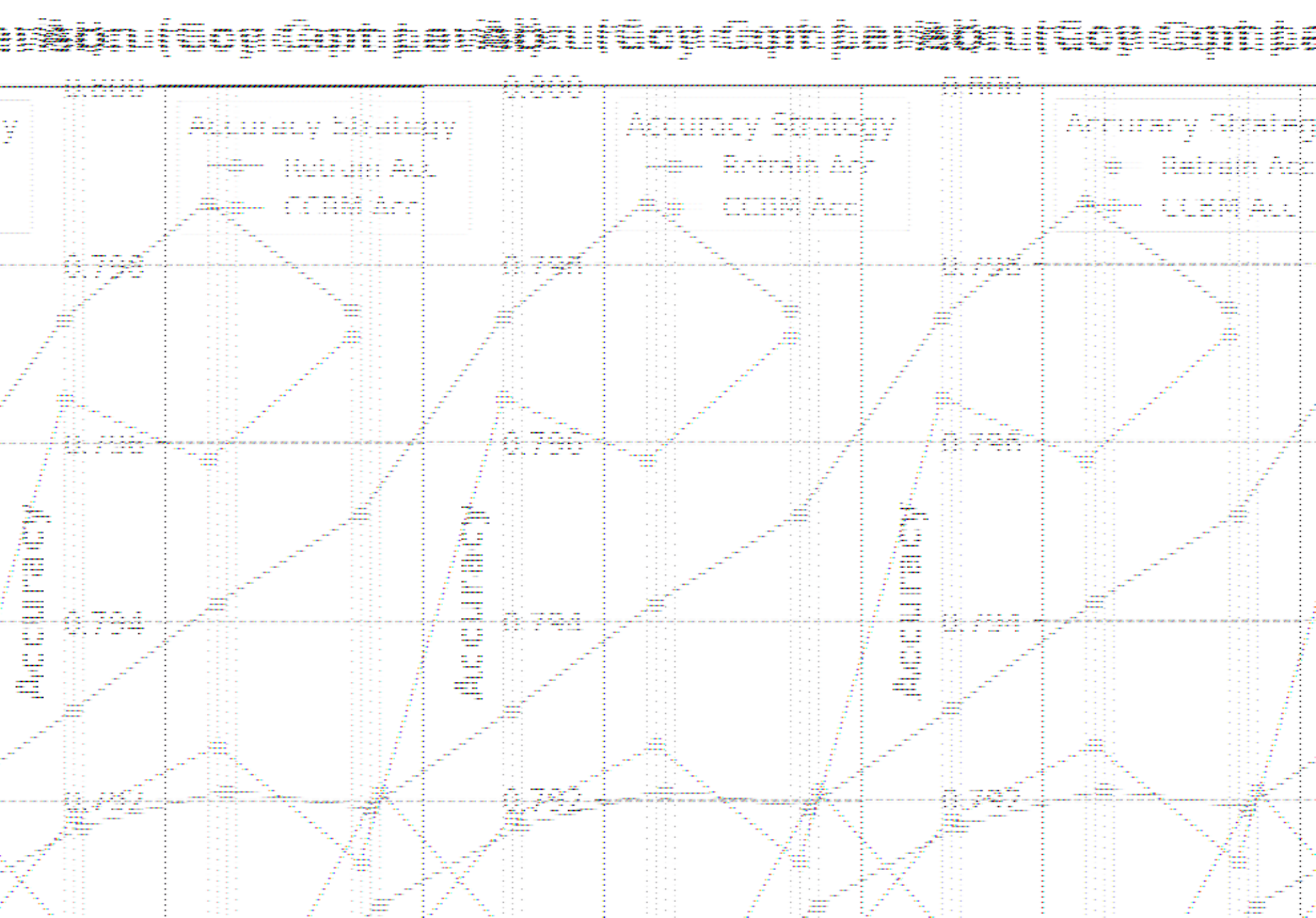}
    \caption{The accuracy of the edited model compared with retrained.}
    \label{fig:pe_con_acc}
\end{subfigure}
\hfill
\begin{subfigure}[b]{0.4\textwidth}
    \includegraphics[width=\linewidth]{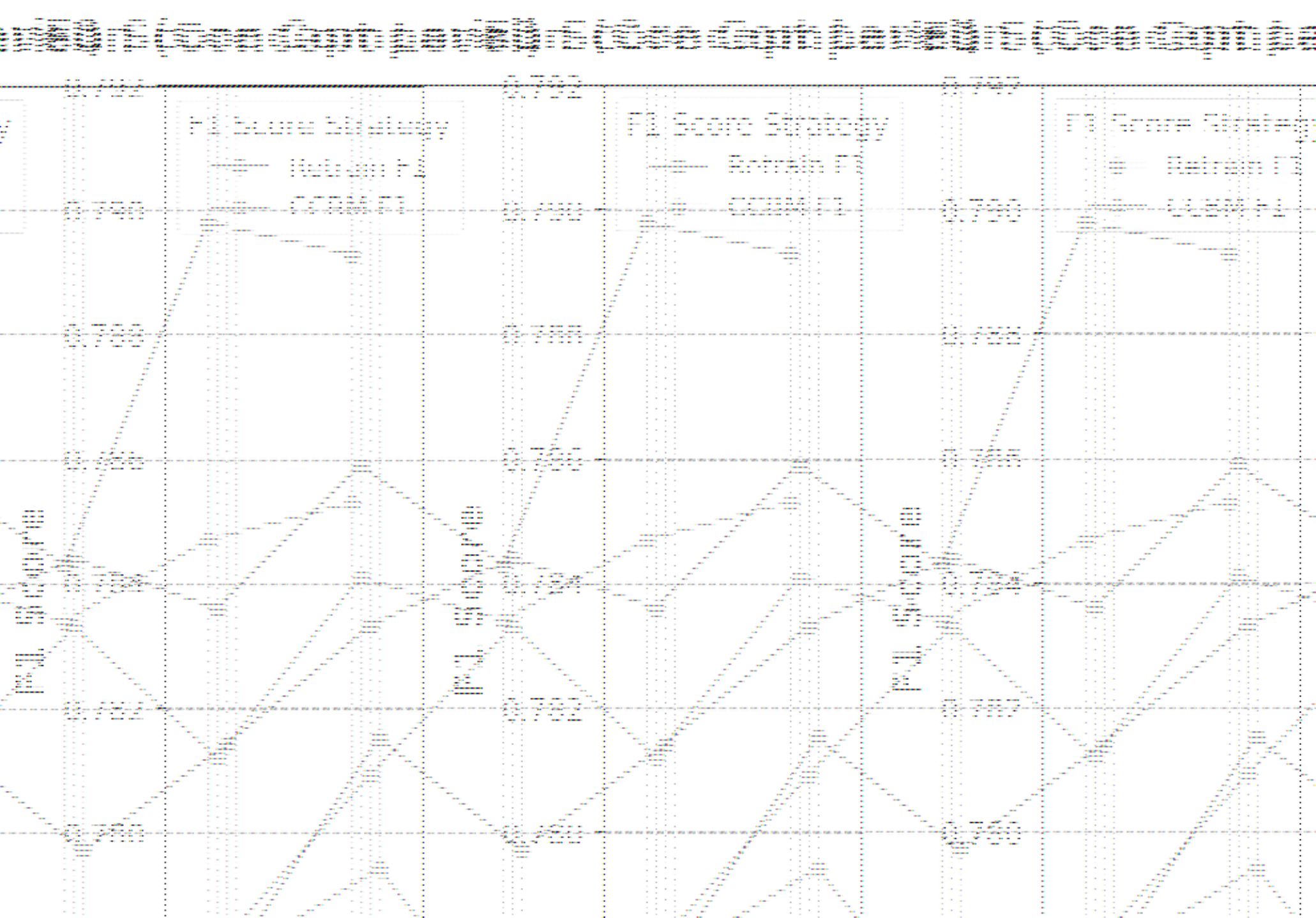}
    \caption{The F1 score of the edited model compared with retrained.}
    \label{fig:pe_con_f1}
\end{subfigure}
\vspace{-4pt}
\caption{Accuracy and F1 score difference of the edited model compared with retrained at concept level.}
\label{fig:pe_con}
\vspace{-8pt}
\end{figure}

\begin{figure}[!ht]
\centering
\begin{subfigure}[b]{0.4\textwidth}
    \includegraphics[width=\linewidth]{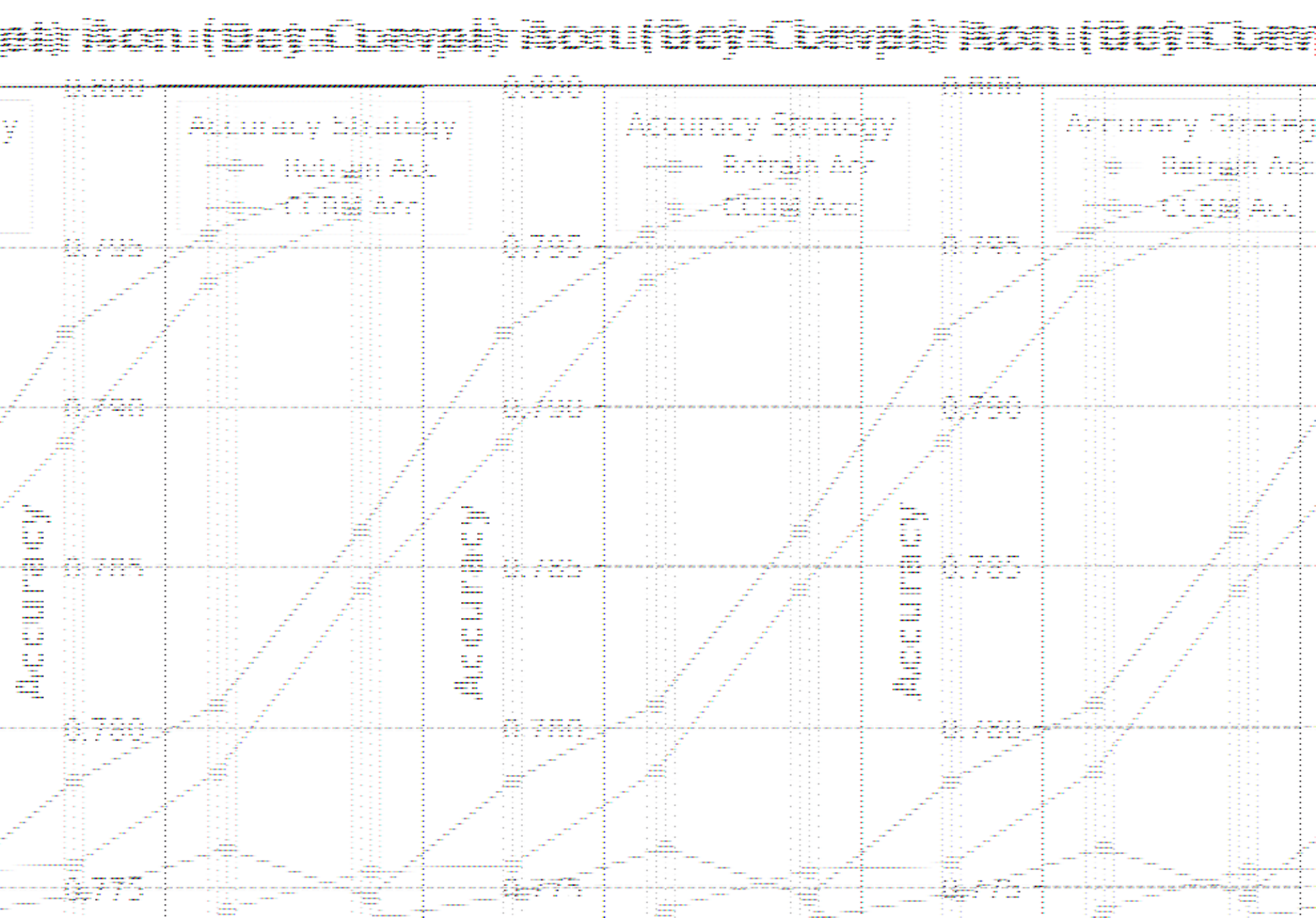}
    \caption{The accuracy of the edited model compared with retrained.}
    \label{fig:pe_da_acc}
\end{subfigure}
\hfill
\begin{subfigure}[b]{0.4\textwidth}
    \includegraphics[width=\linewidth]{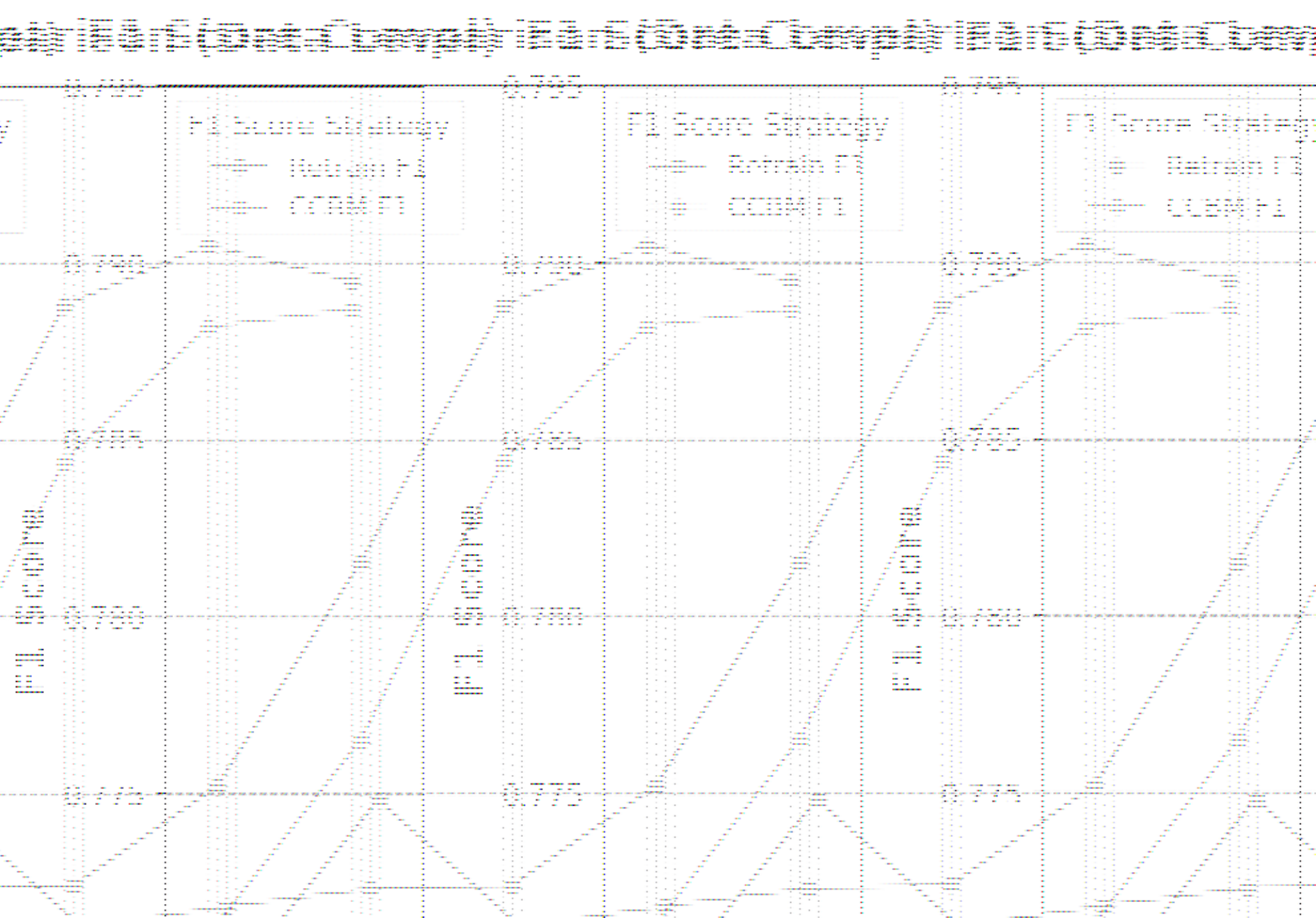}
    \caption{The F1 score of the edited model compared with retrained.}
    \label{fig:pe_da_f1}
\end{subfigure}
\vspace{-4pt}
\caption{Accuracy and F1 score difference of the edited model compared with retrained at data level.}
\label{fig:pe_da}
\vspace{-8pt}
\end{figure}

\begin{figure}[!ht]
\centering
\begin{subfigure}[b]{0.4\textwidth}
    \includegraphics[width=\linewidth]{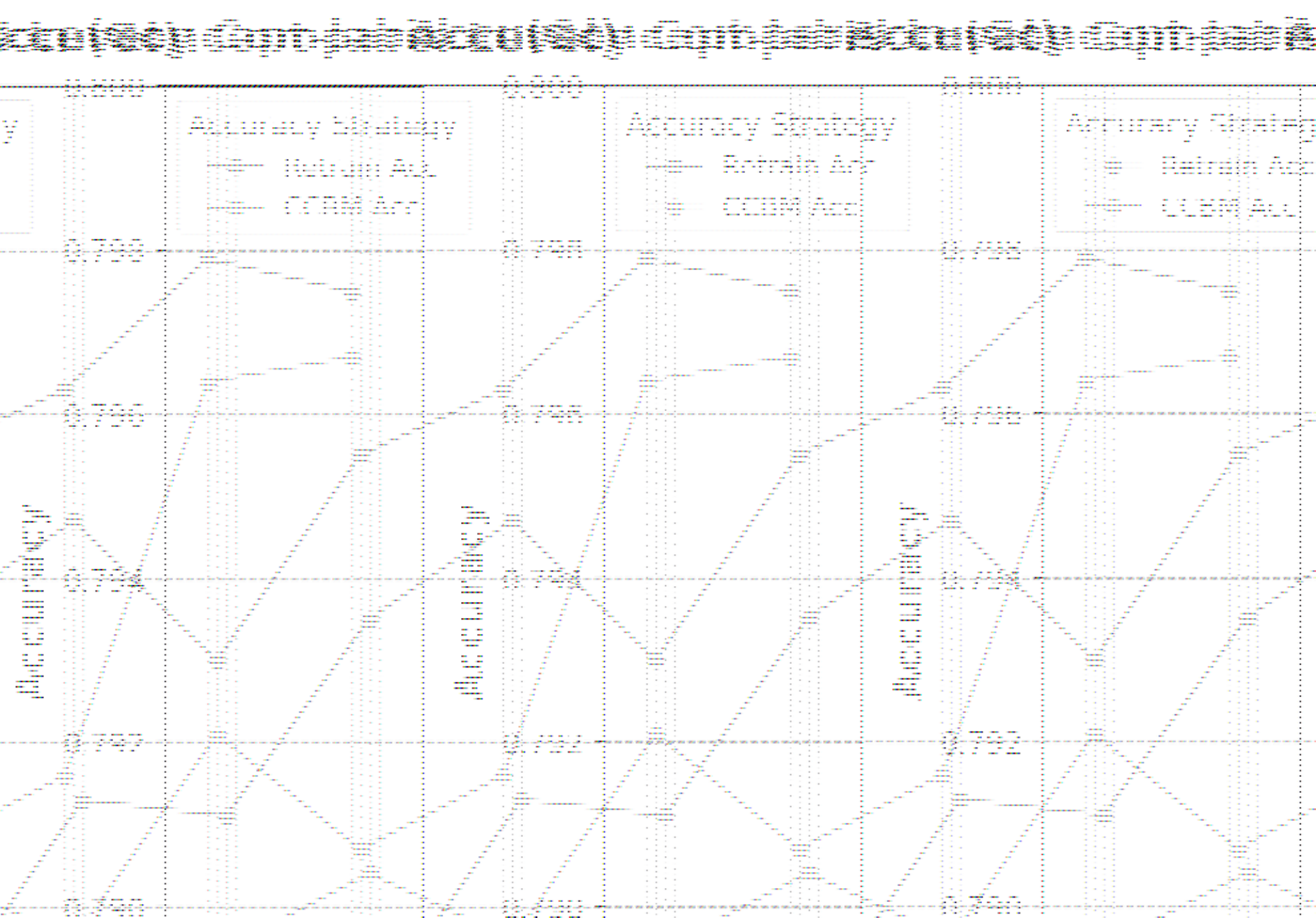}
    \caption{The accuracy of the edited model compared with retrained.}
    \label{fig:pe_con_la_acc}
\end{subfigure}
\hfill
\begin{subfigure}[b]{0.4\textwidth}
    \includegraphics[width=\linewidth]{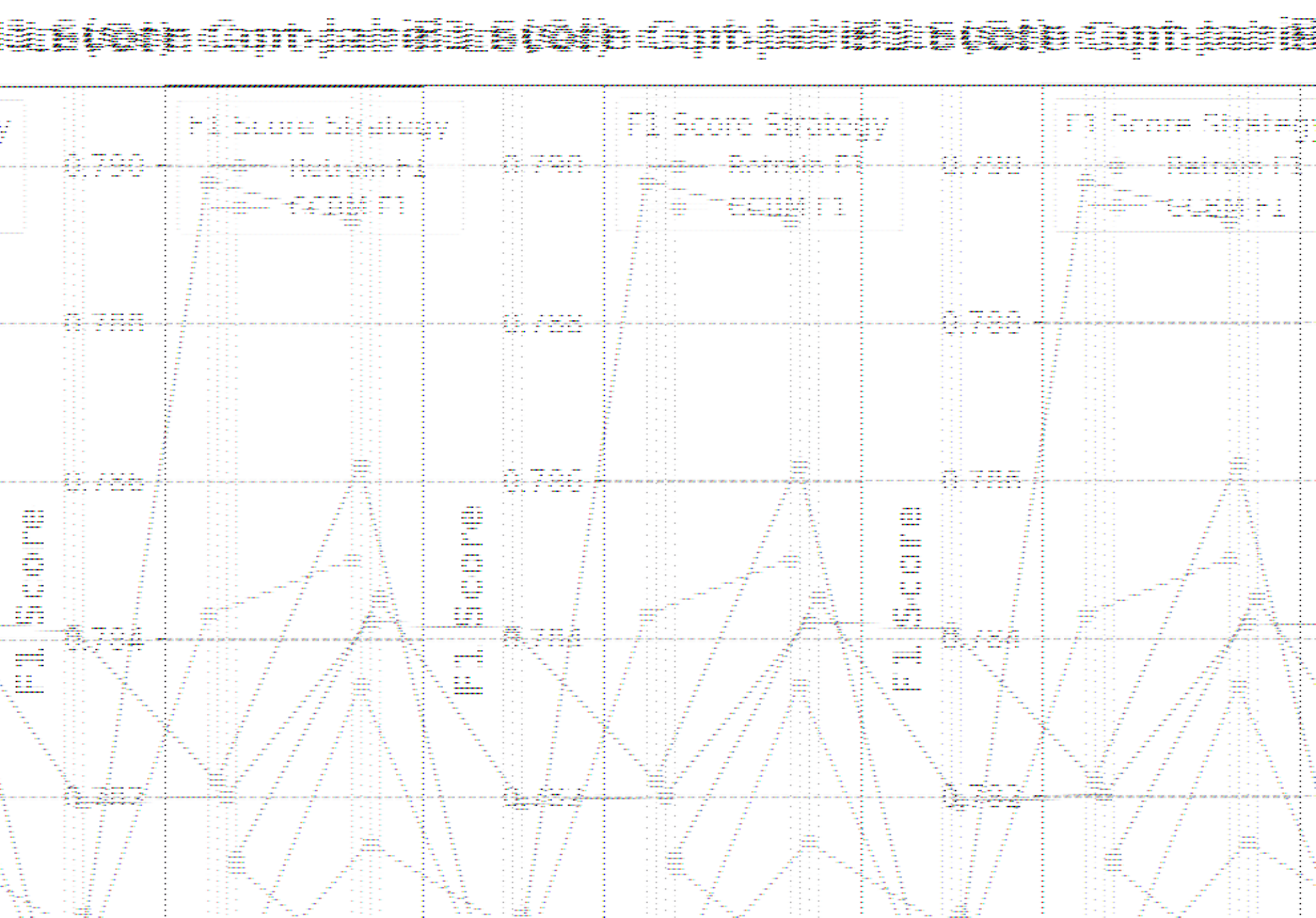}
    \caption{The F1 score of the edited model compared with retrained.}
    \label{fig:pe_con_la_f1}
\end{subfigure}
\vspace{-4pt}
\caption{Accuracy and F1 score difference of the edited model compared with retrained at concept-label level.}
\label{fig:pe_con_la}
\vspace{-8pt}
\end{figure}

%% file: sections/6_conclusion.tex
\section{Conclusion}
\label{sec:conclusion}

In this paper, we presented **Controllable Concept Bottleneck Models (CCBMs)**, a unified framework designed to address the critical limitation of staticity in traditional interpretable models. Recognizing that real-world deployment requires continuous adaptation, we formalized the problem of post-hoc model controllability across three distinct granularities: \textit{concept-label-level} for correcting annotation errors, \textit{concept-level} for evolving semantic ontologies, and \textit{data-level} for managing the training data lifecycle. Notably, our framework supports bidirectional data operations, enabling both the rigorous \textit{unlearning} of sensitive or harmful samples and the efficient \textit{addition} of new data for incremental learning.

Methodologically, we derived mathematically rigorous closed-form approximations based on influence functions, allowing for precise parameter updates without the prohibitive cost of retraining from scratch. To ensure scalability for high-dimensional models, we further integrated Eigenvalue-corrected Kronecker-Factored Approximate Curvature (EK-FAC) to accelerate Hessian computations. 

Extensive experiments on multiple benchmarks demonstrated that CCBMs achieve a superior trade-off between utility and efficiency. Our method delivers model updates that are theoretically aligned with ground-truth retraining while reducing computational time by orders of magnitude. Furthermore, we empirically validated the privacy benefits of our approach through membership inference attacks, confirming its capability to effectively erase data footprints. We believe CCBMs pave the way for more sustainable, adaptive, and trustworthy Explainable AI systems capable of evolving alongside human knowledge and dynamic data environments.

%% file: sections/7_impact_statement.tex
\section{Impact Statement}
This research enhances methodologies for data attribution in contrastive learning. While we acknowledge the importance of evaluating societal impacts, our analysis suggests that there are no immediate ethical risks requiring specific mitigation measures beyond standard practices in machine learning.

%% file: IEEEbiography.tex
\begin{IEEEbiography}[{\includegraphics[width=1in,height=1.25in,clip,keepaspectratio]{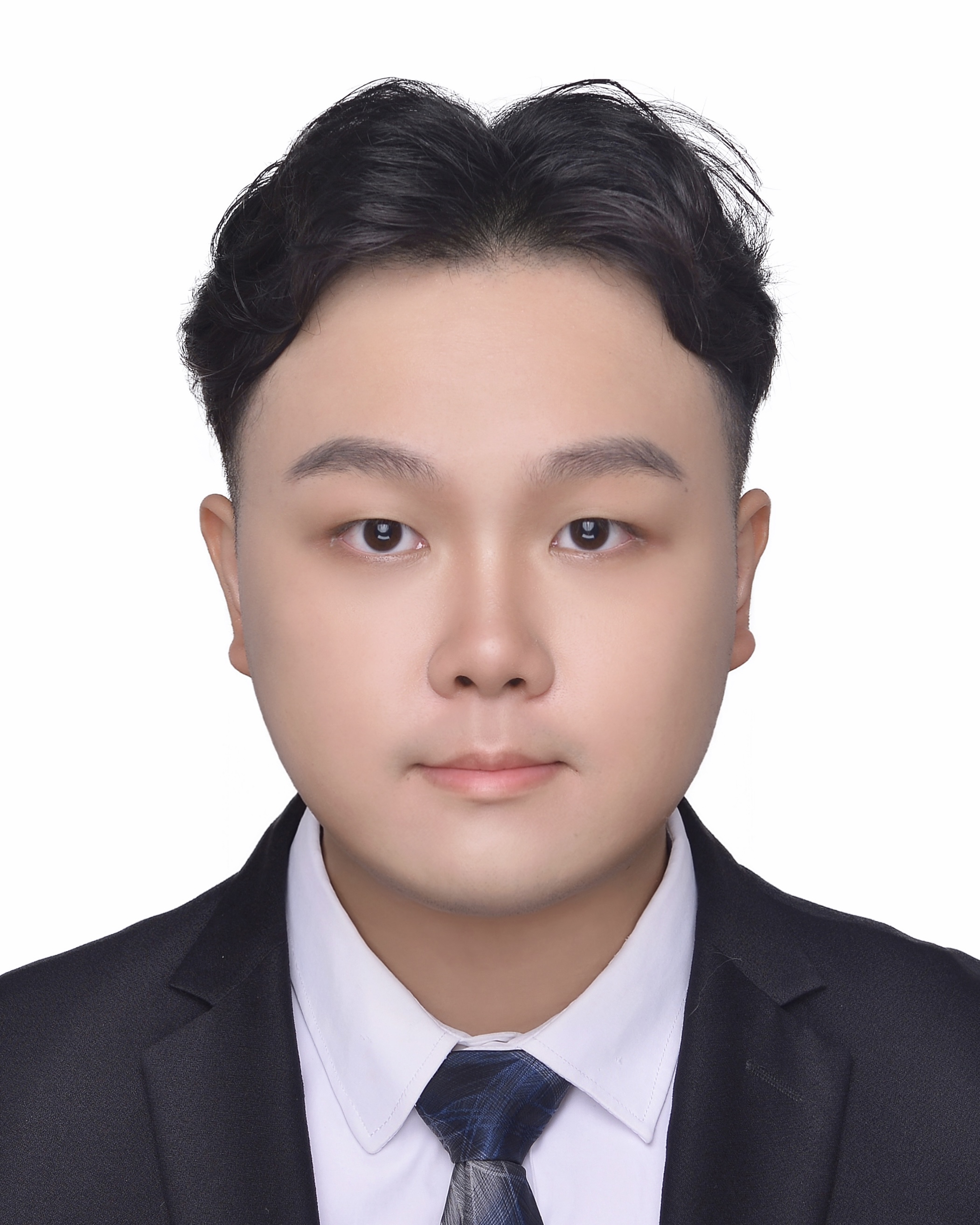}}]{Hongbin Lin} received the MPhil degree in Artificial Intelligence from The Hong Kong University of Science and Technology, Hong Kong, China in 2025. He is currently a research assistant at the Machine Learning Department at Mohamed
bin Zayed University of Artificial Intelligence (MBZUAI). His research interests include multimodal large language models and deep learning theory.
\end{IEEEbiography}

\begin{IEEEbiography}
[{\includegraphics[width=1in,height=1.25in,clip,keepaspectratio]{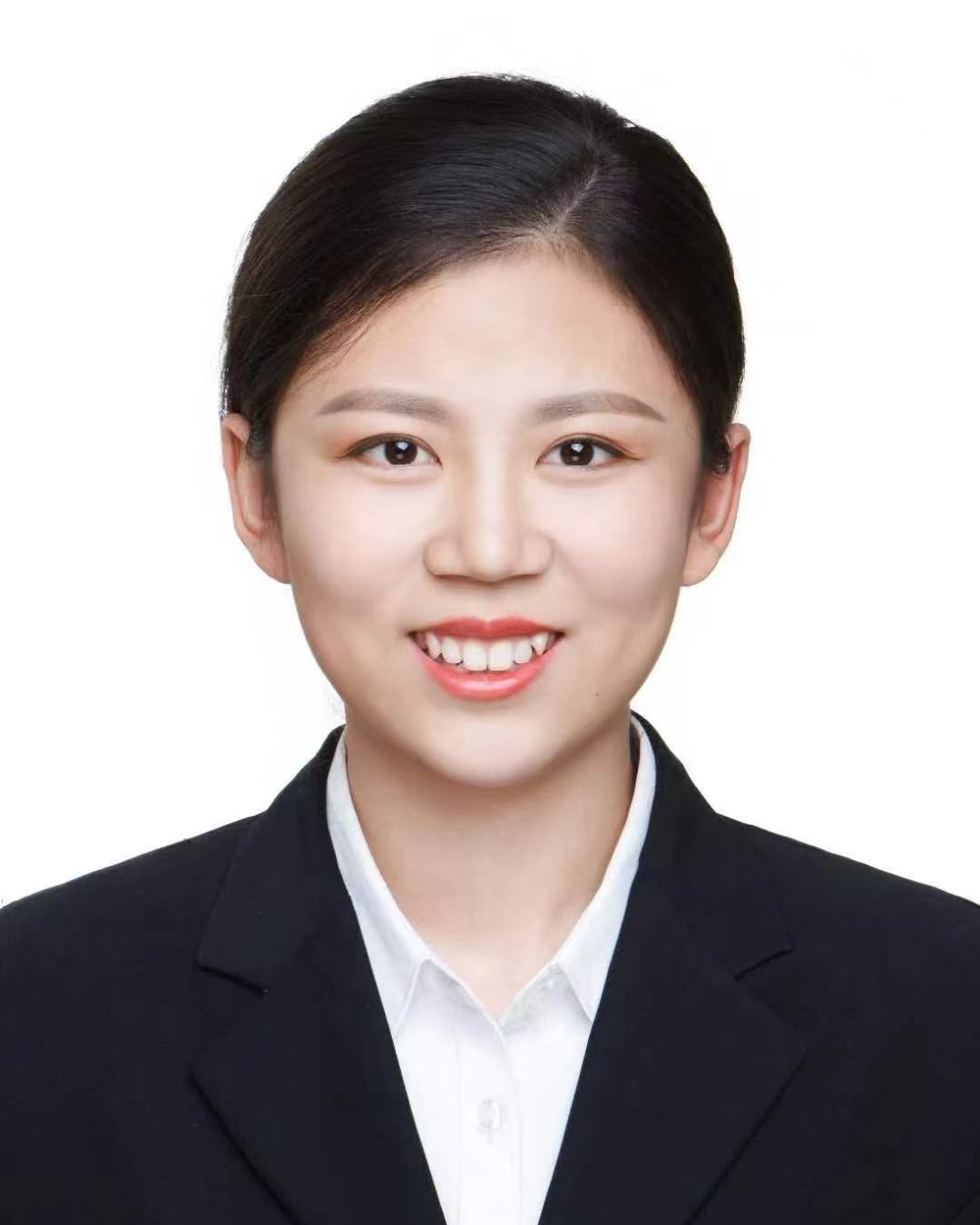}}]{Chenyang Ren} received the M.S. degree in Applied Mathematics from Shanghai Jiao Tong University, Shanghai, China in 2025. She is currently pursuing the Ph.D. degree with the Department of Mathematics at The Hong Kong University of Science and Technology, Hong Kong SAR, China. Her research interests lie in artificial intelligence for materials science, machine learning, and scientific computing.
\end{IEEEbiography}

\begin{IEEEbiography}[{\includegraphics[width=1in,height=1.25in,clip,keepaspectratio]{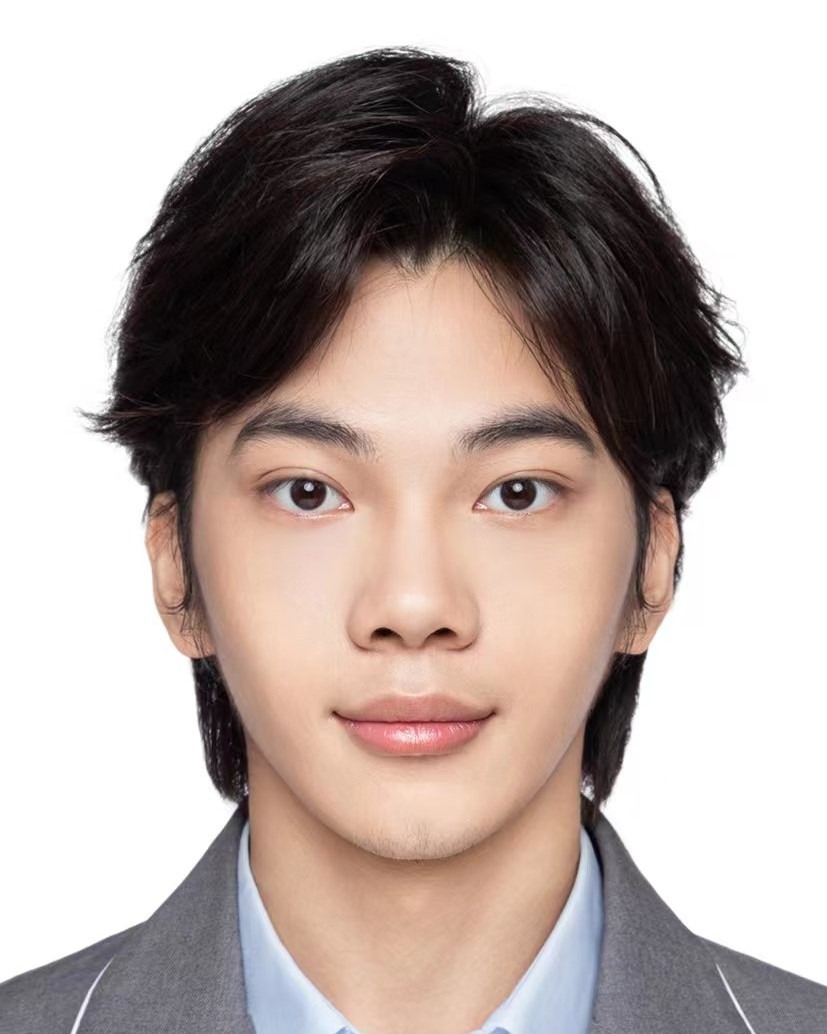}}]{Juangui Xu} received the B.Eng. degree in Cyber Space Security from Guangzhou University, China in 2024. He is currently purchasing the M.Sc. degree with the Department of Language Science and Technology, Saarland University, Germany. His research interests include Computer Vision and Large Language Models.
\end{IEEEbiography}

\begin{IEEEbiography}[{\includegraphics[width=1in,height=1.25in,clip,keepaspectratio]{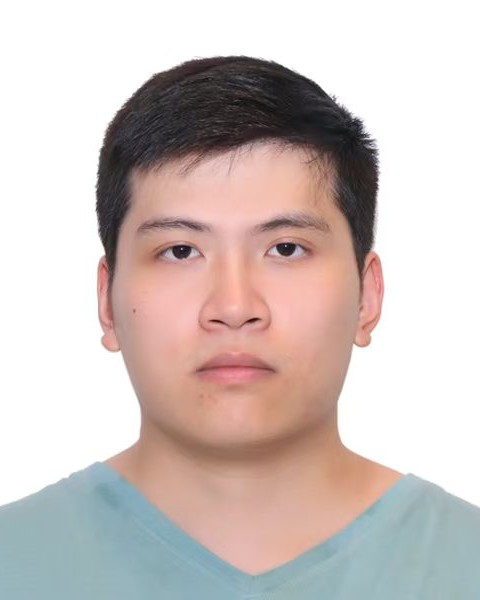}}]{Zhengyu Hu} is an M.Phil. student in Artificial Intelligence at the Hong Kong University of Science and Technology (Guangzhou), advised by Prof. Hui Xiong. His research spans machine learning and data mining, with a recent focus on large language models, preference evaluation, multi-agent systems, and graph learning. He has published over ten papers in top-tier conferences such as NeurIPS, ICML, KDD, EMNLP, ICCV, and WWW. He has also served as a reviewer for major conferences including KDD, NeurIPS, ICLR, ICML, CVPR, ACL, AAAI, AISTATS, WWW, NAACL, COLING, ECCV, and EMNLP, as well as for journals including DMLR and IEEE TNNLS.
\end{IEEEbiography}

\begin{IEEEbiography}[{\includegraphics[width=1in,height=1.25in,clip,keepaspectratio]{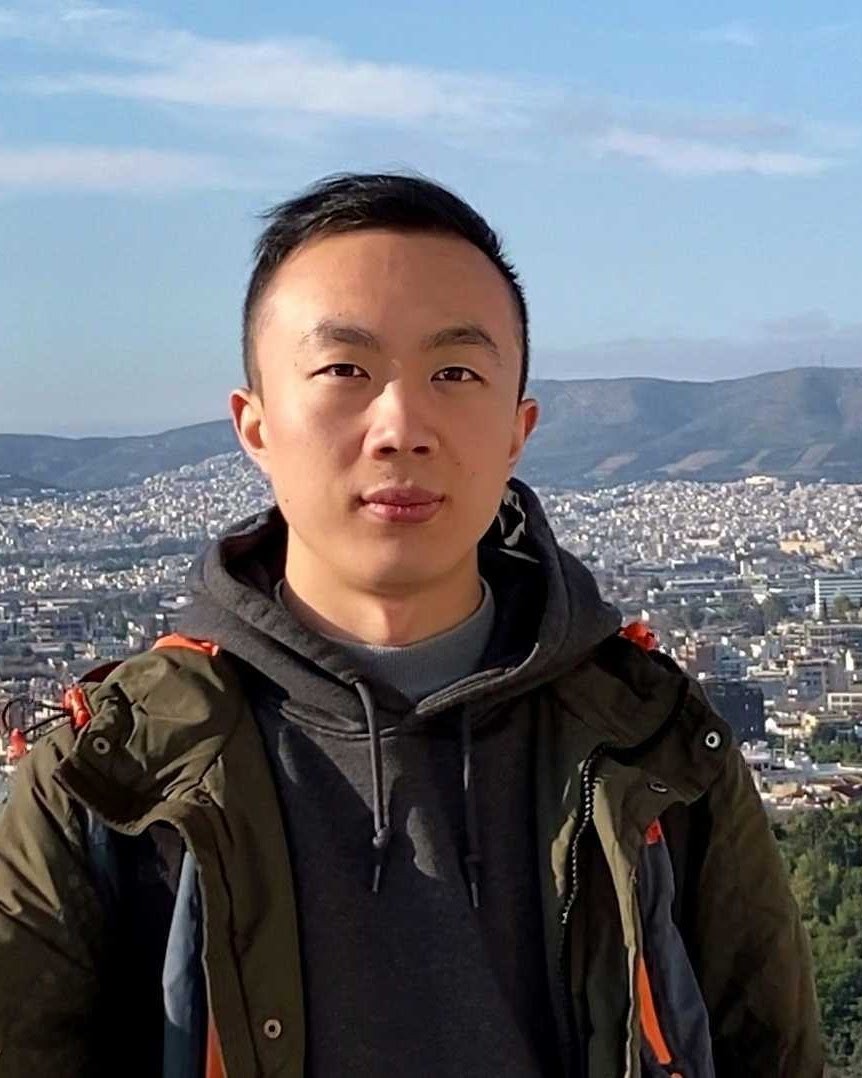}}]{Cheng-Long Wang} is a final-year Ph.D. student in Computer Science at KAUST, advised by Dr. Di Wang, and he has worked closely with Dr. Yinzhi Cao. His research focuses on machine unlearning, data/LLM memorization, and related privacy, copyright, and security challenges.
\end{IEEEbiography}

\begin{IEEEbiography}[{\includegraphics[width=1in,height=1.25in,clip,keepaspectratio]{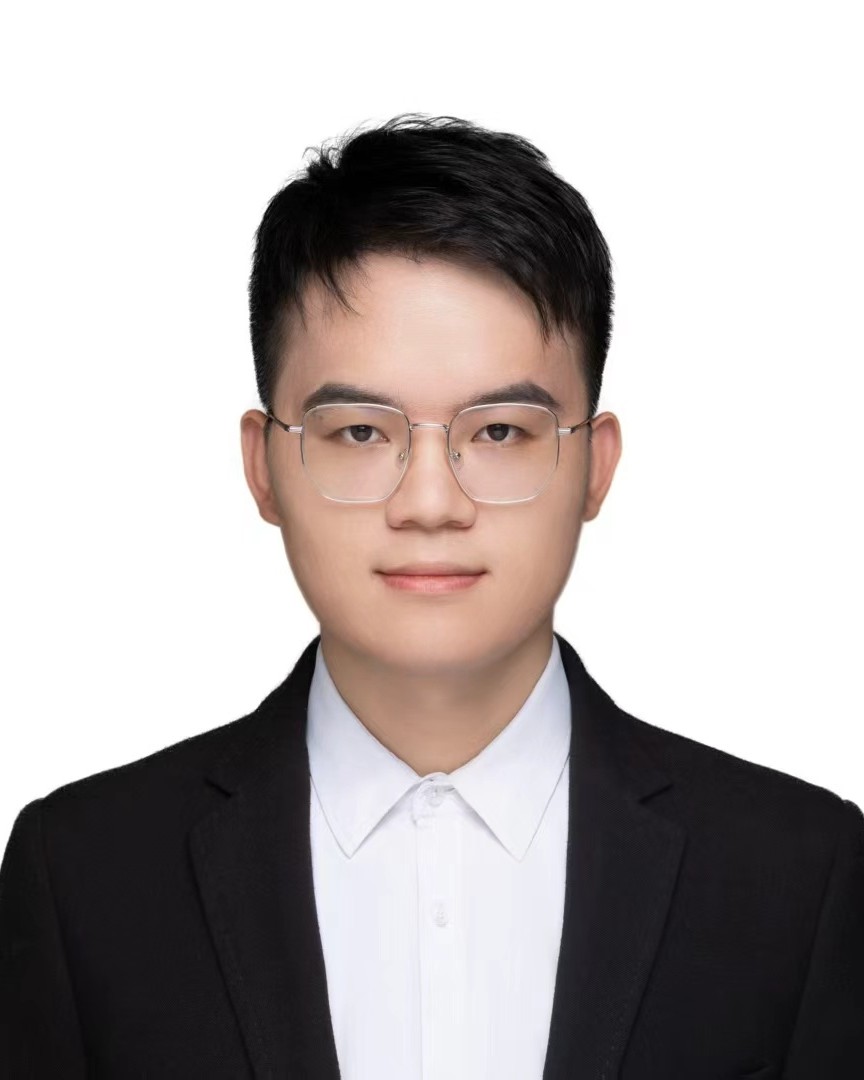}}]{Yao Shu} received the B.Eng. degree from Huazhong University of Science and Technology, China, in 2017, and the Ph.D. degree from the National University of Singapore, Singapore, in 2022. He is currently a Tenure-Track Assistant Professor in the Artificial Intelligence Thrust, The Hong Kong University of Science and Technology (Guangzhou), China. His research interests include deep learning theory and optimization algorithms.
\end{IEEEbiography}

\begin{IEEEbiography}[{\includegraphics[width=1in,height=1.25in,clip,keepaspectratio]{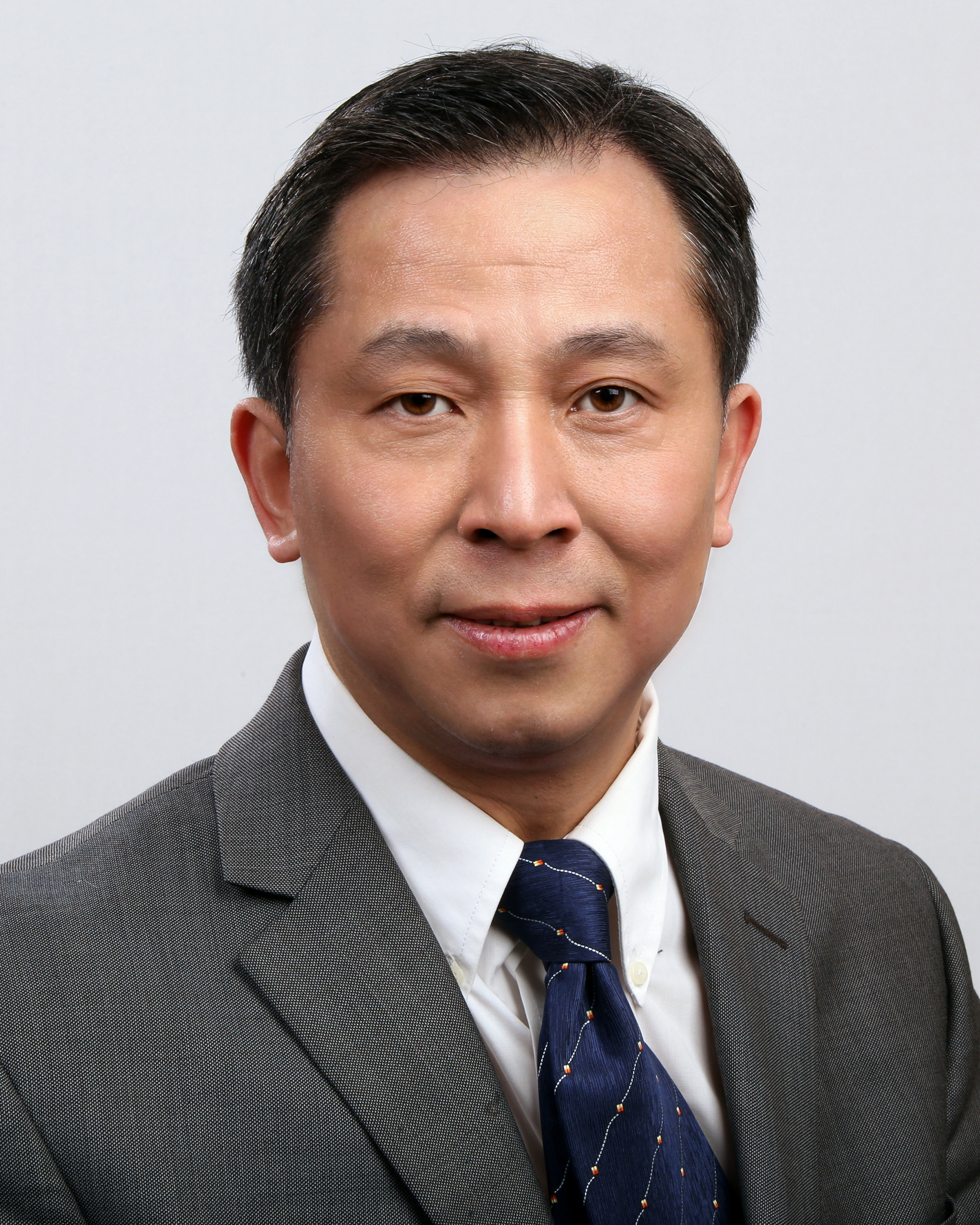}}]{Hui Xiong} (Fellow, IEEE) received the Ph.D. degree in computer science from the University of Minnesota. He is a chair professor, associate vice president (Knowledge Transfer), and the Thrust of Artificial Intelligence at The Hong Kong University of Science and Technology (Guangzhou), China. He is also with the Department of Computer Science and Engineering at The Hong Kong University of Science and Technology, Hong Kong SAR, China. His research interests span artificial intelligence, data mining, and mobile computing. He has served on numerous organization and program committees for conferences, including as program co-chair for the Industrial and Government Track of the 18th ACM SIGKDD International Conference on Knowledge Discovery and Data Mining (KDD), program co-chair for the IEEE 2013 International Conference on Data Mining (ICDM), general co-chair for the 2015 IEEE International Conference on Data Mining (ICDM), and program co-chair of the Research Track for the 2018 ACM SIGKDD International Conference on Knowledge Discovery and Data Mining. He received several awards, such as the 2021 AAAI Best Paper Award and the 2011 IEEE ICDM Best Research Paper Award. For his significant contributions to data mining and mobile computing, he was elected as a Fellow of AAAS in 2020.
\end{IEEEbiography}

\begin{IEEEbiography}[{\includegraphics[width=1in,height=1.25in,clip,keepaspectratio]{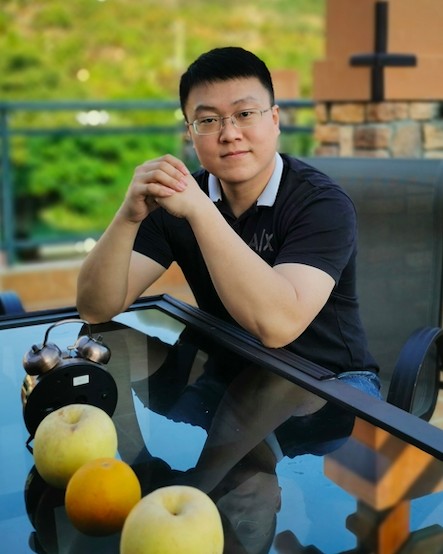}}]{Jingfeng Zhang} (Member, IEEE) received the bachelor’s degree in computer science from the Taishan College, Shandong University, Jinan, China, in 2016, and the Ph.D. degree in computer science from the National University of Singapore, Singapore, in 2020.,From 2021 to 2022, he was a Post-Doctoral Researcher with RIKEN Center for Advanced Intelligence Project (AIP) (RIKEN AIP), Chuo City, Japan, where he was a Research Scientist in 2023. He is currently a Lecturer and a Ph.D./Doctoral Accredited Supervisor with The University of Auckland and also a Visiting Scientist with RIKEN AIP.,Dr. Zhang received the JST Award, the JSPS Grants-in-Aid for Scientific Research (KAKENHI) for Early-Career Scientists, and the RIKEN Ohbu Award in 2022.
\end{IEEEbiography}

\begin{IEEEbiography}[{\includegraphics[width=1in,height=1.25in,clip,keepaspectratio]{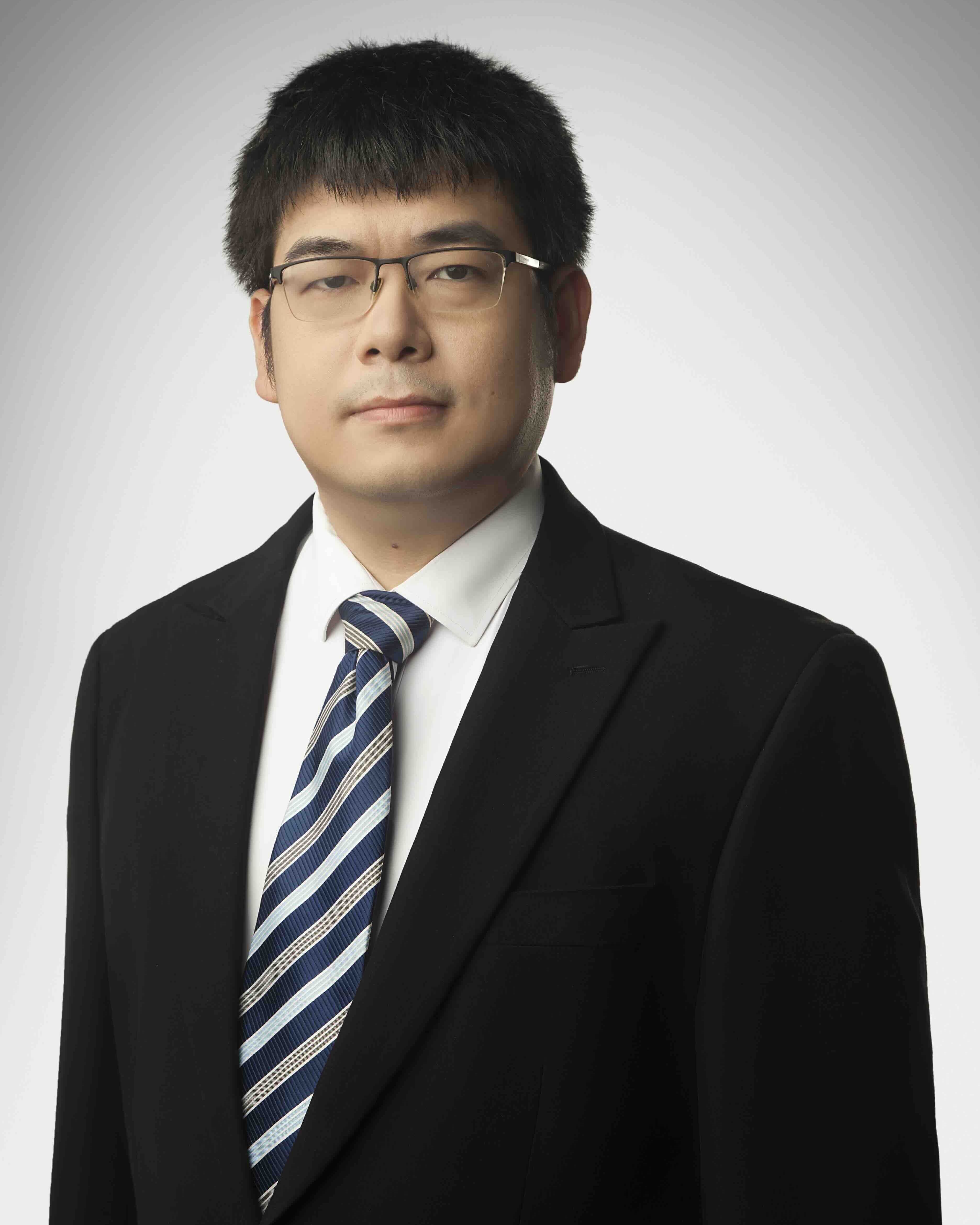}}]{Di Wang} is an Assistant Professor of Computer Science and an affiliated faculty of Statistics in the Division of Computer, Electrical and Mathematical Sciences and Engineering (CEMSE) at the King Abdullah University of Science and Technology (KAUST), start from Spring 2021. He is also the PI of Provable Responsible AI and Data Analytics (PRADA) Lab, and a member of Center of Excellence on Generative AI. Before that, he got his Ph.D degree in Computer Science at the State University of New York (SUNY) at Buffalo in 2020 under supervision of Dr. Jinhui Xu. Before my Ph.D study I took my Master degree in Mathematics at University of Western Ontario in 2015, and he received his Bachelor degree in Mathematics and Applied Mathematics at Shandong University in 2014.
\end{IEEEbiography}

\begin{IEEEbiography}[{\includegraphics[width=1in,height=1.25in,clip,keepaspectratio]{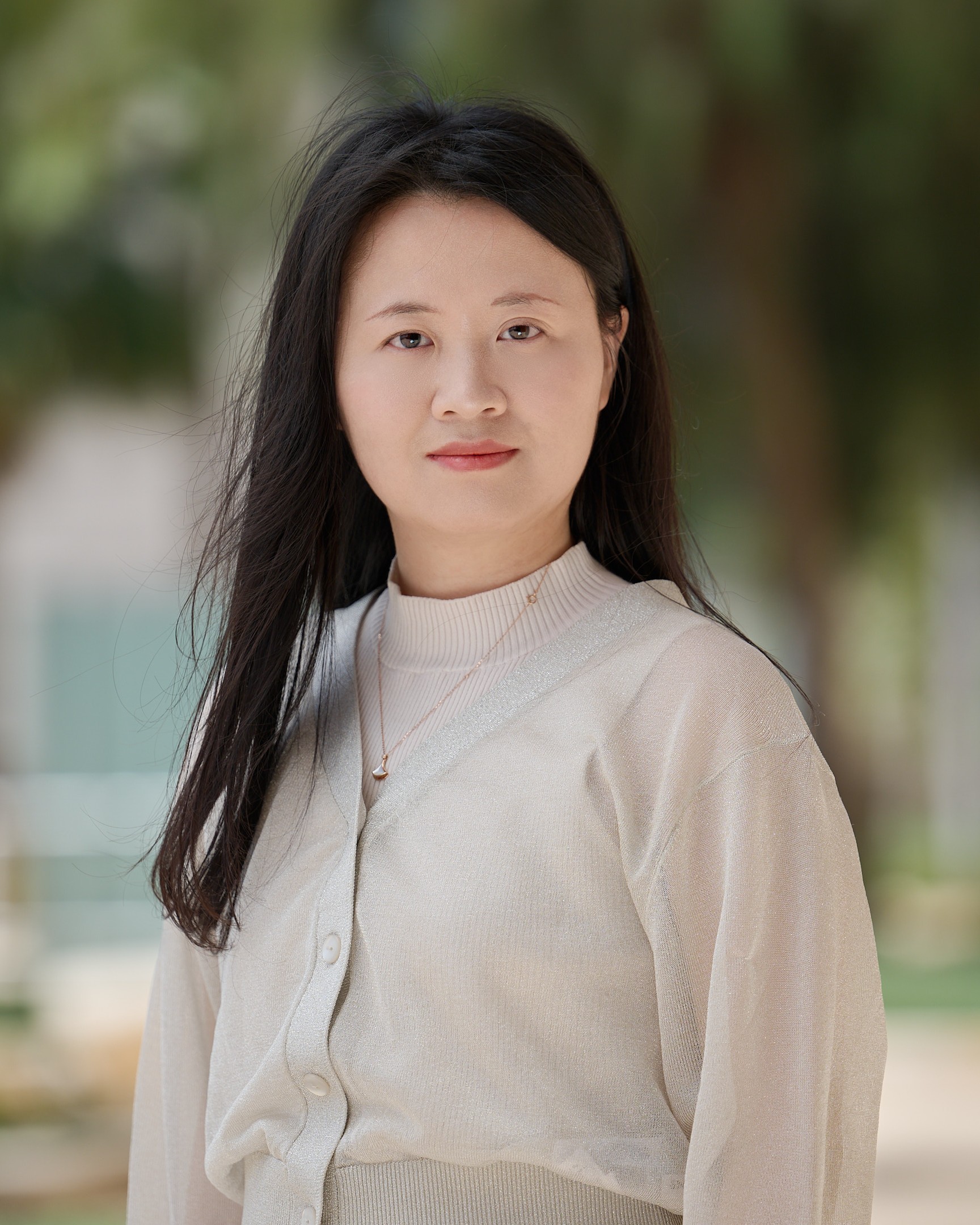}}]{Lijie Hu} is an Assistant Professor in the Machine Learning Department at Mohamed bin Zayed University of Artificial Intelligence (MBZUAI). She received her Ph.D. in Computer Science from King Abdullah University of Science and Technology (KAUST) in 2025, where she had the privilege of being advised by Prof. Di Wang, who leads the PRADA Lab (Provable Responsible AI and Data Analytics Lab). Prior to that, she obtained her Master’s degree in Mathematics from Renmin University of China.
\end{IEEEbiography}

%% file: appendix/7_appendix_a.tex
\section{Influence Function}
Consider a neural network $\hat{\theta}=\arg\min_\theta \sum_{i=1}^n \ell(z_i, \theta)$ with loss function $L$ and dataset $D=\{z_i\}_{i=1}^n$. That is $\hat{\theta}$ minimize the empirical risk
\begin{equation*}
    R(\theta) = \sum_{i=1}^n L(z_i, \theta)
\end{equation*}
Assume $R$ is strongly convex in $\theta$. Then $\theta$ is uniquely defined. If we remove a point $z_m$ from the training dataset, the parameters become $\hat{\theta}_{-z_m} = \arg\min_\theta \sum_{i\neq m} L(z_i, \theta)$. Up-weighting $z_m$ by $\epsilon$ small enough, then the revised risk ${R(\theta)}^{'}= \frac{1}{n}\sum_{i=1}^{n} L(z_{i} ; \theta) + \epsilon L(z_m ; \theta)$ is still strongly convex.
Then the response function $\hat{\theta}_{\epsilon, -z_m} = {R(\theta)}^{'}$ is also uniquely defined. The parameter change is denoted as $\Delta_{\epsilon} = \hat{\theta}_{\epsilon, -z_m} - \hat{\theta}$.
Since $\hat{\theta}_{\epsilon, -z_m}$ is the minimizer of ${R(\theta)}^{'}$, we have the first-order optimization condition as
\begin{equation*}
    \nabla_{\hat{\theta}_{\epsilon, -z_m}} R(\theta) + \epsilon \cdot \nabla_{\hat{\theta}_{\epsilon, -z_m}} L(z_m, \hat{\theta}_{\epsilon, -z_m}) = 0
\end{equation*}
Since $\hat{\theta}_{\epsilon, -z_m} \rightarrow \hat{\theta} as \epsilon \rightarrow 0$, we perform a Taylor expansion of the right-hand side:
\begin{equation*}
 \left[ \nabla R(\hat{\theta}) + \epsilon \nabla L(z_m, \hat{\theta}) \right] + \left[ \nabla^2 R(\hat{\theta}) + \epsilon \nabla^2 L(z_m, \hat{\theta}) \right] \Delta_{\epsilon} \approx 0
\end{equation*}
Noting $\epsilon \nabla^2 L(z_m, \hat{\theta}) \Delta_{\epsilon}$ is $o(\|\Delta_{\epsilon}\|)$ term, which is smaller than other parts, we drop it in the following analysis. Then the Taylor expansion equation becomes
\begin{equation*}
 \left[ \nabla R(\hat{\theta}) + \epsilon \nabla L(z_m, \hat{\theta}) \right] +  \nabla^2 R(\hat{\theta}) \cdot \Delta_{\epsilon} \approx 0
\end{equation*}
Solving for $\Delta_{\epsilon}$, we obtain:
$$\Delta_{\epsilon} = - \left[ \nabla^2 R(\hat{\theta}) + \epsilon \nabla^2 L(z, \hat{\theta}) \right]^{-1} \left[ \nabla R(\hat{\theta}) + \epsilon \nabla L(z, \hat{\theta}) \right].$$
Remember $\theta$ minimizes $R$, then $\nabla R(\hat{\theta}) = 0$. Dropping $o(\epsilon)$ term, we have
$$\Delta_{\epsilon} = - \epsilon \nabla^2 R(\hat{\theta}) ^{-1}   \nabla L(z, \hat{\theta}).$$
\begin{equation*}
\left. \frac{d \hat{\theta}_{\epsilon, -z_m}}{d \epsilon} \right|_{\epsilon=0} =\left. \frac{d \Delta_{\epsilon}}{d \epsilon} \right|_{\epsilon=0} = -H_{\hat{\theta}}^{-1} \nabla L(z, \hat{\theta}) \equiv \mathcal{I}_{up, params}(z).
\end{equation*}
Besides, we can obtain the approximation of $\hat{\theta}_{-z_m}$ directly by $\hat{\theta}_{-z_m} \approx \hat{\theta}+\mathcal{I}_{up, params}(z)$.

\section{Acceleration for Influence Function} \label{sec:further}

\paragraph{EK-FAC.}
EK-FAC method relies on two approximations to the Fisher information matrix, equivalent to $G_{\hat{\theta}}$ in our setting, which makes it feasible to compute the inverse of the matrix.

Firstly, assume that the derivatives of the weights in different layers are uncorrelated, which implies that $G_{\hat{\theta}}$ has a block-diagonal structure.
Suppose $\hat{g}_\theta$ can be denoted by $\hat{g}_{\theta}(x) = g_{\theta_L} \circ \cdots \circ g_{\theta_l} \circ \cdots \circ g_{\theta_1}(x)$ where $l \in [L]$. We fold the bias into the weights and vectorize the parameters in the $l$-th layer into a vector $\theta_{l}\in\mathbb{R}^{d_l}$, $d_l\in \mathbb{N}$ is the number of $l$-th layer parameters.  
Then $G_{\hat{\theta}}$ can be reaplcaed by $\left(G_1(\hat{\theta}), \cdots, G_L(\hat{\theta})\right)$, where $G_l(\hat{\theta}) \triangleq n^{-1}\sum_{i=1}^n\nabla_{\hat{\theta}           _l}\ell_i \nabla_{\theta_l}\ell_i^{\mathrm{T}}$. Denote $h_{l}$, $o_l$ as the output and pre-activated output of $l$-th layer. Then $G_l(\theta)$ can be approximated by
\begin{equation*}
\begin{split}
{G}_l(\theta) \approx \hat{G}_l(\theta) 
& \triangleq \frac{1}{n} \sum_{i=1}^{n} h_{l-1}\left(x_{i}\right) h_{l-1}\left(x_{i}\right)^{T} \\
& \otimes \frac{1}{n} \sum_{i=1}^{n} \nabla_{o_{l}} \ell_{i} \nabla_{o_{l}} \ell_{i}^{T} \triangleq {\Omega}_{l-1} \otimes {\Gamma_{l}}.
\end{split}
\end{equation*}
Furthermore, in order to accelerate transpose operation and introduce the damping term, perform eigenvalue decomposition of matrix ${\Omega}_{l-1}$ and ${\Gamma_{l}}$ and obtain the corresponding decomposition results as $Q_{\Omega} {\Lambda}_{{\Omega}} {Q}_{{\Omega}}^{\top}$ and ${Q}_{\Gamma}{\Lambda}_{\Gamma} {Q}_{\Gamma}^{\top}$. Then the inverse of $ \hat{H}_l(\theta)$ can be obtained by
\begin{equation*}
\begin{split}
\hat{H}_l(\theta)^{-1} 
&\approx \left(\hat{G}_{l}\left(\hat{g}\right)+\lambda_{l} I_{d_{l}}\right)^{-1} \\
&= \left(Q_{\Omega_{l-1}} \otimes Q_{\Gamma_{l}}\right)
   \left(\Lambda_{\Omega_{l-1}} \otimes \Lambda_{\Gamma_{l}}+\lambda_{l} I_{d_{l}}\right)^{-1} \\
&\quad \times \left(Q_{\Omega_{l-1}} \otimes Q_{\Gamma_{l}}\right)^{\mathrm{T}}.
\end{split}
\end{equation*}
Besides, \cite{george2018fast} proposed a new method that corrects the error in equation \ref{apa:ekfac:1} which sets the $i$-th diagonal element of $\Lambda_{\Omega_{l-1}} \otimes \Lambda_{\Gamma_{l}}$ as $\Lambda_{ii}^{*} = n^{-1} \sum_{j =1}^n \left( \left( Q_{\Omega_{l-1}}\otimes Q_{\Gamma_{l}}\right)\nabla_{\theta_l}\ell_j\right)^2_i.$

\subsection{EK-FAC for CBMs}\label{sec:FAC_CBM}
In our CBM model, the label predictor is a single linear layer, and Hessian computing costs are affordable. However, the concept predictor is based on Resnet-18, which has many parameters. Therefore, we perform EK-FAC for $\hat{g}$.
\begin{equation*}
    \hat{g} = \argmin_g \sum_{j=1}^{k}  L_{C_j} = \argmin_g \sum_{j=1}^k\sum_{i=1}^n L_{C}(g^j(x_i),c_i^j),
\end{equation*}
we define $H_{\hat{g}} = \nabla^2_{\hat{g}}\sum_{i,j} L_{C_j}(g(x_i), c_i)$ as the Hessian matrix of the loss function with respect to the parameters. 

To this end, consider the $l$-th layer of $\hat{g}$ which takes as input a layer of activations $\{a_{j,t}\}$ where $j\in\{1, 2, \ldots, J\}$ indexes the input map and $t \in \mathcal{T}$ indexes the spatial location which is typically a 2-D grid. This layer is parameterized by a set of weights $W = \left(w_{i, j, \delta}\right)$ and biases $b = \left(b_{i}\right)$, where $i \in\{1, \ldots, I\}$ indexes the output map, and $\delta \in \Delta$ indexes the spatial offset (from the center of the filter).

The convolution layer computes a set of pre-activations as 
\begin{equation*}
    [S_l]_{i,t} = s_{i, t}=\sum_{\delta \in \Delta} w_{i, j, \delta} a_{j, t+\delta}+b_{i}.
\end{equation*}
Denote the loss derivative with respect to $s_{i,t}$ as 
\begin{equation*}
    \mathcal{D}s_{i, t} = \frac{\partial \sum L_{C_j}}{\partial s_{i, t}},
\end{equation*}
which can be computed during backpropagation.

The activations are actually stored as $A_{l-1}$ of dimension $|\mathcal{T}|\times J$. Similarly, the weights are stored as an $I \times|\Delta| J$ array $W_l$. The straightforward implementation of convolution, though highly parallel in theory, suffers from poor memory access patterns. Instead, efficient implementations typically leverage what is known as the expansion operator $\llbracket \cdot \rrbracket$. For instance, 
$\llbracket A_{l-1} \rrbracket$ is a $|\mathcal{T}|\times J|\Delta|$ matrix, defined as
\begin{equation*}
    \llbracket {{A}}_{l-1} \rrbracket_{t, j|\Delta|+\delta}=\left[{{A}}_{l-1}\right]_{(t+\delta), j}=a_{j, t+\delta},
\end{equation*}

In order to fold the bias into the weights, we need to add a homogeneous coordinate (i.e. a column of all $1$’s) to the expanded activations $\llbracket A_{l-1} \rrbracket$ and denote this as $\llbracket A_{l-1} \rrbracket_{\mathrm{H}}$. Concatenating the bias vector to the weights matrix, then we have $\theta_l = (b_l, W_l)$.

Then, the approximation for $H_{\hat{g}}$ is given as:
\begin{align*}
G^{(l)}(\hat{g}) 
&= \mathbb{E}\left[\mathcal{D} w_{i, j, \delta} \mathcal{D} w_{i^{\prime}, j^{\prime}, \delta^{\prime}}\right] \\
&= \mathbb{E}\left[\left(\sum_{t \in \mathcal{T}} a_{j, t+\delta} \mathcal{D} s_{i, t}\right)
   \left(\sum_{t^{\prime} \in \mathcal{T}} a_{j^{\prime}, t^{\prime}+\delta^{\prime}} \mathcal{D} s_{i^{\prime}, t^{\prime}}\right)\right] \\
&\approx \mathbb{E}\left[\llbracket {A}_{l-1} \rrbracket_{\mathrm{H}}^{\top} \llbracket {A}_{l-1} \rrbracket_{\mathrm{H}}\right] \otimes \frac{1}{|\mathcal{T}|} \mathbb{E}\left[\mathcal{D} {S}_{l}^{\top} \mathcal{D} {S}_{l}\right] \\
&\triangleq \Omega_{l-1}\otimes \Gamma_{l}.
\end{align*}
Estimate the expectation using the mean of the training set, 
\begin{align*}
G^{(l)}(\hat{g}) 
&\approx \frac{1}{n} \sum_{i=1}^{n} \left( \llbracket {A}^i_{l-1} \rrbracket_{\mathrm{H}}^{\top} \llbracket {A}^i_{l-1} \rrbracket_{\mathrm{H}} \right) \otimes \frac{1}{n} \sum_{i=1}^{n} \left( \frac{1}{|\mathcal{T}|} \mathcal{D} {{S}^i_{l}}^{\top} \mathcal{D} {S}^i_{l} \right) \\
&\triangleq \hat{\Omega}_{l-1} \otimes \hat{\Gamma}_{l}.
\end{align*}
Furthermore, if the factors $\hat{\Omega}_{l-1}$ and $\hat{\Gamma}_{l}$ have eigen decomposition $Q_{\Omega} {\Lambda}_{{\Omega}} {Q}_{{\Omega}}^{\top}$ and ${Q}_{\Gamma}{\Lambda}_{\Gamma} {Q}_{\Gamma}^{\top}$, respectively, then the eigen decomposition of $\hat{\Omega}_{l-1}\otimes \hat{\Gamma_{l}}$ can be written as:  
\begin{align*}  
\hat{\Omega}_{l-1}\otimes \hat{\Gamma}_{l} & ={Q}_{  {\Omega}} {\Lambda}_{  {\Omega}} {Q}_{  {\Omega}}^{\top} \otimes {Q}_{  {\Gamma}} {\Lambda}_{  {\Gamma}} {Q}_{  {\Gamma}}^{\top} \\  
& =\left({Q}_{  {\Omega}} \otimes {Q}_{   {\Gamma}}\right)\left(  {\Lambda}_{   {\Omega}} \otimes  {\Lambda}_{   {\Gamma}}\right)\left( {Q}_{   {\Omega}} \otimes  {Q}_{   {\Gamma}}\right)^{\top}.  
\end{align*}  

Since subsequent inverse operations are required and the current approximation for $G^{(l)}(\hat{g})$ is PSD, we actually use a damped version as 
\begin{equation}\label{apa:ekfac:1}
\begin{split}
{\hat{G}^{l}}(\hat{g})^{-1} 
&= \left(G_{l}\left(\hat{g}\right) + \lambda_{l} I_{d_{l}}\right)^{-1} \\
&= \left(Q_{\Omega_{l-1}} \otimes Q_{\Gamma_{l}}\right) \cdot \left(\Lambda_{\Omega_{l-1}} \otimes \Lambda_{\Gamma_{l}} + \lambda_{l} I_{d_{l}}\right)^{-1} \\
&\quad \cdot \left(Q_{\Omega_{l-1}} \otimes Q_{\Gamma_{l}}\right)^{\mathsf{T}}.
\end{split}
\end{equation}

Besides, \cite{george2018fast} proposed a new method that corrects the error in equation \ref{apa:ekfac:1} which sets the $i$-th diagonal element of $\Lambda_{\Omega_{l-1}} \otimes \Lambda_{\Gamma_{l}}$ as
\begin{equation*}
    \Lambda_{ii}^{*} = n^{-1} \sum_{j =1}^n \left( \left( Q_{\Omega_{l-1}}\otimes Q_{\Gamma_{l}}\right)\nabla_{\theta_l}\ell_j\right)^2_i.
\end{equation*}

%% file: appendix/7_appendix_alg.tex
\section{Algorithm}
\begin{breakablealgorithm}
\caption{Concept-label-level CCBM \label{alg:1}}
\begin{algorithmic}[1]
\STATE {\bf Input:} Dataset $\mathcal{D} = \{ (x_i, y_i, c_i) \}_{i=1}^{n}$, original concept predictor $\hat{f}$, and label predictor $\hat{g}$, a set of erroneous data $D_e$ and its associated index set $S_e$.
\STATE For the index $(w, r)$ in $S_e$, correct ${c_w^r}$ to the right label ${c_w^r}^{\prime}$ for the $w$-th data $(x_w, y_w, c_w)$.
\STATE Compute the Hessian matrix of the loss function respect to $\hat{g}$:
$$H_{\hat{g}} = \nabla^2_{\hat{g}}  \sum_{i,j}  L_{C_j}(\hat{g}^j(x_i),c_i^j).$$
\STATE Update concept predictor $\Tilde{g}$:
\begin{equation*}
\begin{split}
\Tilde{g} 
&= \hat{g} - H^{-1}_{\hat{g}} \cdot \sum_{(w,r)\in S_e} \Bigl( \nabla_{\hat{g}} L_{C_r}\left(\hat{g}^r(x_w), {c_w^r}^{\prime}\right) \\
&\quad - \nabla_{\hat{g}} L_{C_r}\left(\hat{g}^r(x_w), c_w^r\right) \Bigr).
\end{split}
\end{equation*}
\STATE Compute the Hessian matrix of the loss function respect to $\hat{f}$:
$$H_{\hat{f}} = \nabla^2_{\hat{f}}\sum_{i=1}^n L_{Y_i}(\hat{f}, \hat{g}).$$
\STATE Update label predictor $\Tilde{f}$:
\begin{align*}
\Tilde{f} 
&= \hat{f} 
+ H^{-1}_{\hat{f}} \cdot \nabla_f \sum_{i=1}^n L_{Y} \left( \hat{f} \left( \hat{g} ( x_{i} ) \right), y_{i} \right) \\
&- H^{-1}_{\hat{f}} \cdot \nabla_f \sum_{l=1}^n \left( L_Y \left( \hat{f} \left( \Tilde{g}(x_l) \right), y_{l} \right) \right).
\end{align*}
\STATE {\bf Return:} $\Tilde{f}$, $\Tilde{g}$. 
\end{algorithmic}
\end{breakablealgorithm}

\begin{breakablealgorithm}
\caption{Concept-level CCBM \label{alg:2}}
\begin{algorithmic}[1]
    \STATE {\bf Input:} Dataset $\mathcal{D} = \{ (x_i, y_i, c_i) \}_{i=1}^{ n}$, original concept predictor $\hat{f}$, label predictor $\hat{g}$ and the to be removed concept index set $M$.
    \STATE For $r\in M$, set $p_r = 0$ for all the data $z\in \mathcal{D}$.
    \STATE Compute the Hessian matrix of the loss function respect to $\hat{g}$:
    $$H_{\hat{g}} = \nabla^2_{\hat{g}}  \sum_{j\notin M}\sum_{i=1}^n  L_{C_j}(\hat{g}^j(x_i),c_i^j).$$
    \STATE Update concept predictor $\Tilde{g}^*$:
    \begin{equation*}
        \Tilde{g}^* = \hat{g} - H^{-1}_{\hat{g}} \cdot \nabla_{\hat{g}} \sum_{j\notin M} \sum_{i=1}^{n} L_{C_j} ( \hat{g}^{j} (x_i), c_i^j).
    \end{equation*}
    \STATE Compute the Hessian matrix of the loss function respect to $\hat{f}$:
    $$H_{\hat{f}} = \nabla^2_{\hat{f}}  \sum_{i=1}^n  L_{Y}(\hat{f}(\hat{g}(x_i), y_i).$$
    \STATE Update label predictor $\Tilde{f}$:
    \begin{equation*}
    \Tilde{f} =  \hat{f}-H_{\hat{f}}^{-1} \cdot   \nabla_{\hat{f}} \sum_{l   =1}^{n} L_{Y} \left( \hat{f} \left( \Tilde{g}^*(x_l) \right)  , y_l \right).
    \end{equation*}
    \STATE Map $\Tilde{g}^*$ to $\Tilde{g}$ by removing the $r$-th row of the matrix in the final layer of $\Tilde{g}^*$ for $r\in M$.
    \STATE {\bf Return:}$\Tilde{f}$, $\Tilde{g}$.
\end{algorithmic}
\end{breakablealgorithm}

\begin{breakablealgorithm}
\caption{Data-level CCBM (Removal) \label{alg:3}}
\begin{algorithmic}[1]
    \STATE {\bf Input:} Dataset $\mathcal{D} = \{ (x_i, y_i, c_i) \}_{i=1}^{N}$, original concept predictor $\hat{f}$, label predictor $\hat{g}$, and the to be removed data index set $G$.
    \STATE For $r\in G$, remove the $r$-th data $(x_r, y_r, c_r)$ from $\mathcal{D}$ and define the new dataset as $\mathcal{S}$.
    \STATE Compute the Hessian matrix of the loss function with respect to $\hat{g}$:
    $$H_{\hat{g}} = \nabla^2_{\hat{g}}  \sum_{i,j}  L_{C_j}(\hat{g}^j(x_i),c_i^j).$$
    \STATE Update concept predictor $\Tilde{g}$:
    \begin{equation*}
        \Tilde{g} = \hat{g} + H^{-1}_{\hat{g}} \cdot  \sum_{r\in G}\nabla_{g} L_{C} (\hat{g}(x_r), c_r)
    \end{equation*}
    \STATE Update label predictor $\Tilde{f}$.
    Compute the Hessian matrix of the loss function with respect to $\hat{f}$:
$$H_{\hat{f}} = \nabla^2_{\hat{f}}  \sum_{i=1}^n  L_Y(\hat{f}(\hat{g}(x_i), y_i).$$
\STATE   Compute $A$ as:
    \begin{equation*}
           A =  H^{-1}_{\hat{f}} \cdot  \sum_{i\in [n]-G} \nabla_{\hat{f}} L_{Y}\left(\hat{f}(\hat{g}(x_i)), y_i\right)
    \end{equation*}

\STATE Obtain $\bar{f}$ as 
    \begin{equation*}
        \bar{f} = \hat{f} + A
    \end{equation*}
\STATE Compute the Hessian matrix of the loss function concerning $\bar{f}$:
$$H_{\bar{f}} = \nabla^2_{\bar{f}}  \sum_{i\in [n]-G}  L_Y(\bar{f}(\hat{g}(x_i)),y_i).$$
\STATE Compute $B$ as
\begin{equation*}
\begin{split}
B = & -H^{-1}_{\bar{f}} \cdot \sum_{i\in [n]-G} \nabla_{\hat{f}} \Bigl( 
        L_Y \left( \bar{f}(\Tilde{g}(x_i)), y_i \right) \\
    & - L_Y \left( \bar{f}(\hat{g}(x_i)), y_i \right) \Bigr).
\end{split}
\end{equation*}
\STATE Update the label predictor $\Tilde{f}$ as: $\Tilde{f} =  \hat{f} + A + B.$
    \STATE {\bf Return: $\Tilde{f}$, $\Tilde{g}$}.  
\end{algorithmic}
\end{breakablealgorithm}

\begin{breakablealgorithm}
\caption{Data-level CCBM (Addition) \label{alg:3}}
\begin{algorithmic}[1]
    \STATE {\bf Input:} Dataset $\mathcal{D} = \{ (x_i, y_i, c_i) \}_{i=1}^{N}$, original concept predictor $\hat{f}$, label predictor $\hat{g}$, and the data set to be added $\tilde{\mathcal{D}} = \{ (\tilde{x}_i, \tilde{y}_i, \tilde{c}_i) \}_{i=1}^{M}$.
    \STATE Compute the Hessian matrix of the loss function with respect to $\hat{g}$:
    $$H_{\hat{g}} = \nabla^2_{\hat{g}}  \sum_{i,j}  L_{C_j}(\hat{g}^j(x_i),c_i^j).$$
    \STATE Update concept predictor $\Tilde{g}$:
    \begin{equation*}
        \Tilde{g} = \hat{g} - H^{-1}_{\hat{g}} \cdot  \sum_{i\in [M]}\nabla_{g} L_{C} (\hat{g}(\tilde{x}_i), \tilde{c}_i)
    \end{equation*}
    \STATE Update label predictor $\Tilde{f}$.
    Compute the Hessian matrix of the loss function with respect to $\hat{f}$:
$$H_{\hat{f}} = \nabla^2_{\hat{f}}  \sum_{i=1}^n  L_Y(\hat{f}(\hat{g}(x_i), y_i).$$
\STATE   Compute $A$ as:
\begin{equation*}
\begin{split}
A = & - H^{-1}_{\hat{f}} \cdot \Biggl( \sum_{i\in [N]} \nabla_{\hat{f}} L_{Y}\left(\hat{f}(\hat{g}(x_i)), y_i\right) \\
    & + \sum_{i\in [M]} \nabla_{\hat{f}} L_{Y}\left(\hat{f}(\hat{g}(\tilde{x}_i)), \tilde{y}_i\right) \Biggr)
\end{split}
\end{equation*}

\STATE Obtain $\bar{f}$ as 
    \begin{equation*}
        \bar{f} = \hat{f} + A
    \end{equation*}
\STATE Compute the Hessian matrix of the loss function concerning $\bar{f}$:
\begin{equation*}
\begin{split}
H_{\bar{f}} = \nabla^2_{\bar{f}} \Biggl( 
    & \sum_{i\in [N]} L_Y \left( \bar{f}(\hat{g}(x_i)), y_i \right) \\
    & + \sum_{i\in [M]} L_Y \left( \bar{f}(\hat{g}(\tilde{x}_i)), \tilde{y}_i \right) 
\Biggr).
\end{split}
\end{equation*}
\STATE Compute $B$ as
\begin{equation*}
\begin{split}
B = & -H^{-1}_{\bar{f}} \cdot \sum_{i \in [n]\setminus G} \nabla_{\hat{f}} \Bigl( 
        L_Y \bigl( \bar{f}(\Tilde{g}(x_i)), y_i \bigr) \\
    & - L_Y \bigl( \bar{f}(\hat{g}(x_i)), y_i \bigr) \Bigr).
\end{split}
\end{equation*}
\STATE Update the label predictor $\Tilde{f}$ as: $\Tilde{f} =  \hat{f} + A + B.$
    \STATE {\bf Return: $\Tilde{f}$, $\Tilde{g}$}.  
\end{algorithmic}
\end{breakablealgorithm}

\begin{breakablealgorithm}
\caption{EK-FAC for Concept Predictor $g$ \label{alg:ekg}}
\begin{algorithmic}[1]
    \STATE {\bf Input:} Dataset $\mathcal{D} = \{ (x_i, y_i, c_i) \}_{i=1}^{N}$, original concept predictor $\hat{g}$.
\FOR{the $l$-th convolution layer of $\hat{g}$:}     
\STATE Define the input activations $\{a_{j,t}\}$, weights $W = \left(w_{i, j, \delta}\right)$, and biases $b = \left(b_{i}\right)$ of this layer;
\STATE Obtain the expanded activations $\llbracket {{A}}_{l-1} \rrbracket$ as: 
$$    \llbracket {{A}}_{l-1} \rrbracket_{t, j|\Delta|+\delta}=\left[{{A}}_{l-1}\right]_{(t+\delta), j}=a_{j, t+\delta},$$
\STATE Compute the pre-activations:
$$[S_l]_{i,t} = s_{i, t}=\sum_{\delta \in \Delta} w_{i, j, \delta} a_{j, t+\delta}+b_{i}. $$
\STATE During the backpropagation process, obtain the $\mathcal{D}s_{i, t}$ as:
\begin{equation*}
    \mathcal{D}s_{i, t} = \frac{\partial \sum_{j=1}^k\sum_{i=1}^n L_{C_j}}{\partial s_{i, t}}
\end{equation*}
\STATE Compute $\hat{\Omega}_{l-1}$ and $\hat{\Gamma_{l}}$:
\begin{align*}
    \hat{\Omega}_{l-1} =& \frac{1}{n} \sum_{i=1}^{n}\left(\llbracket {A}^i_{l-1} \rrbracket_{\mathrm{H}}^{\top} \llbracket {A}^i_{l-1} \rrbracket_{\mathrm{H}}\right)\\
    \hat{\Gamma_{l}} =& \frac{1}{n} \sum_{i=1}^{n}\left(\frac{1}{|\mathcal{T}|}\mathcal{D} {{S}^i_{l}}^{\top} \mathcal{D} {S}^i_{l}\right)
\end{align*}
\STATE Perform eigenvalue decomposition of $\hat{\Omega}_{l-1}$ and $\hat{\Gamma}_{l}$, obtain $Q_{\Omega}, {\Lambda}_{{\Omega}}, {Q}_{\Gamma}, {\Lambda}_{\Gamma}$, which satisfies 
\begin{align*}
    \hat{\Omega}_{l-1} &= Q_{\Omega} {\Lambda}_{{\Omega}} {Q}_{{\Omega}}^{\top}\\
    \hat{\Gamma}_{l} &= {Q}_{\Gamma}{\Lambda}_{\Gamma} {Q}_{\Gamma}^{\top}
\end{align*}
\STATE Define a diagonal matrix 
$\Lambda$ and compute the diagonal element as
$$    \Lambda_{ii}^{*} = n^{-1} \sum_{j =1}^n \left( \left( Q_{\Omega_{l-1}}\otimes Q_{\Gamma_{l}}\right)\nabla_{\theta_l}L_{C_j}\right)^2_i.$$
\STATE Compute $\hat{H}_l^{-1}$ as 
\begin{equation*}
    \hat{H}_l^{-1} = \left(Q_{\Omega_{l-1}} \otimes Q_{\Gamma_{l}}\right)\left(\Lambda+\lambda_{l} I_{d_{l}}\right)^{-1}\left(Q_{\Omega_{l-1}} \otimes Q_{\Gamma_{l}}\right)^{\mathrm{T}}
\end{equation*}
\ENDFOR
\STATE Splice $H_l$ sequentially into large diagonal matrices
$$\hat{H}_{\hat{g}}^{-1} = \left(\begin{array}{ccc}
\hat{H}_1^{-1} & & \mathbf{0} \\
& \ddots & \\
\mathbf{0} & & \hat{H}_d^{-1}
\end{array}\right)$$
where $d$ is the number of the convolution layer of the concept predictor.
\STATE {\bf Return: the inverse Hessian matrix $\hat{H}_{\hat{g}}^{-1}$}.  
\end{algorithmic}
\end{breakablealgorithm}

\begin{breakablealgorithm}
\caption{EK-FAC for Label Predictor $f$ \label{alg:ekf}}
\begin{algorithmic}[1]
    \STATE {\bf Input:} Dataset $\mathcal{D} = \{ (x_i, y_i, c_i) \}_{i=1}^{N}$, original label predictor $\hat{f}$.
 \STATE Denote the pre-activated output of $\hat{f}$ as $f^{\prime}$, 
    Compute $A$ as 
    \begin{equation*}
        A = \frac{1}{n} \cdot \sum_{i=1}^n \hat{g}(x_i)\cdot\hat{g}(x_i)^{\mathrm{T}}  
    \end{equation*}
    \STATE Comput $B$ as:
    \begin{equation*}
        B= \frac{1}{n} \cdot \sum_{i=1}^n \nabla_{f^{\prime}}L_Y(\hat{f}\left(\hat{g}(x_i)\right), y_i)\cdot {\nabla_{f^{\prime}}L_Y(\hat{f}\left(\hat{g}(x_i)\right), y_i)}^{\mathrm{T}}
    \end{equation*}
    \STATE Perform eigenvalue decomposition of AA and BB, obtain $Q_{A}, {\Lambda}_{{A}}, {Q}_{B}, {\Lambda}_{B}$, which satisfies 
        \begin{align*}
           A &= Q_{A} {\Lambda}_{{A}} {Q}_{{A}}^{\top}\\
            B &= {Q}_{B}{\Lambda}_{B} {Q}_{B}^{\top}
        \end{align*}
\STATE Define a diagonal matrix 
$\Lambda$ and compute the diagonal element as
$$    \Lambda_{ii}^{*} = n^{-1} \sum_{j =1}^n \left( \left( Q_{A}\otimes Q_{B}\right)\nabla_{\hat{f}}L_{Y_j}\right)^2_i.$$
\STATE Compute $\hat{H}_{\hat{f}}^{-1}$ as 
\begin{equation*}
   \hat{H}_{\hat{f}}^{-1} = \left(Q_{A} \otimes Q_{B}\right)\left(\Lambda+\lambda I_{d}\right)^{-1}\left(Q_{A} \otimes Q_{B}\right)^{\mathrm{T}}
\end{equation*}
    \STATE {\bf Return: the inverse Hessian matrix $\hat{H}_{\hat{f}}^{-1}$}.  
\end{algorithmic}
\end{breakablealgorithm}

\begin{breakablealgorithm}
\caption{EK-FAC Concept-label-level CCBM \label{alg:4}}
\begin{algorithmic}[1]
\STATE {\bf Input:} Dataset $\mathcal{D} = \{ (x_i, y_i, c_i) \}_{i=1}^{N}$, original concept predictor $\hat{f}$, label predictor $\hat{g}$, and the to be removed data index set $G$, and damping parameter $\lambda$.
\STATE For $r\in G$, remove the $r$-th data $(x_r, y_r, c_r)$ from $\mathcal{D}$ and define the new dataset as $\mathcal{S}$.
\STATE {\bf Use EK-FAC method in algorithm \ref{alg:ekg} to accelerate iHVP problem for $\hat{g}$ and obtain the inverse Hessian matrix $\hat{H}_{\hat{g}}^{-1}$}
    \STATE Update concept predictor $\Tilde{g}$:
\begin{equation*}
\begin{split}
\tilde{g} 
&= \hat{g} - H^{-1}_{\hat{g}} \cdot \sum_{(w,r)\in S_e} \Bigl( \nabla_{\hat{g}} L_{C_r}\left(\hat{g}^r(x_w), {c_w^r}^{\prime}\right) \\
&\quad - \nabla_{\hat{g}} L_{C_r}\left(\hat{g}^r(x_w), c_w^r\right) \Bigr).
\end{split}
\end{equation*}
    \STATE {\bf Use EK-FAC method in algorithm \ref{alg:ekf} to accelerate iHVP problem for $\hat{f}$ and obtain $\hat{H}_{\hat{f}}^{-1}$}
    \STATE Update label predictor $\Tilde{f}$:
\begin{align*}
\Tilde{f} 
&= \hat{f} 
+ H^{-1}_{\hat{f}} \cdot \nabla_f \sum_{i=1}^n L_{Y} \left( \hat{f} \left( \hat{g} (x_{i}) \right), y_{i} \right) \\
&- H^{-1}_{\hat{f}} \cdot \nabla_f \sum_{l=1}^n \left( L_Y \left( \hat{f} \left( \tilde{g}(x_l) \right), y_{l} \right) \right).
\end{align*}
    \STATE {\bf Return: $\Tilde{f}$, $\Tilde{g}$}.  
\end{algorithmic}
\end{breakablealgorithm}

\begin{breakablealgorithm}
\caption{EK-FAC Concept-level CCBM \label{alg:5}}
\begin{algorithmic}[1]
\STATE {\bf Input:} Dataset $\mathcal{D} = \{ (x_i, y_i, c_i) \}_{i=1}^{ n}$, original concept predictor $\hat{f}$, label predictor $\hat{g}$ and the to be removed concept index set $M$, and damping parameter $\lambda$.
\STATE For $r\in M$, set $p_r = 0$ for all the data $z\in \mathcal{D}$.
\STATE {\bf Use EK-FAC method in algorithm \ref{alg:ekg} to accelerate iHVP problem for $\hat{g}$ and obtain the inverse Hessian matrix $\hat{H}_{\hat{g}}^{-1}$}
    \STATE Update concept predictor $\Tilde{g}$:
 \begin{equation*}
        \Tilde{g}^* = \hat{g} - H^{-1}_{\hat{g}} \cdot \nabla_{\hat{g}} \sum_{j\notin M} \sum_{i=1}^{n} L_{C_j} ( \hat{g}^{j} (x_i), c_i^j).
    \end{equation*}
   \STATE {\bf Use EK-FAC method in algorithm \ref{alg:ekf} to accelerate iHVP problem for $\hat{f}$ and obtain $\hat{H}_{\hat{f}}^{-1}$}
    \STATE Update label predictor $\Tilde{f}$:
    \begin{equation*}
    \Tilde{f} =  \hat{f}-H_{\hat{f}}^{-1} \cdot   \nabla_{\hat{f}} \sum_{l   =1}^{n} L_{Y} \left( \hat{f} \left( \Tilde{g}^*(x_l) \right)  , y_l \right).
    \end{equation*}
    \STATE Map $\Tilde{g}^*$ to $\Tilde{g}$ by removing the $r$-th row of the matrix in the final layer of $\Tilde{g}^*$ for $r\in M$.
    \STATE {\bf Return: $\Tilde{f}$, $\Tilde{g}$.}
\end{algorithmic}
\end{breakablealgorithm}

\begin{breakablealgorithm}
    \caption{EK-FAC Data-level CCBM \label{alg:6}}
    \begin{algorithmic}[1]
\STATE {\bf Input:} Dataset $\mathcal{D} = \{ (x_i, y_i, c_i) \}_{i=1}^{n}$, original concept predictor $\hat{f}$, and label predictor $\hat{g}$, a set of erroneous data $D_e$ and its associated index set $S_e$, and damping parameter $\lambda$.
\STATE For the index $(w, r)$ in $S_e$, correct ${c_w^r}$ to the right label ${c_w^r}^{\prime}$ for the $w$-th data $(x_w, y_w, c_w)$.
\STATE {\bf Use EK-FAC method in algorithm \ref{alg:ekg} to accelerate iHVP problem for $\hat{g}$ and obtain the inverse Hessian matrix $\hat{H}_{\hat{g}}^{-1}$}
\STATE Update concept predictor $\Tilde{g}$:
\begin{equation*}
\begin{split}
\tilde{g} 
&= \hat{g} - H^{-1}_{\hat{g}} \cdot \sum_{(w,r)\in S_e} \Bigl( \nabla_{\hat{g}} L_{C_r}\left(\hat{g}^r(x_w), {c_w^r}^{\prime}\right) \\
&\quad - \nabla_{\hat{g}} L_{C_r}\left(\hat{g}^r(x_w), c_w^r\right) \Bigr).
\end{split}
\end{equation*}
\STATE {\bf Use EK-FAC method in algorithm \ref{alg:ekf} to accelerate iHVP problem for $\hat{f}$ and obtain $H_{\hat{f}}^{-1}$}
Compute $A$ as:
\begin{equation*}
       A =  H^{-1}_{\hat{f}} \cdot  \sum_{i\in [n]-G} \nabla_{\hat{f}} L_{Y}\left(\hat{f}(\hat{g}(x_i)), y_i\right)
\end{equation*}

Obtain $\bar{f}$ as 
\begin{equation*}
    \bar{f} = \hat{f} + A
\end{equation*}
    \STATE {\bf Use EK-FAC method in algorithm \ref{alg:ekf} to accelerate iHVP problem for $\bar{f}$ and obtain ${H}_{\bar{f}}^{-1}$}
Compute $B^{\prime}$ as
\begin{equation*}
    B^{\prime} = -H^{-1}_{\bar{f}}\cdot \sum_{i\in [n]-G}\nabla_{\hat{f}}\left(    L_Y(\bar{f}(\Tilde{g}(x_i)), y_i) - L_Y(\bar{f}(\hat{g}(x_i)), y_i)\right)
\end{equation*}
Update the label predictor $\Tilde{f}$ as: $\Tilde{f} =  \hat{f} + A + B^{\prime}$.
\STATE {\bf Return: $\Tilde{f}$, $\Tilde{g}$}.  
\end{algorithmic}
\end{breakablealgorithm}

%% file: appendix/7_appendix_b_bound.tex
   \subsection{Theoretical Bound for the Influence Function}\label{app:bound_cc}
 Consider the dataset $\mathcal{D}=\{(x_i,c_i,y_i\}{i=1}^n$, the loss function of the concept predictor $g$ is defined as:
 \begin{align*}
L_\text{Total}(\mathcal{D};g)
=\sum_{i=1}^n L_{C}(g(x_i),{c}_i)+ \frac{\delta}{2}\cdot \|g\|^2\\
=
\sum_{i=1}^n \sum_{j=1}^k L_{C}^j(g(x_i),{c_i}) + \frac{\delta}{2}\cdot \|g\|^2\\
 =
\sum_{i=1}^n\sum_{j=1}^k g^j(x_i)^\top \log({c_i}^j)+ \frac{\delta}{2}\cdot \|g\|^2.
\end{align*}

Mathematically, we have a set of erroneous data $D_e$ and its associated index set $S_e\subseteq [n]\times [k]$ such that for each $(w, r)\in S_e$, we have $(x_w, y_w, {c}_w)\in D_e$ with $c_w^r$ is mislabeled and $\tilde{c}_w^r$ is corrected concept label. Thus, our goal is to estimate the retrained CBM. The retrained concept predictor and label predictor will be represented in the following manner. 
\begin{equation}\label{app:concept-label:g}
\begin{split}
\hat{g}_{e} = \argmin \Biggl[ 
    & \sum_{(i, j) \notin S_e} L^j_{C}\left(g(x_i), c_i\right) \\
    & + \sum_{(i, j) \in S_e} L^j_{C}\left(g(x_i), \tilde{c}_i \right) + \frac{\delta}{2} \cdot \|g\|^2 
\Biggr],
\end{split}
\end{equation}
Define the corrected dataset as $\mathcal{D}^*$. Then the loss function with the influence of erroneous data $D_e$ removed becomes
\begin{equation}\label{eq:approxiated_loss}
\begin{split}
L^-(\mathcal{D}^*; g) 
&= \sum_{(i, j) \notin S_e} L^j_{C}\left(g(x_i), c_i\right) \\
&\quad + \sum_{(i, j) \in S_e} L^j_{C}\left(g(x_i), \tilde{c}_i \right) \quad + \frac{\delta}{2} \cdot \|g\|^2.
\end{split}
\end{equation}

Assume $\hat{g} = \argmin  L_\text{Total}(\mathcal{D};g)$ is the original model parameter, and $\hat{g}_{e}(\mathcal{D}^*)$ is the minimizer of $L^-(\mathcal{D}^*; g)$, which is obtained from retraining. Denote $\bar{g}_{e}(\mathcal{D}^*)$ as the updated model with the influence of erroneous data $D_e$ removed and is obtained by the influence function method in theorem \ref{app:thm:concept-label:g}, which is an estimation for $\hat{g}_{e}(\mathcal{D}^*)$.

To simplify the problem, we concentrate on the removal of erroneous data $D_e$ and neglect the process of adding the corrected data back. Once we obtain the bound for $\hat{g}_{e}(\mathcal{D}^*) - \bar{g}_{e}(\mathcal{D}^*)$ under this circumstance, the bound for the case where the corrected data is added back can naturally be derived using a similar approach. For brevity, we use the same notations.

Then, the loss function $L^-(\mathcal{D}^*; g)$ becomes
\begin{equation}\label{eq:approxiated_loss}
    \begin{split}
       L^-(\mathcal{D}^*; g) = \sum_{(i, j) \notin S_e} L^j_{C}\left(g(x_i),c_i\right)+ \frac{\delta}{2}\cdot \|g\|^2\\
       =  L_\text{Total}(\mathcal{D};g) -\sum_{(i, j) \in S_e} L^j_{C}\left(g(x_i),c_i\right)  
    \end{split}
\end{equation}
And the definition of $\bar{g}_{e}(\mathcal{D}^*)$ becomes
\begin{equation}
     \hat{g} + H^{-1}_{\hat{g}}  \cdot \sum_{(w,r)\in S_e}    G^r_C(x_w,{c}_w;\hat{g}) 
\end{equation}
where $H_{\hat{g}} = \nabla^2_{\hat{g}}  \sum_{i,j}  L^j_{C}(\hat{g}(x_i),c_i) + \delta \cdot I$ is the Hessian matrix of the loss function with respect to $\hat{g}$. Here $\delta \cdot I$ is a small damping term for ensuring positive definiteness of the Hessian.
Introducing the damping term into the Hessian is essentially equivalent to adding a regularization term to the initial loss function. Consequently, $\delta$ can also be interpreted as the regularization strength.

In this part, we will study the error between the estimated influence given by the  theorem \ref{app:thm:concept-label:g} method and retraining. We use the parameter changes as the evaluation metric:
\begin{equation}\label{app:theorem_total}
    \left|\left(\bar{g}_{e} -\hat{g}\right) - \left( \hat{g}_e - \hat{g}\right)\right| =  \left|\bar{g}_{e} -  \hat{g}_e\right|
\end{equation}

\begin{assumption}The loss $L_{C}(x,c;g)$
\begin{equation*}
    L_{C}(x,c;g;j) =  L^j_C(g(x), c)
\end{equation*}
is convex and twice-differentiable in $g$, with positive regularization $\delta > 0$. There exists $C_H \in \mathbb{R}$ such that
$$\| \nabla^2_{g} L_{C}(x,c;g_1) - \nabla^2_{g} L_{C}(x,c;g_2)\|_{2} \leq C_H \| g_1 - g_2 \|_2$$
for all $(x, c) \in \mathcal{D}=\{(x_i,c_i)\}_{i=1}^n $, $j\in[k]$ and $g_1, g_2 \in \Gamma$. 
\end{assumption}

Then the function $L'(\mathcal{D}, S_e; g)$:
\begin{equation*}
    L'(\mathcal{D}, S_e; g) = \sum_{(i, j) \in S_e} L^j_{C}\left(g(x_i),c_i\right) = \sum_{(i, j) \in S_e} L_{C}(x_i,c_i;g;j)
\end{equation*}
is convex and twice-differentiable in $g$, with some positive regularization. Then we have
$$\| \nabla^2_{g}  L'(\mathcal{D}, S_e; g_1) - \nabla^2_{g}  L'(\mathcal{D}, S_e; g_2)\|_{2} \leq |S_e|\cdot C_H \| g_1 - g_2 \|_2$$
for  $g_1, g_2 \in \Gamma$.

\begin{corollary}
    \begin{equation*}
        \|\nabla^2_{g}  L^-(\mathcal{D}^*; g_1) - \nabla^2_{g}  L^-(\mathcal{D}^*; g_2)\|_2\leq  
     \left((nk+|S_e|)\cdot C_H \right)\|g_1-g_2\| 
    \end{equation*}
    Define $C_H^- \triangleq(nk+|S_e|)\cdot C_H$
\end{corollary}

\begin{definition}
Define $|\mathcal{D}|$ as the number of pairs 
\begin{equation*}
        C'_{L} = \left\| \nabla_{g} L'(\mathcal{D}, S_e; \hat{g})\right\|_2,
\end{equation*}
\begin{equation*}
    \sigma'_{\text{min}} = \text{smallest singular value of } \nabla^2_{g}  L^-(\mathcal{D}^*; \hat{g}),
\end{equation*}
\begin{equation*}
    \sigma_{\text{min}} = \text{smallest singular value of } \nabla^2_{g}  L_{\text{Total}}(\mathcal{D}; \hat{g}),
\end{equation*}
\end{definition}
Based on above corollaries and assumptions, we derive the following theorem.

\begin{theorem}\label{app:bound_cc_the}
    We obtain the error between the actual influence and our predicted influence as follows:
    \begin{equation*}
        \begin{split}
             &\left\|\hat{g}_e(\mathcal{D}^*) - \bar{g}_e(\mathcal{D}^*)\right\|\\
             \leq & \frac{C_H^-    {C'_{L}}^2}{2 (\sigma'_{\text{min}} + \delta)^3} + \left|\frac{2\delta+\sigma_{\text{min}}+\sigma'_{\text{min}}}{\left(\delta+ \sigma'_{\text{min}}\right)\cdot\left(\delta+ \sigma_{\text{min}}\right)}\right| \cdot C_L'.
        \end{split}
    \end{equation*}
\end{theorem}
\begin{proof}

We will use the one-step Newton approximation as an intermediate step. Define $\Delta g_{Nt}(\mathcal{D}^*)$ as
\begin{equation*}
    \Delta g_{Nt}(\mathcal{D}^*)\triangleq H_{\delta}^{-1}\cdot \nabla_{ g}L'(\mathcal{D}, S_e; \hat{g}),
\end{equation*}
 where $H_{\delta} = \delta \cdot I + \nabla_{ g}^2 L^-(\mathcal{D}^*;\hat{g})$ is the regularized empirical Hessian at $\hat{ g}$ but reweighed after removing the influence of wrong data. Then the one-step Newton approximation for $\hat{ g}(\mathcal{D}^*)$ is defined as $ g_{Nt}(\mathcal{D}^*) \triangleq \Delta g_{Nt}(\mathcal{D}^*) + \hat{ g}$.

In the following, we will separate the error between $\bar{g}_e(\mathcal{D}^*)$ and $\hat{g}_e(\mathcal{D}^*)$ into the following two parts:
\begin{equation*}
\begin{split}
    \hat{g}_e(\mathcal{D}^*) - \bar{g}_e(\mathcal{D}^*) = \underbrace{\hat{g}_e(\mathcal{D}^*) - g_{Nt}(\mathcal{D}^*)}_{\text{Err}_{\text{Nt, act}}(\mathcal{D}^*)} \\
 + \underbrace{\left(g_{Nt}(\mathcal{D}^*)-\hat{ g}\right) - \left(\bar{g}_e(\mathcal{D}^*) - \hat{ g}\right)}_{\text{Err}_{\text{Nt, if}}(\mathcal{D}^*)}
\end{split}
\end{equation*}

Firstly, in {\bf Step $1$}, we will derive the bound for Newton-actual error ${\text{Err}_{\text{Nt, act}}(\mathcal{D}^*)}$.
Since $L^-( g)$ is strongly convex with parameter $\sigma'_{\text{min}} + \delta$ and minimized by 
$\hat{g}_e(\mathcal{D}^*)$, we can bound the distance
$\left\|\hat{g}_e(\mathcal{D}^*) - { g}_{Nt}(\mathcal{D}^*)\right\|_2$ in terms of the norm of the gradient at ${ g}_{Nt}$:
\begin{equation}\label{bound:1}
    \left\|\hat{g}_e(\mathcal{D}^*) - { g}_{Nt}(\mathcal{D}^*)\right\|_2 \leq \frac{2}{\sigma'_{\text{min}} + \delta} \left\|\nabla_{ g} L^- \left({ g}_{Nt}(\mathcal{D}^*)\right)\right\|_2
\end{equation}
Therefore, the problem reduces to bounding $\left\|\nabla_{ g} L^- \left({ g}_{Nt}(\mathcal{D}^*)\right)\right\|_2$.
Noting that $\nabla_{ g}L'(\hat{ g}) = -\nabla_{ g}L^-$. This is because $\hat{ g}$ minimizes $L^- + L'$, that is, $$\nabla_{ g}L^-(\hat{ g}) + \nabla_{ g}L'(\hat{ g}) = 0.$$ Recall that $\Delta g_{Nt}= H_{\delta}^{-1}\cdot \nabla_{ g}L'(\mathcal{D}, S_e; \hat{g}) = -H_{\delta}^{-1}\cdot \nabla_{ g}L^-(\mathcal{D}^*; \hat{g})$.
Given the above conditions, we can have this bound for $\text{Err}_{\text{Nt, act}}(-\mathcal{D}^*)$.
\begin{equation}\label{bound:2}
\begin{split}
&\left\|\nabla_{g} L^- \left({g}_{Nt}(\mathcal{D}^*)\right)\right\|_2 \\
= &\left\|\nabla_{g} L^- \left(\hat{g} + \Delta{g}_{Nt}(\mathcal{D}^*)\right)\right\|_2 \\
= &\left\|\nabla_{g} L^- \left(\hat{g} + \Delta g_{N_t}(\mathcal{D}^*)\right) - \nabla_{g} L^- \left(\hat{g}\right) \right. \\
&\quad \left. - \nabla_{g}^2 L^- \left(\hat{g}\right) \cdot\quad\Delta g_{N_t}(\mathcal{D}^*)\right\|_2 \\
= &\left\|\int_0^1 \left(\nabla_{g}^2 L^- \left(\hat{g} + t\cdot \Delta g_{Nt}(\mathcal{D}^*)\right) \right. \right. \\
&\quad \left. \left. - \nabla_{g}^2 L^- \left(\hat{g}\right)\right) \Delta g_{Nt}(\mathcal{D}^*) \, dt\right\|_2 \\
\leq &\frac{C_H^-}{2} \left\|\Delta g_{Nt}(\mathcal{D}^{*})\right\|_2^2 \\
= &\frac{C_H^-}{2} \left\|\left[\nabla_{g}^2 L^-(\hat{g})\right]^{-1} \nabla_{g} L^-(\hat{g})\right\|_2^2 \\
\leq &\frac{C_{H}^{-}}{2 (\sigma'_{\text{min}} + \delta)^2} \left\|\nabla_{g} L^-(\hat{g})\right\|_2^2 \\
= &\frac{C_H^-}{2 (\sigma'_{\text{min}} + \delta)^2} \left\|\nabla_{g} L'(\hat{g})\right\|_2^2 \\
\leq &\frac{C_{H}^{-} {C'_{L}}^2}{2 (\sigma'_{\text{min}} + \delta)^2}.
\end{split}
\end{equation}

Now we come to {\bf Step $2$} to bound ${\text{Err}_{\text{Nt, if}}(-\mathcal{D}^*)}$, and we will bound the difference in parameter change between Newton and our CCBM method.
\begin{align*}
&\left\|\left(g_{Nt}(\mathcal{D}^*)-\hat{g}\right) - \left(\bar{g}_e(\mathcal{D}^*) - \hat{g}\right)\right\| \\
= &\left\| \left[\left(\delta \cdot I + \nabla_{g}^2 L^- \left(\hat{g}\right)\right)^{-1} \right. \right. + \left.\left. \left(\delta \cdot I + \nabla_{g}^2 L_{\text{Total}} \left(\hat{g}\right)\right)^{-1}\right] \right. \\
&\quad \left. \cdot \nabla_{g} L'(\mathcal{D}, S_e; \hat{g}) \right\|
\end{align*}
For simplification, we use matrix $A$, $B$ for the following substitutions:
\begin{align*}
    A = {\delta \cdot I+ \nabla_{ g}^2 L^- \left(\hat{ g}\right)}\\
    B = {\delta \cdot I+ \nabla_{ g}^2 L_{\text{Total}} \left(\hat{ g}\right)}
\end{align*}
And $A$ and $B$ are positive definite matrices with the following properties
\begin{align*}
\delta + \sigma'_{\text{min}} \prec A \prec \delta +   \sigma'_{\text{max}}\\
\delta + \sigma_{\text{min}} \prec B \prec \delta +   \sigma_{\text{max}}\\
\end{align*}
Therefore, we have
\begin{equation}\label{bound:3}
    \begin{split}
             &\left\|{\left(g_{Nt}(\mathcal{D}^*)-\hat{ g}\right) - \left(\bar{g}_e(\mathcal{D}^*) - \hat{ g}\right)}\right\| \\
     =&\left\|\left(A^{-1}+B^{-1}\right)\cdot \nabla_{ g}L^-(\mathcal{D}^*; \hat{g})\right\| \\
     \leq & \left\|A^{-1}+B^{-1}\right\|\cdot \left\|\nabla_{ g}L^-(\mathcal{D}^*; \hat{g})\right\|\\
     \leq & \left|\frac{2\delta+\sigma_{\text{min}}+\sigma'_{\text{min}}}{\left(\delta+ \sigma'_{\text{min}}\right)\cdot\left(\delta+ \sigma_{\text{min}}\right)}\right|\cdot\left\|\nabla_{ g}L^-(\mathcal{D}^*; \hat{g})\right\|\\
     \leq& \left|\frac{2\delta+\sigma_{\text{min}}+\sigma'_{\text{min}}}{\left(\delta+ \sigma'_{\text{min}}\right)\cdot\left(\delta+ \sigma_{\text{min}}\right)}\right| \cdot C_L'
    \end{split}
\end{equation}
By combining the conclusions from Step I and Step II in Equations \ref{bound:1}, \ref{bound:2} and \ref{bound:3}, we obtain the error between the actual influence and our predicted influence as follows:
\begin{equation*}
    \begin{split}
         &\left\|\hat{g}_e(\mathcal{D}^*) - \bar{g}_e(\mathcal{D}^*)\right\|\\
         \leq & \frac{C_H^-    {C'_{L}}^2}{2 (\sigma'_{\text{min}} + \delta)^3} + \left|\frac{2\delta+\sigma_{\text{min}}+\sigma'_{\text{min}}}{\left(\delta+ \sigma'_{\text{min}}\right)\cdot\left(\delta+ \sigma_{\text{min}}\right)}\right| \cdot C_L'.
    \end{split}
\end{equation*}
\end{proof}

\begin{remark}
Theorem \ref{app:bound_cc_the} reveals one key finding about influence function estimation: The estimation error scales inversely with the regularization parameter $\delta$ ($ \mathcal{O}(1/\delta)$), indicating that increased regularization improves approximation accuracy.
\end{remark}

\begin{remark}
 In CBM, retraining is the most accurate way to handle the removal of a training data point. For the concept predictor, we derive a theoretical error bound for an influence function-based approximation. However, the label predictor differs. As a single-layer linear model, the label predictor is computationally inexpensive to retrain. However, its input depends on the concept predictor, making theoretical analysis challenging due to: (1) Input dependency: Changes in the concept predictor affect the label predictor's input, coupling their updates. (2) Error propagation: Errors from the concept predictor propagate to the label predictor, introducing complex interactions. Given the label predictor's low retraining cost, direct retraining is more practical and accurate. Thus, we focus our theoretical analysis on the concept predictor.
\end{remark}

%% file: appendix/7_appendix_b.tex
\section{Proof of Concept-label-level Influence}
\label{sec:appendix:b}
We have a set of erroneous data $D_e$ and its associated index set $S_e\subseteq [n]\times [k]$ such that for each $(w, r)\in S_e$, we have $(x_w, y_w, c_w)\in D_e$ with $c_w^r$ is mislabeled and $\tilde{c}_w^r$ is its corrected concept label. Thus, our goal is to approximate the new CBM without retraining.

\paragraph{Proof Sketch.} Our goal is to edit $\hat{g}$ and $\hat{f}$ to $\hat{g}_{e}$ and $\hat{f}_{e}$. (i) First, we introduce new parameters $\hat{g}_{\epsilon, e}$ that minimize a modified loss function with a small perturbation $\epsilon$. (ii) Then, we perform a Newton step around $\hat{g}$ and obtain an estimate for $\hat{g}_{e}$. (iii) Then, we consider changing the concept predictor at one data point $(x_{i_c}, y_{i_c}, c_{i_c})$ and retraining the model to obtain a new label predictor $\hat{f}_{i_c}$, obtain an approximation for $\hat{f}_{i_c}$. (iv) Next, we iterate $i_c$ over $1, 2, \cdots, n$, sum all the equations together, and perform a Newton step around $\hat{f}$ to obtain an approximation for $\hat{f}_{e}$. (v) Finally, we bring the estimate of $\hat{g}$ into the equation for $\hat{f}_{e}$ to obtain the final approximation.

\begin{theorem}\label{app:thm:concept-label:g}
The retrained concept predictor $\hat{g}_{e}$ defined by 
\begin{equation}\label{app:concept-label:g}
\begin{split}
\hat{g}_{e} = \argmin \Biggl[ 
    & \sum_{(i, j) \notin S_e} L_{C}\left(g^j(x_i), c_i^j\right) \\
    & + \sum_{(i, j) \in S_e} L_{C}\left(g^j(x_i), \tilde{c}_i^j \right)
\Biggr],
\end{split}
\end{equation}
can be approximated by: 
\begin{equation}
\begin{split}
\hat{g}_{e} \approx \bar{g}_{e} 
&\triangleq \hat{g} - H^{-1}_{\hat{g}} \cdot \sum_{(w,r)\in S_e} \Bigl( \nabla_{\hat{g}} L_C\left(\hat{g}^r(x_w), \tilde{c}_w^r\right) \\
&\quad - \nabla_{\hat{g}} L_C\left(\hat{g}^r(x_w), c_w^r\right) \Bigr),
\end{split}
\end{equation}
where $H_{\hat{g}} = \nabla^2_{\hat{g}}  \sum_{i,j}  L_{C}(\hat{g}^j(x_i),c_i^j)$ is the Hessian matrix of the loss function respect to $\hat{g}$.
\end{theorem}

\begin{proof}
For the index $(w, r) \in S_e$, indicating the $r$-th concept of the $w$-th data is wrong, we correct this concept $c_w^r$ to $\tilde{c}_w^r$. Rewrite $\hat{g}_{e}$ as
\begin{equation}\label{min_g}
\begin{split}
\hat{g}_{e} = \argmin \Biggl[ 
    & \sum_{i, j} L_{C}\left(g^j(x_i), c_i^j\right) + \sum_{(w,r)\in S_e} L_{C}\left(g^r(x_w), \tilde{c}_w^r\right) \\
    & - \sum_{(w,r)\in S_e} L_{C}\left(g^r(x_w), c_w^r\right)
\Biggr].
\end{split}
\end{equation}

To approximate this effect, define new parameters $\hat{g}_{\epsilon, e}$ as
\begin{equation}\label{min}
\begin{split}
\hat{g}_{\epsilon, e} \triangleq \argmin \Biggl[ 
    & \sum_{i,j} L_{C}\left(g^j(x_i), c_i^j\right) \\
    + \sum_{(w,r)\in S_e} \epsilon \cdot L_{C}\left(g^r(x_w), \tilde{c}_w^r\right) 
    & - \sum_{(w,r)\in S_e} \epsilon \cdot L_{C}\left(g^r(x_w), c_w^r\right)
\Biggr].
\end{split}
\end{equation}

Then, because $\hat{g}_{\epsilon, e}$ minimizes equation \ref{min}, we have
\begin{equation*}
\begin{split}
\nabla_{\hat{g}} \sum_{i, j} L_{C}\left(\hat{g}^j_{\epsilon, e}(x_i), c_i^j\right) 
+ \sum_{(w,r)\in S_e} \epsilon \cdot \nabla_{\hat{g}} L_{C}\left(\hat{g}_{\epsilon, e}^r(x_w), \tilde{c}_w^r\right) \\
- \sum_{(w,r)\in S_e} \epsilon \cdot \nabla_{\hat{g}} L_{C}\left(\hat{g}_{\epsilon, e}^r(x_w), c_w^r\right) = 0.
\end{split}
\end{equation*}

Perform a Taylor expansion of the above equation at $\hat{g}$,
\begin{align}\label{tl_31}
\begin{split}
&\quad\nabla_{\hat{g}} \sum_{i,j} L_{C}\left(\hat{g}^j(x_i), c_i^j\right) + \sum_{(w,r)\in S_e} \epsilon \cdot \nabla_{\hat{g}} L_C\left(\hat{g}^r(x_w), \tilde{c}_w^r\right) \\
&\quad- \sum_{(w,r)\in S_e} \epsilon \cdot \nabla_{\hat{g}} L_C\left(\hat{g}^r(x_w), c_w^r\right) \\
&\quad+ \nabla^2_{\hat{g}} \sum_{i,j} L_C\left(\hat{g}^j(x_i), c_i^j\right) \cdot (\hat{g}_{\epsilon, e} - \hat{g}) \approx 0.
\end{split}
\end{align}

Because of equation \ref{app:concept-label:g}, the first term of equation \ref{tl_31} equals $0$. Then we have
\begin{equation*}
\begin{split}
\hat{g}_{\epsilon, e} - \hat{g} 
&= - \sum_{(w,r)\in S_e} \epsilon \cdot H^{-1}_{\hat{g}} \cdot \left( \nabla_{\hat{g}} L_C\left(\hat{g}^r(x_w), \tilde{c}_w^r\right) \right. \\
&\quad \left. - \nabla_{\hat{g}} L_C\left(\hat{g}^r(x_w), c_w^r\right) \right),
\end{split}
\end{equation*}
where 
\begin{equation*}
    H_{\hat{g}} = \nabla^2_{\hat{g}}  \sum_{i,j}  L_C\left(\hat{g}^j(x_i),c_i^j\right).
\end{equation*}

Then, we do a Newton step around $\hat{g}$ and obtain
\begin{equation}\label{3-esti-g}
\begin{split}
\hat{g}_{e} \approx \bar{g}_{e} 
&\triangleq \hat{g} - H^{-1}_{\hat{g}} \cdot \sum_{(w,r)\in S_e} \Bigl( \nabla_{\hat{g}} L_C\left(\hat{g}^r(x_w), \tilde{c}_w^r\right) \\
&\quad - \nabla_{\hat{g}} L_C\left(\hat{g}^r(x_w), c_w^r\right) \Bigr).
\end{split}
\end{equation}
\end{proof}

\begin{theorem}
The retrained label predictor $\hat{f}_{e}$ defined by 
\begin{equation*}
    \hat{f}_{e} = \argmin \left[\sum_{i=1}^n L_{Y} \left(f\left(\hat{g}_{e}\left(x_i\right)\right), y_i\right)\right]
\end{equation*}
can be approximated by: 
\begin{align*}
\hat{f}_{e} 
&\approx \bar{f}_{e} = \hat{f} + H^{-1}_{\hat{f}} \cdot \nabla_f \sum_{i=1}^n L_{Y_i} \left(\hat{f}, \hat{g}\right) \\
&\quad - H^{-1}_{\hat{f}} \cdot \nabla_f \sum_{i=1}^n L_{Y_i} \left(\hat{f}, \bar{g}_{e}\right),
\end{align*}
where $H_{\hat{f}} = \nabla^2_{\hat{f}}\sum_{i=1}^n L_{Y_i}(\hat{f}, \hat{g})$ is the Hessian matrix of the loss function respect to $\hat{f}$, $L_{Y_i}(\hat{f},\hat{g}) \triangleq L_{Y}(\hat{f}(\hat{g}(x_i)), y_i)$, and $\bar{g}_{e}$ is given in Theorem \ref{app:thm:concept-label:g}.
\end{theorem}

\begin{proof}
Now we come to deduce the edited label predictor towards $\hat{f}_{e}$. 

First, we consider only changing the concept predictor at one data point $(x_{i_c}, y_{i_c}, c_{i_c})$ and retrain the model to obtain a new label predictor $\hat{f}_{i_c}$.
\begin{equation*}
\begin{split}
\hat{f}_{i_c} = \argmin \Biggl[ 
    & \sum_{i=1}^n L_Y \left(f\left(\hat{g}\left(x_i\right)\right), y_i\right) + L_Y \left(f\left(\hat{g}_{e}\left(x_{i_c}\right)\right), y_{i_c}\right) \\
    & - L_Y \left(f\left(\hat{g}\left(x_{i_c}\right)\right), y_{i_c}\right)
\Biggr].
\end{split}
\end{equation*}

We rewrite the above equation as follows:
\begin{equation*}
\begin{split}
\hat{f}_{i_c} = \argmin \Biggl[ 
    & \sum_{i=1}^n L_Y \left(f\left(\hat{g}\left(x_i\right)\right), y_i\right) + L_Y \left(f\left(\hat{g}_{e}\left(x_{i_c}\right)\right), y_{i_c}\right) \\
    & - L_Y \left(f\left(\hat{g}\left(x_{i_c}\right)\right), y_{i_c}\right)
\Biggr].
\end{split}
\end{equation*}

We define $\hat{f}_{\epsilon, i_c}$ as:
\begin{equation*}
\begin{split}
\hat{f}_{\epsilon, i_c} = \argmin \Biggl[ 
    & \sum_{i=1}^n L_Y \left(f\left(\hat{g}\left(x_i\right)\right), y_i\right) \\
    & + \epsilon \cdot L_Y \left(f\left(\hat{g}_{e}\left(x_{i_c}\right)\right), y_{i_c}\right) \\
    & - \epsilon \cdot L_Y \left(f\left(\hat{g}\left(x_{i_c}\right)\right), y_{i_c}\right)
\Biggr].
\end{split}
\end{equation*}
Derive with respect to $f$ at both sides of the above equation. we have
\begin{equation*}
\begin{split}
\nabla_{\hat{f}}\sum_{i=1}^n L_Y &\left(\hat{f}_{\epsilon,i_c}\left(\hat{g}(x_i)\right), y_i\right) \\
&+ \epsilon \cdot \nabla_{\hat{f}} L_Y \left(\hat{f}_{\epsilon,i_c}\left(\hat{g}_e(x_{i_c})\right), y_{i_c}\right) \\
&- \epsilon \cdot \nabla_{\hat{f}} L_Y \left(\hat{f}_{\epsilon,i_c}\left(\hat{g}(x_{i_c})\right), y_{i_c}\right) \\
&= 0
\end{split}
\end{equation*}

Perform a Taylor expansion of the above equation at $\hat{f}$,
\begin{align*}
    &\nabla_{\hat{f}}\sum_{i=1}^n L_Y \left(\hat{f}\left(\hat{g}\left(x_i\right)\right), y_i\right) +\epsilon\cdot \nabla_{\hat{f}} L_Y \left(\hat{f}\left(\hat{g}_{ e}\left(x_{i_c}\right)\right), y_{i_c}\right)\\
        &-\epsilon\cdot\nabla_{\hat{f}}  L_Y \left(\hat{f}\left(\hat{g}\left(x_{i_c}\right)\right), y_{i_c}\right) \\
        &+ \nabla^2_{\hat{f}}\sum_{i=1}^n L_Y \left(\hat{f}\left(\hat{g}\left(x_i\right)\right), y_i\right)\cdot \left(\hat{f}_{\epsilon, i_c}-\hat{f}\right)= 0 
\end{align*}

Then we have
\begin{equation*}
\begin{split}
\hat{f}_{\epsilon, i_c} - \hat{f} 
&\approx -\epsilon \cdot H^{-1}_{\hat{f}} \cdot \nabla_f \left( L_Y \left( \hat{f} \left( \hat{g}_{e} \left(x_{i_c} \right) \right), y_{i_c} \right) \right. \\
&\quad \left. - L_Y \left( \hat{f} \left( \hat{g} \left(x_{i_c} \right) \right), y_{i_c} \right) \right),
\end{split}
\end{equation*}
where $H^{-1}_{\hat{f}} = \nabla^2_{\hat{f}}\sum_{i=1}^nL_Y \left(\hat{f}\left(\hat{g}\left(x_i\right)\right), y_i\right)$.

Iterate $i_c$ over $1, 2,\cdots, n$, and sum all the equations together, we can obtain:
\begin{equation*}
\begin{split}
\hat{f}_{\epsilon, e} - \hat{f} 
&\approx -\epsilon \cdot H^{-1}_{\hat{f}} \cdot \sum_{i=1}^n \nabla_f \left( L_Y \left( \hat{f} \left( \hat{g}_{e} \left(x_{i} \right) \right), y_{i} \right) \right. \\
&\quad \left. - L_Y \left( \hat{f} \left( \hat{g} \left(x_{i} \right) \right), y_{i} \right) \right).
\end{split}
\end{equation*}

Perform a Newton step around $\hat{f}$ and we have
\begin{equation}\label{3-esti-f}
\begin{split}
\hat{f}_{e} \approx \hat{f} 
& - H^{-1}_{\hat{f}} \cdot \sum_{i=1}^n \nabla_f \left( L_Y \left( \hat{f} \left( \hat{g}_{e} \left(x_{i} \right) \right), y_{i} \right) \right. \\
& \left. - L_Y \left( \hat{f} \left( \hat{g} \left(x_{i} \right) \right), y_{i} \right) \right).
\end{split}
\end{equation}

Bringing the edited \ref{3-esti-g} of $g$ into equation \ref{3-esti-f}, we have
\begin{align*}
\hat{f}_{e} 
&\approx \hat{f} - H^{-1}_{\hat{f}} \cdot \sum_{i=1}^n \nabla_f \Bigl( L_Y \left( \hat{f} \left( \bar{g}_{e} (x_{i}) \right), y_{i} \right) \\
&\quad - L_Y \left( \hat{f} \left( \hat{g} (x_{i}) \right), y_{i} \right) \Bigr) \\
&= \hat{f} - H^{-1}_{\hat{f}} \cdot \sum_{i=1}^n \nabla_f \left( L_{Y_i} (\hat{f}, \bar{g}_{e}) - L_{Y_i} (\hat{f}, \hat{g}) \right) \\
&= \hat{f} + H^{-1}_{\hat{f}} \cdot \nabla_f \sum_{i=1}^n L_{Y_i} (\hat{f}, \hat{g}) \\
&\quad - H^{-1}_{\hat{f}} \cdot \nabla_f \sum_{i=1}^n L_{Y_i} (\hat{f}, \bar{g}_e) \\
&\triangleq \bar{f}_{e}.
\end{align*}
\end{proof}

%% file: appendix/7_appendix_c_bound.tex
   \subsection{Theoretical Bound for the Influence Function}\label{app:bound_c}
 Consider the dataset $\mathcal{D}=\{(x_i,c_i,y_i\}_{i=1}^n$, the loss function of the concept predictor $g$ is defined as:
\begin{equation*}
\begin{split}
L_\text{Total}(\mathcal{D};g) 
&= \sum_{i=1}^n L_{C}(g(x_i), c_i) + \frac{\delta}{2} \cdot \|g\|^2 \\
&= \sum_{i=1}^n \sum_{j=1}^k L_{C}^j(g(x_i), c_i^j) + \frac{\delta}{2} \cdot \|g\|^2 \\
&= \sum_{i=1}^n \sum_{j=1}^k g^j(x_i)^\top \log(c_i^j) + \frac{\delta}{2} \cdot \|g\|^2.
\end{split}
\end{equation*}

Mathematically, we have a set of erroneous concepts need to be removed, which are denoted as $p_r$ for $r\in M$. Then the retrained concept predictor becomes
\begin{equation*}
\hat{g}_{-p_M} =  \argmin_{g'} \sum_{j\notin M}\sum_{i=1}^n L^j_{C}({g'}(x_i),c_i)+ \frac{\delta}{2}\cdot \|g\|^2.
\end{equation*}
We map it to $\hat{g}^{*}_{-p_M}$ as $\hat{g}_{-p_M}$ to $\hat{g}^{*}_{-p_M} \triangleq \mathrm{P}(\hat{g}_{-p_M})$, which has the same amount of parameters as $\hat{g}$ and has the same predicted concepts $\hat{g}^{*}_{-p_M}(j)$ as $\hat{g}_{-p_M}(j)$ for all $j\in [d_i]-M$. We achieve this effect by inserting a zero row vector into the $r$-th row of the matrix in the final layer of $\hat{g}_{-p_M}$ for $r\in M$. Thus, we can see that the mapping $P$ is one-to-one. Moreover, assume the parameter space of $\hat{g}$ is $T$ and that of $\hat{g}^{*}_{-p_M}$, $T_0$ is the subset of $T$. Noting that $\hat{g}^{*}_{-p_M}$ is the optimal model of the following objective function:
\begin{equation*}
\hat{g}^{*}_{-p_M} = \argmin_{g^{\prime} \in T_0} \sum_{j \notin M} \sum_{i=1}^{n} L^j_{C} ( g'(x_i), c_i)+ \frac{\delta}{2}\cdot \|g\|^2.
\end{equation*}
Then the loss function with the influence of erroneous concepts removed becomes
\begin{equation}
    \begin{split}
       L^-(\mathcal{D}; g) = \sum_{j \notin M} \sum_{i=1}^{n} L^j_{C} ( g'(x_i), c_i)+ \frac{\delta}{2}\cdot \|g\|^2 \\
       =  L_\text{Total}(\mathcal{D};g) -\sum_{j \in M}\sum_{i=1}^{n} L^j_{C}\left(g(x_i),c_i\right).
    \end{split}
\end{equation}

Assume $\hat{g} = \argmin  L_\text{Total}(\mathcal{D};g)$ is the original model parameter. $\hat{g}_{-p_M}(\mathcal{D})$ and $\hat{g}^{*}_{-p_M}(\mathcal{D})$ is the minimizer of $L^-(\mathcal{D}; g)$, which is obtained from retraining in different parameter space. $\hat{g}^{*}_{-p_M}(\mathcal{D})$ shares the same dimensionality as the original model. Because $\hat{g}_{-p_M}(\mathcal{D})$ and $\hat{g}^{*}_{-p_M}(\mathcal{D})$ produces identical outputs given identical inputs, to simplify the proof, we use $\hat{g}^{*}_{-p_M}(\mathcal{D})$ as the retrained model. 

Denote $\bar{g}_{-p_M}$ as the updated model with the influence of erroneous concepts removed and is obtained by the influence function method in theorem \ref{apc:thm:1}, which is an estimation for $\hat{g}^{*}_{-p_M}(\mathcal{D})$.
\begin{equation*}
    \bar{g}_{-p_M}(\mathcal{D})\triangleq \hat{g}  - H^{-1}_{\hat{g}} \cdot \sum_{j\notin M}\sum_{i=1}^n G_C^j(x_i,c_i;\hat{g}),
\end{equation*}

In this part, we will study the error between the estimated influence given by the theorem \ref{apc:thm:1} method and $\hat{g}^{*}_{-p_M}(\mathcal{D})$. We use the parameter changes as the evaluation metric:
\begin{equation}\label{app:theorem_total}
    \left|\left(\bar{g}_{-p_M} -\hat{g}\right) - \left( \hat{g}^{*}_{-p_M} - \hat{g}\right)\right| =  \left|\bar{g}_{-p_M} -  \hat{g}^{*}_{-p_M}\right|
\end{equation}

\begin{assumption}The loss $L_{C}(x,c;g;j)$
\begin{equation*}
    L_{C}(\mathcal{D};g;j) =  \sum_{i=1}^n L^j_C(g(x_i), c_i).
\end{equation*}
is convex and twice-differentiable in $g$, with positive regularization $\delta > 0$. There exists $C_H \in \mathbb{R}$ such that
$$\| \nabla^2_{g} L_{C}(\mathcal{D};g_1;j) - \nabla^2_{g} L_{C}(\mathcal{D};g_2;j)\|_{2} \leq C_H \| g_1 - g_2 \|_2$$
for all  $j\in [k]$ and $g_1, g_2 \in \Gamma$. 
\end{assumption}


\begin{definition}
\begin{equation*}
        C'_{L} = \max_{j}\left\| \nabla_{g} L_C(\mathcal{D}; \hat{g};j)\right\|_2,
\end{equation*}
\begin{equation*}
    \sigma'_{\text{min}} = \text{smallest singular value of } \nabla^2_{g}  L^-(\mathcal{D}; \hat{g}),
\end{equation*}
\begin{equation*}
    \sigma_{\text{min}} = \text{smallest singular value of } \nabla^2_{g}  L_{\text{Total}}(\mathcal{D}; \hat{g}),
\end{equation*}
\begin{equation}
    L'(\mathcal{D},M;g) = \sum_{j\in M }   L_{C}(\mathcal{D};g;j)
\end{equation}
\end{definition}

\begin{corollary}
\begin{equation}
    L^-(\mathcal{D}; g) = L_{\text{Total}}(\mathcal{D};g) - L'(\mathcal{D},M;g)
\end{equation}
    \begin{equation*}
        \|\nabla^2_{g}  L^-(\mathcal{D}; g_1) - \nabla^2_{g}  L^-(\mathcal{D}; g_2)\|_2\leq  
     \left((k+|M|)\cdot C_H\right)\|g_1-g_2\| 
    \end{equation*}
    Define $C_H^- \triangleq(k+|M|)\cdot C_H$
\end{corollary}

Based on above corollaries and assumptions, we derive the following theorem.

\begin{theorem}\label{app:bound_c_the}
    We obtain the error between the actual influence and our predicted influence as follows:
    \begin{equation*}
        \begin{split}
             &\left\|\hat{g}^{*}_{-p_M}(\mathcal{D} ) - \bar{g}_{-p_M}(\mathcal{D} )\right\|\\
             \leq & \frac{C_H^-    {C'_{L}|M|}^2}{2 (\sigma'_{\text{min}} + \delta)^3} + \left|\frac{2\delta+\sigma_{\text{min}}+\sigma'_{\text{min}}}{\left(\delta+ \sigma'_{\text{min}}\right)\cdot\left(\delta+ \sigma_{\text{min}}\right)}\right| \cdot C_L'|M|.
        \end{split}
    \end{equation*}
\end{theorem}
\begin{proof}

We will use the one-step Newton approximation as an intermediate step. Define $\Delta g_{Nt}(\mathcal{D} )$ as
\begin{equation*}
    \Delta g_{Nt}(\mathcal{D} )\triangleq H_{\delta}^{-1}\cdot \nabla_{ g}L'(\mathcal{D}, M; \hat{g}),
\end{equation*}
 where $H_{\delta} = \delta \cdot I + \nabla_{ g}^2 L^-(\mathcal{D} ;\hat{g})$ is the regularized empirical Hessian at $\hat{ g}$ but reweighed after removing the influence of wrong data. Then the one-step Newton approximation for $\hat{g}^{*}_{-p_M}(\mathcal{D} )$ is defined as $ g_{Nt}(\mathcal{D} ) \triangleq \Delta g_{Nt}(\mathcal{D} ) + \hat{ g}$.

In the following, we will separate the error between $\bar{g}_{-p_M}(\mathcal{D} )$ and $\hat{g}^{*}_{-p_M}(\mathcal{D} )$ into the following two parts:
\begin{equation*}
\begin{split}
\hat{g}^{*}_{-p_M}(\mathcal{D}) - \bar{g}_{-p_M}(\mathcal{D}) 
&= \underbrace{\hat{g}^{*}_{-p_M}(\mathcal{D}) - g_{Nt}(\mathcal{D})}_{\text{Err}_{\text{Nt, act}}(\mathcal{D})} \\
&\quad + \underbrace{\left(g_{Nt}(\mathcal{D}) - \hat{g}\right) - \left(\bar{g}_{-p_M}(\mathcal{D}) - \hat{g}\right)}_{\text{Err}_{\text{Nt, if}}(\mathcal{D})}
\end{split}
\end{equation*}

Firstly, in {\bf Step $1$}, we will derive the bound for Newton-actual error ${\text{Err}_{\text{Nt, act}}(\mathcal{D} )}$.
Since $L^-( g)$ is strongly convex with parameter $\sigma'_{\text{min}} + \delta$ and minimized by 
$\hat{g}^{*}_{-p_M}(\mathcal{D} )$, we can bound the distance
$\left\|\hat{g}^{*}_{-p_M}(\mathcal{D} ) - { g}_{Nt}(\mathcal{D} )\right\|_2$ in terms of the norm of the gradient at ${ g}_{Nt}$:
\begin{equation}\label{bound:1}
    \left\|\hat{g}^{*}_{-p_M}(\mathcal{D} ) - { g}_{Nt}(\mathcal{D} )\right\|_2 \leq \frac{2}{\sigma'_{\text{min}} + \delta} \left\|\nabla_{ g} L^- \left({ g}_{Nt}(\mathcal{D} )\right)\right\|_2
\end{equation}
Therefore, the problem reduces to bounding $\left\|\nabla_{ g} L^- \left({ g}_{Nt}(\mathcal{D} )\right)\right\|_2$.
Noting that $\nabla_{ g}L'(\hat{ g}) = -\nabla_{ g}L^-$. This is because $\hat{ g}$ minimizes $L^- + L'$, that is, $$\nabla_{ g}L^-(\hat{ g}) + \nabla_{ g}L'(\hat{ g}) = 0.$$ Recall that $\Delta g_{Nt}= H_{\delta}^{-1}\cdot \nabla_{ g}L'(\mathcal{D}, S_e; \hat{g}) = -H_{\delta}^{-1}\cdot \nabla_{ g}L^-(\mathcal{D} ; \hat{g})$.
Given the above conditions, we can have this bound for $\text{Err}_{\text{Nt, act}}(-\mathcal{D} )$.
\begin{equation}\label{bound:2}
    \begin{split}
            &\left\|\nabla_{ g} L^- \left({ g}_{Nt}(\mathcal{D} )\right)\right\|_2\\
    = & \left\|\nabla_{ g} L^- \left(\hat{ g} + \Delta{ g}_{Nt}(\mathcal{D} )\right)\right\|_2\\
    = & \left\|\nabla_{ g} L^- \left(\hat{ g} + \Delta  g_{N_t}(\mathcal{D} )\right) - \nabla_{ g} L^- \left(\hat{ g} \right) -  \nabla_{ g}^2 L^- \left(\hat{ g}\right) \cdot  \Delta  g_{N_t}(\mathcal{D} )\right\|_2\\
     = & \left\|\int_0^1 \left(\nabla_{ g}^2 L^- \left(\hat{ g} + t\cdot \Delta g_{Nt}(\mathcal{D} )\right) - \nabla_{ g}^2 L^- \left(\hat{ g}\right)\right) \Delta g_{Nt}(\mathcal{D} ) \, dt\right\|_2\\
    \leq & \frac{C_H^-}{2} \left\|\Delta  g_{Nt}(\mathcal{D}^{*})\right\|_2^2 =  \frac{C_H^-}{2} \left\|\left[\nabla_{ g}^2 L^-(\hat{ g})\right]^{-1} \nabla_{ g} L^-(\hat{ g})\right\|_2^2\\
    \leq & \frac{C_{H}^{-}}{2 (\sigma'_{\text{min}} + \delta)^2} \left\|\nabla_{ g} L^-(\hat{ g})\right\|_2^2 = \frac{C_H^-}{2 (\sigma'_{\text{min}} + \delta)^2} \left\|\nabla_{ g} L'(\hat{ g})\right\|_2^2\\
    \leq &\frac{C_{H}^{-}    {C'_{L}|M|}^2}{2 (\sigma'_{\text{min}} + \delta)^2}.
    \end{split}
\end{equation}

Now we come to {\bf Step $2$} to bound ${\text{Err}_{\text{Nt, if}}(-\mathcal{D} )}$, and we will bound the difference in parameter change between Newton and our CCBM method.
\begin{align*}
&\left\|\left(g_{Nt}(\mathcal{D}) - \hat{g}\right) - \left(\bar{g}_{-p_M}(\mathcal{D}) - \hat{g}\right)\right\| \\
&= \left\| \left[ \left(\delta \cdot I + \nabla_{g}^2 L^- (\hat{g})\right)^{-1} \right. \left. + \left(\delta \cdot I + \nabla_{g}^2 L_{\text{Total}} (\hat{g})\right)^{-1} \right] \right. \\
&\quad \left. \cdot \nabla_{g} L'(\mathcal{D}, S_e; \hat{g}) \right\|
\end{align*}
For simplification, we use matrix $A$, $B$ for the following substitutions:
\begin{align*}
    A = {\delta \cdot I+ \nabla_{ g}^2 L^- \left(\hat{ g}\right)}\\
    B = {\delta \cdot I+ \nabla_{ g}^2 L_{\text{Total}} \left(\hat{ g}\right)}
\end{align*}
And $A$ and $B$ are positive definite matrices with the following properties
\begin{align*}
\delta + \sigma'_{\text{min}} \prec A \prec \delta +   \sigma'_{\text{max}}\\
\delta + \sigma_{\text{min}} \prec B \prec \delta +   \sigma_{\text{max}}\\
\end{align*}
Therefore, we have
\begin{equation}\label{bound:3}
    \begin{split}
             &\left\|{\left(g_{Nt}(\mathcal{D} )-\hat{ g}\right) - \left(\bar{g}_{-p_M}(\mathcal{D} ) - \hat{ g}\right)}\right\| \\
     =&\left\|\left(A^{-1}+B^{-1}\right)\cdot \nabla_{ g}L^-(\mathcal{D} ; \hat{g})\right\| \\
     \leq & \left\|A^{-1}+B^{-1}\right\|\cdot \left\|\nabla_{ g}L^-(\mathcal{D} ; \hat{g})\right\|\\
     \leq & \left|\frac{2\delta+\sigma_{\text{min}}+\sigma'_{\text{min}}}{\left(\delta+ \sigma'_{\text{min}}\right)\cdot\left(\delta+ \sigma_{\text{min}}\right)}\right|\cdot\left\|\nabla_{ g}L^-(\mathcal{D} ; \hat{g})\right\|\\
     \leq& \left|\frac{2\delta+\sigma_{\text{min}}+\sigma'_{\text{min}}}{\left(\delta+ \sigma'_{\text{min}}\right)\cdot\left(\delta+ \sigma_{\text{min}}\right)}\right| \cdot C_L'|M |
    \end{split}
\end{equation}
By combining the conclusions from Step I and Step II in Equations \ref{bound:1}, \ref{bound:2} and \ref{bound:3}, we obtain the error between the actual influence and our predicted influence as follows:
\begin{equation*}
    \begin{split}
         &\left\|\hat{g}^{*}_{-p_M}(\mathcal{D} ) - \bar{g}_{-p_M}(\mathcal{D} )\right\|\\
         \leq & \frac{C_H^-    {C'_{L}|M|}^2}{2 (\sigma'_{\text{min}} + \delta)^3} + \left|\frac{2\delta+\sigma_{\text{min}}+\sigma'_{\text{min}}}{\left(\delta+ \sigma'_{\text{min}}\right)\cdot\left(\delta+ \sigma_{\text{min}}\right)}\right| \cdot C_L'|M|.
    \end{split}
\end{equation*}
\end{proof}

\begin{remark}
    Theorem \ref{app:bound_c_the} reveals one key finding about influence function estimation: The estimation error scales inversely with the regularization parameter $\delta$ ($ \mathcal{O}(1/\delta)$), indicating that increased regularization improves approximation accuracy. Besides, the error bound is linearly increasing with the number of removed concepts $|M|$. This implies that the estimation error increases with the number of erroneous concepts removed.
\end{remark}

%% file: appendix/7_appendix_c.tex
\section{Concept-level Influence}
\subsection{Proof of Concept-level Influence Function}
\label{sec:appendix:c}
We address situations that delete $p_r$ for $r\in M$ concept removed dataset. Our goal is to estimate $\hat{g}_{-p_M}$, $\hat{f}_{-p_M}$, which is the concept and label predictor trained on the $p_r$ for $r\in M$ concept removed dataset. 

\paragraph{Proof Sketch.} The main ideas are as follows: (i) First, we define a new predictor $\hat{g}^*_{p_M}$, which has the same dimension as $\hat{g}$ and the same output as $\hat{g}_{-p_M}$. Then deduce an approximation for $\hat{g}^*_{p_M}$. (ii) Then, we consider setting $p_r=0$ instead of removing it, we get $\hat{f}_{p_M=0}$, which is equivalent to $\hat{f}_{-p_M}$ according to lemma \ref{apc:lm:1}. We estimate this new predictor as a substitute. (iii) Next, we assume we only use the updated concept predictor $\hat{g}_{p_M}^*$ for one data $(x_{i_r}, y_{i_r}, c_{i_r})$ and obtain a new label predictor $\hat{f}_{ir}$, and obtain a one-step Newtonian iterative approximation of $\hat{f}_{ir}$ with respect to $\hat{f}$. (iv) Finally, we repeat the above process for all data points and combine the estimate of $\hat{g}$ in Theorem $\ref{apc:thm:1}$, we obtain a closed-form solution of the influence function for $\hat{f}$.

First, we introduce our following lemma:
\begin{lemma}\label{apc:lm:1}
For the concept bottleneck model, if the label predictor utilizes linear transformations of the form $\hat{f} \cdot c$ with input $c$, then, for each $r\in M$, we remove the $r$-th concept from c and denote the new input as $c^{\prime}$. Set the $r$-th concept to 0 and denote the new input as $c^0$. Then we have $\hat{f}_{-p_M} \cdot c^{\prime} = \hat{f}_{p_M=0} \cdot c^0$ for any c. 
\end{lemma}

\begin{proof}
Assume the parameter space of $\hat{f}_{-p_M}$ and $\hat{f}_{p_M=0}$ are $\Gamma$ and $\Gamma_0$, respectively, then there exists a surjection  $P:\Gamma\rightarrow\Gamma_0$. For any $\theta\in\Gamma$, $P(\theta)$ is the operation that removes the $r$-th row of $\theta$ for $r\in M$. Then we have: 
\begin{equation*}
P(\theta) \cdot c^{\prime} = \sum_{t\notin M} \theta[j] 
\cdot c^{\prime}[j] = \sum_{t} \theta[t]  \mathbb{I} \{ t \notin M \} c[t] =\theta \cdot c^0. 
\end{equation*}
Thus, the loss function $L_Y (\theta, c^0) = L_Y(P(\theta), c^{\prime})$ of both models is the same for every sample in the second stage. Besides, by formula derivation, we have, for $\theta^{\prime}\in\Gamma_0$, for any $\theta$ in $P^{-1}(\theta^{\prime})$,
\begin{equation*}
    \frac{\partial L_Y (\theta, c^0)}{\partial \theta} = \frac{\partial L_Y(P(\theta), c^{\prime})}{\partial \theta^{\prime}}
\end{equation*}
Thus, if the same initialization is performed, $\hat{f}_{-p_M} \cdot c^{\prime} = \hat{f}_{p_M = 0} \cdot c^0$ for any $c$ in the dataset.
\end{proof}

\begin{theorem}\label{apc:thm:1}
For the retrained concept predictor $\hat{g}_{-p_M}$ defined as:
\begin{equation}\label{app:concept-level:g}
\hat{g}_{-p_M} = \argmin_{g'} \sum_{j\notin M}\sum_{i=1}^n L^j_{C}({g'}(x_i),c_i), 
\end{equation}
we map it to $\hat{g}^{*}_{-p_M}$ as 
\begin{equation}\label{app:concept-level:g^*}
\hat{g}^{*}_{-p_M} = \argmin_{g^{\prime} \in T_0} \sum_{j \notin M} \sum_{i=1}^{n} L^j_{C} ( g'(x_i), c_i).
\end{equation}
And we can edit the initial $\hat{g}$ to $\hat{g}^*_{-p_M}$, defined as:
\begin{equation*}
    \bar{g}^*_{-p_M}\triangleq \hat{g}  - H^{-1}_{\hat{g}} \cdot \sum_{j\notin M}\sum_{i=1}^n D_C^j(x_i,c_i;\hat{g}),
\end{equation*}
where $H_{\hat{g}} = \nabla_{g} \sum_{j\notin M} \sum_{i=1}^{n} L^j_{C} ( \hat{g} (x_i), c_i)$.
Then, by removing all zero rows inserted during the mapping phase, we can naturally approximate $\hat{g}_{-p_M}\approx  \mathrm{P}^{-1}(\hat{g}^{*}_{-p_M})$.
\end{theorem}

\begin{theorem}\label{apc:thm:1}
For the retrained concept predictor $\hat{g}_{-p_M}$ defined by 
\begin{equation*}
\hat{g}_{-p_M} =  \argmin_{g'} \sum_{j\notin M}\sum_{i=1}^n L^j_{C}({g'}(x_i),c_i),    
\end{equation*}
we map it to $\hat{g}^{*}_{-p_M}$ as
\begin{equation*}
\hat{g}^{*}_{-p_M} = \argmin_{g^{\prime} \in T_0} \sum_{j \notin M} \sum_{i=1}^{n} L^j_{C} ( g'(x_i), c_i).
\end{equation*}
And we can edit the initial $\hat{g}$ to $\hat{g}^*_{-p_M}$, defined as:
\begin{equation}\label{app:concept-level:g}
    \bar{g}_{-p_M}\triangleq \hat{g}  - H^{-1}_{\hat{g}} \cdot \sum_{j\notin M}\sum_{i=1}^n D_C^j(x_i,c_i;\hat{g}),
\end{equation}
where $H_{\hat{g}} = \nabla_{g} \sum_{j\notin M} \sum_{i=1}^{n} D^j_{C} (x_i, c_i; \hat{g})$.
Then, by removing all zero rows inserted during the mapping phase, we can naturally approximate $\hat{g}_{-p_M}\approx  \mathrm{P}^{-1}(\hat{g}^{*}_{-p_M})$.
\end{theorem}

\begin{proof}
At this level, we consider the scenario that removes a set of mislabeled concepts or introduces new ones. Because after removing concepts from all the data, the new concept predictor has a different dimension from the original. We denote $g^j(x_i)$ as the $j$-th concept predictor with $x_i$, and $c_i^j$ as the $j$-th concept in data $z_i$. For simplicity, we treat $g$ as a collection of $k$ concept predictors and separate different columns as a vector $g^j(x_i)$. Actually, the neural network gets $g$ as a whole. 

For the comparative purpose, we introduce a new notation $\hat{g}_{-p_M}^{*}$. Specifically, we define weights of $\hat{g}$ and $\hat{g}_{-p_M}^{*}$ for the last layer of the neural network as follows.
\begin{equation*}
\begin{split}
\hat{g}_{-p_M}(x) 
&= \underbrace{
\begin{pmatrix}
  w_{11} & w_{12} & \cdots & w_{1d_i}  \\
  w_{21} & w_{22} & \cdots & w_{2d_i} \\
  \vdots & \vdots & \ddots & \vdots \\
  w_{(k-1)1} & w_{(k-1)2} & \cdots & w_{(k-1)d_i}
\end{pmatrix}
}_{(k-1) \times d_i} 
\cdot 
\underbrace{
\begin{pmatrix}
x^{1} \\ x^{2} \\ \vdots \\ x^{d_i}
\end{pmatrix}
}_{d_i \times 1} \\
&= \underbrace{
\begin{pmatrix}
c_1 \\ \vdots \\ c_{r-1} \\ c_{r+1} \\ \vdots \\ c_k
\end{pmatrix}
}_{(k-1) \times 1}
\end{split}
\end{equation*}

\begin{equation*}
\begin{split}
\hat{g}^{*}_{-p_M}(x) 
&= \underbrace{
\begin{pmatrix}
  w_{11} & w_{12} & \cdots & w_{1d_i}  \\
  \vdots & \vdots & \ddots & \vdots \\
  w_{(r-1)1} & w_{(r-1)2} & \cdots & w_{(r-1)d_i} \\
  0 & 0 & \cdots & 0 \\
  w_{(r+1)1} & w_{(r+1)2} & \cdots & w_{(r+1)d_i} \\
  \vdots & \vdots & \ddots & \vdots \\
  w_{k1} & w_{k2} & \cdots & w_{kd_i}
\end{pmatrix}
}_{k \times d_i} 
\cdot 
\underbrace{
\begin{pmatrix}
x^{1} \\ \vdots \\ x^{r-1} \\ x^{r} \\ x^{r+1} \\ \vdots \\ x^{d_i}
\end{pmatrix}
}_{d_i \times 1} \\
&= \underbrace{
\begin{pmatrix}
c_1 \\ \vdots \\ c_{r-1} \\ 0 \\ c_{r+1} \\ \vdots \\ c_k
\end{pmatrix}
}_{k \times 1},
\end{split}
\end{equation*}
where $r$ is an index from the index set $M$.

Firstly, we want to edit to  $\hat{g}^{*}_{-p_M}\in T_0 = \{w_{\text{final}} = 0\}  \subseteq T$ based on $\hat{g}$, where $w_{\text{final}}$ is the parameter of the final layer of neural network. Let us take a look at the definition of $\hat{g}^{*}_{-p_M}$:
\begin{equation*}
\hat{g}^{*}_{-p_M} = \argmin_{g' \in T_0} \sum_{j\notin M} \sum_{i=1}^{n} L^{j}_C ( g'(x_i), c_i).
\end{equation*}

Then, we separate the $r$-th concept-related item from the rest and rewrite $\hat{g}$ as the following form:
\begin{equation*}
\hat{g} = \argmin_{g \in T} \left[\sum_{j \notin M} \sum_{i=1}^{n} L^j_C ( g(x_i), c_i) + \sum_{r\in M}\sum_{i=1}^{n} L^r_C ( g(x_i), c_i)\right].
\end{equation*}

Then, if the $r$-th concept part is up-weighted by some small $\epsilon$, this gives us the new parameters 
$\hat{g}_{\epsilon, p_M}$, which we will abbreviate as $\hat{g}_{\epsilon}$ below. 
\begin{equation*}
\begin{split}
\hat{g}_{\epsilon, p_M} \triangleq \argmin_{g \in \mathcal{T}} \Biggl[ 
    & \sum_{j \notin M} \sum_{i=1}^{n} L^j_C \left( g(x_i), c_i \right) \\
    & + \epsilon \cdot \sum_{r \in M} \sum_{i=1}^{n} L^r_C \left( g(x_i), c_i \right) 
\Biggr].
\end{split}
\end{equation*}

Obviously, when $\epsilon \rightarrow 0$, $\hat{g}_{\epsilon} \to \hat{g}^{*}_{-p_M}$. We can obtain the minimization conditions from the definitions above.
\begin{equation}
\label{con:1}
\nabla_{\hat{g}^{*}_{-p_M}} \sum_{j\notin M} \sum_{i=1}^{n} L^j_{C} ( \hat{g}^{*}_{-p_M} (x_i), c_i) = 0.
\end{equation}
\begin{equation*}
\nabla_{\hat{g}_{\epsilon}} \sum_{j\notin M} \sum_{i=1}^{n} L^j_{C} (\hat{g}_{\epsilon}(x_i), c_i) + \epsilon \cdot \nabla_{\hat{g}_{\epsilon}} \sum_{r\in M}\sum_{i=1}^{n} L^r_{C} ( \hat{g}_{\epsilon}(x_i), c_i) =0.
\end{equation*}

Perform a first-order Taylor expansion of equation \ref{con:1} with respect to $\hat{g}_{\epsilon}$, then we get
\begin{equation*}
\begin{split}
\nabla_{g} \sum_{j \notin M} \sum_{i=1}^{n} L^j_{C} \left( \hat{g}_{\epsilon}(x_i), c_i \right) +\\
 \nabla^2_{g} \sum_{j \notin M} \sum_{i=1}^{n} L^j_{C} \left( \hat{g}_{\epsilon}(x_i), c_i \right) \cdot \left( \hat{g}^{*}_{-p_M} - \hat{g}_{\epsilon} \right) &\approx 0.
\end{split}
\end{equation*}
Then we have
\begin{equation*}
\hat{g}^{*}_{-p_M} - \hat{g}_{\epsilon} = - H^{-1}_{\hat{g}_{\epsilon}} \cdot \nabla_g \sum_{j\notin M} \sum_{i=1}^{n} L^j_{C} ( \hat{g}_{\epsilon}(x_i), c_i).
\end{equation*}
Where $H_{\hat{g}_{\epsilon}} = \nabla^2_{g} \sum_{j\notin M} \sum_{i=1}^{n} L^j_{C} ( \hat{g}_{\epsilon} (x_i), c_i)$. 

We can see that: 

When $\epsilon=0$,
\begin{equation*}
\hat{g}_{\epsilon}=\hat{g}^{*}_{-p_M},    
\end{equation*}

When $\epsilon=1$, $\hat{g}_{\epsilon}=\hat{g}$, 
\begin{equation*}
\hat{g}^{*}_{-p_M} - \hat{g} \approx - H^{-1}_{\hat{g}} \cdot \nabla_g \sum_{j\notin M} \sum_{i=1}^{n} L^j_{C} ( \hat{g} (x_i), c_i ),
\end{equation*}
where $H_{\hat{g}} = \nabla^2_{g} \sum_{j\notin M} \sum_{i=1}^{n} L^j_{C} ( \hat{g} (x_i), c_i)$.

Then, an approximation of $\hat{g}^{*}_{-p_M}$ is obtained.
\begin{equation}
    \hat{g}^{*}_{-p_M} \approx \hat{g}  - H^{-1}_{\hat{g}} \cdot \nabla_g \sum_{j\notin M} \sum_{i=1}^{n} L^j_{C} ( \hat{g} (x_i), c_i).
\end{equation}
Recalling the definition of the gradient:
\begin{equation*}
    G_C^j(x_i,c_i;\hat{g}) = L^j_{C} ( \hat{g} (x_i), c_i) )= \hat{g}^j (x_i)^\top\cdot\log(c_i^j).
\end{equation*}
Then the approximation of $\hat{g}^{*}_{-p_M}$ becomes
\begin{equation*}
    \bar{g}_{-p_M}\triangleq \hat{g}  - H^{-1}_{\hat{g}} \cdot \sum_{j\notin M}\sum_{i=1}^n G_C^j(x_i,c_i;\hat{g}),
\end{equation*}
\end{proof}

\begin{theorem}
For the retrained label predictor $\hat{f}_{-p_M}$ defined as
\begin{equation*}
\hat{f}_{-p_M}=\argmin_{f^{\prime}} \sum_{i=1}^{n} L_{Y}=\argmin_{f'} \sum_{i=1}^{n} L_Y(f^{\prime}(\hat{g}_{-p_M}(x_{i})), y_{i}),
\end{equation*}
We can consider its equivalent version $\hat{f}_{p_M=0}$ as:
\begin{equation*}
    \hat{f}_{p_M=0} = \argmin_{f} \sum_{i = 1}^n L_{Y} \left(f\left(\hat{g}^*_{-p_M}(x_i)\right),y_i\right),
\end{equation*}
which can be edited by
\begin{equation*}
    \hat{f}_{p_M=0} \approx \bar{f}_{p_M=0} \triangleq \hat{f}-H_{\hat{f}}^{-1} \cdot    \sum_{l=1}^{n}G_Y(x_l;\bar{g}^*_{-p_M},\hat{f}), 
\end{equation*}
where $H_{\hat{f}} = \nabla_{\hat{f}}\sum_{i=1}^n G_Y(x_l;\bar{g}^*_{-p_M},\hat{f})$ is the Hessian matrix. Deleting the $r$-th dimension of $\bar{f}_{p_M=0}$ for $r\in M$, then we can map it to $\bar{f}_{-p_M}$, which is the approximation of the final edited label predictor $\hat{f}_{-p_M}$ under concept level.
\end{theorem}

\begin{proof}
Now, we come to the approximation of $\hat{f}_{-p_M}$. Noticing that the input dimension of $f$ decreases to $k-|M|$. We consider setting $p_r=0$ for all data points in the training phase of the label predictor and get another optimal model $\hat{f}_{p_M=0}$. From lemma \ref{apc:lm:1}, we know that for the same input $x$, $\hat{f}_{p_M=0}(x) = \hat{f}_{-p_M}$. And the values of the corresponding parameters in $\hat{f}_{p_M=0}$ and $\hat{f}_{-p_M}$ are equal.

Now, let us consider how to edit the initial $\hat{f}$ to $\hat{f}_{p_M=0}$.
Firstly, assume we only use the updated concept predictor $\hat{g}^{*}_{-p_M}$ for one data $(x_{i_r}, y_{i_r}, c_{i_r})$ and obtain the following $\hat{f}_{ir}$, which is denoted as
\begin{equation*}
\begin{split}
\hat{f}_{ir} = \argmin_f \Biggl[ 
    & \sum_{i=1}^{n} L_Y \left( f \left( \hat{g} (x_i) \right), y_i \right) \\
    & + L_Y \left( f \left( \hat{g}^{*}_{-p_M} (x_{ir}) \right), y_{ir} \right) \\
    & - L_Y \left( f \left( \hat{g} (x_{ir}) \right), y_{ir} \right)
\Biggr].
\end{split}
\end{equation*}

Then up-weight the $i_r$-th data by some small $\epsilon$ and have the following new parameters:
\begin{equation*}
\begin{split}
\hat{f}_{\epsilon,ir} = \argmin_f \Biggl[ 
    & \sum_{i=1}^{n} L_Y \left( f \left( \hat{g} (x_i) \right), y_i \right) \\
    & + \epsilon \cdot L_Y \left( f \left( \hat{g}^{*}_{-p_M} (x_{ir}) \right), y_{ir} \right) \\
    & - \epsilon \cdot L_Y \left( f \left( \hat{g} (x_{ir}) \right), y_{ir} \right)
\Biggr].
\end{split}
\end{equation*}

Deduce the minimized condition subsequently,
\begin{equation*}
\begin{split}
\nabla_f \sum_{i=1}^{n} & L_Y \left( \hat{f}_{ir} \left( \hat{g} (x_i) \right), y_i \right) \\
& + \epsilon \cdot \nabla_f L_Y \left( \hat{f}_{ir} \left( \hat{g}^{*}_{-p_M} (x_{ir}) \right), y_{ir} \right) \\
& - \epsilon \cdot \nabla_f L_Y \left( \hat{f}_{ir} \left( \hat{g} (x_{ir}) \right), y_{ir} \right) \\
&= 0.
\end{split}
\end{equation*}

If we expand first term of $\hat{f}$, which $\hat{f}_{ir, \epsilon} \to \hat{f} (\epsilon \to 0) $, then
\begin{align*}
\nabla_f \sum_{i=1}^{n}  L_Y\left(\hat{f}(\hat{g}(x_i)), y_i\right) + \epsilon \cdot \nabla_f L_Y\left(\hat{f}(\hat{g}^{*}_{-p_M}(x_{ir})), y_{ir}\right) \\
- \epsilon \cdot \nabla_f L_Y\left(\hat{f}(\hat{g}(x_{ir})), y_{ir}\right)+ \left(\nabla^2_f \sum_{i=1}^{n} L_Y\left(\hat{f}(\hat{g}(x_i)), y_i\right)\right) \\
 \cdot (\hat{f}_{ir,\epsilon} - \hat{f}) = 0.
\end{align*}

Note that $\nabla_f \sum_{i=1}^{n} L_Y( \hat{f}( \hat{g} (x_i)), y_i) = 0$. Thus we have
\begin{equation*}
\begin{split}
\hat{f}_{ir,\epsilon} - \hat{f} 
&= H^{-1}_{\hat{f}} \cdot \epsilon \\
&\quad \cdot \left( \nabla_f L_Y \left( \hat{f} \left( \hat{g}^{*}_{-p_M} (x_{ir}) \right), y_{ir} \right) \right. \\
&\quad \left. - \nabla_f L_Y \left( \hat{f} \left( \hat{g} (x_{ir}) \right), y_{ir} \right) \right).
\end{split}
\end{equation*}

We conclude that 
\begin{equation*}
\begin{split}
\left.\frac{\mathrm{d}\hat{f}_{\epsilon,ir}}{\mathrm{d}\epsilon}\right|_{\epsilon = 0} 
= H^{-1}_{\hat{f}} \cdot \Bigl( & \nabla_{\hat{f}}  L_Y ( \hat{f}( \hat{g}^{*}_{-p_M} (x_{ir})), y_{ir}) \\
& - \nabla_{\hat{f}}  L_Y ( \hat{f}( \hat{g}(x_{ir})), y_{ir}) \Bigr).
\end{split}
\end{equation*}

Perform a one-step Newtonian iteration at $\hat{f}$ and we get the approximation of $\hat{f}_{i_r}$.
\begin{equation*}
    \hat{f}_{ir} \approx \hat{f} + H^{-1}_{\hat{f}} \cdot \begin{aligned}[t]
    \Bigl( & \nabla_{\hat{f}}  L_Y ( \hat{f}( \hat{g}(x_{ir})), y_{ir}) \\
    & - \nabla_{\hat{f}}  L_Y ( \hat{f}( \hat{g}^{*}_{-p_M} (x_{ir})), y_{ir}) \Bigr).
    \end{aligned}
\end{equation*}

Reconsider the definition of $\hat{f}_{i_r}$, we use the updated concept predictor $\hat{g}_{-p_M}^*$ for one data $(x_{i_r}, y_{i_r}, c_{i_r})$. Now we carry out this operation for all the other data and estimate $\hat{f}_{p_M=0}$. Combining the minimization condition from the definition of $\hat{f}$, we have
\begin{align}
    \hat{f}_{p_M=0} 
    &\approx \hat{f} + H^{-1}_{\hat{f}} \cdot \Bigg( \nabla_{\hat{f}} \sum_{i=1}^{n} L_Y \big( \hat{f}(\hat{g}(x_{i})), y_{i}\big) \notag \\
    &\quad - \nabla_{\hat{f}} \sum_{i=1}^{n} L_Y \big( \hat{f}(\hat{g}^{*}_{-p_M}(x_{i})), y_{i}\big) \Bigg) \\
    &= \hat{f} + H^{-1}_{\hat{f}} \cdot \Bigg( - \nabla_{\hat{f}} \sum_{i=1}^{n} L_Y \big( \hat{f}(\hat{g}^{*}_{-p_M}(x_{i})), y_{i}\big) \Bigg) \notag \\
    &= \hat{f} - H^{-1}_{\hat{f}} \sum_{l=1}^n \nabla_{\hat{f}} L_Y \big( \hat{f}(\hat{g}^{*}_{-p_M}(x_l)), y_l \big). \label{eq:10}
\end{align}

Theorem \ref{apc:thm:1} gives us the edited version of $\hat{g}^{*}_{-p_M}$. Substitute it into equation \ref{eq:10}, and we get the final closed-form edited label predictor under concept level:
\begin{equation*}
    \hat{f}_{p_M=0} \approx \bar{f}_{p_M=0} \triangleq\hat{f}-H_{\hat{f}}^{-1} \cdot   \nabla_{\hat{f}} \sum_{l=1}^{n} L_{Y_l} \left( \hat{f}, \bar{g}^*_{-p_M} \right), 
\end{equation*}
where $H_{\hat{f}} = \nabla^2_{\hat{f}}\sum_{i=1}^n L_{Y_i}(\hat{f}, \hat{g})$ is the Hessian matrix of the loss function respect to is the Hessian matrix of the loss function respect to $\hat{f}$.
Recalling the definition of the gradient:
\begin{equation*}
    G_Y(x_l;\bar{g}^*_{-p_M},\hat{f}) = \nabla_{\hat{f}} L_{Y} \left(\hat{f}\left(\bar{g}^*_{-p_M}(x_l)\right),y_l\right),
\end{equation*}
then the approximation becomes
\begin{equation*}
    \hat{f}_{p_M=0} \approx \bar{f}_{p_M=0} \triangleq \hat{f}-H_{\hat{f}}^{-1} \cdot    \sum_{l=1}^{n}G_Y(x_l;\bar{g}^*_{-p_M},\hat{f}). 
\end{equation*}
\end{proof}

%% file: appendix/7_appendix_c_update.tex
\clearpage

Notation update:

1. concept-label correction:

The wrong data-concept-label index set $S_e$, the corresponding concept predictor $\hat{g}^{\mathrm{corr} S_e}$.

2. concept removal:

The wrong concept index set $M$, the corresponding concept predictor $\hat{g}^{\\ \mathcal{C}_M}$.

3. data removal/addition:

$\mathcal{D}_{\text{rem}}$ and $\mathcal{D}_{\text{add}}$, the corresponding concept predictor $\hat{g}^{\setminus\mathcal{D}_{\text{rem}}}$ and $\hat{g}^{\cup \mathcal{D}_{\text{add}}}$.

\subsection{Data Addition} 

\begin{theorem}\label{app:data-level-add:g}
For dataset $\mathcal{D} = \{ (x_i, y_i, c_i) \}_{i=1}^{N}$, given a set of new data $\mathcal{D}_{\mathrm{add}} = \{\tilde{z}_s = (\tilde{x}_s, \tilde{y}_s, \tilde{c}_s)\}_{s=1}^M$ to be added. Suppose the updated concept predictor $\hat{g}^{\cup \mathcal{D}_{\text{add}}}$ is defined by
\begin{equation}
\begin{split}
    \hat{g}^{\cup \mathcal{D}_{\text{add}}} =&  \argmin_{g}\sum_{j\in[K]}\sum_{i\in [N]}  L_{C_j}({g}({x}_{i}), {c}_{i}) \\
    & + \sum_{j\in[K]}\sum_{s\in [M]}  L_{C_j}({g}(\tilde{x}_{S}), \tilde{c}_{S}).
    \end{split}
\end{equation}
Define $L_{C}(\hat{g}(x_{i}), c_{i}) \triangleq \sum_{j=1}^K  L_{C_j}(\hat{g}(x_{i}), c_{i})$.
Then we have the following approximation for $\hat{g}^{\cup \mathcal{D}_{\text{add}}}$
\begin{equation}
\hat{g}^{\cup \mathcal{D}_{\text{add}}} \approx \bar{g}^{\cup \mathcal{D}_{\text{add}}} \triangleq\hat{g} - H^{-1}_{\hat{g}} \cdot \sum_{s\in [M]}\nabla_{g} L_{C} (\hat{g}(\tilde{x}_s), \tilde{c}_s), 
\end{equation}
where $H_{\hat{g}} = \nabla^2_{\hat{g}}  \sum_{i=1}^N  L_{C}(\hat{g}(x_i),c_i)$ is the Hessian matrix of the loss function respect to $\hat{g}$.
\end{theorem}

\begin{proof}
Firstly, recalling the definition of $\hat{g}^{\cup \mathcal{D}_{\text{add}}}$ as 
\begin{equation*}
    \hat{g}^{\cup \mathcal{D}_{\text{add}}} = \argmin_{g} \left[ \sum_{\substack{i=1}}^{N} L_C({g}(x_{i}), c_{i}) + \sum_{s=1}^M L_C(g(\tilde{x}_s), \tilde{c}_s)\right],
\end{equation*}

Then we up-weighted the $s$-th data by some $\epsilon$ and have a new predictor $\hat{g}_{\epsilon}^{\cup \mathcal{D}_{\text{add}}}$:
\begin{equation}\label{app-thm2:g1}
\argmin_{g} \left[ \sum_{i=1}^{N} L_C ( g(x_i), c_i) + \epsilon \cdot \sum_{s=1}^M L_C ( g(\tilde{x}_s), \tilde{c}_s)\right].
\end{equation}

Because $\hat{g}_{\epsilon}^{\cup \mathcal{D}_{\text{add}}}$ minimizes the right side of equation~\eqref{app-thm2:g1}, we have 
\begin{equation*}
\nabla_{g}  \sum_{i=1}^{N} L_Y (\hat{g}_{\epsilon}^{\cup \mathcal{D}_{\text{add}}}(x_i), c_i) + \epsilon \cdot\nabla_{g} \sum_{s=1}^ML_Y ( \hat{g}_{\epsilon}^{\cup \mathcal{D}_{\text{add}}}(\tilde{x}_s), \tilde{c}_s) =0. 
\end{equation*}

When $\epsilon\rightarrow 0$, $\hat{g}_{\epsilon}^{\cup \mathcal{D}_{\text{add}}}\rightarrow\hat{g}$.
So we can perform a first-order Taylor expansion with respect to ${g}$, and we have
\begin{equation}\label{app-thm2:g:1}
\begin{split}
0 \approx &\nabla_{g}  \sum_{i=1}^{N} L_C ( \hat{g} (x_i), c_i) 
+
\epsilon \cdot \nabla_{g}\sum_{s=1}^M L_C (\hat{g}(\tilde{x}_s), \tilde{c}_s) \\
&+
\nabla^2_{g}  \sum_{i=1}^{N} L_C ( \hat{g} (x_i), c_i) \cdot (\hat{g}_{\epsilon}^{\cup \mathcal{D}_{\text{add}}} - \hat{g}) .
\end{split}
\end{equation}

Recap the definition of $\hat{g}$:
\begin{equation*}
    \hat{g} = \argmin_g  \sum_{i=1}^n L_Y(g(x_i),c_i),
\end{equation*}

Then, the first term of the right side of~\eqref{app-thm2:g:1} equals $0$. Let $\epsilon\rightarrow0$, then we have
\begin{equation*}
\left.\frac{\mathrm{d}\hat{g}_{\epsilon}^{\cup \mathcal{D}_{\text{add}}}}{\mathrm{d}\epsilon}\right|_{\epsilon = 0} = -H^{-1}_{\hat{g}} \cdot \sum_{s=1}^M\nabla_{g} L_C (\hat{g}(\tilde{x}_s), \tilde{c}_s),
\end{equation*}
where $H_{\hat{g}} = \nabla^2_{\hat{g}}  \sum_{i=1}^N  L_{C}(\hat{g}(x_i),c_i)$.

Remember when $\epsilon\rightarrow0$, $\hat{g}_{\epsilon}^{\cup \mathcal{D}_{\text{add}}}\rightarrow\hat{g}^{\cup \mathcal{D}_{\text{add}}}$.
Perform a Newton step at $\hat{g}$, then we obtain the method to edit the original concept predictor under concept level: 
\begin{equation*}
\hat{g}^{\cup \mathcal{D}_{\text{add}}} \approx \bar{g}^{\cup \mathcal{D}_{\text{add}}}\triangleq\hat{g} - H^{-1}_{\hat{g}} \cdot \sum_{s=1}^M\nabla_{g} L_C (\hat{g}(\tilde{x}_s), \tilde{c}_s).
\end{equation*}
\end{proof}

\begin{theorem}
For dataset $\mathcal{D} = { (x_i, y_i, c_i) }{i=1}^{N}$, given a set of new data $\mathcal{D}{\mathrm{add}} = {\tilde{z}s = (\tilde{x}s, \tilde{y}s, \tilde{c}s)}{s=1}^M$ to be added. Suppose the updated label predictor $\hat{f}^{\cup \mathcal{D}_{\text{add}}}$ is defined by 
\begin{equation}\label{app-data:f-def} 
\begin{split} 
\hat{f}^{\cup \mathcal{D}_{\text{add}}} =& \argmin_{f} \sum_{i=1}^N L_{Y}(f(\hat{g}^{\cup \mathcal{D}_{\text{add}}}(x_{i})), y_{i}) \\
& + \sum_{s=1}^{M} L_{Y}(f(\hat{g}^{\cup \mathcal{D}_{\text{add}}}(\tilde{x}_{s})), \tilde{y}_{s}).
\end{split} 
\end{equation} 
Then we have the following approximation for $\hat{f}^{\cup \mathcal{D}_{\text{add}}}$: \begin{equation} \hat{f}^{\cup \mathcal{D}_{\text{add}}} \approx \bar{f}^{\cup \mathcal{D}_{\text{add}}} \triangleq \hat{f} - H^{-1}{\hat{f}} \cdot \sum_{s\in [M]} \nabla_{f} L_{Y} (\hat{f}(\tilde{x}s), \tilde{y}s), \end{equation} where $H{\hat{f}} = \nabla^2{\hat{f}} \sum_{i=1}^N L_{Y}(\hat{f}(x_i),y_i)$ is the Hessian matrix of the loss function with respect to $\hat{f}$. \end{theorem}


\begin{proof}
We can see that there is a huge gap between $\hat{f}^{\cup \mathcal{D}_{\text{add}}}$ and $\hat{f}$. Thus, firstly, we define $\tilde{f}^{\cup \mathcal{D}_{\text{add}}}$ as 
\begin{equation*}
\argmin_f \sum_{i=1}^{N} L_Y \left(  f (\hat{g}(x_i)), y_i  \right) + \sum_{s=1}^M L_Y \left( f ( \hat{g} (\tilde{x}_s) ) , \tilde{y}_s   \right).
\end{equation*}

Then, we define $\tilde{f}_{\epsilon}^{\cup \mathcal{D}_{\text{add}}}$ as follows to estimate $\tilde{f}^{\cup \mathcal{D}_{\text{add}}}$.
\begin{equation*}
    \tilde{f}_{\epsilon}^{\cup \mathcal{D}_{\text{add}}} =  \argmin_f \sum_{i=1}^{N} L_Y \left(  f (\hat{g}(x_i)), y_i  \right) + \epsilon \cdot \sum_{s=1}^M L_Y \left( f ( \hat{g} (\tilde{x}_s) ) , \tilde{y}_s   \right).
\end{equation*}

From the minimization condition, we have
\begin{equation*}
\begin{split}
    0 = &\nabla_{{f}} \sum_{i=1}^{N} L_Y \left(  \tilde{f}_{\epsilon}^{\cup \mathcal{D}_{\text{add}}} (\hat{g}(x_i)), y_i  \right) \\
    & + \epsilon \cdot \sum_{s=1}^M\nabla_{{f}}L_Y \left(\tilde{f}_{\epsilon}^{\cup \mathcal{D}_{\text{add}}} ( \hat{g} (\tilde{x}_s) ) , \tilde{y}_s   \right).
    \end{split}
\end{equation*}
Perform a first-order Taylor expansion at $\hat{f}$,
\begin{align*}
& \nabla_{\hat{f}} \sum_{i=1}^{N} L_Y \left(  \hat{f} (\hat{g}(x_i)), y_i  \right) + \epsilon \cdot \nabla_{\hat{f}} \sum_{s=1}^M L_Y \left(\hat{f} ( \hat{g} (\tilde{x}_s) ) , \tilde{y}_s   \right)  \\
& + \nabla^2_{\hat{f}} \sum_{i=1}^{N} L_Y \left(  \hat{f} (\hat{g}(x_i)), y_i  \right) \left(\tilde{f}_{\epsilon}^{\cup \mathcal{D}_{\text{add}}} - \hat{f}\right)= 0.
\end{align*}

Then $\tilde{f}^{\cup \mathcal{D}_{\text{add}}}$ can be approximated by
\begin{equation}\label{ap:approx:f_piao}
    \tilde{f}^{\cup \mathcal{D}_{\text{add}}} \approx \hat{f} - H^{-1}_{\hat{f}} \cdot \sum_{s=1}^M\nabla_{\hat{f}} L_Y \left(\hat{f} ( \hat{g} (\tilde{x}_s) ) , \tilde{y}_s   \right).
\end{equation}
Define $A\triangleq H^{-1}_{\hat{f}} \cdot \sum_{s=1}^M\nabla_{\hat{f}} L_Y \left(\hat{f} ( \hat{g} (\tilde{x}_s) ) , \tilde{y}_s   \right)$
Then the approximation of $\tilde{f}^{\cup \mathcal{D}_{\text{add}}}$ is defined as
\begin{equation}\label{app:th:3:f1}
\bar{f}^{\cup \mathcal{D}_{\text{add}}} = \hat{f} - A
\end{equation}

Then we estimate the difference between $\hat{f}^{\cup \mathcal{D}_{\text{add}}}$ and $\tilde{f}^{\cup \mathcal{D}_{\text{add}}}$. Recall the definition of $\tilde{f}^{\cup \mathcal{D}_{\text{add}}}$ as 
\begin{equation}\label{app-data-add:thm:2:f:2}
\argmin_f \sum_{i=1}^{N} L_Y \left(  f (\hat{g}(x_i)), y_i  \right) + \sum_{s=1}^{M} L_Y \left(  f (\hat{g}(\tilde{x}_s)), \tilde{y}_s  \right).
\end{equation}

Compare equation~\eqref{app-data:f-def} with~\eqref{app-data-add:thm:2:f:2}, we still need to define an intermediary predictor $\hat{f}^{\cup \mathcal{D}_{\text{add}}}_{s_t}$ as 
\begin{align*}
    \hat{f}^{\cup \mathcal{D}_{\text{add}}}_{s_t} 
    &= \argmin_f \Biggl[\,\sum_{i=1}^N L_{Y} \bigl( f(\hat{g}(x_i)), y_i \bigr) \\
    & + \sum_{s=1}^M L_{Y} \bigl( f(\hat{g}(\tilde{x}_s)), \tilde{y}_s \bigr) + L_{Y} \bigl( f(\hat{g}_{\epsilon}^{\cup \mathcal{D}_{\text{add}}}(\tilde{x}_{s_t})), \tilde{y}_{s_t} \bigr)\\
    & - L_{Y} \bigl( f(\hat{g}(\tilde{x}_{s_t})), \tilde{y}_{s_t} \bigr) \Biggr].
\end{align*}

Up-weight the $s_t$ data by some $\epsilon$, we define $ \hat{f}^{\cup \mathcal{D}_{\text{add}}}_{s_t, \epsilon}$ as
\begin{equation*}
\begin{aligned}
\hat{f}^{\cup \mathcal{D}_{\text{add}}}_{s_t, \epsilon} 
= \argmin_{f} \Biggl[ 
    &\sum_{i=1}^N L_{Y} \bigl( f(\hat{g}(x_i)), y_i \bigr) \\
    + \sum_{s=1}^M L_{Y} \bigl( f(\hat{g}(\tilde{x}_s)), \tilde{y}_s \bigr) &+ \epsilon \cdot L_{Y} \bigl( f(\hat{g}^{\cup \mathcal{D}_{\text{add}}}(\tilde{x}_{s_t})), \tilde{y}_{s_t} \bigr) \\
    &- \epsilon \cdot L_{Y} \bigl( f(\hat{g}(\tilde{x}_{s_t})), \tilde{y}_{s_t} \bigr) 
\Biggr].
\end{aligned}
\end{equation*}

Then, from the minimization condition, we have 
\begin{equation}\label{add-f:1} 
\begin{split}
0 = \nabla_{{f}}\sum_{i=1}^N L_{Y_i} \left(  \hat{f}^{\cup \mathcal{D}_{\text{add}}}_{s_t, \epsilon},\hat{g}\right) +\sum_{s=1}^M L_{Y_i}\left(\hat{f}^{\cup \mathcal{D}_{\text{add}}}_{s_t, \epsilon},\hat{g}\right) \\
+ \epsilon \cdot\nabla_{{f}} L_{Y_{st}} \left(  \hat{f}^{\cup \mathcal{D}_{\text{add}}}_{s_t, \epsilon},  \hat{g}^{\cup \mathcal{D}_{\text{add}}}\right) - \epsilon \cdot\nabla_{\hat{f}} L_{Y_{ir}} \left(   \hat{f}^{\cup \mathcal{D}_{\text{add}}}_{s_t, \epsilon}, \hat{g} \right).
\end{split}
\end{equation}

When $\epsilon\rightarrow 0$, $\hat{f}^{\cup \mathcal{D}_{\text{add}}}_{s_t, \epsilon}\rightarrow \Tilde{f}^{\cup \mathcal{D}_{\text{add}}}$.
Then we perform a Taylor expansion at $\Tilde{f}^{\cup \mathcal{D}_{\text{add}}}$ of equation \ref{add-f:1} and have
\begin{equation*}
\begin{split}
0 \approx \nabla_{{f}}\sum_{i=1}^N L_{Y_i} \left(  \Tilde{f}^{\cup \mathcal{D}_{\text{add}}},\hat{g}\right) +\sum_{s=1}^M L_{Y_i}\left(\Tilde{f}^{\cup \mathcal{D}_{\text{add}}},\hat{g}\right) \\
+ \epsilon \cdot\nabla_{{f}} L_{Y_{st}} \left(  \Tilde{f}^{\cup \mathcal{D}_{\text{add}}},  \hat{g}^{\cup \mathcal{D}_{\text{add}}}\right) \cdot \left(\hat{f}^{\cup \mathcal{D}_{\text{add}}}_{s_t, \epsilon} - \Tilde{f}^{\cup \mathcal{D}_{\text{add}}}\right)\\
- \epsilon \cdot\nabla_{\hat{f}} L_{Y_{ir}} \left(  \Tilde{f}^{\cup \mathcal{D}_{\text{add}}}, \hat{g} \right) \cdot \left(\hat{f}^{\cup \mathcal{D}_{\text{add}}}_{s_t, \epsilon} - \Tilde{f}^{\cup \mathcal{D}_{\text{add}}}\right).
\end{split}
\end{equation*}

Organizing the above equation gives
\begin{align*}
\hat{f}^{\cup \mathcal{D}_{\text{add}}}_{s_t, \epsilon} 
&- \tilde{f}^{\cup \mathcal{D}_{\text{add}}} \\
\approx & -\epsilon \cdot H^{-1}_{\tilde{f}} \nabla_{f} \Biggl[ 
   L_{Y_{s_t}} \left( \tilde{f}^{\cup \mathcal{D}_{\text{add}}}, \hat{g}^{\cup \mathcal{D}_{\text{add}}} \right) \\
   & - L_{Y_{s_t}} \left( \tilde{f}^{\cup \mathcal{D}_{\text{add}}}, \hat{g} \right) 
\Biggr],
\end{align*}
where $$H_{\Tilde{f}} = \nabla^2_{\hat{f}}\sum_{i=1}^N L_{Y_i} \left(  \Tilde{f}^{\cup \mathcal{D}_{\text{add}}},\hat{g}\right) +\sum_{s=1}^M L_{Y_i}\left(\Tilde{f}^{\cup \mathcal{D}_{\text{add}}},\hat{g}\right). $$

When $\epsilon = 1$, $\hat{f}^{\cup \mathcal{D}_{\text{add}}}_{s_t, -1} = \hat{f}^{\cup \mathcal{D}_{\text{add}}}_{s_t}$. Then we perform a Newton iteration with step size $1$ at $\Tilde{f}^{\cup \mathcal{D}_{\text{add}}}$,
\begin{equation*}
\begin{split}
     \hat{f}^{\cup \mathcal{D}_{\text{add}}}_{s_t} - \Tilde{f}^{\cup \mathcal{D}_{\text{add}}} &\approx -H^{-1}_{\Tilde{f}} \nabla_{{f}}  \left(L_{Y_{st}} \left(  \Tilde{f}^{\cup \mathcal{D}_{\text{add}}},  \hat{g}^{\cup \mathcal{D}_{\text{add}}}\right) \right. \\
     &\left. - L_{Y_{st}}\left(  \Tilde{f}^{\cup \mathcal{D}_{\text{add}}}, \hat{g} \right) \right),
\end{split}
\end{equation*}

Iterate $s_t$ through set $[M]$, and we have
\begin{equation}\label{ap:add-approx:hat_f}
\begin{split}
     \hat{f}^{\cup \mathcal{D}_{\text{add}}} - \Tilde{f}^{\cup \mathcal{D}_{\text{add}}} 
     &\approx -H^{-1}_{\Tilde{f}} \cdot \sum_{s=1}^M\nabla_{{f}} \left(L_{Y_{s}} \left( \Tilde{f}^{\cup \mathcal{D}_{\text{add}}}, \hat{g}^{\cup \mathcal{D}_{\text{add}}}\right) \right. \\
     &\left. - L_{Y_{s}}\left( \Tilde{f}^{\cup \mathcal{D}_{\text{add}}}, \hat{g} \right) \right)
\end{split}
\end{equation}

The edited version of $\hat{g}^{\cup \mathcal{D}_{\text{add}}}$ has been deduced as $\bar{g}^{\cup \mathcal{D}_{\text{add}}}$ in theorem \ref{app:data-level-add:g}, substituting this approximation into \eqref{ap:add-approx:hat_f}, then we have
\begin{equation}
\begin{split}
     \hat{f}^{\cup \mathcal{D}_{\text{add}}} - \Tilde{f}^{\cup \mathcal{D}_{\text{add}}}  
     &\approx -H^{-1}_{\Tilde{f}} \cdot \nabla_{\hat{f}} \sum_{s=1}^M \left[ L_{Y_s} \left( \Tilde{f}^{\cup \mathcal{D}_{\text{add}}}, \bar{g}^{\cup \mathcal{D}_{\text{add}}}\right) \right. \\
     &\left. - L_{Y_s} \left( \Tilde{f}^{\cup \mathcal{D}_{\text{add}}}, \hat{g}\right) \right].
\end{split}
\end{equation}
Noting that we cannot obtain $\hat{f}^{\cup \mathcal{D}_{\text{add}}}$ and $H_{\Tilde{f}}$ directly because we do not retrain the label predictor but edit it to $\bar{f}^{\cup \mathcal{D}_{\text{add}}}$ as a substitute. Therefore, we approximate $\hat{f}^{\cup \mathcal{D}_{\text{add}}}$ with $\bar{f}^{\cup \mathcal{D}_{\text{add}}}$ and $H_{\Tilde{f}}$ with  $H_{\bar{f}}$ which is defined by:
\begin{equation*}
    H_{\bar{f}} = \nabla^2_{\bar{f}}\sum_{i=1}^N  L_{Y_{i}} \left( \bar{f}^{\cup \mathcal{D}_{\text{add}}}, \hat{g}\right) + \nabla^2_{\bar{f}}\sum_{s=1}^M  L_{Y_{s}} \left( \bar{f}^{\cup \mathcal{D}_{\text{add}}}, \hat{g}\right)  
\end{equation*}
Then we define $B$ as
\begin{equation}\label{app:add-th:3:f2}
    B \triangleq H^{-1}_{\bar{f}} \cdot \nabla_{\hat{f}}\left(  \sum_{s=1}^M L_{Y_s} \left(  \Tilde{f}^{\cup \mathcal{D}_{\text{add}}}, \bar{g}^{\cup \mathcal{D}_{\text{add}}}\right) - L_{Y_s} \left(  \Tilde{f}^{\cup \mathcal{D}_{\text{add}}}, \hat{g}\right) \right)
\end{equation}
Combining \eqref{app:th:3:f1} and \eqref{app:th:3:f2}, then we deduce the final closed-form edited label predictor as
\begin{equation*}
    \bar{f}^{\cup \mathcal{D}_{\text{add}}}_* =\bar{f}^{\cup \mathcal{D}_{\text{add}}} - B = \hat{f} - A - B. 
\end{equation*}
\end{proof}

%% file: appendix/7_appendix_d_bound.tex
\subsection{Theoretical Bound for the Influence Function}\label{app:bound_d_the}
Consider the dataset $\mathcal{D}=\{(x_i,c_i,y_i\}_{i=1}^n$, the loss function of the concept predictor $g$ is defined as:
\begin{equation*}
\begin{split}
 & L_\text{Total}(\mathcal{D};g)=\sum_{i=1}^n L_{C}(g(x_i),{c}_i) + \frac{\delta}{2}\cdot \|g\|^2\\
 &= \sum_{i=1}^n \sum_{j=1}^k L_{C}^j(g(x_i),{c_i})+ \frac{\delta}{2}\cdot \|g\|^2 =
\sum_{i=1}^n\sum_{j=1}^k g^j(x_i)^\top \log({c_i}^j)\\
&+ \frac{\delta}{2}\cdot \|g\|^2.
\end{split}
\end{equation*}

Mathematically, we have a set of erroneous data $z_r = (x_r, y_r, c_r)$, $r\in G$ need to be removed. Then the retrained concept predictor becomes

\begin{equation*}
    \hat{g}_{-z_G} = \argmin_{g}\sum_{j=1}^k\sum_{i\in [n]-G}L^j_{C}(g(x_i), c_i)+ \frac{\delta}{2}\cdot \|g\|^2.
\end{equation*}
Define the new dataset as $\mathcal{D}^* = \{(x_i,c_i,y_i)\}_{i\in [n]-G}$, then the loss function with the influence of erroneous data removed becomes
\begin{equation}
   \begin{split}
      L^-(\mathcal{D}^*; g) = \sum_{j=1}^k\sum_{i\in [n]-G}L^j_{C}(g(x_i), c_i)+ \frac{\delta}{2}\cdot \|g\|^2 \\
      = L_\text{Total}(\mathcal{D};g) -\sum_{j=1}^k\sum_{i\in G} L^j_{C}\left(g(x_i),c_i\right).
   \end{split}
\end{equation}

Assume $\hat{g} = \argmin  L_\text{Total}(\mathcal{D};g)$ is the original model parameter. $\hat{g}_{-z_G}$ is the minimizer of $L^-(\mathcal{D}^*; g)$. Denote $\bar{g} _{-z_G}$ as the updated model with the influence of erroneous data removed and is obtained by the influence function method in theorem \ref{app:data-level:g}, which is an estimation for $\hat{g}_{-z_G}$.
\begin{equation}
    \hat{g}_{-z_G} \approx \bar{g}_{-z_G} \triangleq\hat{g} + H^{-1}_{\hat{g}} \cdot \sum_{r\in G}  \sum_{j=1}^M G^j_C(x_r,c_r;\hat{g}), 
    \end{equation}

In this part, we will study the error between the estimated influence given by the theorem \ref{app:data-level:g} method and $\hat{g}_{-z_G}$. We use the parameter changes as the evaluation metric:
\begin{equation}
   \left|\left(\bar{g}_{-z_G} -\hat{g}\right) - \left(  \hat{g}_{-z_G} - \hat{g}\right)\right| =  \left|\bar{g}_{-z_G} -   \hat{g}_{-z_G}\right|
\end{equation}

\begin{assumption}The loss $L_{C}(x,c;g;j)$
\begin{equation*}
   L_{C}(x, c;g) =  \sum_{j=1}^k L^j_C(g(x), c).
\end{equation*}
is convex and twice-differentiable in $g$, with positive regularization $\delta > 0$. There exists $C_H \in \mathbb{R}$ such that
$$\| \nabla^2_{g} L_{C}(x, c;g_1) - \nabla^2_{g} L_{C}(x, c;g_2)\|_{2} \leq C_H \| g_1 - g_2 \|_2$$
for all  $(x,c)\in \mathcal{D}$ and $g_1, g_2 \in \Gamma$. 
\end{assumption}


\begin{definition}
\begin{equation*}
       C'_{L} = \left\| \nabla_{g} L_C(\mathcal{D}; \hat{g})\right\|_2,
\end{equation*}
\begin{equation*}
   \sigma'_{\text{min}} = \text{smallest singular value of } \nabla^2_{g}  L^-(\mathcal{D}; \hat{g}),
\end{equation*}
\begin{equation*}
   \sigma_{\text{min}} = \text{smallest singular value of } \nabla^2_{g}  L_{\text{Total}}(\mathcal{D}; \hat{g}),
\end{equation*}
\begin{equation}
   L'(\mathcal{D},G;g) = \sum_{i\in G }   L_{C}(x_i, c_i;g)
\end{equation}
\end{definition}

\begin{corollary}
\begin{equation}
   L^-(\mathcal{D}; g) = L_{\text{Total}}(\mathcal{D};g) - L'(\mathcal{D},G;g)
\end{equation}
   \begin{equation*}
       \|\nabla^2_{g}  L^-(\mathcal{D}; g_1) - \nabla^2_{g}  L^-(\mathcal{D}; g_2)\|_2\leq  
    \left((n+|G|)\cdot C_H\right)\|g_1-g_2\| 
   \end{equation*}
   Define $C_H^- \triangleq(n+|G|)\cdot C_H$
\end{corollary}

Based on above corollaries and assumptions, we derive the following theorem.

\begin{theorem}
   We obtain the error between the actual influence and our predicted influence as follows:
   \begin{equation*}
    \begin{split}
         &\left\| \hat{g}_{-z_G}(\mathcal{D} ) - \bar{g}_{-z_G}(\mathcal{D} )\right\|\\
         \leq & \frac{C_H^-    {C'_{L}|G|}^2}{2 (\sigma'_{\text{min}} + \delta)^3} + \left|\frac{2\delta+\sigma_{\text{min}}+\sigma'_{\text{min}}}{\left(\delta+ \sigma'_{\text{min}}\right)\cdot\left(\delta+ \sigma_{\text{min}}\right)}\right| \cdot C_L'|G|.
    \end{split}
 \end{equation*}
\end{theorem}
\begin{proof}

We will use the one-step Newton approximation as an intermediate step. Define $\Delta g_{Nt}(\mathcal{D} )$ as
\begin{equation*}
   \Delta g_{Nt}(\mathcal{D} )\triangleq H_{\delta}^{-1}\cdot \nabla_{ g}L'(\mathcal{D}, G; \hat{g}),
\end{equation*}
where $H_{\delta} = \delta \cdot I + \nabla_{ g}^2 L^-(\mathcal{D} ;\hat{g})$ is the regularized empirical Hessian at $\hat{ g}$ but reweighed after removing the influence of wrong data. Then the one-step Newton approximation for $ \hat{g}_{-z_G}(\mathcal{D} )$ is defined as $ g_{Nt}(\mathcal{D} ) \triangleq \Delta g_{Nt}(\mathcal{D} ) + \hat{ g}$.

In the following, we will separate the error between $\bar{g}_{-z_G}(\mathcal{D} )$ and $ \hat{g}_{-z_G}(\mathcal{D} )$ into the following two parts:
\begin{equation*}
\begin{split}
    \hat{g}_{-z_G}(\mathcal{D}) - \bar{g}_{-z_G}(\mathcal{D}) 
    &= \underbrace{ \hat{g}_{-z_G}(\mathcal{D}) - g_{Nt}(\mathcal{D})}_{\text{Err}_{\text{Nt, act}}(\mathcal{D})} \\
    &\quad + \underbrace{\left(g_{Nt}(\mathcal{D})-\hat{g}\right) - \left(\bar{g}_{-z_G}(\mathcal{D}) - \hat{g}\right)}_{\text{Err}_{\text{Nt, if}}(\mathcal{D})}
\end{split}
\end{equation*}

Firstly, in {\bf Step $1$}, we will derive the bound for Newton-actual error ${\text{Err}_{\text{Nt, act}}(\mathcal{D} )}$.
Since $L^-( g)$ is strongly convex with parameter $\sigma'_{\text{min}} + \delta$ and minimized by 
$ \hat{g}_{-z_G}(\mathcal{D} )$, we can bound the distance
$\left\| \hat{g}_{-z_G}(\mathcal{D} ) - { g}_{Nt}(\mathcal{D} )\right\|_2$ in terms of the norm of the gradient at ${ g}_{Nt}$:
\begin{equation}\label{bound:1}
   \left\| \hat{g}_{-z_G}(\mathcal{D} ) - { g}_{Nt}(\mathcal{D} )\right\|_2 \leq \frac{2}{\sigma'_{\text{min}} + \delta} \left\|\nabla_{ g} L^- \left({ g}_{Nt}(\mathcal{D} )\right)\right\|_2
\end{equation}
Therefore, the problem reduces to bounding $\left\|\nabla_{ g} L^- \left({ g}_{Nt}(\mathcal{D} )\right)\right\|_2$.
Noting that $\nabla_{ g}L'(\mathcal{D}, G;\hat{ g}) = -\nabla_{ g}L^-$. This is because $\hat{ g}$ minimizes $L^- + L'$, that is, $$\nabla_{ g}L^-(\hat{ g}) + \nabla_{ g}L'(\mathcal{D}, G;\hat{ g}) = 0.$$ Recall that $\Delta g_{Nt}= H_{\delta}^{-1}\cdot \nabla_{ g}L'(\mathcal{D}, G;\hat{ g}) = -H_{\delta}^{-1}\cdot \nabla_{ g}L^-(\mathcal{D} ; \hat{g})$.
Given the above conditions, we can have this bound for $\text{Err}_{\text{Nt, act}}(-\mathcal{D} )$.
\begin{equation}\label{bound:2}
   \begin{split}
           &\left\|\nabla_{ g} L^- \left({ g}_{Nt}(\mathcal{D} )\right)\right\|_2\\
   = & \left\|\nabla_{ g} L^- \left(\hat{ g} + \Delta{ g}_{Nt}(\mathcal{D} )\right)\right\|_2\\
   = & \left\|\nabla_{ g} L^- \left(\hat{ g} + \Delta  g_{N_t}(\mathcal{D} )\right) - \nabla_{ g} L^- \left(\hat{ g} \right) -  \nabla_{ g}^2 L^- \left(\hat{ g}\right) \cdot  \Delta  g_{N_t}(\mathcal{D} )\right\|_2\\
    = & \left\|\int_0^1 \left(\nabla_{ g}^2 L^- \left(\hat{ g} + t\cdot \Delta g_{Nt}(\mathcal{D} )\right) - \nabla_{ g}^2 L^- \left(\hat{ g}\right)\right) \Delta g_{Nt}(\mathcal{D} ) \, dt\right\|_2\\
   \leq & \frac{C_H^-}{2} \left\|\Delta  g_{Nt}(\mathcal{D}^{*})\right\|_2^2 =  \frac{C_H^-}{2} \left\|\left[\nabla_{ g}^2 L^-(\hat{ g})\right]^{-1} \nabla_{ g} L^-(\hat{ g})\right\|_2^2\\
   \leq & \frac{C_{H}^{-}}{2 (\sigma'_{\text{min}} + \delta)^2} \left\|\nabla_{ g} L^-(\hat{ g})\right\|_2^2 = \frac{C_H^-}{2 (\sigma'_{\text{min}} + \delta)^2} \left\|\nabla_{ g} L'(\mathcal{D}, G;\hat{ g})\right\|_2^2\\
   \leq &\frac{C_{H}^{-}    {C'_{L}|G|}^2}{2 (\sigma'_{\text{min}} + \delta)^2}.
   \end{split}
\end{equation}

Now we come to {\bf Step $2$} to bound ${\text{Err}_{\text{Nt, if}}(-\mathcal{D} )}$, and we will bound the difference in parameter change between Newton and our CCBM method.
\begin{align*}
       &\left\| \left(g_{Nt}(\mathcal{D}) - \hat{g}\right) - \left(\bar{g}_{-z_G}(\mathcal{D}) - \hat{g}\right) \right\| \\
       &= \left\| \left[ \left( \delta \cdot I + \nabla_{g}^2 L^- \left(\hat{g}\right) \right)^{-1} \right. \right. \\
       &\quad \left. \left. + \left( \delta \cdot I + \nabla_{g}^2 L_{\text{Total}} \left(\hat{g}\right) \right)^{-1} \right] \cdot \nabla_{g} L'(\mathcal{D}, G; \hat{g}) \right\|
\end{align*}
For simplification, we use matrix $A$, $B$ for the following substitutions:
\begin{align*}
   A = {\delta \cdot I+ \nabla_{ g}^2 L^- \left(\hat{ g}\right)}\\
   B = {\delta \cdot I+ \nabla_{ g}^2 L_{\text{Total}} \left(\hat{ g}\right)}
\end{align*}
And $A$ and $B$ are positive definite matrices with the following properties
\begin{align*}
\delta + \sigma'_{\text{min}} \prec A \prec \delta +   \sigma'_{\text{max}}\\
\delta + \sigma_{\text{min}} \prec B \prec \delta +   \sigma_{\text{max}}\\
\end{align*}
Therefore, we have
\begin{equation}\label{bound:3}
   \begin{split}
            &\left\|{\left(g_{Nt}(\mathcal{D} )-\hat{ g}\right) - \left(\bar{g}_{-z_G}(\mathcal{D} ) - \hat{ g}\right)}\right\| \\
    =&\left\|\left(A^{-1}+B^{-1}\right)\cdot \nabla_{ g}L^-(\mathcal{D} ; \hat{g})\right\| \\
    \leq & \left\|A^{-1}+B^{-1}\right\|\cdot \left\|\nabla_{ g}L^-(\mathcal{D} ; \hat{g})\right\|\\
    \leq & \left|\frac{2\delta+\sigma_{\text{min}}+\sigma'_{\text{min}}}{\left(\delta+ \sigma'_{\text{min}}\right)\cdot\left(\delta+ \sigma_{\text{min}}\right)}\right|\cdot\left\|\nabla_{ g}L^-(\mathcal{D} ; \hat{g})\right\|\\
    \leq& \left|\frac{2\delta+\sigma_{\text{min}}+\sigma'_{\text{min}}}{\left(\delta+ \sigma'_{\text{min}}\right)\cdot\left(\delta+ \sigma_{\text{min}}\right)}\right| \cdot C_L'|G |
   \end{split}
\end{equation}
By combining the conclusions from Step I and Step II in Equations \ref{bound:1}, \ref{bound:2} and \ref{bound:3}, we obtain the error between the actual influence and our predicted influence as follows:
\begin{equation*}
   \begin{split}
        &\left\| \hat{g}_{-z_G}(\mathcal{D} ) - \bar{g}_{-z_G}(\mathcal{D} )\right\|\\
        \leq & \frac{C_H^-    {C'_{L}|G|}^2}{2 (\sigma'_{\text{min}} + \delta)^3} + \left|\frac{2\delta+\sigma_{\text{min}}+\sigma'_{\text{min}}}{\left(\delta+ \sigma'_{\text{min}}\right)\cdot\left(\delta+ \sigma_{\text{min}}\right)}\right| \cdot C_L'|G|.
   \end{split}
\end{equation*}
\end{proof}

\begin{remark}
 The error bound is linearly increasing with the number of removed data $|G|$. This implies that the estimation error increases with the number of erroneous data removed.
\end{remark}

%% file: appendix/7_appendix_d.tex
\section{Proof of Data-level Influence}
\label{sec:appendix:d}
\subsection{Data Removal}
We address situations that for dataset $\mathcal{D} = \{ (x_i, y_i, c_i) \}_{i=1}^{n}$, given a set of data $z_r = (x_r, y_r, c_r)$, $r\in G$ to be removed. Our goal is to estimate $\hat{g}_{-z_G}$, $\hat{f}_{-z_G}$, which is the concept and label predictor trained on the $z_r$ for $r \in G$ removed dataset. 

\paragraph{Proof Sketch.} (i) First, we estimate the retrained concept predictor $\hat{g}_{-z_G}$. (ii) Then, we define a new label predictor $\Tilde{f}_{-z_G}$ and estimate $\Tilde{f}_{-z_G}-\hat{f}$. (iii) Next, in order to reduce computational complexity, use the lemma method to obtain the approximation of the Hessian matrix of $\Tilde{f}_{-z_G}$. (iv) Next, we compute the difference $\hat{f}_{-z_G} - \Tilde{f}_{-z_G}$ as 
\begin{equation*}
\begin{split}
-H^{-1}_{\Tilde{f}_{-z_G}} \cdot \left( \nabla_{\hat{f}} L_Y \left( \Tilde{f}_{-z_G}(\hat{g}_{-z_G}(x_{i_r})), y_{i_r} \right) \right. \\
\left. - \nabla_{\hat{f}} L_Y \left( \Tilde{f}_{-z_G}(\hat{g}(x_{i_r})), y_{i_r} \right) \right).
\end{split}
\end{equation*}
(v) Finally, we divide $\hat{f}_{-z_G} - \hat{f}$, which we actually concerned with, into $\left ( \hat{f}_{-z_G} - \Tilde{f}_{-z_G} \right ) + \left ( \Tilde{f}_{-z_G} - \hat{f} \right )$.


\begin{theorem}\label{app:data-level:g}
For dataset $\mathcal{D} = \{ (x_i, y_i, c_i) \}_{i=1}^{n}$, given a set of data $z_r = (x_r, y_r, c_r)$, $r\in G$ to be removed. Suppose the updated concept predictor $\hat{g}_{-z_G}$ is defined by 
\begin{equation*}
    \hat{g}_{-z_G} = \argmin_{g}\sum_{j\in[k]}\sum_{i\in [n]-G}  L_{C_j}(\hat{g}(x_{i}), c_{i})
\end{equation*}
where $L_{C}(\hat{g}(x_{i}), c_{i}) \triangleq \sum_{j=1}^k  L_{C_j}(\hat{g}(x_{i}), c_{i})$.
Then we have the following approximation for $\hat{g}_{-z_G}$
\begin{equation}
\hat{g}_{-z_G} \approx \bar{g}_{-z_G} \triangleq\hat{g} + H^{-1}_{\hat{g}} \cdot \sum_{r\in G}\nabla_{g} L_{C} (\hat{g}(x_r), c_r), 
\end{equation}
where $H_{\hat{g}} = \nabla^2_{\hat{g}}  \sum_{i,j}  L_{C}(\hat{g}^j(x_i),c_i^j)$ is the Hessian matrix of the loss function respect to $\hat{g}$.
\end{theorem}

\begin{proof}
Firstly, we rewrite $\hat{g}_{-z_G}$ as 
\begin{equation*}
    \hat{g}_{-z_G} = \argmin_{g} \left[\sum_{\substack{i=1}}^{n} L_C(\hat{g}(x_{i}), c_{i}) - \sum_{r\in G} L_C(g(x_r), c_r)\right],
\end{equation*}

Then we up-weighted the $r$-th data by some $\epsilon$ and have a new predictor $\hat{g}_{-z_G, \epsilon}$, which is abbreviated as $\hat{g}_{\epsilon}$:
\begin{equation}\label{thm2:g1}
    \hat{g}_{\epsilon} \triangleq \argmin_{g} \left[ \sum_{i=1}^{n} L_C ( g(x_i), c_i) - \epsilon \cdot \sum_{r\in G}L_C ( g(x_r), c_r)\right].
\end{equation}

Because $\hat{g}_{\epsilon}$ minimizes the right side of equation \ref{thm2:g1}, we have 
\begin{equation*}
\nabla_{\hat{g}_{\epsilon}}  \sum_{i=1}^{n} L_Y (\hat{g}_{\epsilon}(x_i), c_i) - \epsilon \cdot\nabla_{\hat{g}_{\epsilon}} \sum_{r\in G}L_Y ( \hat{g}_{\epsilon}(x_r), c_r) =0. 
\end{equation*}

When $\epsilon\rightarrow 0$, $\hat{g}_{\epsilon}\rightarrow\hat{g}$.
So we can perform a first-order Taylor expansion with respect to $\hat{g}$, and we have
\begin{equation}\label{thm2:g:1}
\begin{split}
\nabla_{g} \sum_{i=1}^{n} L_C \bigl( \hat{g} (x_i), c_i \bigr) 
& - \epsilon \cdot \nabla_{g}\sum_{r\in G} L_C \bigl(\hat{g}(x_r), c_r \bigr) \\
& + \nabla^2_{g} \sum_{i=1}^{n} L_C \bigl( \hat{g} (x_i), c_i \bigr) \cdot (\hat{g}_{\epsilon} - \hat{g})  
\approx 0.
\end{split}
\end{equation}

Recap the definition of $\hat{g}$:
\begin{equation*}
    \hat{g} = \argmin_g  \sum_{i=1}^n L_Y(g(x_i),c_i),
\end{equation*}

Then, the first term of equation \ref{thm2:g:1} equals $0$. Let $\epsilon\rightarrow0$, then we have
\begin{equation*}
\left.\frac{\mathrm{d}\hat{g}_{\epsilon}}{\mathrm{d}\epsilon}\right|_{\epsilon = 0} = H^{-1}_{\hat{g}} \cdot \sum_{r\in G}\nabla_{g} L_C (\hat{g}(x_r), c_r),
\end{equation*}
where $H^{-1}_{\hat{g}} = \nabla^2_{g}  \sum_{i=1}^{n} \ell ( \hat{g} (x_i), c_i)$.

Remember when $\epsilon\rightarrow0$, $\hat{g}_{\epsilon}\rightarrow\hat{g}_{-z_G}$.
Perform a Newton step at $\hat{g}$, then we obtain the method to edit the original concept predictor under concept level: 
\begin{equation*}
\hat{g}_{-z_G} \approx \bar{g}_{-z_G}\triangleq\hat{g} + H^{-1}_{\hat{g}} \cdot \sum_{r\in G}\nabla_{g} L_C (\hat{g}(x_r), c_r).
\end{equation*}
\end{proof}

\begin{theorem}\label{app:thm:data-level:f}
For dataset $\mathcal{D} = \{ (x_i, y_i, c_i) \}_{i=1}^{n}$, given a set of data $z_r = (x_r, y_r, c_r)$, $r\in G$ to be removed. The label predictor $\hat{f}_{-z_G}$ trained on the revised dataset becomes
\begin{equation}\label{app:data-level:f}
    \hat{f}_{-z_G} = \argmin_{f}\sum_{i\in [n]-G} L_{Y_i}(f, \hat{g}_{-z_G}).
\end{equation}
The intermediate label predictor $\Tilde{f}_{-z_G}$ is defined by
\begin{equation*}
    \Tilde{f}_{-z_G} = \argmin \sum_{i\in [n]-G} L_{Y_i}(f, \hat{g}),
\end{equation*}
Then $\Tilde{f}_{-z_G}-\hat{f}$ can be approximated by
\begin{equation}\label{app:approx:f_piao}
    \Tilde{f}_{ -z_G} -\hat{f}  \approx H^{-1}_{\hat{f}} \cdot \sum_{i\in[n]-G}\nabla_{\hat{f}}L_{Y_i}(\hat{f},\hat{g}) \triangleq A_G.
\end{equation}
We denote the edited version of $\Tilde{f}_{-z_G}$ as $\bar{f}^*_{-z_G} \triangleq \hat{f} + A_G$. And $\hat{f}_{-z_G}-\Tilde{f}_{-z_G}$ can be approximated by
\begin{equation}\label{app:approx:hat_f}
\begin{split}
\hat{f}_{-z_G} - \Tilde{f}_{-z_G} &\approx -H^{-1}_{\bar{f}^*_{-z_G}} \cdot \left(\nabla_{\hat{f}} \sum_{i\in [n]-G}L_{Y_i} \left(\bar{f}^*_{-z_G}, \bar{g}_{-z_G}\right) \right. \\
&\left. \quad - \nabla_{\hat{f}} \sum_{i\in [n]-G} L_{Y_i} \left( \bar{f}^*_{-z_G}, \hat{g}\right)\right) \triangleq B_G,
\end{split}
\end{equation}
where $H_{\bar{f}^*_{-z_G}} = \nabla_{\bar{f}}\sum_{i\in[n]-G}L_{Y_i}\left(  \bar{f}^*_{-z_G}, \hat{g}\right)$ is the Hessian matrix of the loss function on the intermediate dataset concerning $\bar{f}^*_{-z_G}$. Then, the final edited label predictor $\bar{f}_{-z_G}$ can be obtained by
\begin{equation}
    \bar{f}_{-z_G} = \bar{f}^*_{-z_G} + B_G = \hat{f} + A_G + B_G. 
\end{equation}
\end{theorem}

\begin{proof}
We can see that there is a huge gap between $\hat{f}_{-z_G}$ and $\hat{f}$. Thus, firstly, we define $\Tilde{f}_{-z_G}$ as 
\begin{equation*}
\Tilde{f}_{-z_G} = \argmin_f \sum_{i=1}^{n} L_Y \left(  f (\hat{g}(x_i)), y_i  \right) - \sum_{r\in G}L_Y \left( f ( \hat{g} (x_r) ) , y_r   \right).
\end{equation*}

Then, we define $\Tilde{f}_{\epsilon, -z_G}$ as follows to estimate $\Tilde{f}_{-z_G}$.
\begin{equation*}
    \Tilde{f}_{\epsilon, -z_G} =  \argmin_f \sum_{i=1}^{n} L_Y \left(  f (\hat{g}(x_i)), y_i  \right) - \epsilon \cdot \sum_{r\in G} L_Y \left( f ( \hat{g} (x_r) ) , y_r   \right).
\end{equation*}

From the minimization condition, we have
\begin{equation*}
\begin{split}
    \nabla_{\Tilde{f}} \sum_{i=1}^{n} L_Y \left( \Tilde{f}_{\epsilon, -z_G} (\hat{g}(x_i)), y_i \right) \\
    - \epsilon \cdot \sum_{r\in G}\nabla_{\Tilde{f}}L_Y \left(\Tilde{f}_{\epsilon, -z_G} ( \hat{g} (x_r) ), y_r \right) 
    = 0.
\end{split}
\end{equation*}

Perform a first-order Taylor expansion at $\hat{f}$, 
\begin{align*}
& \nabla_{\hat{f}} \sum_{i=1}^{n} L_Y \left(  \hat{f} (\hat{g}(x_i)), y_i  \right) -  \epsilon \cdot \nabla_{\hat{f}} \sum_{r\in G} L_Y \left(\hat{f} ( \hat{g} (x_r) ) , y_r   \right)  \\
& + \nabla^2_{\hat{f}} \sum_{i=1}^{n} L_Y \left(  \hat{f} (\hat{g}(x_i)), y_i  \right) \cdot \left(\Tilde{f}_{\epsilon, -z_G} - \hat{f}\right)= 0.
\end{align*}

Then $\Tilde{f}_{-z_G}$ can be approximated by
\begin{equation}\label{ap:approx:f_piao}
    \Tilde{f}_{ -z_G} \approx \hat{f} + H^{-1}_{\hat{f}} \cdot \sum_{r\in G}\nabla_{\hat{f}} L_Y \left(\hat{f} ( \hat{g} (x_r) ) , y_r   \right)\triangleq A_G.
\end{equation}

Then the edit version of $\Tilde{f}_{-z_G}$ is defined as
\begin{equation}\label{app:th:3:f1}
    \bar{f}^*_{-z_G} = \hat{f} + A_G
\end{equation}

Then we estimate the difference between $\hat{f}_{-z_G}$ and $\Tilde{f}_{ -z_G}$. Rewrite $\Tilde{f}_{-z_G}$ as 
\begin{equation}\label{app:thm:2:f:2}
\Tilde{f}_{-z_G} = \argmin_f \sum_{i\in S}^{n} L_Y \left(  f (\hat{g}(x_i)), y_i  \right),
\end{equation}
where $S\triangleq[n]-G$.

Compare equation \ref{app:data-level:f} with \ref{app:thm:2:f:2}, we still need to define an intermediary predictor $\hat{f}_{-z_G, ir}$ as 
\begin{align*}
    \hat{f}_{-z_G, ir} = \argmin_f \left[\sum_{\substack{i\in S \\i\neq i_r}} L_{Y_i} \left(  f, \hat{g}(x_i)\right) + L_{Y_{ir}} \left(  f, \hat{g}_{-z_G}\right)\right]\\
    = \argmin_f \left[\sum_{\substack{i\in S}} L_{Y_i}\left(  f, \hat{g}\right) + L_{Y_{ir}} \left(  f, \hat{g}_{-z_G} \right) -  L_{Y_{ir}} \left(  f, \hat{g}\right) \right].
\end{align*}

Up-weight the $i_r$ data by some $\epsilon$, we define $\hat{f}_{\epsilon,-z_G,i_r}$ as
\begin{equation*}
\begin{split}
\hat{f}_{\epsilon,-z_G,ir} = \arg\min_{f} \Biggl[ & \sum_{i \in S} L_{Y_i}(f,\hat{g}) + \epsilon \cdot L_{Y_{ir}}(f,\hat{g}_{-z_G}) \\
& - \epsilon \cdot L_{Y_{ir}}(f,\hat{g}) \Biggr].
\end{split}
\end{equation*}

We denote $\hat{f}_{\epsilon, -z_G, ir}$ as $\hat{f}_{\epsilon}^*$ in the following proof. Then, from the minimization condition, we have 
\begin{equation}\label{f:1}
\begin{split}
\nabla_{\hat{f}}\sum_{\substack{i\in S}} L_{Y_i} &\left( \hat{f}_{\epsilon}^*,\hat{g}\right) + \epsilon \cdot\nabla_{\hat{f}} L_{Y_{ir}} \left( \hat{f}_{\epsilon}^*, \hat{g}_{-z_G}\right) \\
&- \epsilon \cdot\nabla_{\hat{f}} L_{Y_{ir}} \left( \hat{f}_{\epsilon}^*, \hat{g}(x_{i_r}) \right).
\end{split}
\end{equation}

When $\epsilon\rightarrow 0$, $\hat{f}_{\epsilon}^*\rightarrow \Tilde{f}_{-z_G}$.
Then we perform a Taylor expansion at $\Tilde{f}_{-z_G}$ of equation \ref{f:1} and have
\begin{align*}
\nabla_{\hat{f}}\sum_{\substack{i\in S}} L_{Y_{i}} &\left( \Tilde{f}_{-z_G},\hat{g}\right) 
+ \epsilon \cdot \nabla_{\hat{f}} L_{Y_{ir}} \left( \Tilde{f}_{-z_G}, \hat{g}_{-z_G}\right) \\
&- \epsilon \cdot \nabla_{\hat{f}} L_{Y_{ir}} \left( \Tilde{f}_{-z_G}, \hat{g}\right) \\
&+ \nabla^2_{\hat{f}}\sum_{\substack{i\in S}} L_{Y_{i}} \left( \Tilde{f}_{-z_G}, \hat{g}\right) 
\cdot \left( \hat{f}_{\epsilon}^* - \Tilde{f}_{-z_G} \right) 
\approx 0.
\end{align*}

Organizing the above equation gives
\begin{equation*}
\begin{split}
     \hat{f}_{\epsilon}^* - \Tilde{f}_{-z_G} \approx -\epsilon \cdot H^{-1}_{\Tilde{f}_{-z_G}} \cdot \left( \nabla_{\hat{f}}  L_{Y_{ir}} \left(  \Tilde{f}_{-z_G}, \hat{g}_{-z_G} \right) \right. \\
     \left. - \nabla_{\hat{f}}  L_{Y_{ir}} \left(  \Tilde{f}_{-z_G},\hat{g}\right) \right),
\end{split}
\end{equation*}
where $H_{\Tilde{f}_{-z_G}} = \nabla^2_{\hat{f}}\sum_{\substack{i\in S }}  L_{Y_{i}} \left( \Tilde{f}_{-z_G}, \hat{g}\right)$.

When $\epsilon = 1$, $\hat{f}_{\epsilon}^* = \hat{f}_{-z_G, ir}$. Then we perform a Newton iteration with step size $1$ at $\Tilde{f}_{-z_G}$,
\begin{equation*}
\begin{split}
     \hat{f}_{-z_G, ir} - \Tilde{f}_{-z_G} \approx -H^{-1}_{\Tilde{f}_{-z_G}} \cdot \left( \nabla_{\hat{f}}  L_{Y_{ir}} \left(  \Tilde{f}_{-z_G}, \hat{g}_{-z_G}\right) \right. \\
     \left. - \nabla_{\hat{f}}  L_{Y_{ir}} \left(  \Tilde{f}_{-z_G}, \hat{g}\right) \right)
\end{split}
\end{equation*}
Iterate $i_r$ through set $S$, and we have
\begin{equation}\label{ap:approx:hat_f}
\begin{split}
     \hat{f}_{-z_G} - \Tilde{f}_{-z_G} \approx -H^{-1}_{\Tilde{f}_{-z_G}} \cdot \left( \nabla_{\hat{f}} \sum_{i\in S} L_{Y_{i}} \left(  \Tilde{f}_{-z_G}, \hat{g}_{-z_G}\right) \right. \\
     \left. - \nabla_{\hat{f}} \sum_{i\in S}  L_{Y_i} \left(  \Tilde{f}_{-z_G}, \hat{g}\right) \right)
\end{split}
\end{equation}

The edited version of $\hat{g}_{-z_G}$ has been deduced as $\bar{g}_{-z_G}$ in theorem \ref{app:data-level:g}, substituting this approximation into \eqref{ap:approx:hat_f}, then we have
\begin{equation}
\begin{split}
     \hat{f}_{-z_G} - \Tilde{f}_{-z_G} \approx -H^{-1}_{\Tilde{f}_{-z_G}} \cdot \left( \nabla_{\hat{f}} \sum_{i\in S} L_{Y_i} \left(  \Tilde{f}_{-z_G}, \bar{g}_{-z_G}\right) \right. \\
     \left. - \nabla_{\hat{f}} \sum_{i\in S}  L_{Y_i} \left(  \Tilde{f}_{-z_G}, \hat{g}\right) \right).
\end{split}
\end{equation}
Noting that we cannot obtain $\hat{f}_{-z_G}$ and $H_{\Tilde{f}_{-z_G}}$ directly because we do not retrain the label predictor but edit it to $\bar{f}^*_{-z_G}$ as a substitute. Therefore, we approximate $\hat{f}_{-z_G}$ with $\bar{f}^*_{-z_G}$ and $H_{\Tilde{f}_{-z_G}}$ with  $H_{\bar{f}^*_{-z_G}}$ which is defined by:
\begin{equation*}
    H_{\bar{f}^*_{-z_G}} = \nabla^2_{\hat{f}}\sum_{\substack{i\in S }}  L_{Y_{i}} \left( \bar{f}^*_{-z_G}, \hat{g}\right)
\end{equation*}
Then we define $B_G$ as
\begin{equation}\label{app:th:3:f2}
\begin{split}
    B_G \triangleq -H^{-1}_{\bar{f}^*_{-z_G}} \cdot \left( \nabla_{\hat{f}} \sum_{i\in S} L_{Y_i} \left(  \bar{f}^*_{-z_G}, \bar{g}_{-z_G}\right) \right. \\
    \left. - \nabla_{\hat{f}} \sum_{i\in S}  L_{Y_i} \left(  \bar{f}^*_{-z_G}, \hat{g}\right) \right)
\end{split}
\end{equation}
Combining \eqref{app:th:3:f1} and \eqref{app:th:3:f2}, then we deduce the final closed-form edited label predictor as
\begin{equation*}
    \bar{f}_{-z_G} = \bar{f}^*_{-z_G} + B_G = \hat{f} + A_G + B_G. 
\end{equation*}
\end{proof}

\subsubsection{Data Addition} 

\begin{theorem}\label{app:data-level-add:g}
For dataset $\mathcal{D} = \{ (x_i, y_i, c_i) \}_{i=1}^{N}$, given a set of new data $\mathcal{D}_{\mathrm{add}} = \{\tilde{z}_s = (\tilde{x}_s, \tilde{y}_s, \tilde{c}_s)\}_{s=1}^M$ to be added. Suppose the updated concept predictor $\hat{g}^{\cup \mathcal{D}_{\text{add}}}$ is defined by
\begin{equation}
\begin{split}
    \hat{g}^{\cup \mathcal{D}_{\text{add}}} =&  \argmin_{g}\sum_{j\in[K]}\sum_{i\in [N]}  L_{C_j}({g}({x}_{i}), {c}_{i}) \\
    & + \sum_{j\in[K]}\sum_{s\in [M]}  L_{C_j}({g}(\tilde{x}_{S}), \tilde{c}_{S}).
    \end{split}
\end{equation}
Define $L_{C}(\hat{g}(x_{i}), c_{i}) \triangleq \sum_{j=1}^K  L_{C_j}(\hat{g}(x_{i}), c_{i})$.
Then we have the following approximation for $\hat{g}^{\cup \mathcal{D}_{\text{add}}}$
\begin{equation}
\hat{g}^{\cup \mathcal{D}_{\text{add}}} \approx \bar{g}^{\cup \mathcal{D}_{\text{add}}} \triangleq\hat{g} - H^{-1}_{\hat{g}} \cdot \sum_{s\in [M]}\nabla_{g} L_{C} (\hat{g}(\tilde{x}_s), \tilde{c}_s), 
\end{equation}
where $H_{\hat{g}} = \nabla^2_{\hat{g}}  \sum_{i=1}^N  L_{C}(\hat{g}(x_i),c_i)$ is the Hessian matrix of the loss function respect to $\hat{g}$.
\end{theorem}

\begin{proof}
Firstly, recalling the definition of $\hat{g}^{\cup \mathcal{D}_{\text{add}}}$ as 
\begin{equation*}
    \hat{g}^{\cup \mathcal{D}_{\text{add}}} = \argmin_{g} \left[ \sum_{\substack{i=1}}^{N} L_C({g}(x_{i}), c_{i}) + \sum_{s=1}^M L_C(g(\tilde{x}_s), \tilde{c}_s)\right],
\end{equation*}

Then we up-weighted the $s$-th data by some $\epsilon$ and have a new predictor $\hat{g}_{\epsilon}^{\cup \mathcal{D}_{\text{add}}}$:
\begin{equation}\label{app-thm2:g1}
\argmin_{g} \left[ \sum_{i=1}^{N} L_C ( g(x_i), c_i) + \epsilon \cdot \sum_{s=1}^M L_C ( g(\tilde{x}_s), \tilde{c}_s)\right].
\end{equation}

Because $\hat{g}_{\epsilon}^{\cup \mathcal{D}_{\text{add}}}$ minimizes the right side of equation~\eqref{app-thm2:g1}, we have 
\begin{equation*}
\nabla_{g}  \sum_{i=1}^{N} L_Y (\hat{g}_{\epsilon}^{\cup \mathcal{D}_{\text{add}}}(x_i), c_i) + \epsilon \cdot\nabla_{g} \sum_{s=1}^ML_Y ( \hat{g}_{\epsilon}^{\cup \mathcal{D}_{\text{add}}}(\tilde{x}_s), \tilde{c}_s) =0. 
\end{equation*}

When $\epsilon\rightarrow 0$, $\hat{g}_{\epsilon}^{\cup \mathcal{D}_{\text{add}}}\rightarrow\hat{g}$.
So we can perform a first-order Taylor expansion with respect to ${g}$, and we have
\begin{equation}\label{app-thm2:g:1}
\begin{split}
0 \approx &\nabla_{g}  \sum_{i=1}^{N} L_C ( \hat{g} (x_i), c_i) 
+
\epsilon \cdot \nabla_{g}\sum_{s=1}^M L_C (\hat{g}(\tilde{x}_s), \tilde{c}_s) \\
&+
\nabla^2_{g}  \sum_{i=1}^{N} L_C ( \hat{g} (x_i), c_i) \cdot (\hat{g}_{\epsilon}^{\cup \mathcal{D}_{\text{add}}} - \hat{g}) .
\end{split}
\end{equation}

Recap the definition of $\hat{g}$:
\begin{equation*}
    \hat{g} = \argmin_g  \sum_{i=1}^n L_Y(g(x_i),c_i),
\end{equation*}

Then, the first term of the right side of~\eqref{app-thm2:g:1} equals $0$. Let $\epsilon\rightarrow0$, then we have
\begin{equation*}
\left.\frac{\mathrm{d}\hat{g}_{\epsilon}^{\cup \mathcal{D}_{\text{add}}}}{\mathrm{d}\epsilon}\right|_{\epsilon = 0} = -H^{-1}_{\hat{g}} \cdot \sum_{s=1}^M\nabla_{g} L_C (\hat{g}(\tilde{x}_s), \tilde{c}_s),
\end{equation*}
where $H_{\hat{g}} = \nabla^2_{\hat{g}}  \sum_{i=1}^N  L_{C}(\hat{g}(x_i),c_i)$.

Remember when $\epsilon\rightarrow0$, $\hat{g}_{\epsilon}^{\cup \mathcal{D}_{\text{add}}}\rightarrow\hat{g}^{\cup \mathcal{D}_{\text{add}}}$.
Perform a Newton step at $\hat{g}$, then we obtain the method to edit the original concept predictor under concept level: 
\begin{equation*}
\hat{g}^{\cup \mathcal{D}_{\text{add}}} \approx \bar{g}^{\cup \mathcal{D}_{\text{add}}}\triangleq\hat{g} - H^{-1}_{\hat{g}} \cdot \sum_{s=1}^M\nabla_{g} L_C (\hat{g}(\tilde{x}_s), \tilde{c}_s).
\end{equation*}
\end{proof}

\begin{theorem}
For dataset $\mathcal{D} = { (x_i, y_i, c_i) }{i=1}^{N}$, given a set of new data $\mathcal{D}{\mathrm{add}} = {\tilde{z}s = (\tilde{x}s, \tilde{y}s, \tilde{c}s)}{s=1}^M$ to be added. Suppose the updated label predictor $\hat{f}^{\cup \mathcal{D}_{\text{add}}}$ is defined by 
\begin{equation}\label{app-data:f-def} 
\begin{split} 
\hat{f}^{\cup \mathcal{D}_{\text{add}}} =& \argmin_{f} \sum_{i=1}^N L_{Y}(f(\hat{g}^{\cup \mathcal{D}_{\text{add}}}(x_{i})), y_{i}) \\
& + \sum_{s=1}^{M} L_{Y}(f(\hat{g}^{\cup \mathcal{D}_{\text{add}}}(\tilde{x}_{s})), \tilde{y}_{s}).
\end{split} 
\end{equation} 
Then we have the following approximation for $\hat{f}^{\cup \mathcal{D}_{\text{add}}}$: \begin{equation} \hat{f}^{\cup \mathcal{D}_{\text{add}}} \approx \bar{f}^{\cup \mathcal{D}_{\text{add}}} \triangleq \hat{f} - H^{-1}{\hat{f}} \cdot \sum_{s\in [M]} \nabla_{f} L_{Y} (\hat{f}(\tilde{x}s), \tilde{y}s), \end{equation} where $H{\hat{f}} = \nabla^2{\hat{f}} \sum_{i=1}^N L_{Y}(\hat{f}(x_i),y_i)$ is the Hessian matrix of the loss function with respect to $\hat{f}$. \end{theorem}


\begin{proof}
We can see that there is a huge gap between $\hat{f}^{\cup \mathcal{D}_{\text{add}}}$ and $\hat{f}$. Thus, firstly, we define $\tilde{f}^{\cup \mathcal{D}_{\text{add}}}$ as 
\begin{equation*}
\argmin_f \sum_{i=1}^{N} L_Y \left(  f (\hat{g}(x_i)), y_i  \right) + \sum_{s=1}^M L_Y \left( f ( \hat{g} (\tilde{x}_s) ) , \tilde{y}_s   \right).
\end{equation*}

Then, we define $\tilde{f}_{\epsilon}^{\cup \mathcal{D}_{\text{add}}}$ as follows to estimate $\tilde{f}^{\cup \mathcal{D}_{\text{add}}}$.
\begin{equation*}
\begin{split}
    \tilde{f}_{\epsilon}^{\cup \mathcal{D}_{\text{add}}} =  \argmin_f \sum_{i=1}^{N} L_Y \left(  f (\hat{g}(x_i)), y_i  \right) \\
    + \epsilon \cdot \sum_{s=1}^M L_Y \left( f ( \hat{g} (\tilde{x}_s) ) , \tilde{y}_s   \right).
\end{split}
\end{equation*}

From the minimization condition, we have
\begin{equation*}
\begin{split}
    0 = &\nabla_{{f}} \sum_{i=1}^{N} L_Y \left(  \tilde{f}_{\epsilon}^{\cup \mathcal{D}_{\text{add}}} (\hat{g}(x_i)), y_i  \right) \\
    & + \epsilon \cdot \sum_{s=1}^M\nabla_{{f}}L_Y \left(\tilde{f}_{\epsilon}^{\cup \mathcal{D}_{\text{add}}} ( \hat{g} (\tilde{x}_s) ) , \tilde{y}_s   \right).
    \end{split}
\end{equation*}
Perform a first-order Taylor expansion at $\hat{f}$,
\begin{align*}
& \nabla_{\hat{f}} \sum_{i=1}^{N} L_Y \left(  \hat{f} (\hat{g}(x_i)), y_i  \right) + \epsilon \cdot \nabla_{\hat{f}} \sum_{s=1}^M L_Y \left(\hat{f} ( \hat{g} (\tilde{x}_s) ) , \tilde{y}_s   \right)  \\
& + \nabla^2_{\hat{f}} \sum_{i=1}^{N} L_Y \left(  \hat{f} (\hat{g}(x_i)), y_i  \right) \left(\tilde{f}_{\epsilon}^{\cup \mathcal{D}_{\text{add}}} - \hat{f}\right)= 0.
\end{align*}

Then $\tilde{f}^{\cup \mathcal{D}_{\text{add}}}$ can be approximated by
\begin{equation}\label{ap:approx:f_piao}
    \tilde{f}^{\cup \mathcal{D}_{\text{add}}} \approx \hat{f} - H^{-1}_{\hat{f}} \cdot \sum_{s=1}^M\nabla_{\hat{f}} L_Y \left(\hat{f} ( \hat{g} (\tilde{x}_s) ) , \tilde{y}_s   \right).
\end{equation}
Define $A\triangleq H^{-1}_{\hat{f}} \cdot \sum_{s=1}^M\nabla_{\hat{f}} L_Y \left(\hat{f} ( \hat{g} (\tilde{x}_s) ) , \tilde{y}_s   \right)$
Then the approximation of $\tilde{f}^{\cup \mathcal{D}_{\text{add}}}$ is defined as
\begin{equation}\label{app:th:3:f1}
\bar{f}^{\cup \mathcal{D}_{\text{add}}} = \hat{f} - A
\end{equation}

Then we estimate the difference between $\hat{f}^{\cup \mathcal{D}_{\text{add}}}$ and $\tilde{f}^{\cup \mathcal{D}_{\text{add}}}$. Recall the definition of $\tilde{f}^{\cup \mathcal{D}_{\text{add}}}$ as 
\begin{equation}\label{app-data-add:thm:2:f:2}
\argmin_f \sum_{i=1}^{N} L_Y \left(  f (\hat{g}(x_i)), y_i  \right) + \sum_{s=1}^{M} L_Y \left(  f (\hat{g}(\tilde{x}_s)), \tilde{y}_s  \right).
\end{equation}

Compare equation~\eqref{app-data:f-def} with~\eqref{app-data-add:thm:2:f:2}, we still need to define an intermediary predictor $\hat{f}^{\cup \mathcal{D}_{\text{add}}}_{s_t}$ as 
\begin{align*}
    \hat{f}^{\cup \mathcal{D}_{\text{add}}}_{s_t} 
    = \argmin_f \Biggl[ 
    &\sum_{i=1}^N L_{Y} \left( f(\hat{g}(x_i)), y_i \right) \\
    &+ \sum_{s=1}^M L_{Y} \left( f(\hat{g}(\tilde{x}_s)), \tilde{y}_s \right) \\
    &+ L_{Y} \left( f(\hat{g}_{\epsilon}^{\cup \mathcal{D}_{\text{add}}}(\tilde{x}_{s_t})), \tilde{y}_{s_t} \right) \\
    &- L_{Y} \left( f(\hat{g}(\tilde{x}_{s_t})), \tilde{y}_{s_t} \right)
    \Biggr]
\end{align*}

Up-weight the $s_t$ data by some $\epsilon$, we define $ \hat{f}^{\cup \mathcal{D}_{\text{add}}}_{s_t, \epsilon}$ as
\begin{equation*}
\begin{split}
    \hat{f}^{\cup \mathcal{D}_{\text{add}}}_{s_t, \epsilon} = \argmin_f \Biggl[ 
    & \sum_{i=1}^N L_{Y} \left(f(\hat{g}(x_i)), y_i\right) \\
    & + \sum_{s=1}^M L_{Y} \left(f(\hat{g}(\tilde{x}_s)), \tilde{y}_s\right) \\
    & + \epsilon \cdot L_{Y} \left(f(\hat{g}^{\cup \mathcal{D}_{\text{add}}}(\tilde{x}_{s_t})), \tilde{y}_{s_t}\right) \\
    & - \epsilon \cdot L_{Y} \left(f(\hat{g}(\tilde{x}_{s_t})), \tilde{y}_{s_t}\right) \Biggr].
\end{split}
\end{equation*}

Then, from the minimization condition, we have 
\begin{equation}\label{add-f:1} 
\begin{split}
0 = \nabla_{{f}}\sum_{i=1}^N L_{Y_i} \left(  \hat{f}^{\cup \mathcal{D}_{\text{add}}}_{s_t, \epsilon},\hat{g}\right) +\sum_{s=1}^M L_{Y_i}\left(\hat{f}^{\cup \mathcal{D}_{\text{add}}}_{s_t, \epsilon},\hat{g}\right) \\
+ \epsilon \cdot\nabla_{{f}} L_{Y_{st}} \left(  \hat{f}^{\cup \mathcal{D}_{\text{add}}}_{s_t, \epsilon},  \hat{g}^{\cup \mathcal{D}_{\text{add}}}\right) - \epsilon \cdot\nabla_{\hat{f}} L_{Y_{ir}} \left(   \hat{f}^{\cup \mathcal{D}_{\text{add}}}_{s_t, \epsilon}, \hat{g} \right).
\end{split}
\end{equation}

When $\epsilon\rightarrow 0$, $\hat{f}^{\cup \mathcal{D}_{\text{add}}}_{s_t, \epsilon}\rightarrow \Tilde{f}^{\cup \mathcal{D}_{\text{add}}}$.
Then we perform a Taylor expansion at $\Tilde{f}^{\cup \mathcal{D}_{\text{add}}}$ of equation \ref{add-f:1} and have
\begin{equation*}
\begin{split}
0 \approx \nabla_{{f}}\sum_{i=1}^N L_{Y_i} \left(  \Tilde{f}^{\cup \mathcal{D}_{\text{add}}},\hat{g}\right) +\sum_{s=1}^M L_{Y_i}\left(\Tilde{f}^{\cup \mathcal{D}_{\text{add}}},\hat{g}\right) \\
+ \epsilon \cdot\nabla_{{f}} L_{Y_{st}} \left(  \Tilde{f}^{\cup \mathcal{D}_{\text{add}}},  \hat{g}^{\cup \mathcal{D}_{\text{add}}}\right) \cdot \left(\hat{f}^{\cup \mathcal{D}_{\text{add}}}_{s_t, \epsilon} - \Tilde{f}^{\cup \mathcal{D}_{\text{add}}}\right)\\
- \epsilon \cdot\nabla_{\hat{f}} L_{Y_{ir}} \left(  \Tilde{f}^{\cup \mathcal{D}_{\text{add}}}, \hat{g} \right) \cdot \left(\hat{f}^{\cup \mathcal{D}_{\text{add}}}_{s_t, \epsilon} - \Tilde{f}^{\cup \mathcal{D}_{\text{add}}}\right).
\end{split}
\end{equation*}

Organizing the above equation gives
\begin{equation*}
\begin{split}
    \hat{f}^{\cup \mathcal{D}_{\text{add}}}_{s_t, \epsilon} - \Tilde{f}^{\cup \mathcal{D}_{\text{add}}} 
    & \approx -\epsilon \cdot H^{-1}_{\Tilde{f}} \nabla_{{f}} \left( L_{Y_{st}} \left( \Tilde{f}^{\cup \mathcal{D}_{\text{add}}}, \hat{g}^{\cup \mathcal{D}_{\text{add}}} \right) \right. \\
    & \left. - L_{Y_{st}} \left( \Tilde{f}^{\cup \mathcal{D}_{\text{add}}}, \hat{g} \right) \right),
\end{split}
\end{equation*}
where $$H_{\Tilde{f}} = \nabla^2_{\hat{f}}\sum_{i=1}^N L_{Y_i} \left(  \Tilde{f}^{\cup \mathcal{D}_{\text{add}}},\hat{g}\right) +\sum_{s=1}^M L_{Y_i}\left(\Tilde{f}^{\cup \mathcal{D}_{\text{add}}},\hat{g}\right). $$

When $\epsilon = 1$, $\hat{f}^{\cup \mathcal{D}_{\text{add}}}_{s_t, -1} = \hat{f}^{\cup \mathcal{D}_{\text{add}}}_{s_t}$. Then we perform a Newton iteration with step size $1$ at $\Tilde{f}^{\cup \mathcal{D}_{\text{add}}}$,
\begin{equation*}
\begin{split}
    \hat{f}^{\cup \mathcal{D}_{\text{add}}}_{s_t} - \Tilde{f}^{\cup \mathcal{D}_{\text{add}}} 
    & \approx -H^{-1}_{\Tilde{f}} \nabla_{{f}} \left( L_{Y_{st}} \left( \Tilde{f}^{\cup \mathcal{D}_{\text{add}}}, \hat{g}^{\cup \mathcal{D}_{\text{add}}} \right) \right. \\
    & \left. - L_{Y_{st}} \left( \Tilde{f}^{\cup \mathcal{D}_{\text{add}}}, \hat{g} \right) \right),
\end{split}
\end{equation*}

Iterate $s_t$ through set $[M]$, and we have
\begin{equation}\label{ap:add-approx:hat_f}
\begin{split}
    \hat{f}^{\cup \mathcal{D}_{\text{add}}} - \Tilde{f}^{\cup \mathcal{D}_{\text{add}}} 
    & \approx -H^{-1}_{\Tilde{f}} \cdot \sum_{s=1}^M \nabla_{{f}} \Biggl( L_{Y_{s}} \left( \Tilde{f}^{\cup \mathcal{D}_{\text{add}}}, \hat{g}^{\cup \mathcal{D}_{\text{add}}} \right) \\
    & \quad - L_{Y_{s}} \left( \Tilde{f}^{\cup \mathcal{D}_{\text{add}}}, \hat{g} \right) \Biggr)
\end{split}
\end{equation}

The edited version of $\hat{g}^{\cup \mathcal{D}_{\text{add}}}$ has been deduced as $\bar{g}^{\cup \mathcal{D}_{\text{add}}}$ in theorem \ref{app:data-level-add:g}, substituting this approximation into \eqref{ap:add-approx:hat_f}, then we have
\begin{equation}
\begin{split}
    \hat{f}^{\cup \mathcal{D}_{\text{add}}} - \Tilde{f}^{\cup \mathcal{D}_{\text{add}}}  
    & \approx -H^{-1}_{\Tilde{f}} \cdot \nabla_{\hat{f}} \sum_{s=1}^M \Bigl[ L_{Y_s} \left( \Tilde{f}^{\cup \mathcal{D}_{\text{add}}}, \bar{g}^{\cup \mathcal{D}_{\text{add}}} \right) \\
    & \quad - L_{Y_s} \left( \Tilde{f}^{\cup \mathcal{D}_{\text{add}}}, \hat{g} \right) \Bigr].
\end{split}
\end{equation}
Noting that we cannot obtain $\hat{f}^{\cup \mathcal{D}_{\text{add}}}$ and $H_{\Tilde{f}}$ directly because we do not retrain the label predictor but edit it to $\bar{f}^{\cup \mathcal{D}_{\text{add}}}$ as a substitute. Therefore, we approximate $\hat{f}^{\cup \mathcal{D}_{\text{add}}}$ with $\bar{f}^{\cup \mathcal{D}_{\text{add}}}$ and $H_{\Tilde{f}}$ with  $H_{\bar{f}}$ which is defined by:
\begin{equation*}
    H_{\bar{f}} = \nabla^2_{\bar{f}}\sum_{i=1}^N  L_{Y_{i}} \left( \bar{f}^{\cup \mathcal{D}_{\text{add}}}, \hat{g}\right) + \nabla^2_{\bar{f}}\sum_{s=1}^M  L_{Y_{s}} \left( \bar{f}^{\cup \mathcal{D}_{\text{add}}}, \hat{g}\right)  
\end{equation*}
Then we define $B$ as
\begin{equation}\label{app:add-th:3:f2}
    B \triangleq H^{-1}_{\bar{f}} \cdot \nabla_{\hat{f}}\left(  \sum_{s=1}^M L_{Y_s} \left(  \Tilde{f}^{\cup \mathcal{D}_{\text{add}}}, \bar{g}^{\cup \mathcal{D}_{\text{add}}}\right) - L_{Y_s} \left(  \Tilde{f}^{\cup \mathcal{D}_{\text{add}}}, \hat{g}\right) \right)
\end{equation}
Combining \eqref{app:th:3:f1} and \eqref{app:th:3:f2}, then we deduce the final closed-form edited label predictor as
\begin{equation*}
    \bar{f}^{\cup \mathcal{D}_{\text{add}}}_* =\bar{f}^{\cup \mathcal{D}_{\text{add}}} - B = \hat{f} - A - B. 
\end{equation*}
\end{proof}

%% file: appendix/7_appendix_e.tex
\section{Existing Methods}
\label{sec:appendix:e}
In this section, we provide EK-FAC methods for accelerating influence computation.

\subsection{EK-FAC}
In our CBM model, the label predictor is a single linear layer and the cost for Hessian computing is affordable. However, the concept predictor is based on Resnet-18, which has much more parameters. Therefore, we perform EK-FAC for $\hat{g}$.
\begin{equation*}
    \hat{g} = \argmin_g \sum_{j=1}^{k}  L_{C_j} = \argmin_g \sum_{j=1}^k\sum_{i=1}^n \ell(g^j(x_i),c_i^j),
\end{equation*}
we define $H_{\hat{g}} = \nabla^2_{\hat{g}}\sum L_{C_j}$ as the Hessian matrix of the loss function with respect to the parameters.

To this end, consider the $l$-th layer of $\hat{g}$ which takes as input a layer of activations $\{a_{j,t}\}$ where $j\in\{1, 2, \ldots, J\}$ indexes the input map and $t \in \mathcal{T}$ indexes the spatial location which is typically a 2-D grid. And this layer is parameterized by a set of weights $W = \left(w_{i, j, \delta}\right)$ and biases $b = \left(b_{i}\right)$, where $i \in\{1, \ldots, I\}$ indexes the output map, and $\delta \in \Delta$ indexes the spatial offset (from the center of the filter).

The convolution layer computes a set of pre-activations as 
\begin{equation*}
    [S_l]_{i,t} = s_{i, t}=\sum_{\delta \in \Delta} w_{i, j, \delta} a_{j, t+\delta}+b_{i}.
\end{equation*}
Denote the loss derivative with respect to $s_{i,t}$ as 
\begin{equation*}
    \mathcal{D}s_{i, t} = \frac{\partial \sum L_{C_j}}{\partial s_{i, t}},
\end{equation*}
which can be computed during backpropagation.

The activations are actually stored as $A_{l-1}$ of dimension $|\mathcal{T}|\times J$. Similarly, the weights are stored as an $I \times|\Delta| J$ array $W_l$. The straightforward implementation of convolution, though highly parallel in theory, suffers from poor memory access patterns. Instead, efficient implementations typically leverage what is known as the expansion operator $\llbracket \cdot \rrbracket$. For instance, 
$\llbracket A_{l-1} \rrbracket$ is a $|\mathcal{T}|\times J|\Delta|$ matrix, defined as
\begin{equation*}
    \llbracket {{A}}_{l-1} \rrbracket_{t, j|\Delta|+\delta}=\left[{{A}}_{l-1}\right]_{(t+\delta), j}=a_{j, t+\delta},
\end{equation*}

In order to fold the bias into the weights, we need to add a homogeneous coordinate (i.e. a column of all $1$’s) to the expanded activations $\llbracket A_{l-1} \rrbracket$ and denote this as $\llbracket A_{l-1} \rrbracket_{\mathrm{H}}$. Concatenating the bias vector to the weights matrix, then we have $\theta_l = (b_l, W_l)$.

Then, the approximation for $H_{\hat{g}}$ is given as:
\begin{align*}
G^{(l)}(\hat{g})  =& \mathbb{E}\left[\mathcal{D} w_{i, j, \delta} \mathcal{D} w_{i^{\prime}, j^{\prime}, \delta^{\prime}}\right]\\
=&\mathbb{E}\left[\left(\sum_{t \in \mathcal{T}} a_{j, t+\delta} \mathcal{D} s_{i, t}\right)\left(\sum_{t^{\prime} \in \mathcal{T}} a_{j^{\prime}, t^{\prime}+\delta^{\prime}} \mathcal{D} s_{i^{\prime}, t^{\prime}}\right)\right]\\
\approx & \mathbb{E}\left[\llbracket {A}_{l-1} \rrbracket_{\mathrm{H}}^{\top} \llbracket {A}_{l-1} \rrbracket_{\mathrm{H}}\right] \\
& \otimes \frac{1}{|\mathcal{T}|} \mathbb{E}\left[\mathcal{D} {S}_{l}^{\top} \mathcal{D} {S}_{l}\right] \triangleq \Omega_{l-1}\otimes \Gamma_{l}.
\end{align*}
Estimate the expectation using the mean of the training set, 
\begin{align*}
G^{(l)}(\hat{g}) \approx & \frac{1}{n} \sum_{i=1}^{n}\left(\llbracket {A}^i_{l-1} \rrbracket_{\mathrm{H}}^{\top} \llbracket {A}^i_{l-1} \rrbracket_{\mathrm{H}}\right) \\ &\otimes\frac{1}{n} \sum_{i=1}^{n}\left(\frac{1}{|\mathcal{T}|}\mathcal{D} {{S}^i_{l}}^{\top} \mathcal{D} {S}^i_{l}\right)\triangleq \hat{\Omega}_{l-1}\otimes \hat{\Gamma_{l}}
\end{align*}
Furthermore, if the factors $\hat{\Omega}_{l-1}$ and $\hat{\Gamma}_{l}$ have eigen decomposition $Q_{\Omega} {\Lambda}_{{\Omega}} {Q}_{{\Omega}}^{\top}$ and ${Q}_{\Gamma}{\Lambda}_{\Gamma} {Q}_{\Gamma}^{\top}$, respectively, then the eigen decomposition of $\hat{\Omega}_{l-1}\otimes \hat{\Gamma_{l}}$ can be written as:  
$$ 
\begin{aligned}  
\hat{\Omega}_{l-1}\otimes \hat{\Gamma}_{l} & ={Q}_{  {\Omega}} {\Lambda}_{  {\Omega}} {Q}_{  {\Omega}}^{\top} \otimes {Q}_{  {\Gamma}} {\Lambda}_{  {\Gamma}} {Q}_{  {\Gamma}}^{\top} \\  
& =\left({Q}_{  {\Omega}} \otimes {Q}_{   {\Gamma}}\right)\left(  {\Lambda}_{   {\Omega}} \otimes  {\Lambda}_{   {\Gamma}}\right)\left( {Q}_{   {\Omega}} \otimes  {Q}_{   {\Gamma}}\right)^{\top}  
\end{aligned}  
$$
Since subsequent inverse operations are required and the current approximation for $G^{(l)}(\hat{g})$  is PSD, we actually use a damped version as 
\begin{align}\label{ape:ekfac:1}
    {\hat{G^{l}}(\hat{g})}^{-1} =& \left(G_{l}\left(\hat{g}\right)+\lambda_{l} I_{d_{l}}\right)^{-1} \\
    =& \left(Q_{\Omega_{l-1}} \otimes Q_{\Gamma_{l}}\right)\left(\Lambda_{\Omega_{l-1}} \otimes \Lambda_{\Gamma_{l}}+\lambda_{l} I_{d_{l}}\right)^{-1}\left(Q_{\Omega_{l-1}} \otimes Q_{\Gamma_{l}}\right)^{\mathrm{T}}
\end{align}
Besides, \cite{george2018fast} proposed a new method that corrects the error in equation \ref{ape:ekfac:1} which sets the $i$-th diagonal element of $\Lambda_{\Omega_{l-1}} \otimes \Lambda_{\Gamma_{l}}$ as
\begin{equation*}
    \Lambda_{ii}^{*} = n^{-1} \sum_{j =1}^n \left( \left( Q_{\Omega_{l-1}}\otimes Q_{\Gamma_{l}}\right)\nabla_{\theta_l}\ell_j\right)^2_i.
\end{equation*}

%% file: appendix/7_appendix_exp.tex
\section{More Visualization Results and Explanation}
\label{sec:appendix:exp}

\noindent{\bf Visualization.} Since CBM is an explainable model, we aim to evaluate the interpretability of our CCBM (compared to the retraining). We will present some visualization results for the concept-level edit. Figure~\ref{fig:vertical_images} presents the top 10 most influential concepts and their corresponding predicted concept labels obtained by our CCBM and the retrain method after randomly deleting concepts for the CUB dataset. (Detailed explanation can be found in Appendix \ref{app:explain}.)
Our CCBM can provide explanations for which concepts are crucial and how they assist the prediction.
Specifically, among the top 10 most important concepts in the ground truth (retraining), CCBM can accurately recognize 9 within them. 
For instance, we correctly identify "has\_upperparts\_color::orange", "has\_upper\_tail\_color::red", and "has\_breast\_color::black" as some of the most important concepts when predicting categories. Additional visualization results under data level and concept-label level on OAI and CUB datasets are included in Appendix \ref{app:visual}.
\input{fig_ICLR/fig4}

\subsection{Explanation for Visualization Results}\label{app:explain}

At the concept level, we remove each concept one at a time, retrain the CBM, and subsequently evaluate the model performance. We rank the concepts in descending order based on the model performance loss. Concepts that, when removed, cause significant changes in model performance are considered influential concepts. The top 10 concepts are shown in the retrain column as illustrated in Figure~\ref{fig:vertical_images}. In contrast, we use our CCBM method instead of the retrain method, as outlined in Algorithm~\ref{alg:5}, and the top 10 concepts are shown in the CCBM column of Figure~\ref{fig:vertical_images}.

To help readers connect the top 10 influential concepts with the input image, we provide visualizations of the data and list the concept labels corresponding to the top 10 influential concepts, which are shown in Figure~\ref{fig:vertical_images},\ref{fig:vertical_images1}, \ref{fig:vertical_images2}.

For the other two levels and for additional datasets, we also conduct a similar procedure, and the corresponding visualization results are presented in Figure~\ref{fig:vertical_images3}, \ref{fig:vertical_images4}, \ref{fig:vertical_images5}, \ref{fig:vertical_images6}, and \ref{fig:vertical_images7}.

\subsection{Visualization Results}\label{app:visual}
We provide our additional visualization results in Figure \ref{fig:vertical_images1}, \ref{fig:vertical_images2}, \ref{fig:vertical_images3}, \ref{fig:vertical_images4}, \ref{fig:vertical_images5}, \ref{fig:vertical_images6}, and \ref{fig:vertical_images7}.

\input{fig_ICLR/app_fig_visual_1}
\input{fig_ICLR/app_fig_visual_2}
\input{fig_ICLR/app_fig_visual_4}
\input{fig_ICLR/app_fig_visual_5}
\input{fig_ICLR/app_fig_visual_6}
\input{fig_ICLR/app_fig_visual_7}
\input{fig_ICLR/app_fig_visual_8}

\section{More Related Work}
\label{app:sec:more_related}
\noindent {\bf Influence Function.} The influence function, initially a staple in robust statistics \cite{cook2000detection,cook1980characterizations}, has seen extensive adoption within machine learning since \cite{koh2017understanding} introduced it to the field. Its versatility spans various applications, including detecting mislabeled data, interpreting models, addressing model bias, and facilitating machine unlearning tasks. Notable works in machine unlearning encompass unlearning features and labels \cite{warnecke2021machine}, minimax unlearning \cite{liu2024certified}, forgetting a subset of image data for training deep neural networks \cite{golatkar2020eternal,golatkar2021mixed}, graph unlearning involving nodes, edges, and features. Recent advancements, such as the LiSSA method \cite{agarwal2017second,kwon2023datainf} and kNN-based techniques \cite{guo2021fastif}, have been proposed to enhance computational efficiency. Besides, various studies have applied influence functions to interpret models across different domains, including natural language processing \cite{han2020explaining} and image classification \cite{basu2021influence}, while also addressing biases in classification models \cite{wang2019repairing}, word embeddings \cite{brunet2019understanding}, and finetuned models \cite{chen2020multi}. Despite numerous studies on influence functions, we are the first to utilize them to construct the editable CBM. Moreover, compared to traditional neural networks, CBMs are more complicated in their influence function. Because we only need to change the predicted output in the traditional influence function. While in CBMs, we should first remove the true concept, then we need to approximate the predicted concept in order to approximate the output. Bridging the gap between the true and predicted concepts poses a significant theoretical challenge in our proof. 

\noindent {\bf Model Unlearning.} 
Model unlearning has gained significant attention in recent years, with various methods~\cite{bourtoule2021machine,brophy2021machine,cao2015towards, chen2022recommendation,chen2022graph} proposed to efficiently remove the influence of certain data from trained machine learning models. Existing approaches can be broadly categorized into exact and approximate unlearning methods. Exact unlearning methods aim to replicate the results of retraining by selectively updating only a portion of the dataset, thereby avoiding the computational expense of retraining on the entire dataset~\cite{sekhari2021remember, chowdhury2024towards}. Approximate unlearning methods, on the other hand, seek to adjust model parameters to approximately satisfy the optimality condition of the objective function on the remaining data~\cite{golatkar2020eternal,guo2019certified,izzo2021approximate}. These methods are further divided into three subcategories: (1) Newton step-based updates that leverage Hessian-related terms [22, 26, 31, 34, 40, 43, 49], often incorporating Gaussian noise to mitigate residual data influence. To reduce computational costs, some works approximate the Hessian using the Fisher information matrix~\cite{golatkar2020eternal} or small Hessian blocks~\cite{mehta2022deep}. (2) Neural tangent kernel (NTK)-based unlearning approximates training as a linear process, either by treating it as a single linear change~\cite{golatkar2020forgetting}. (3) SGD path tracking methods, such as DeltaGrad~\cite{wu2020deltagrad} and unrollSGD~\cite{thudi2022unrolling}, reverse the optimization trajectory of stochastic gradient descent during training. 
Despite their advancements, these methods fail to handle the special architecture of CBMs. Moreover, given the high cost of obtaining data, we sometimes prefer to correct the data rather than remove it, which model unlearning is unable to achieve.

\section{Limitations and Broader Impacts}
\label{sec:limitation_impact}

\subsection{Limitations}
While CCBMs demonstrate significant efficiency and utility, it is important to acknowledge certain limitations inherent to the methodology. 
First, \textbf{Approximation Error in Non-Convex Settings}: Our approach relies on influence functions, which are based on a second-order Taylor expansion of the loss function. This assumes that the loss landscape is locally convex and quadratic. In deep neural networks, where the loss surface is highly non-convex, this approximation may deviate from the exact retraining solution, particularly when the magnitude of the edit (e.g., the number of removed or added samples) is large. Although our sequential editing experiments demonstrate robustness, extremely aggressive modifications might still require periodic full retraining to reset the accumulation error.
Second, \textbf{Computational Overhead of Hessian}: Despite using EK-FAC to accelerate computation, calculating and inverting the Hessian matrix (or Fisher Information Matrix) still incurs a memory and computational cost that scales with the dimension of the parameters. For extremely large-scale foundation models with billions of parameters, further optimization or sparse approximations may be necessary to make CCBMs fully real-time.

\subsection{Broader Impacts}
Our work extends beyond merely correcting model errors; it provides a foundational framework for the \textbf{lifecycle management} of interpretable AI systems.

\noindent\textbf{Enabling Continuous Adaptation via Data Addition.}
Static models are ill-suited for dynamic environments. A critical contribution of CCBM is its support for \textbf{Data Addition}, which empowers models to perform \textit{incremental learning} efficiently. In domains like healthcare or finance, data arrives in streams. CCBMs allow deployed models to assimilate new cases (e.g., a new variant of a disease or a novel financial fraud pattern) continuously without the downtime and energy consumption of retraining. This transforms CBMs from static artifacts into evolving systems that grow alongside human knowledge.

\noindent\textbf{Facilitating Privacy and Compliance.}
With the enforcement of regulations such as General Data Protection Regulation (GDPR) and California Consumer Privacy Act (CCPA), the "Right to be Forgotten" has become a legal imperative. CCBMs provide a mathematically grounded mechanism for \textbf{Data Removal}, allowing organizations to rigorously unlearn sensitive user data or cleanse poisoned samples. Our RMIA experiments confirm that this unlearning is effective in mitigating privacy risks, fostering trust in AI deployment.

\noindent\textbf{Promoting Sustainable "Green AI".}
Retraining modern deep learning models is resource-intensive and has a significant carbon footprint. By replacing retraining with efficient closed-form updates, CCBMs drastically reduce energy consumption (e.g., reducing runtime from hours to minutes as shown in our experiments). This aligns with the goals of Green AI, making high-performance interpretable models more environmentally sustainable.

\noindent\textbf{Empowering Human-AI Collaboration.}
Finally, CCBMs serve as a bridge for interactive Human-AI alignment. By allowing experts (e.g., doctors) to intervene at the \textit{concept-label} or \textit{concept-level}, the model allows users to correct reasoning flaws directly. This "editable" nature ensures that the model remains aligned with expert consensus, significantly enhancing its reliability and adoption in critical decision-making processes.

%% file: fig_ICLR/fig4.tex
\begin{figure}[ht]
    \centering
    \begin{tabular}{cc}
\includegraphics[width=0.49\linewidth]{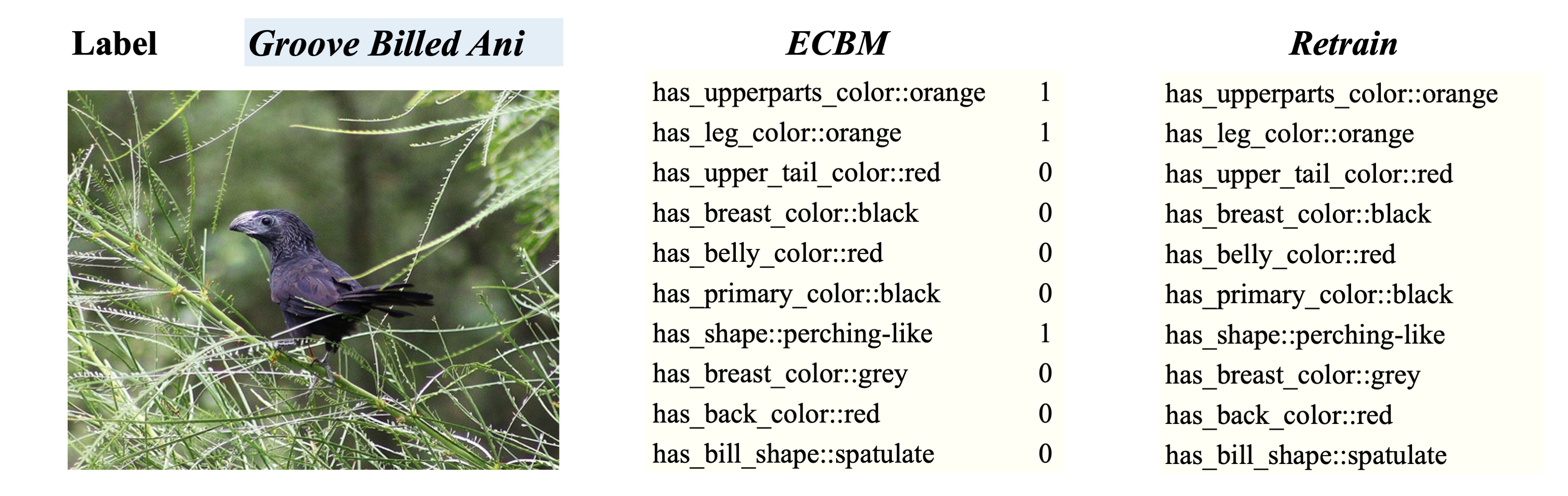} 
\includegraphics[width=0.49\linewidth]{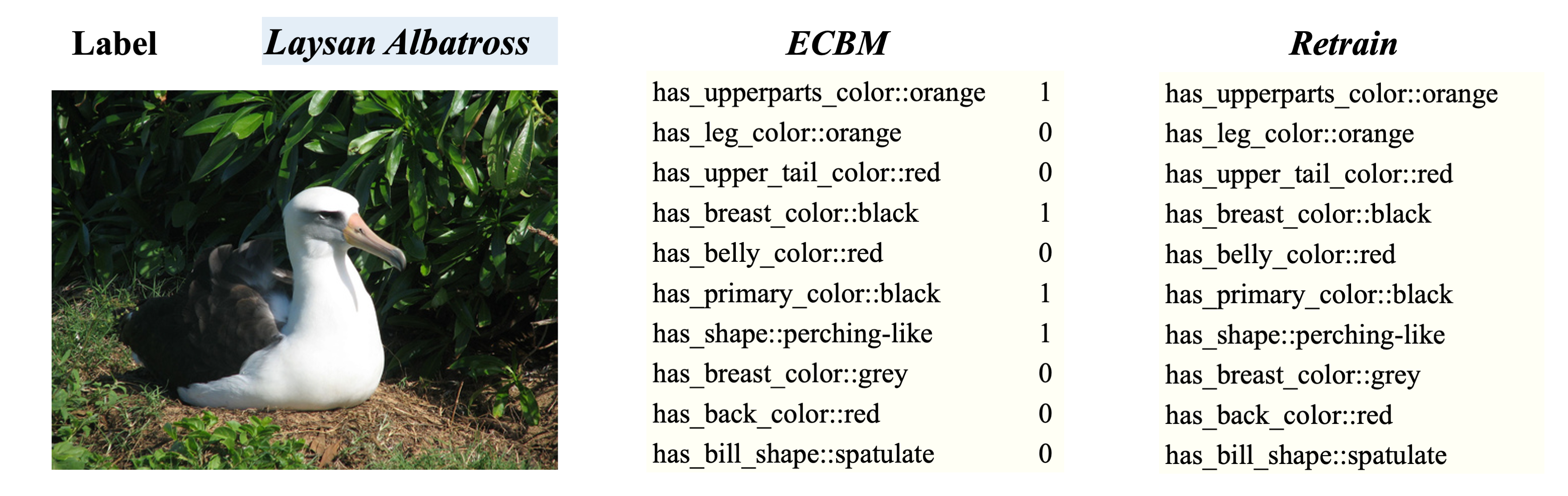}
    \end{tabular}
    \caption{Visualization of the Top 10 Most Influential Concepts for CBM(Identified by CCBM or Retrain) Highlighted on an Extracted Image.}
    \vspace{-8pt}
    \label{fig:vertical_images}
\end{figure}


%% file: fig_ICLR/app_fig_visual_1.tex
\begin{figure}[ht]
    \centering
    \begin{tabular}{cc}
        \includegraphics[width=0.95\linewidth]{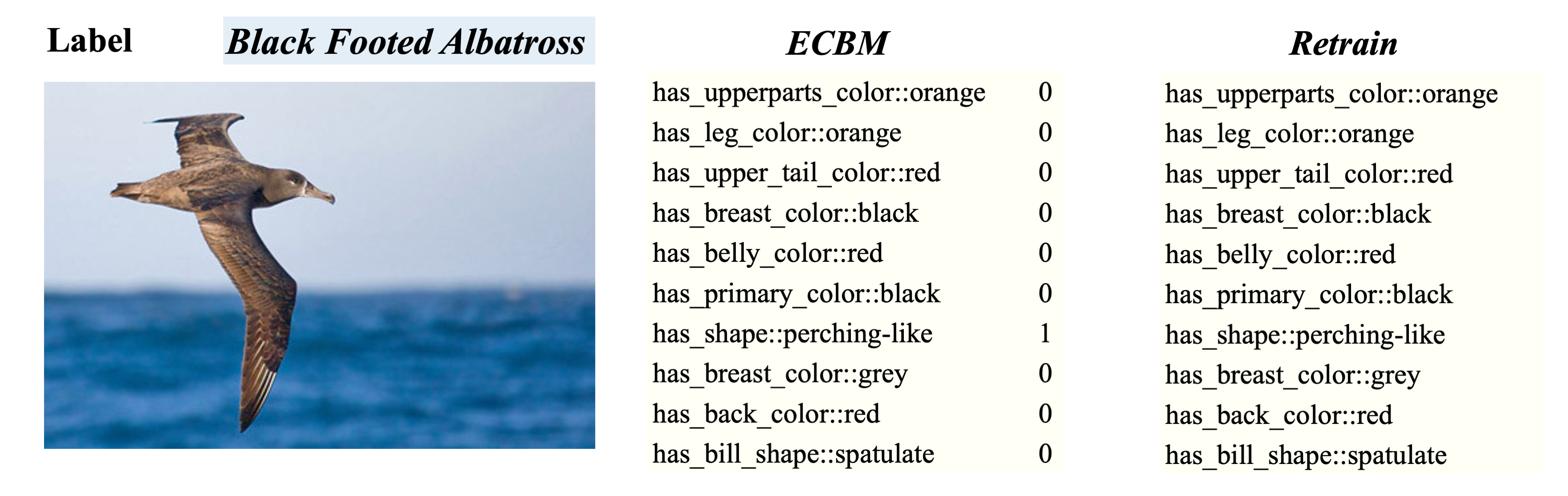}\\
        \includegraphics[width=0.95\linewidth]{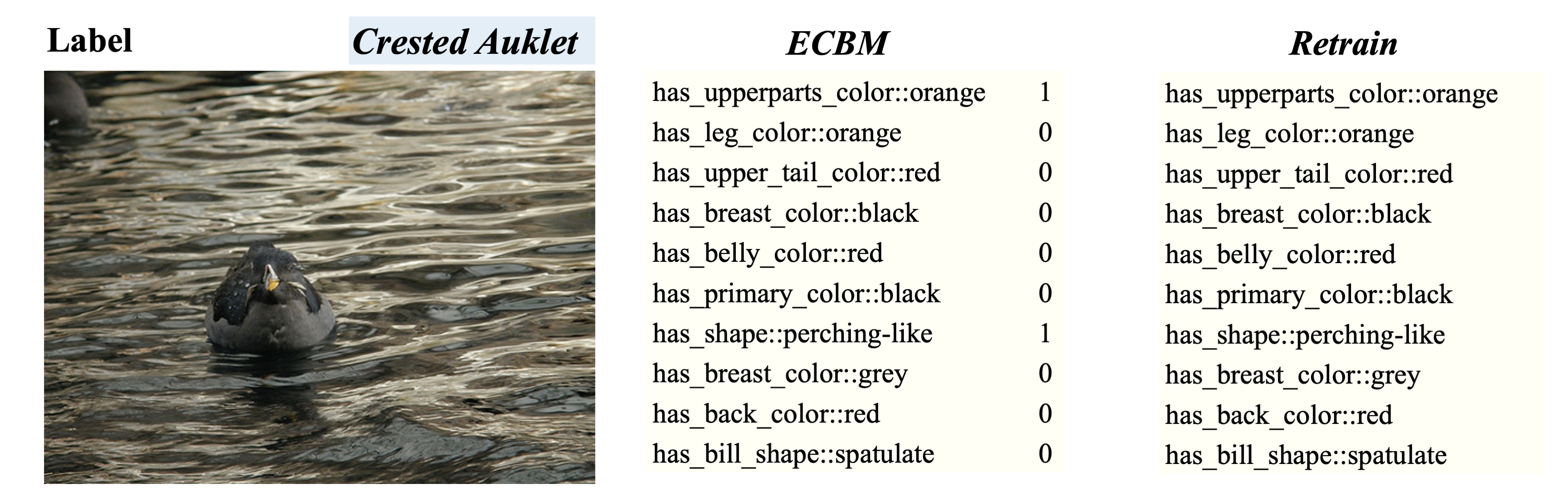}\\
        \includegraphics[width=0.95\linewidth]{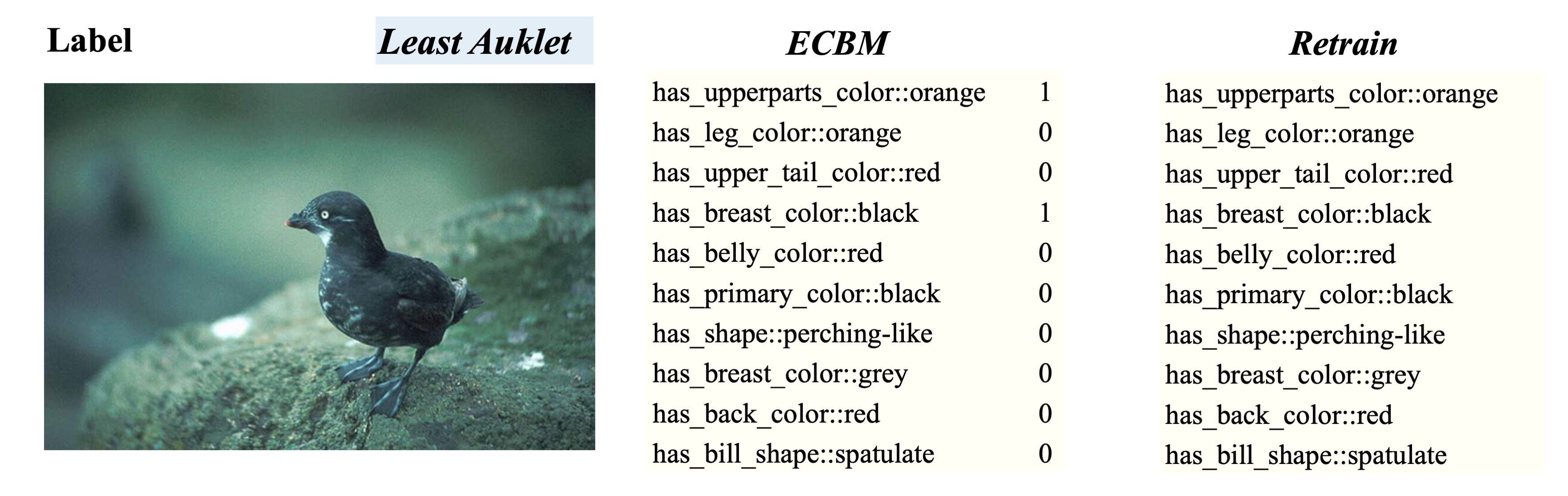}\\
        \includegraphics[width=0.95\linewidth]{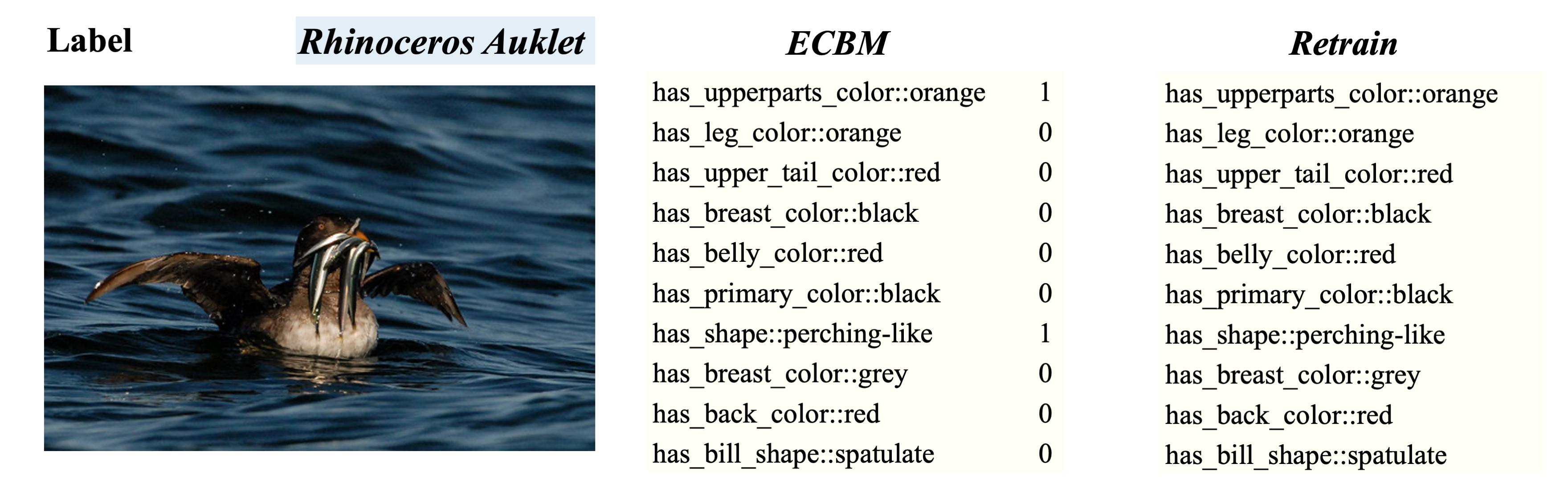}
    \end{tabular}
    \caption{Visualization of the top-10 most influential concepts for different classes in CUB.}
    \label{fig:vertical_images1}
\end{figure}

%% file: fig_ICLR/app_fig_visual_2.tex
\begin{figure}[ht]
    \centering
    \begin{tabular}{cc}
        \includegraphics[width=0.95\linewidth]{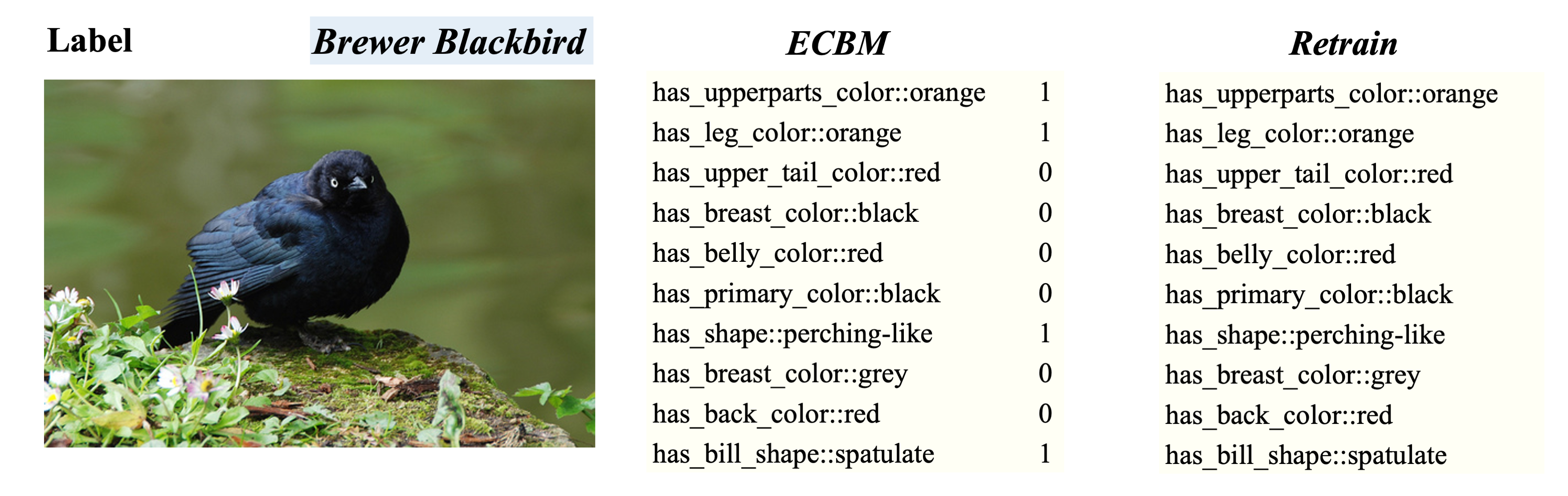}\\
        \includegraphics[width=0.95\linewidth]{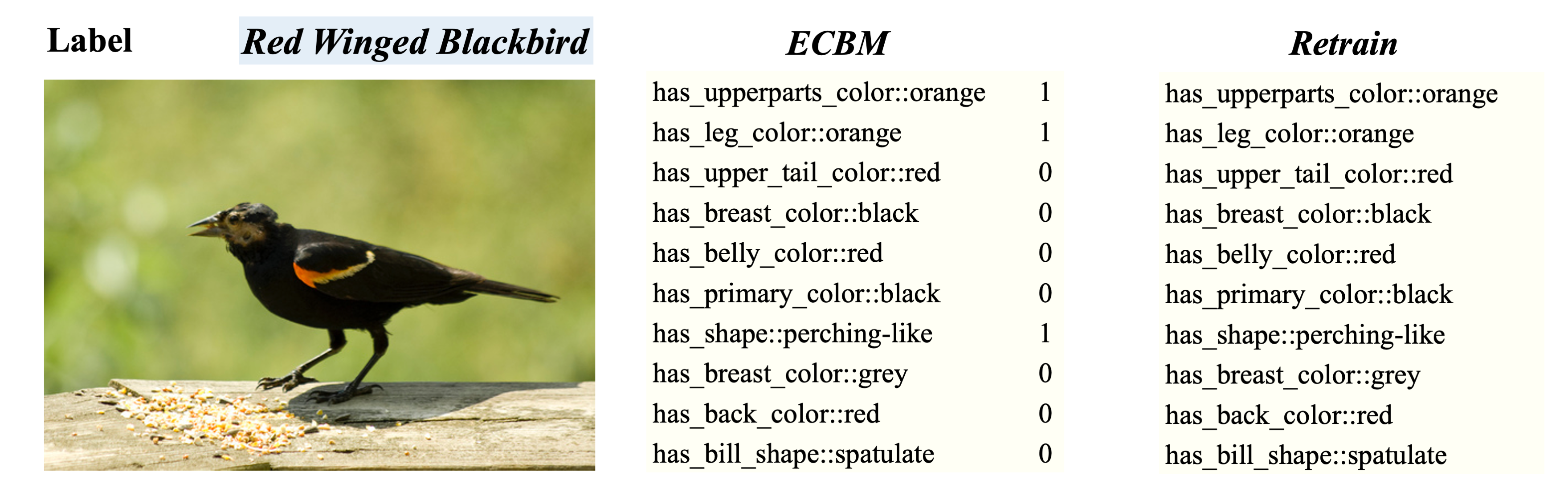}\\
        \includegraphics[width=0.95\linewidth]{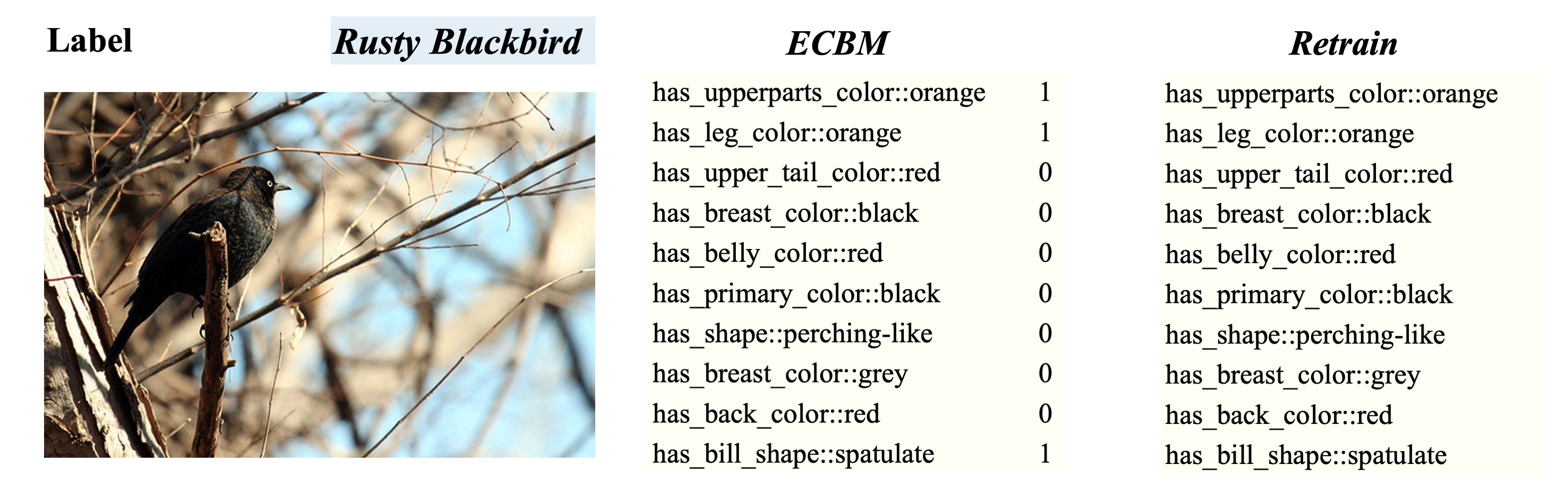}\\
        \includegraphics[width=0.95\linewidth]{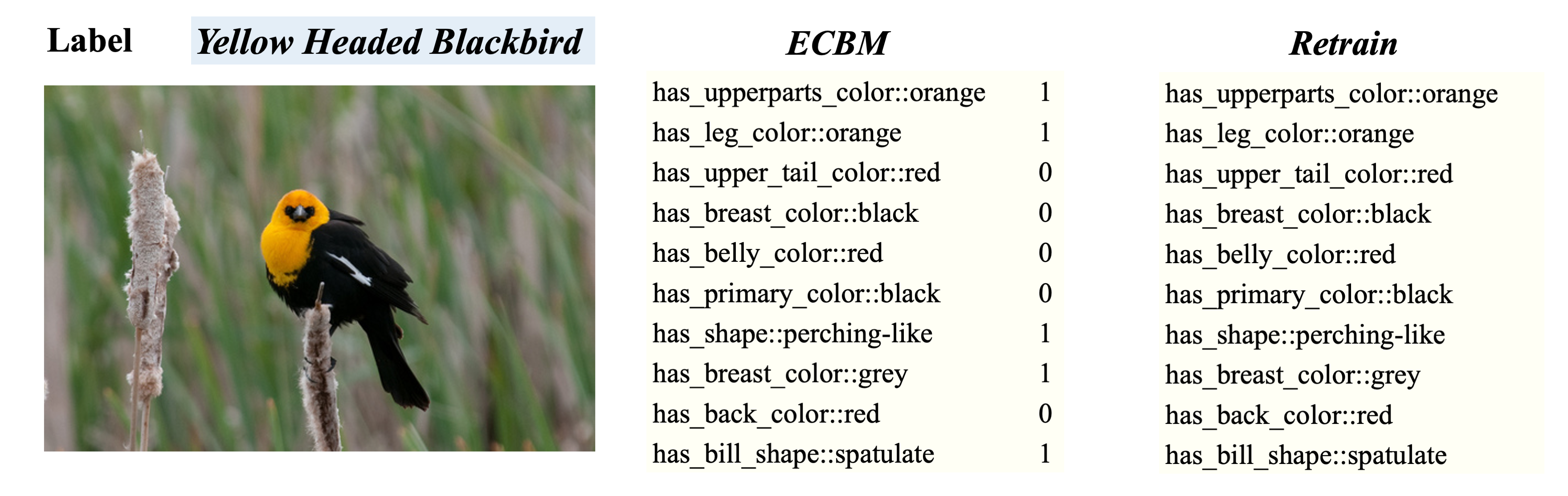}
    \end{tabular}
    \caption{Visualization of the top-10 most influential concepts for different classes in CUB.}
    \label{fig:vertical_images2}
\end{figure}

%% file: fig_ICLR/app_fig_visual_4.tex
\begin{figure}[ht]
    \centering
    \begin{tabular}{cc}
        \includegraphics[width=0.95\linewidth]{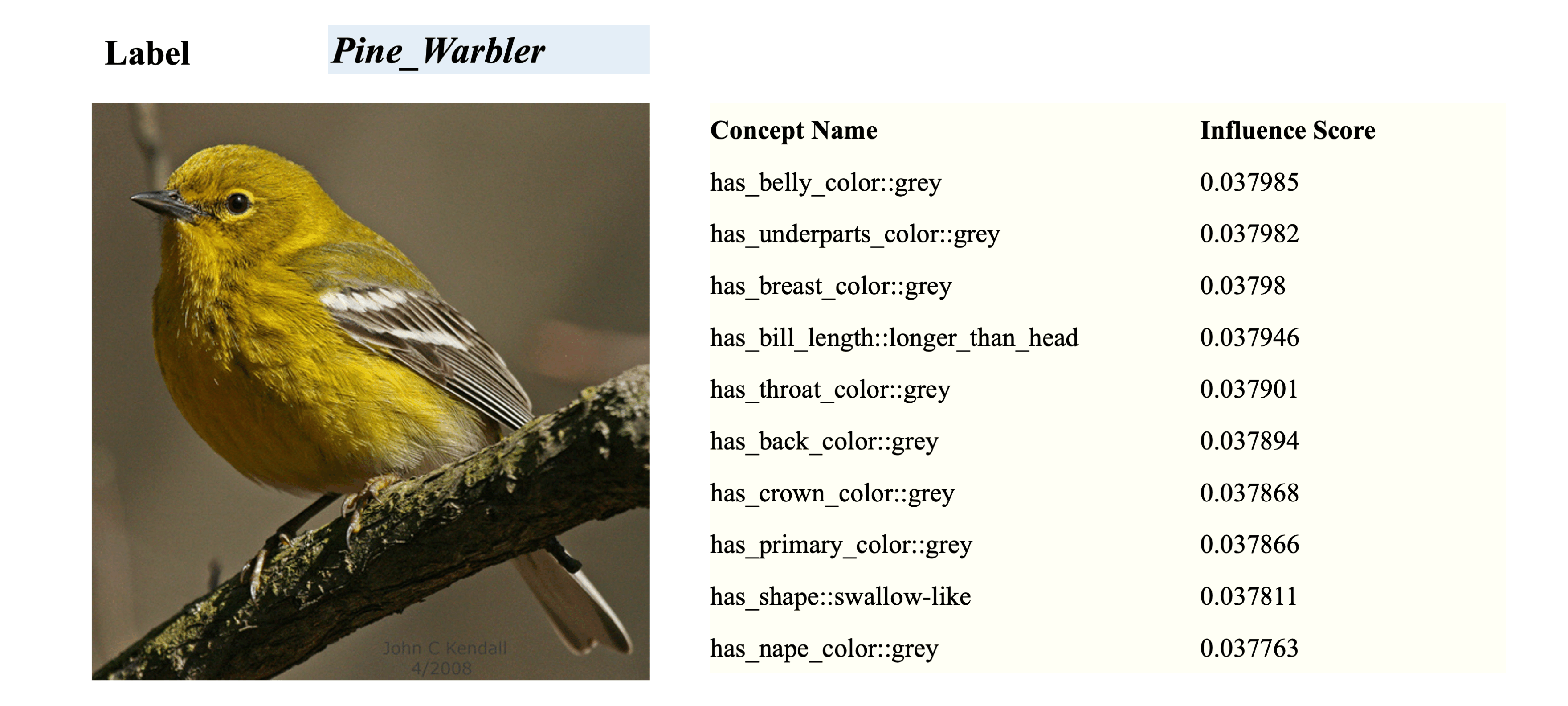}\\
        \includegraphics[width=0.95\linewidth]{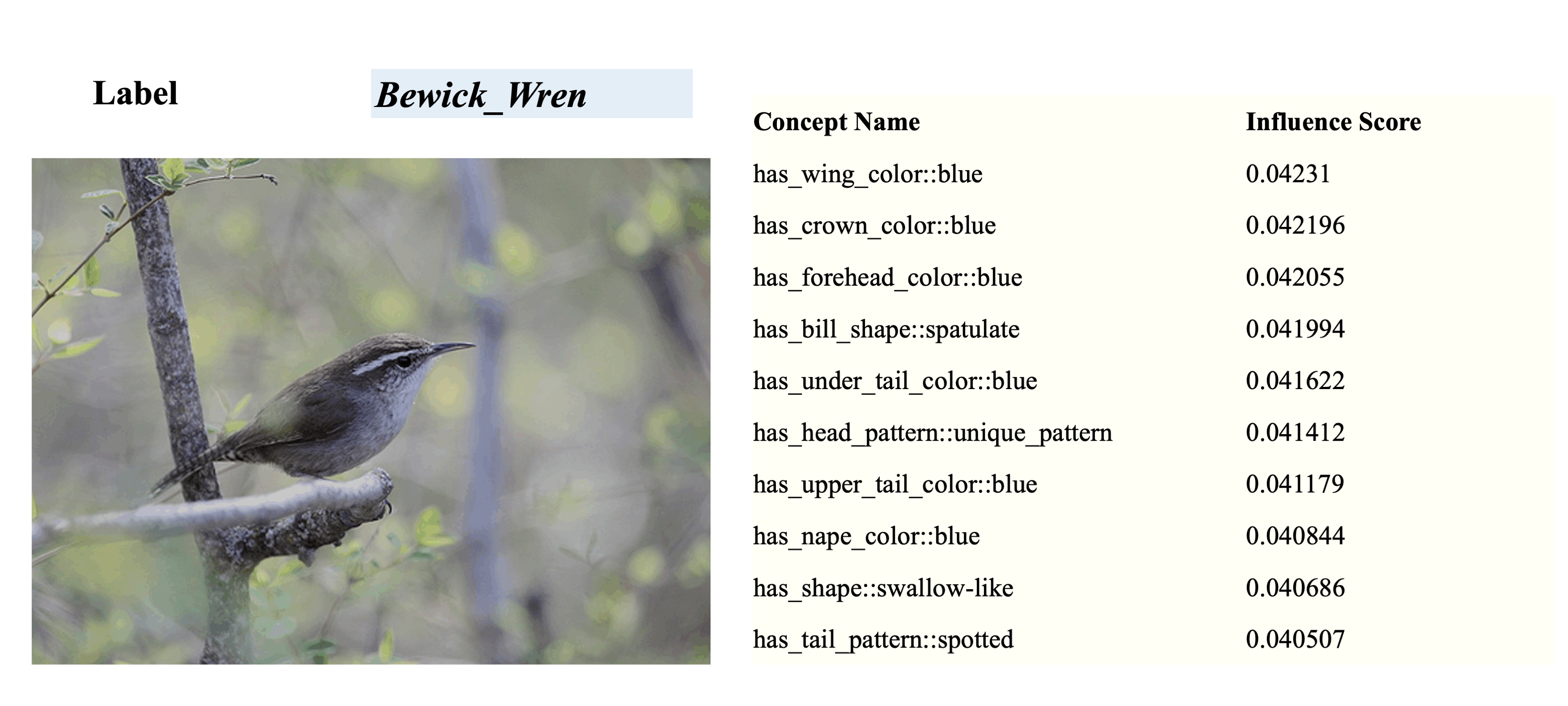}\\
        \includegraphics[width=0.95\linewidth]{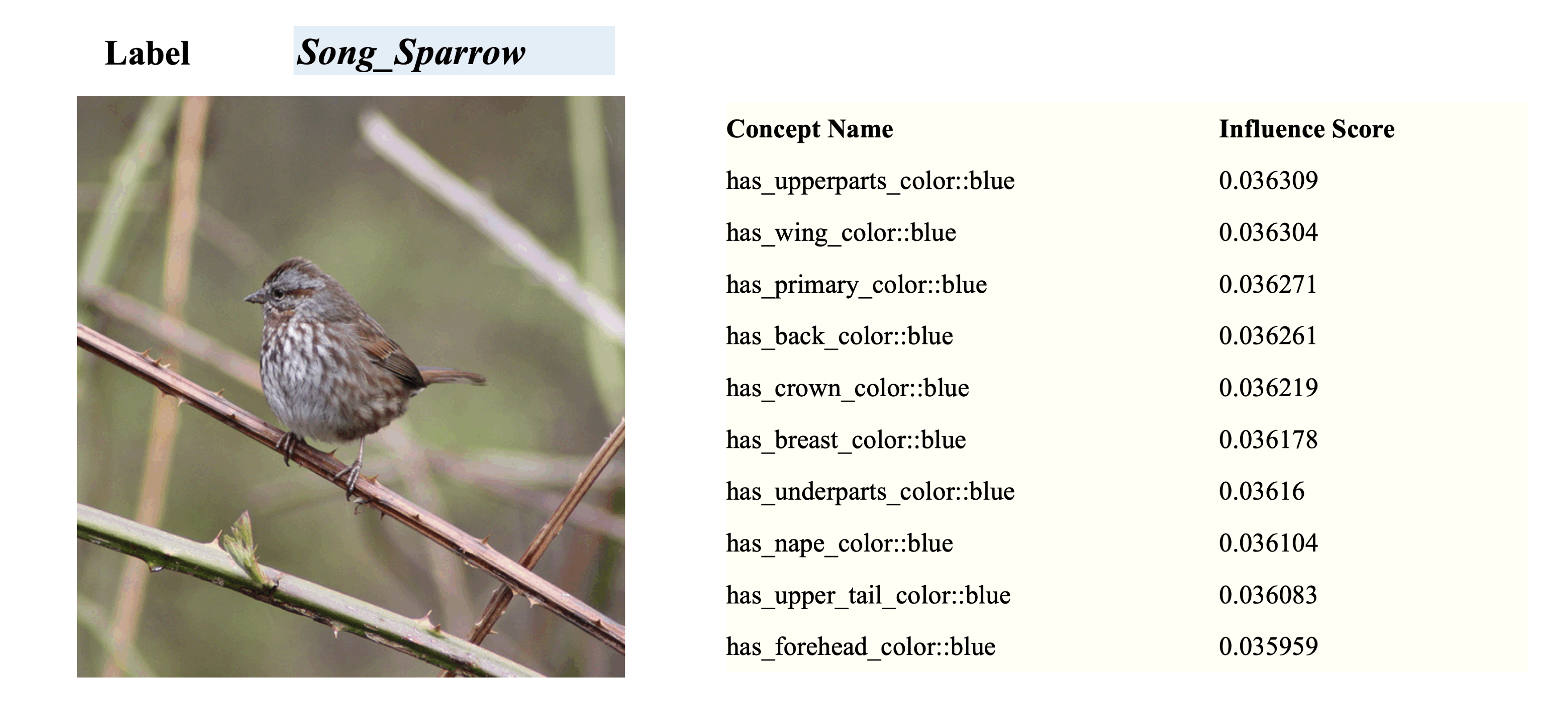}
    \end{tabular}
    \caption{Visualization of the most influential concept label related to different data in CUB.}
    \label{fig:vertical_images3}
\end{figure}

%% file: fig_ICLR/app_fig_visual_5.tex
\begin{figure}[ht]
    \centering
    \begin{tabular}{cc}
        \includegraphics[width=0.95\linewidth]{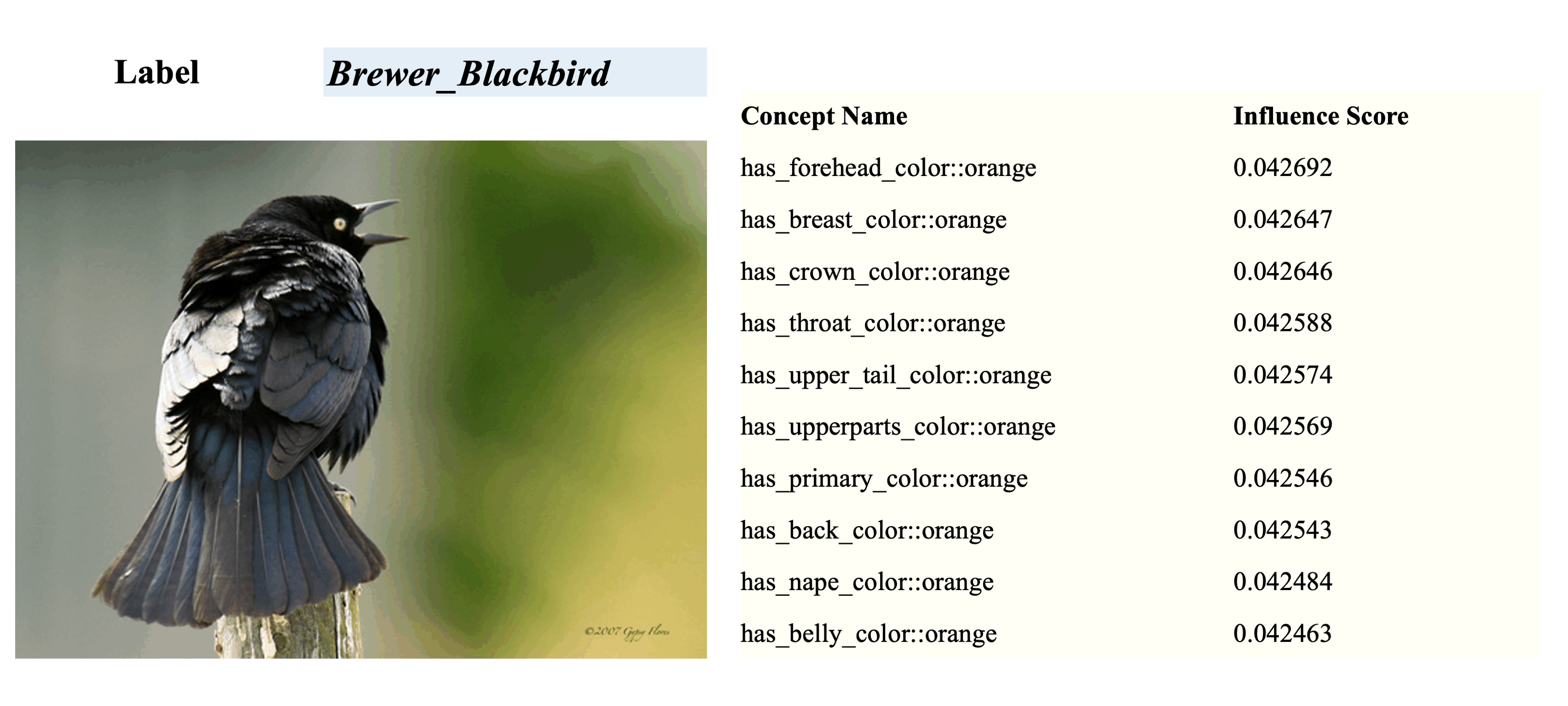}\\
        \includegraphics[width=0.95\linewidth]{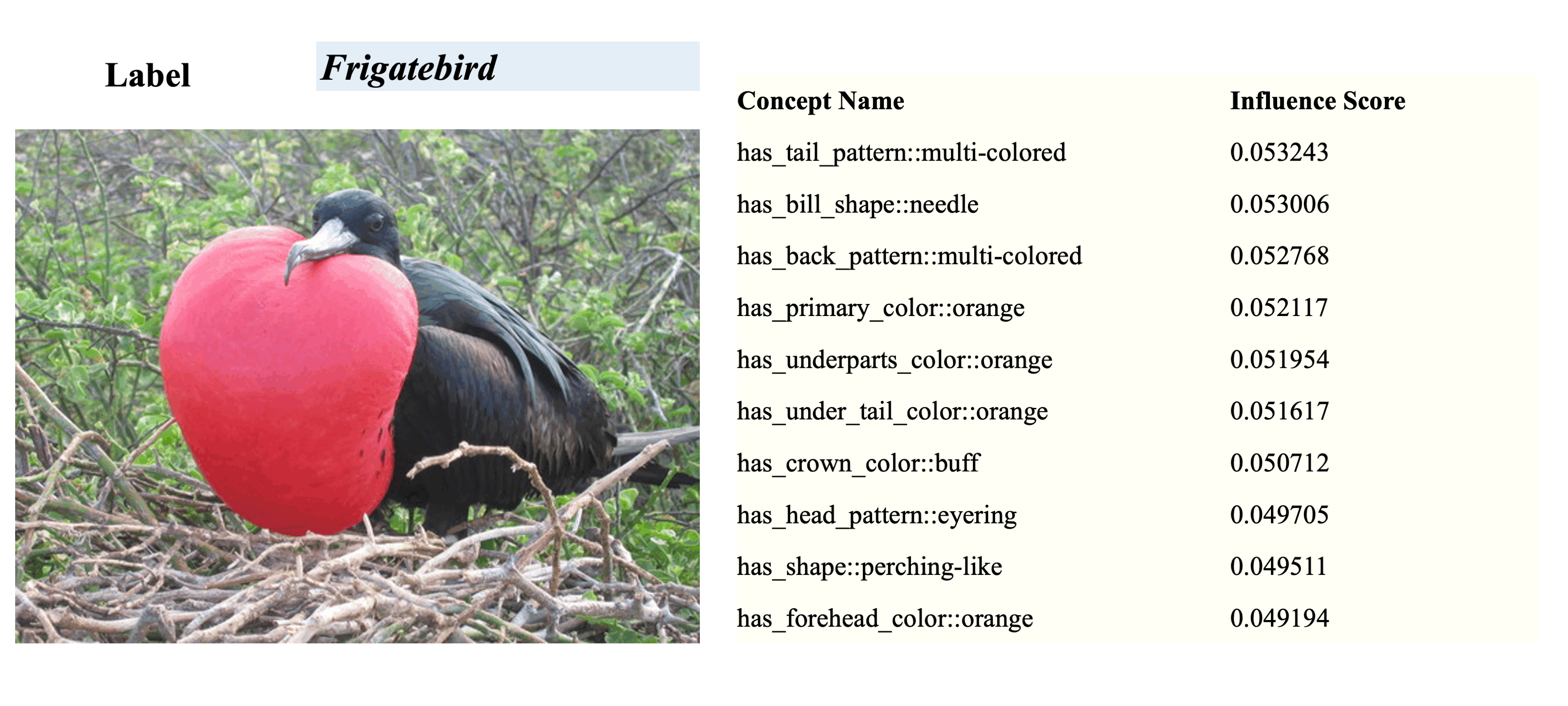}\\
        \includegraphics[width=0.95\linewidth]{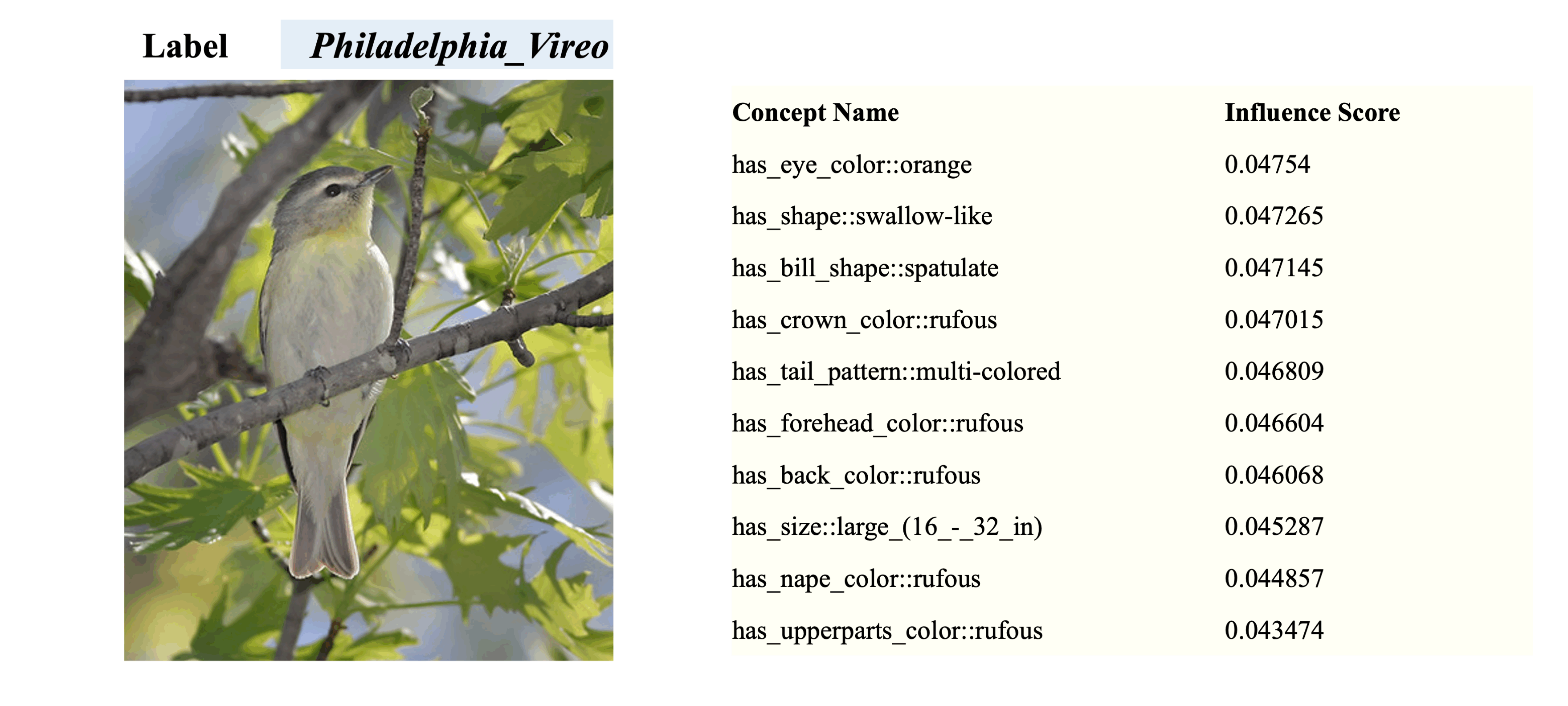}
    \end{tabular}
    \caption{Visualization of the most influential concept label related to different data in CUB.}
    \label{fig:vertical_images4}
\end{figure}

%% file: fig_ICLR/app_fig_visual_6.tex
\begin{figure}[ht]
    \centering
    \begin{tabular}{cc}
        \includegraphics[width=0.80\linewidth]{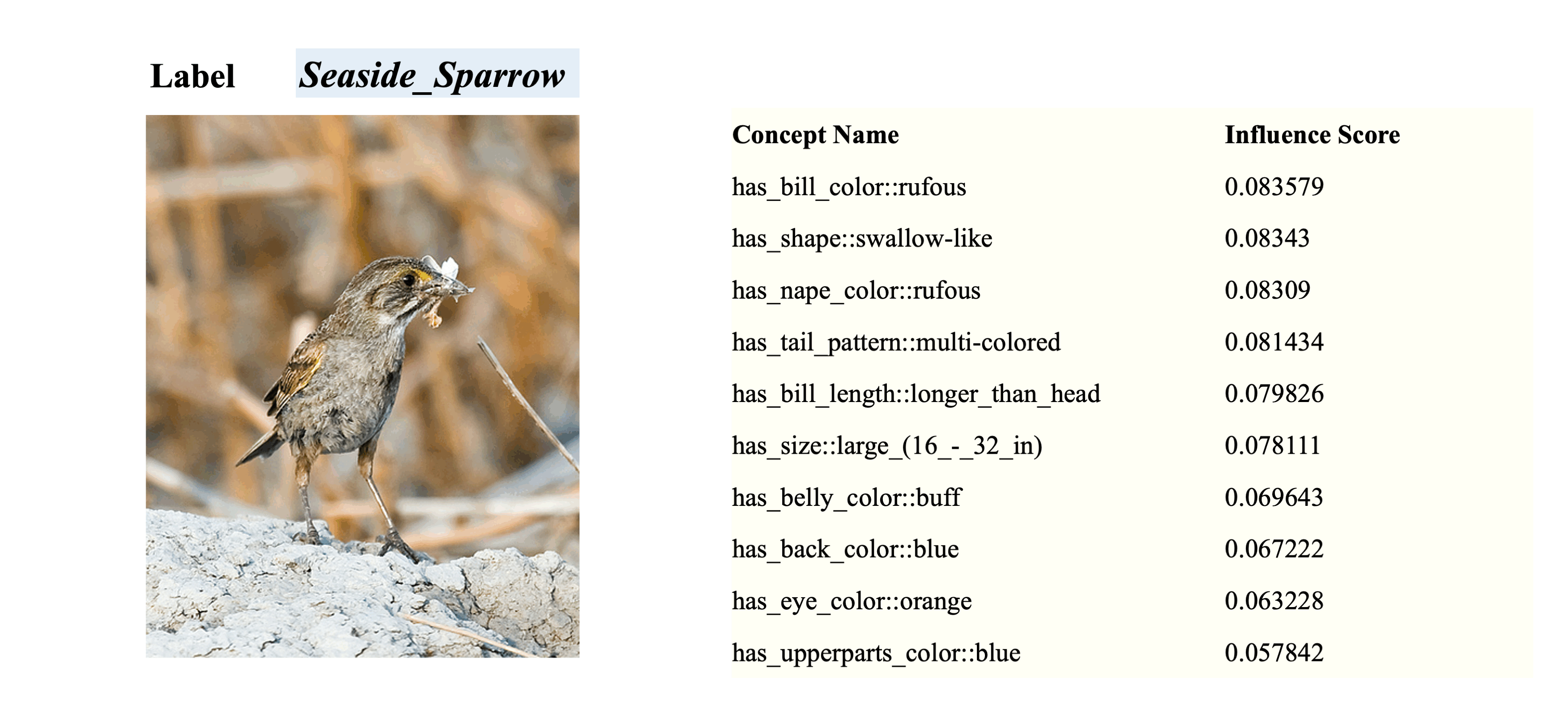}\\
        \includegraphics[width=0.80\linewidth]{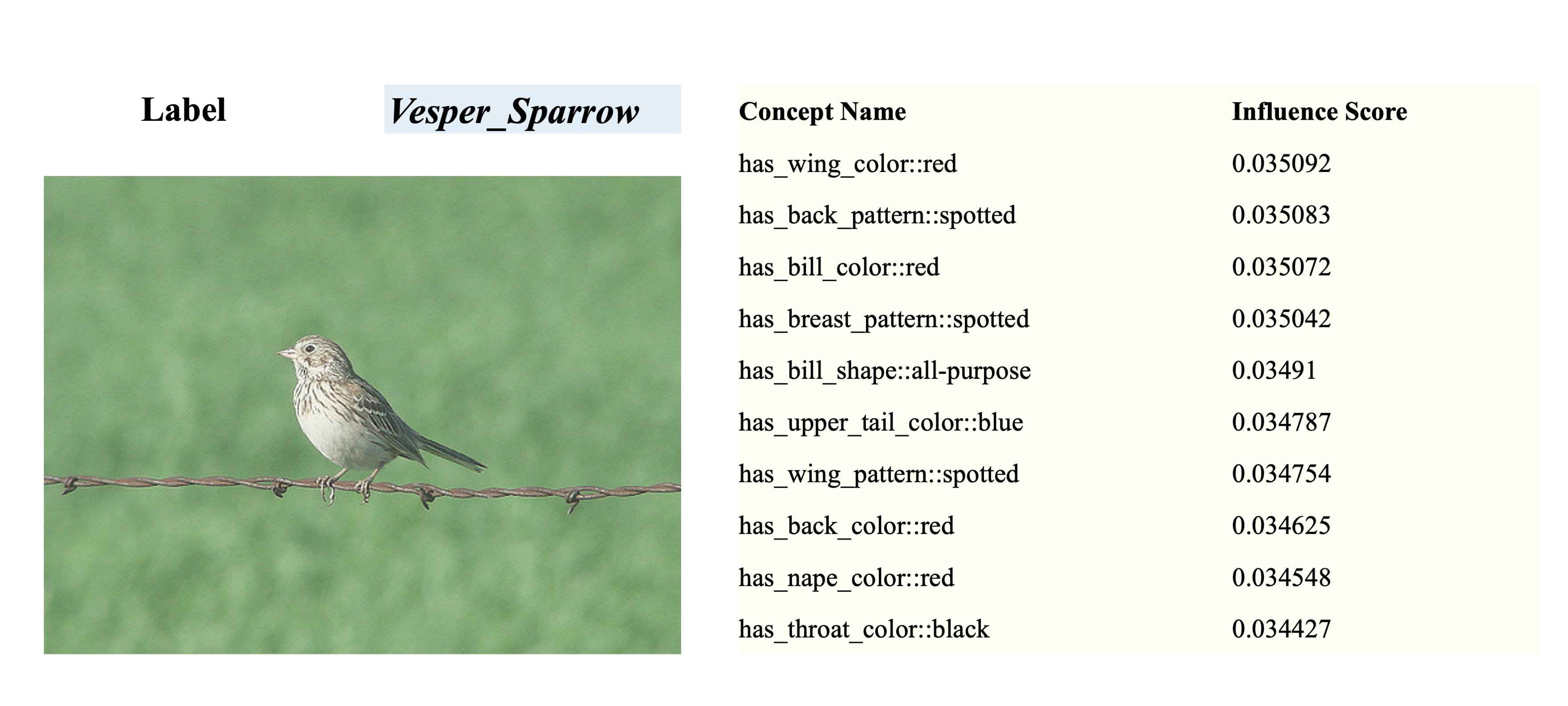}
    \end{tabular}
    \caption{Visualization of the most influential concept label related to different data in CUB.}
    \label{fig:vertical_images5}
\end{figure}

%% file: fig_ICLR/app_fig_visual_7.tex
\begin{figure}[ht]
    \centering
    \begin{tabular}{cc}
        \includegraphics[width=0.80\linewidth]{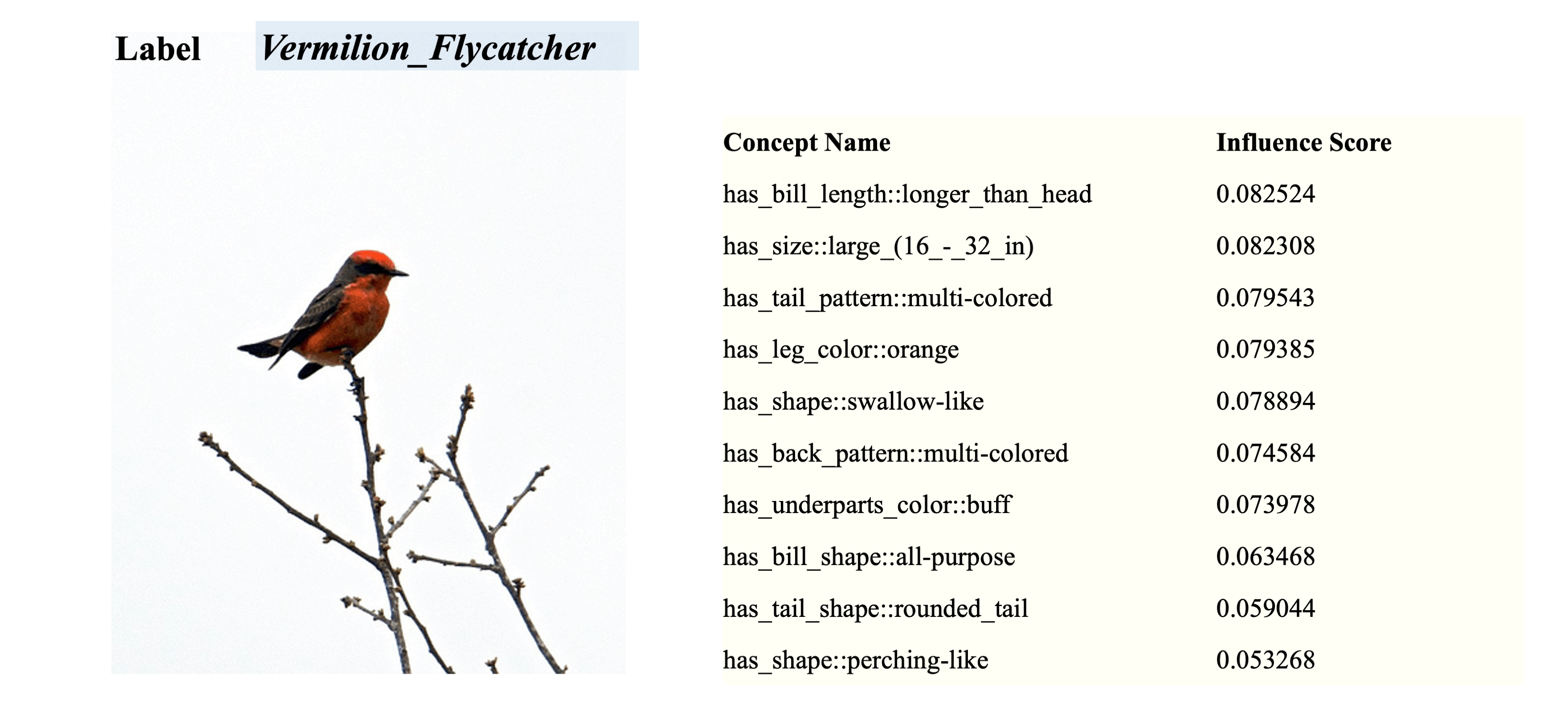}\\
        \includegraphics[width=0.80\linewidth]{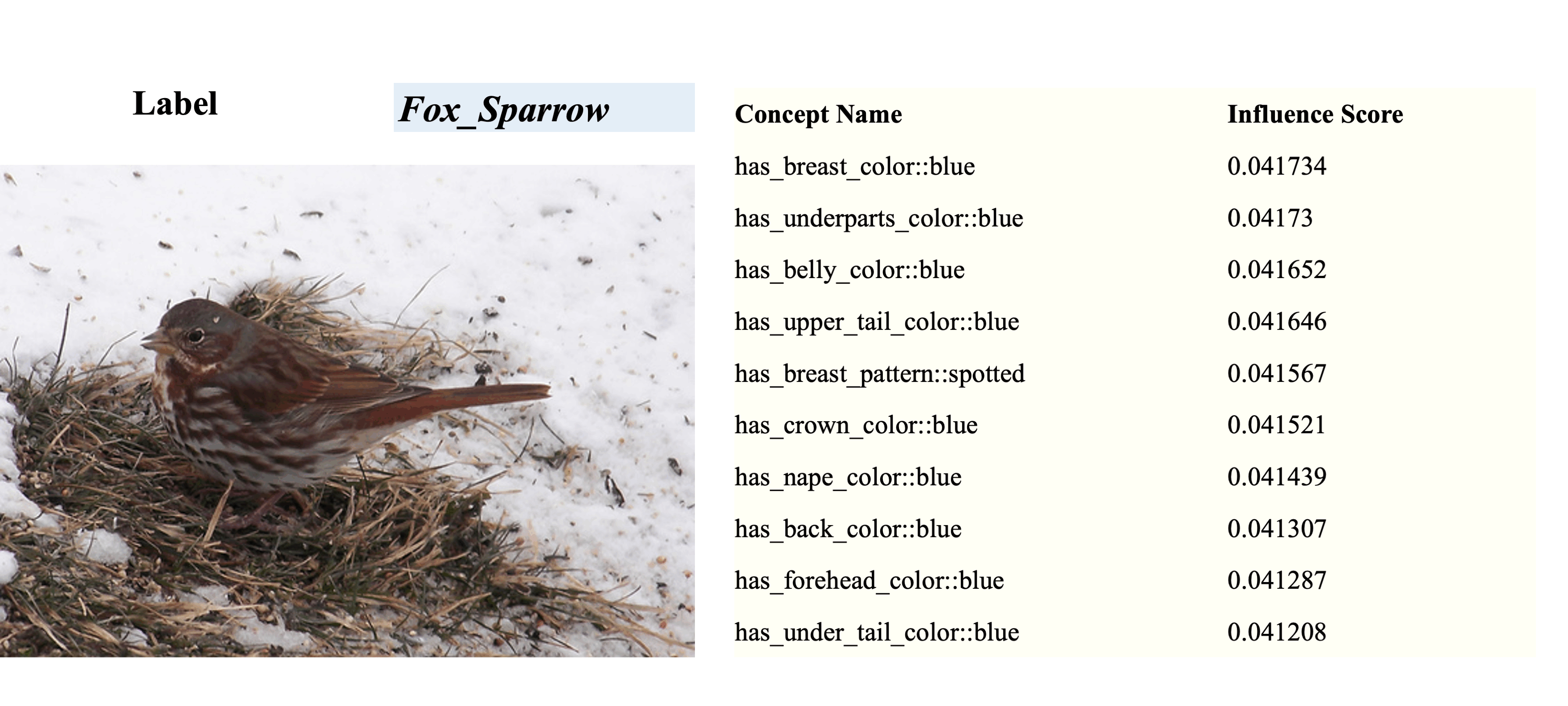}
    \end{tabular}
    \caption{Visualization of the most influential concept label related to different data in CUB.}
    \label{fig:vertical_images6}
\end{figure}

%% file: fig_ICLR/app_fig_visual_8.tex
\begin{figure}[ht]
    \centering
    \begin{tabular}{cc}
        \includegraphics[width=0.95\linewidth]{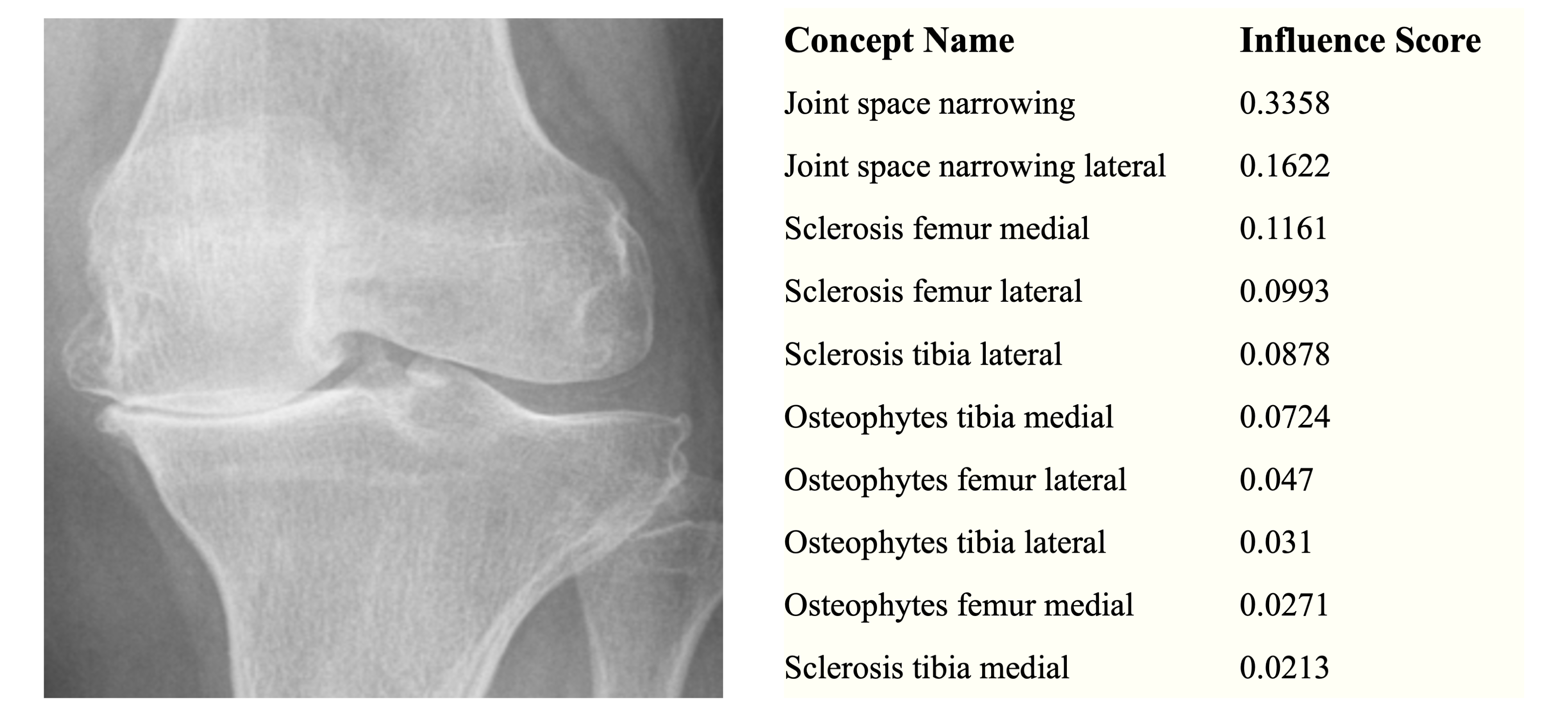}\\
        \includegraphics[width=0.95\linewidth]{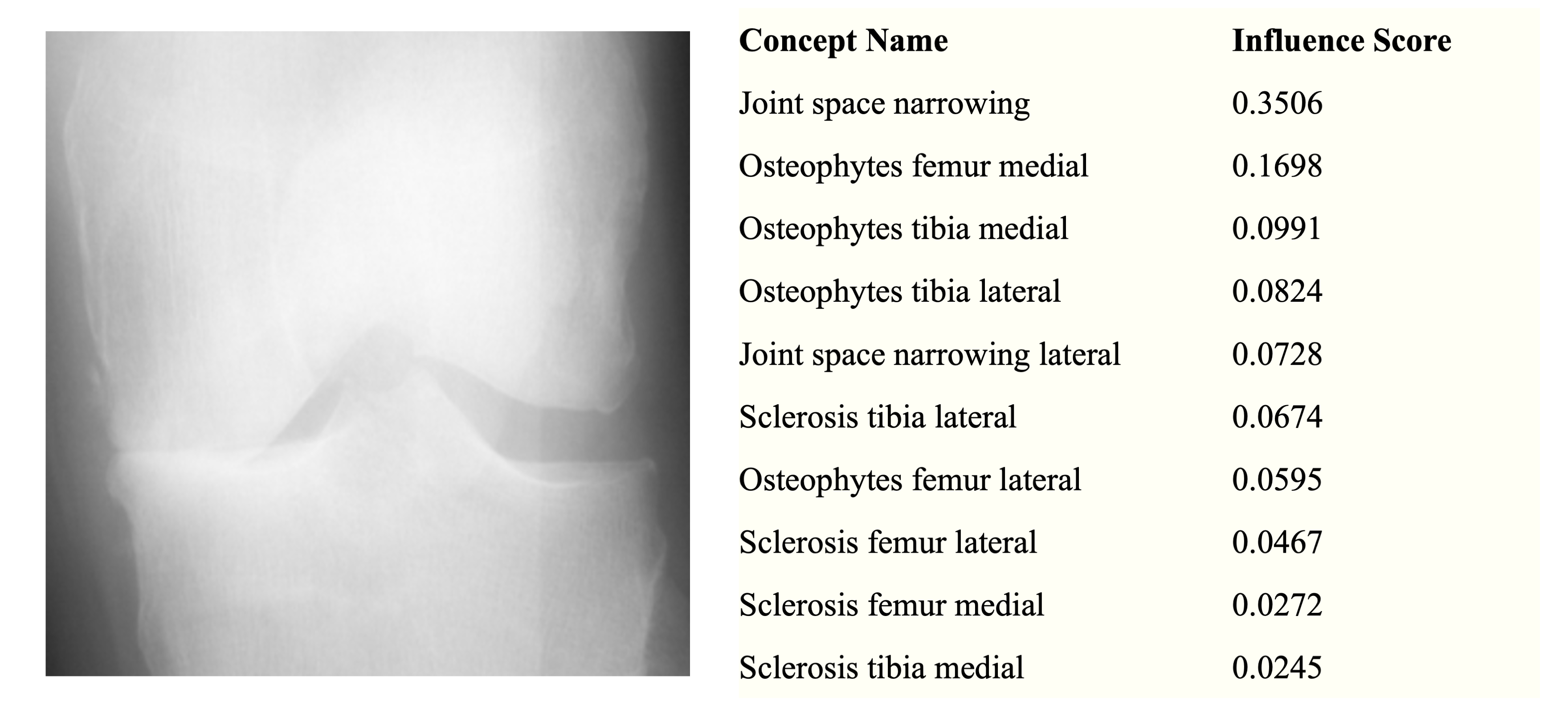}\\
        \includegraphics[width=0.95\linewidth]{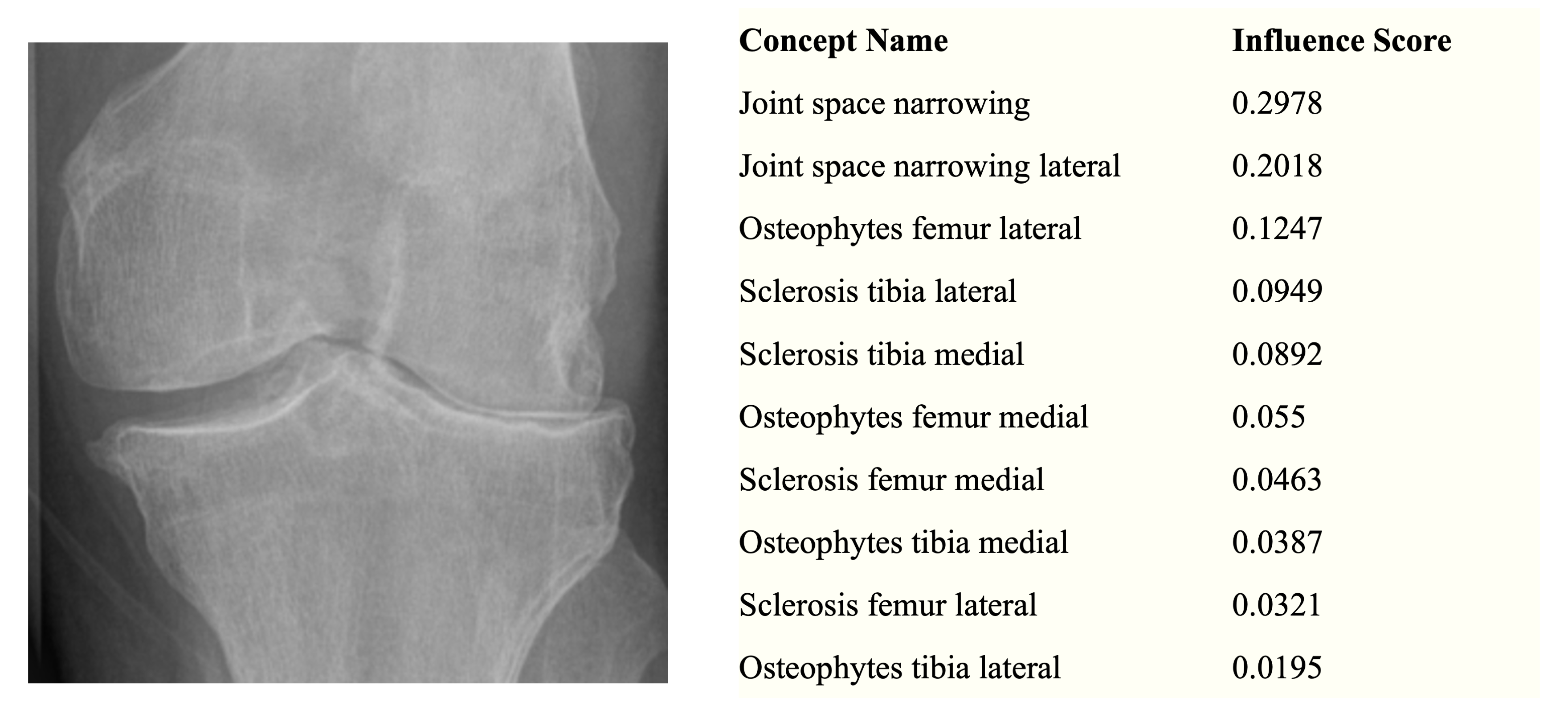}
    \end{tabular}
    \caption{Visualization of the most influential concept label related to different data in OAI.}
    \label{fig:vertical_images7}
\end{figure}